%% file: iclr2025_conference.tex
\documentclass{article} 
\usepackage{iclr2025_conference,times}

\input{math_commands.tex}

\usepackage{enumitem}

\usepackage{hyperref}
\usepackage{url}
\usepackage{amsmath}
\usepackage{amsthm}
\usepackage{mathtools}
\usepackage{amssymb}
\usepackage{graphicx}
\usepackage{booktabs}
\usepackage{wrapfig}
\usepackage{subcaption}
\usepackage{multirow}
\usepackage[export]{adjustbox} 
\usepackage{array}
\newcolumntype{C}[1]{>{\centering\arraybackslash}m{#1}} 

\newtheorem{theorem}{Theorem}[section]
\newtheorem{lemma}[theorem]{Lemma}
\newtheorem{proposition}[theorem]{Proposition}
\newtheorem{corollary}[theorem]{Corollary}
\newtheorem{definition}[theorem]{Definition}

\title{Beyond Gaussian Initializations: \\
Signal Preserving Weight Initialization for Odd-Sigmoid Activations}

\author{Hyunwoo Lee$^1$,\: Hayoung Choi$^1$$^*$,\: Hyunju Kim$^2$\thanks{Corresponding authors.} \\
\textsuperscript{1}Korea Institute for Advanced Study, \:
\textsuperscript{2}Kyungpook National University, \:
\textsuperscript{3}Korea Institute of Energy Technology\\
\textsuperscript{1}\texttt{\{hyunwoolee\}@kias.re.kr}, \:
\textsuperscript{2}\texttt{\{hayoung.choi\}@knu.ac.kr}, \:
\textsuperscript{3}\texttt{hjkim@kentech.ac.kr} 
}

\iclrfinalcopy 
\begin{document}

\maketitle

\begin{abstract}
Activation functions critically influence trainability and expressivity, and recent work has therefore explored a broad range of nonlinearities.
However, widely used Gaussian i.i.d.\ initializations are designed to preserve activation variance under wide or infinite width assumptions.
In deep and relatively narrow networks with sigmoidal nonlinearities, these schemes often drive preactivations into saturation, and collapse gradients.  To address this, we introduce an odd-sigmoid activations and propose an activation aware initialization tailored to any function in this class. Our method remains robust over a wide band of variance scales, preserving both forward signal variance and backpropagated gradient norms even in very deep and narrow networks. Empirically, across standard image benchmarks we find that the proposed initialization is substantially less sensitive to depth, width, and activation scale than Gaussian initializations. In physics informed neural networks (PINNs), scaled odd-sigmoid activations combined with our initialization achieve lower losses than Gaussian based setups, suggesting that diagonal-plus-noise weights provide a practical alternative when Gaussian initialization breaks down.
\end{abstract}

\section{Introduction}

Deep learning has advanced substantially across diverse domains~\citep{lecun2015deep, he2016deep, hwang2022asymmetric}. A central ingredient in
these successes is the choice of activation function, which controls both the expressive
power of a network and the way signals and gradients propagate through depth~\citep{poole2016exponential}.
Recent work has proposed a wide variety of nonlinearities beyond classical sigmoid and 
ReLU (e.g., GELU, Swish, Mish, SELU) to improve trainability, stability, and
accuracy~\citep{hendrycks2016Gaussian,ramachandran2017searching,misra2019mish,
klambauer2017self,murray2022activation,bingham2023efficient,zhang2024deep}. 
However, activation functions and weight initialization are tightly coupled~\citep{he2015delving,lee2024improved}. The scale and shape of the nonlinearity determine how information propagates with depth. As a result, an initialization that is 
not tuned to the chosen activation can drive activations into saturation and cause gradients to vanish or explode, even when the activation itself is reasonable. 
These issues are especially pronounced for sigmoidal activations, which are widely used in sequence models and physics informed neural network~(PINN)~\citep{raissi2019physics}.

\begin{wrapfigure}{r}{0.5\textwidth} 
  \vspace{-1pt}
  \centering
  \includegraphics[width=\linewidth]{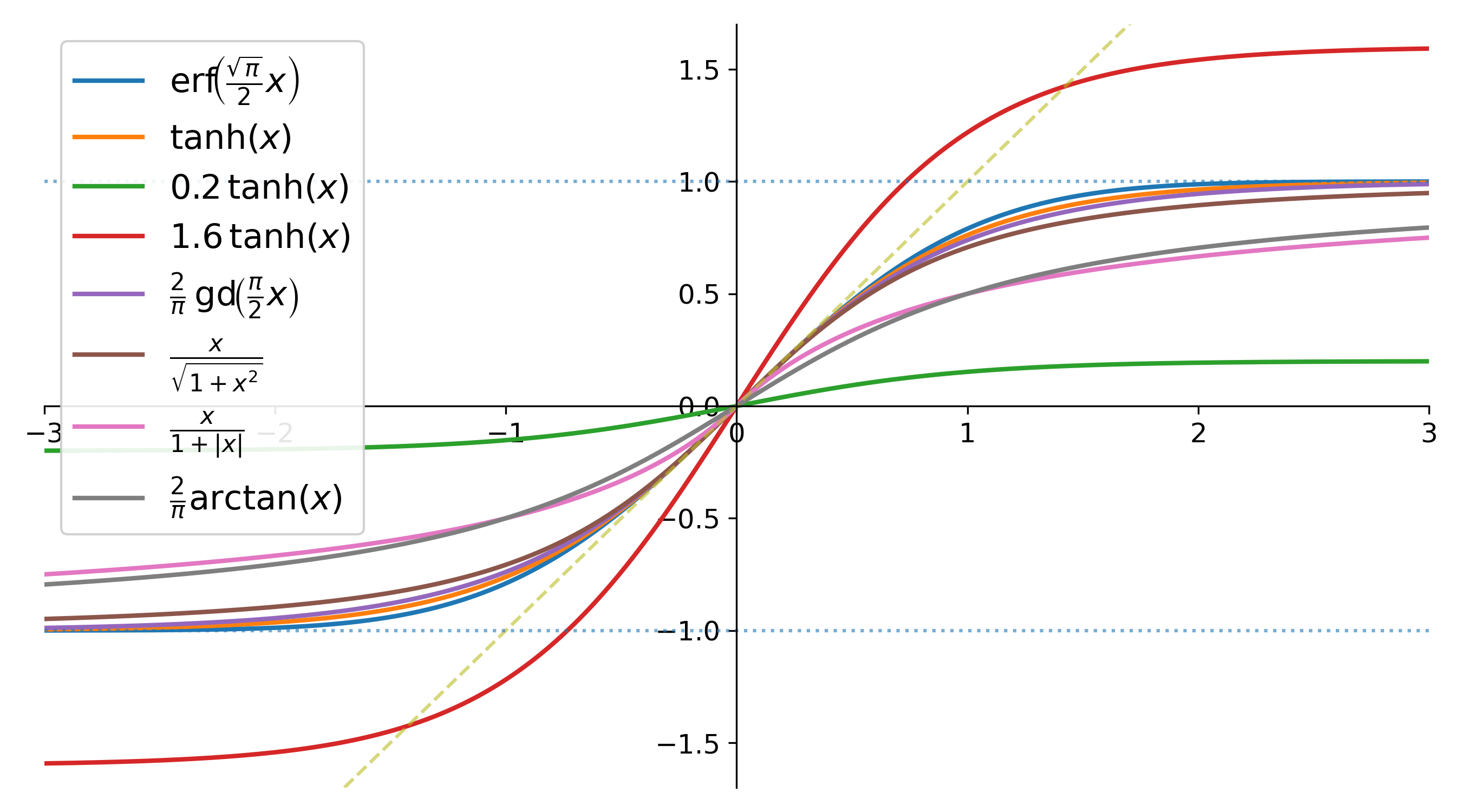}
  \vspace{-1pt}
  \caption{odd–sigmoid activations}
  \label{fig:fig1}
\end{wrapfigure}
Standard Gaussian i.i.d.\ initializations~(Xavier~\citep{glorot2010understanding}, He~\citep{he2015delving}, EOC~\citep{hayou2019impact} are derived under wide or infinite width assumptions and choose a single variance parameter to preserve signal propagation. In practice, we find that this design breaks down in deep and relatively narrow networks and under activation rescaling: small deviations in 
\clearpage
the variance can trigger saturation, or extreme learning rate sensitivity.
Motivated by these observations, we propose an activation aware initialization for fully feedforward neural networks~(FFNNs) with odd–sigmoid activations in the class $\mathcal{F}$~(Figure~\ref{fig:fig1}). For any $f\in\mathcal{F}$ with $\omega := 1/f'(0)$, we study the scalar dynamics of $x \mapsto f(a x)$ and choose the layerwise gains $a$ from a distribution with mean $\omega$ so that trajectories neither collapse to zero nor saturate. This gain design is then implemented via a simple initial weight matrix whose effective gain statistics match the desired behavior. Unlike Gaussian i.i.d. schemes, whose trainability is highly sensitive to the exact variance, the proposed initialization keeps both forward activations and backpropagated gradients in a well conditioned range across depth~\citep{raghu2017expressive}. 
Mean field analysis further shows that the corresponding gradient amplification factor $\chi_\ell$ remains close to $1$ over a significantly wider band of noise scales~\citep{schoenholz2016deep}.

We evaluate the proposed initialization in extensive experiments with odd–sigmoid activations, covering both the regime $\omega \approx 1$ and the more challenging regime $\omega \gg 1$ or $\omega \ll 1$. On standard image classification benchmarks, Gaussian initializations often fail or degrade sharply in deep, narrow, or rescaled-activation networks, whereas our scheme trains stably and is more data-efficient than these Gaussian baselines. We also study scaled and composite activations (e.g., $\alpha f(x)$, $f(\alpha x)$, and positive linear combinations) and observe that our initialization remains stable even for extreme scales (up to $\alpha \approx 10^{9}$). In PINN settings, appropriately scaled odd–sigmoid activations combined with our initialization consistently achieve lower losses than EOC-based baselines and do so without batch normalization (BN)~\citep{ioffe2015batch}, indicating that diagonal–plus–noise weights provide a practical alternative when Gaussian schemes become brittle. Our main contributions are as follows:

\begin{itemize}
  \item We define the odd–sigmoid function class and characterize its properties~(Section~\ref{4.1}).
  \item We view a deep feedforward network as a parallel ensemble of scalar dynamical systems, and use the resulting gain statistics to design a proposed initialization~(Section~\ref{4.2}).
  \item We show that the proposed initialization is more variance robust than standard Gaussian i.i.d.\ initializations and better preserves forward and backward signals~(Section~\ref{section4.3}).
  \item We empirically validate the proposed initialization across a wide range of activations on standard image classification benchmarks and PINNs~(Section~\ref{5.1} and \ref{55}).
\end{itemize}

\section{Related Work}
Classic initialization schemes were originally designed to stabilize layerwise variance, and later work showed that good performance also hinges on well-tuned statistics and momentum, with poor initialization causing saturation, variance collapse, and learning-rate fragility~\citep{Sutskever2013importance,Narkhede2022A}. For sigmoidal networks, follow-up studies adjust the mean or support of the weights or use more spread-out mappings to keep units responsive and accelerate convergence~\citep{Yilmaz2022SuccessfullyAE,Qiao2016MutualInformation,sodhi2014new}. Beyond first-order variance matching, curvature- and correlation-based criteria (Hessian-norm control, depth$\times$width bounds, and propagation analyses in constrained architectures such as input-convex networks and transformers) and architectural tweaks such as ReZero and dynamical-isometry–preserving setups all point toward initializations that are explicitly calibrated to the activation and the network’s depth and shape~\citep{skorski2021revisiting,iyer2023maximal,Hoedt2023PrincipledWI,Noci2022Signal,Bachlechner2021ReZeroIA,lee2024improved,Price2024DeepNN,blumenfeld2020beyond}. Our work makes this activation–initializer coupling explicit by formalizing an odd–sigmoid class and, for any $f$ in this class, deriving a closed-form noise scale that keeps preactivations in a high-gain regime up to a target depth, yielding dispersed forward signals and stable gradients without normalization~\citep{Noci2022Signal,skorski2021revisiting}.

On the activation side, early bounded nonlinearities (Sigmoid, Tanh) suffered from vanishing gradients, whereas ReLU variants improved efficiency but can produce dying neurons~\citep{Adeli2025fNIRS}. Surveys emphasize that no single activation is universally optimal and highlight desirable traits such as continuity, boundedness, monotonicity, symmetry, and smoothness~\citep{Alcaide2025TeLU,Adeli2025fNIRS}, with bounded, odd-symmetric forms helping to center signals and stabilize optimization—directly motivating our initializer. In parallel, adaptive and learned activation families (e.g., parametric activations for DL and PINNs, error-function–corrected ReLU variants, per-neuron self-activating networks, and tailored activations for zeroing neural networks) demonstrate that carefully structured nonlinearities can improve convergence and robustness~\citep{Jagtap2020Adaptive,Wang2023SAF,Ullah2025AHerfReLU,ssrn2025,Liu2025ZNN}, further supporting the need for initialization strategies aligned with activation geometry.

\section{PRELIMINARIES}
In this section, we introduce the notation used throughout the paper, review a simplified elementwise analysis of signal propagation~\citep{woorobust}, and
recall the classical edge of chaos mean field theory~\citep{schoenholz2016deep} and its associated Gaussian i.i.d. EOC initialization~(\citep{hayou2019impact}).

\paragraph{Notation.} We consider a feedforward neural network with $L$ layers. The dataset comprises $K$ pairs $\{(\bm{x}_i,\bm{y}_i)\}_{i=1}^K$, where $\bm{x}_i\in\mathbb{R}^{N_x}$ is the input and $\bm{y}_i\in\mathbb{R}^{N_y}$ is the corresponding target. For $\ell=1,\ldots,L$, the layerwise update is
\[
\bm{h}^{\ell}
= \mathbf{W}^{\ell}\bm{x}^{\ell-1}+\mathbf{b}^{\ell},
\qquad
\bm{x}^{\ell}
= f\!\left(\bm{h}^{\ell}\right)
\in \mathbb{R}^{N_{\ell}},
\]
where $\mathbf{W}^{\ell}\in\mathbb{R}^{N_{\ell}\times N_{\ell-1}}$ is the weight matrix, $\mathbf{b}^{\ell}\in\mathbb{R}^{N_{\ell}}$ is the bias vector, and $f(\cdot)$ denotes an activation function. We write $\mathbf{W}^{\ell}=[w_{ij}^{\ell}]$.

\paragraph{Signal Propagation Analysis.} Following the simplified framework of \citet{woorobust}, we analyze signal propagation in feedforward networks with an elementwise activation $f$.  
For convenience, all layers, including input and output, are assumed to have dimension $n$ (i.e., $N_{\ell}=n$ for all $\ell$).  
A key observation of this approach is that the usual matrix–vector multiplication $\mathbf{W}^{\ell}\bm{x}^{\ell-1}$ can be reformulated elementwise, such that each coordinate is expressed in terms of an effective self-scaling factor.  
Formally, for $\ell=0,\ldots,L-1$ and $i=1,\ldots,n$,  
\begin{equation}\label{eq:prelim-signal}
x_i^{\ell+1}
= f\!\left(a_i^{\ell+1}\,x_i^{\ell}\right),
\qquad
a_i^{\ell+1}
= w_{ii}^{\ell+1}
+ \sum_{j\neq i}\frac{w_{ij}^{\ell+1}x_j^{\ell}}{x_i^{\ell}},
\end{equation}
assuming that $\mathbf{b}^{\ell} = \mathbf{0}$. This elementwise decomposition highlights the role of $a_i^{\ell+1}$ as an effective gain that combines both diagonal and off-diagonal terms and provides a tractable basis for studying signal propagation in deep networks.

\paragraph{Edge of chaos.}
We briefly recall the ``edge of chaos'' viewpoint based on
mean field theory for fully connected networks with standard i.i.d.\ Gaussian
initialization, where the weights are sampled as
$w_{ij}^{\ell}\sim\mathcal N(0,\sigma_w^2/N_{\ell})$~\citep{poole2016exponential, schoenholz2016deep}.
In this setting, the preactivations $h_i^{\ell}$ are modeled as i.i.d.\ Gaussians
with layerwise variance $q^{\ell}$, and both forward signals and backward
gradients admit simple recursions in the infinite width limit. In particular, the mean field analysis of backpropagation shows that the squared gradient norm
obeys
\[
  \chi_{\ell+1}
  \;\approx\;
  \sigma_w^2\,\mathbb E\bigl[f'\bigl(\mathbf{h}^{\ell+1}\bigr)^2\bigr],
\]
where $\mathbf{h}^{\ell+1}$ denotes a typical preactivation at layer $\ell+1$.
The scalar factor $\chi_{\ell+1}$ is the average gradient amplification of
layer $\ell+1$: in the ordered phase ($\chi_{\ell+1}<1$) gradients
vanish exponentially with depth, while in the chaotic phase
($\chi_{\ell+1}>1$) they explode. The critical curve defined by
$\chi_{\ell+1}\approx 1$ is called the edge of chaos and separates
these two regimes. At this edge, information and gradients
can propagate through many layers~\citep{schoenholz2016deep}. 
Building on this perspective, \citep{hayou2019impact} refined the EOC analysis and showed that choosing $(\sigma_w^2,\sigma_b^2)$ on the EOC curve leads to better information propagation and faster training in deep networks.
Throughout this paper, we refer to Gaussian i.i.d.\ initializations with
$(\sigma_w^2,\sigma_b^2)$ on the EOC curve as EOC initialization.

\section{Methodology}\label{4}
We first define the odd–sigmoid function class~(Section~\ref {4.1}). Building on this, we propose a weight initialization for feedforward networks with odd–sigmoid activations~(Section~\ref{4.2}). In Section~\ref{section4.3}, we highlight how the proposed initialization differs from
standard Gaussian i.i.d.\ initializations, and show these differences through
theoretical analysis and numerical experiments. 
The proofs of the theoretical results are provided in Appendix~\ref{Appendix A}, and additional numerical
results are presented in Appendix~\ref{Appendix B}.

\subsection{Odd-Sigmoid Function Class}\label{4.1}
In practical neural networks with sigmoidal activations, both forward signals and backpropagated gradients tend to saturate or vanish as the depth increases. Our goal in this work is to move beyond a single sigmoid and consider a broad class of odd–sigmoid activation functions, and to design a weight initialization that prevents such vanishing for any activation in this class. To this end,
Section~\ref{4.1} introduces the odd–sigmoid function class and establishes its basic properties, which will be used in our subsequent analysis. Formal definitions of the activation functions used in this paper are provided in Appendix~\ref{activation_define}.


\begin{definition}\label{def:odd-sig}
A function $f:\mathbb{R}\to\mathbb{R}$ is an \textbf{odd-sigmoid function} if it satisfies the following:
\begin{itemize}
    \item[(i)] \textbf{Regularity:} $f \in C^{1}(\mathbb{R})$.
    \item[(ii)] \textbf{Odd symmetry:} $f(-x) = -f(x)$ for all $x \in \mathbb{R}$.
    \item[(iii)] \textbf{Boundedness:} $\sup_{x\in\mathbb{R}} |f(x)| < \infty$.
    \item[(iv)] \textbf{Strict monotonicity:} $f'(x) > 0$ for all $x \in \mathbb{R}$.
    \item[(v)] \textbf{Slope decay:} $f'$ is strictly decreasing on $[0,\infty)$.
\end{itemize}
\end{definition}
Denote the class of all odd-sigmoid functions as $\mathcal{F}$. 
In the following, we establish several basic properties of functions in $\mathcal{F}$.


\paragraph{Pitchfork Bifurcation.}
Recall that \(x^\ast\in\mathbb{R}\) is a fixed point of \(f:\mathbb{R}\to\mathbb{R}\) if \(f(x^\ast)=x^\ast\).
For \(f\in\mathcal{F}\) define \(\omega_f:=1/f'(0)>0\) and, unless stated otherwise, write \(\omega\) simply.
Consider \(\phi_a(x):=f(a x)\) for \(a>0\). 
We say that \(\phi_a\) has a pitchfork bifurcation at \(a=\omega\) if it has exactly one fixed point \(\{0\}\) for \(0<a<\omega\) and exactly three fixed points \(\{0,\pm\xi_a\}\) for \(a>\omega\), with the nonzero points occurring as a symmetric pair. We show that this holds for every \(f\in\mathcal{F}\).

\begin{proposition}\label{prop1}
Suppose $f\in\mathcal{F}$ with $\omega:=1/f'(0)$, and for a fixed $a>0$ define $\phi_a(x):=f(ax)$. Then
\begin{itemize}
  \item[(i)] If \( 0<a\le \omega\), then $\phi_a(x)$ has a unique fixed point $x^*=0$.  
  \item[(ii)] If \( a>\omega\), then $\phi_a(x)$ has three distinct fixed points: $x^* = -\xi_a,\;0,\;\xi_a$ such that  $\xi_a > 0$.
\end{itemize}
\end{proposition}

The following theorem establishes convergence properties of \(x_{n+1}=f(a x_n)\) for all \(x_0>0\).

\begin{theorem}\label{thm1}
Suppose $f\in\mathcal{F}$ with $\omega:=1/f'(0)$, and for a fixed $a>0$ define 
\[
 x_0>0, \qquad x_{n+1}=\phi_a(x_n),\qquad n=0,1,2,\dots.
\]
Then the sequence $\{x_n\}$ converges for every $x_0>0$.
Furthermore, 
\begin{itemize}
    \item[$(1)$] if $0 < a \leq \omega$, then $x_n\to 0$ as $n\to \infty$.
    \item[$(2)$] if $a > \omega$, then $x_n\to \xi_a$ as $n\to \infty$.
\end{itemize}
\end{theorem}

According to Theorem~\ref{thm1}, when $a>\omega$, for any initial value $x_0>0$ (resp.\ $x_0<0$), the sequence defined by $x_{n+1}=f(a x_n)$ converges to $\xi_a$ (resp.\ $-\xi_a$) as $n\to\infty$. From a signal propagation viewpoint in FFNN~(Equation~\ref{eq:prelim-signal}), this implies that activations do not vanish as depth increases. 
See Figures~\ref{fig6} and \ref{fig111} for the convergence of the iterates \(x_{n+1}=\phi_a( x_n)\) as a function of \(a\) and the initial value \(x_0\). Proposition~\ref{prop2} and Corollary~\ref{cor2} analyze the iteration with a coefficient sequence \(\{a_m\}\) that may vary across steps. They show that excessively large or small \(a_m\) drives the dynamics into saturation, highlighting the importance of appropriately scaling \(\{a_m\}\) when designing the initialization~(Figure~\ref{fig66}).

\begin{corollary}\label{cor1}
Let $f_1,f_2\in\mathcal{F}$ and let $c_1,c_2\ge0$ with $(c_1,c_2)\neq(0,0)$. If $g=c_1 f_1+c_2 f_2$, then $g\in\mathcal{F}$.
Furthermore, it holds that 
\[
\frac{1}{\omega_g}=\frac{c_1}{\omega_{f_1}}+\frac{c_2}{\omega_{f_2}}.
\]
\end{corollary}

Since \(\mathcal{F}\) is closed under addition, 
more generally, any finite sum \(\sum_{j=1}^M c_jf_j \in\mathcal{F} \) for all $c_j \geq 0$ with $(c_1,\ldots,c_M)\neq \bm{0}$ and \(f_j\in\mathcal{F}\).

\subsection{proposed weight initialization method}\label{4.2}
We build on the theoretical results in Section~\ref{4.1} to design an initialization scheme that keeps forward signals well dispersed and avoids saturation even in very deep networks. We then compare the proposed initialization with standard Gaussian i.i.d.\ schemes~(e.g., Xavier, He, EOC) from both forward and backward pass perspectives, studying signal and gradient propagation through theoretical analysis and experiments.

\begin{figure}[t!]
\centering 
\includegraphics[width=0.63\textwidth]{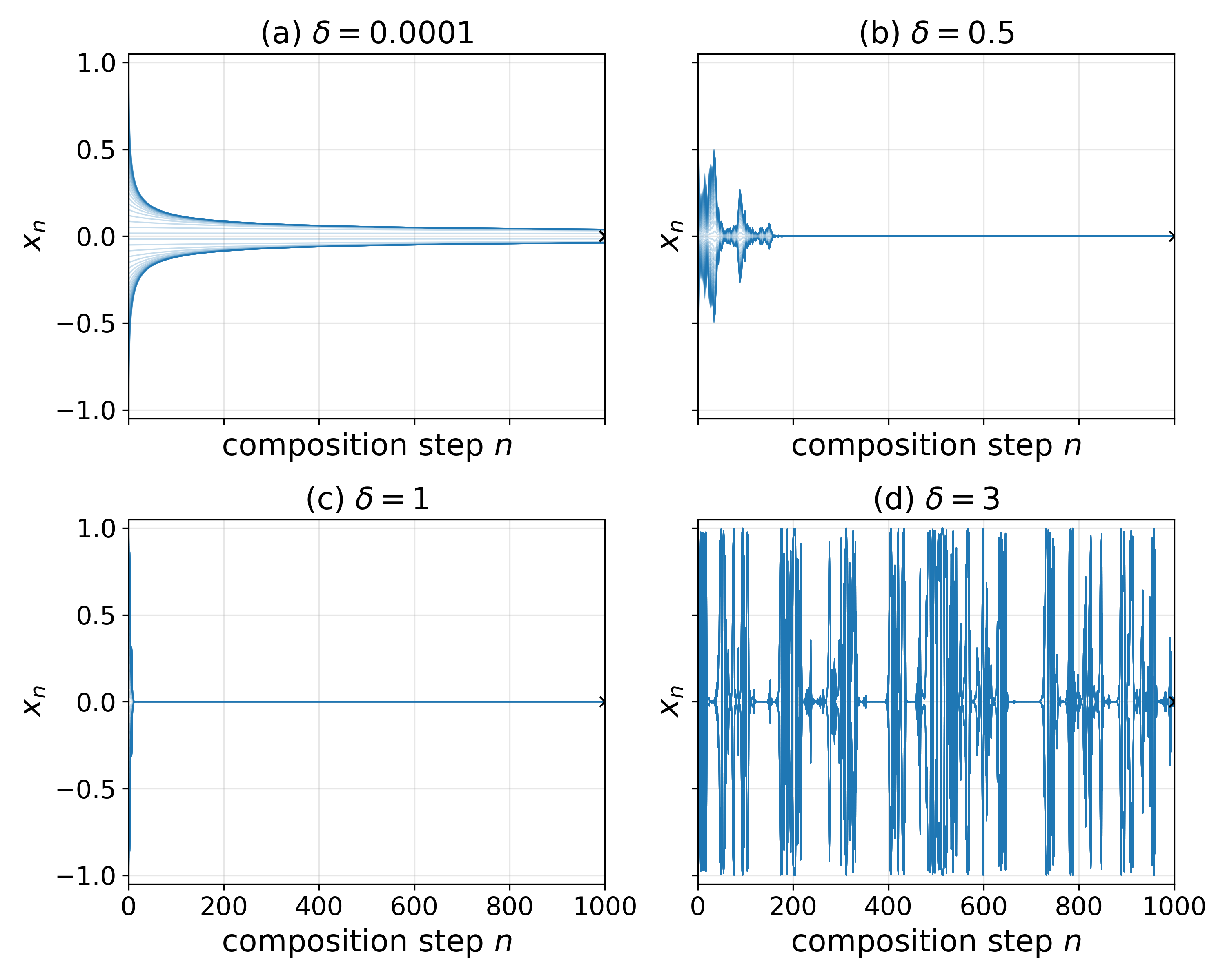}
\caption{
Iterated dynamics of the $\tanh$ activation under randomly varying gains.
For each panel, we fix a sequence $(a_1,\dots,a_{1000})$ sampled i.i.d.\ from
$\mathcal{U}[1-\delta,\,1+\delta]$ and iterate
$x_{n+1} = \tanh(a_{n+1} x_n)$ from 60 initial points $x_0 \in [-1,1]$. Panels (a)--(d) correspond to $\delta=0.0001,\,0.5,\,1,\,3$.}
\label{fig66}
\end{figure}

\paragraph{Proposed Weight Initialization} Consider $f\in \mathcal{F}$ with $\omega := 1/f'(0)>0$ and a target depth $L$.
We initialize each layer as $\mathbf{W}^{\ell} = \mathbf{D}^{\ell} + \mathbf{Z}^{\ell} \in \mathbb{R}^{N_{\ell}\times N_{\ell-1}}$, where $(\mathbf{D}^{\ell})_{ij} = \omega$ if $i \equiv j \pmod{N_{\ell-1}}$ and $0$ otherwise. The noise matrix $\mathbf{Z}^{\ell}$ is sampled with noise scale $\sigma_z$ as
\[
  (\mathbf Z^{\ell})_{ij} \sim \mathcal N\!\Bigl(0,\ \sigma_z^2 / N_{\ell-1}\Bigr).
\]

The value of $\sigma_z$ is determined via a calibrated scalar surrogate model.
Let $p_{\mathrm{real}}=0.4$~(see Section~\ref{b.7}) be the desired negative rate at depth $L$ in the full
network, and let $p_{\mathrm{sur}}(L)$ denote the surrogate target used in the
scalar model. As shown in Appendix~\ref{app_negative}, for shallow depths
the optimal surrogate target remains close to $p_{\mathrm{real}}$, whereas for larger
depths it decays approximately exponentially. In particular, a least-squares fit
for $L\ge 10$ yields
\[
  p_{\mathrm{sur}}(L)\;\\\approx\;2.05\,e^{-0.133L}.
\]
In practice we therefore set
\[
  p(L)
  \;:=\;
  \begin{cases}
    p_{\mathrm{real}}, & L \le L_{th},\\[0.3em]
    p_{\mathrm{sur}}(L), & L > L_{th},
  \end{cases}
\]
with $L_{th}=10$, and use this $p(L)$ in the closed-form calibration
\[
  \sigma^\ast(p(L),L,\omega)
  \;=\;
  -\frac{\omega}{\displaystyle \Phi^{-1}\!\left(\frac{1-(1-2p(L))^{1/L}}{2}\right)},
\]
where $\Phi$ denotes the standard normal cumulative distribution function. We
then set $\sigma_z := \sigma^\ast(p(L),L,\omega)$.

For the learning rate, we use a band proportional to $\omega$; for Adam~\citep{kingma2014adam},
\[
  \eta\ \in\ [\,10^{-5}\omega,\;10^{-3}\omega\,].
\]
We discuss the learning-rate range $\eta$ and the surrogate calibration in
Section~\ref{55} and Appendix~\ref{app_negative}.

\paragraph{Derivation of the initialization.}
Using the elementwise formulation from Section~\ref{4.1}, the usual matrix–vector
product $\mathbf{W}^{\ell}\mathbf{x}^{\ell-1}$ can be rewritten as
\begin{equation}
x_i^{\ell+1}
= f\!\left(a_i^{\ell+1}\,x_i^{\ell}\right),
\qquad
a_i^{\ell+1}
= w_{ii}^{\ell+1}
+ \sum_{j\neq i}\frac{w_{ij}^{\ell+1}x_j^{\ell}}{x_i^{\ell}},
\end{equation}
so that $a_i^{\ell+1}$ plays the role of an effective gain for neuron $i$ in layer $\ell+1$.
Under our proposed initialization, Lemma~\ref{lem1} shows that
$a_i^{\ell+1}$ is approximately Gaussian with mean $\omega$ and a data-dependent variance.

\begin{lemma}\label{lem1}
Using the elementwise formulation in \eqref{eq:prelim-signal} and employing the proposed weight initialization, fix an arbitrary layer $\ell$ and index $i$ such that $x_i^\ell\neq 0$. Then, conditionally on $x^\ell$,
\begin{equation}\label{var_a}
a_i^{\ell+1}\ \sim\ \mathcal N\!\Biggl(\,\omega,\ \frac{\sigma_z^2}{N_\ell}\Bigl(1+\sum_{j\ne i}\Bigl(\frac{x_j^\ell}{x_i^\ell}\Bigr)^2\Bigr)\Biggr).
\end{equation}
Moreover, if $|x_j^\ell|\le M$ for all $j$ and $|x_i^\ell|\ge \varepsilon>0$, then
\[
  \frac{\sigma_z^2}{N_\ell}
  \;\le\;
  \Var\!\bigl(a_i^{\ell+1}\,\big|\,x^\ell\bigr)
  \;\le\;
  \sigma_z^2\,\frac{M^2}{\varepsilon^2}.
\]
\end{lemma}
\medskip

The distribution of the effective gain $a_i^{\ell}$ is crucial for understanding
signal propagation. We first analyze its mean, denoted by $\mu_a$. This mean is
determined by the diagonal entries of the proposed initialization matrix $\mathbf{D}$.
In particular, we are interested in the supercritical regime $\mu_a > \omega$, which is analyzed in Theorem~\ref{thm4.6} below.

\begin{theorem}
\label{thm4.6}
Let $f\in\mathcal F$ be an odd–sigmoid activation with $\omega := 1/f'(0)$, and fix any $\varepsilon>0$.
Consider the feedforward network and proposed initialization,
except that the diagonal element is set to
$a_0 := \omega + \varepsilon,$
and let $a_i^{\ell+1}$ be the effective gain defined in~\eqref{eq:prelim-signal}. Fix a tolerance $\gamma\in(0,1)$ and a finite depth
$L\in\mathbb{N}$. Then there exist a threshold depth
$L_{0}\le L$ and a noise threshold $\sigma_0>0$ such that,
for all $0<\sigma_z\le\sigma_0$,
\begin{equation}\label{eq:gain-var-supercritical}
  \mathbb{P}\Bigl(
    (1-\gamma)\,\sigma_z^2
    \;\le\;
    \Var(a_i^{\ell+1}\mid x^\ell)
    \;\le\;
    (1+\gamma)\,\sigma_z^2
    \quad\text{for all } L_{0}\le \ell<L,\ 1\le i\le N_\ell
  \Bigr)
  \;\ge\; 1-\gamma.
\end{equation}
\end{theorem}

The data dependency of the variance term can be characterized as follows.
For each layer $\ell$ and neuron $i$ with $x_i^\ell(\sigma_z)\neq 0$, define
$
  R_\ell(\sigma_z;i)
  := \frac{\|\mathbf x^\ell(\sigma_z)\|_2^2}{N_\ell\,(x_i^\ell(\sigma_z))^2}.$
Here, $x_i^\ell(\sigma_z)$ denotes the activation at layer $\ell$ under noise scale $\sigma_z$.
Theorem~\ref{thm4.6} implies that when the diagonal mean $\mu_a>\omega$,
there exist a depth threshold $L_0$ and a noise threshold $\sigma_0>0$ such that
$R_\ell(\sigma_z;i)\approx 1$ for all $\ell\ge L_0$ and $\sigma_z\le\sigma_0$.
This means that all neurons in deep layers converge to nearly identical values,
indicating saturation of activations and loss of data dependence~(Figure~\ref{fig6}).
Similarly, for $\mu_a<\omega$, Theorem~\ref{thm4.6} yields the same qualitative behavior. Deeper networks require smaller initialization variance to remain trainable~\citep{poole2016exponential, schoenholz2016deep}. Motivated by this, we set the diagonal mean to the critical value $a_0 = \omega$ when initializing deep networks.

Proposition~\ref{prop2} and Corollary~\ref{cor2} suggest that, for most
gains $a_i^{\ell}$ across layers, if the variance of the gains is either too
small or too large, the activations tend to saturate. 
We define the negative rate at depth $L$ as
$
  \pi_L(\sigma) \;:=\; \mathbb{P}\bigl(x_L < 0 \,\big|\, x_0 > 0\bigr),
$
i.e., the probability that a neuron that starts with a positive activation
ends up with a negative sign at layer $L$. Under the surrogate model introduced below, it suffices to track a single neuron with $x_i^1>0$, since for any $f\in\mathcal{F}$ the activation is sign-preserving and odd. Consequently, the probability that a positive entry becomes negative at layer $L$ is equal to the probability that a negative entry becomes positive, and does not depend on the particular choice of $x_i^1$.

To handle this dependence on the current activations, we approximate the
gain distribution by a scalar surrogate model, $a_i^{\ell+1} \sim \mathcal N(\omega,\sigma_z^2)$, and estimate the negative rate from this Gaussian approximation. Our initialization has a nonzero mean gain $\omega$, which makes sign changes along a given coordinate well defined across layers. We interpret frequent sign flips during forward propagation as a form of information loss and therefore use the surrogate calibration to control the negative rate at the final layer. In practice, we set the noise level so that the empirical negative rate at depth $L$ is close to $p_{\mathrm{real}}=0.4$, preserving most sign information while still retaining sufficient randomness for feature learning.
The next theorem shows how, given a target
negative rate at depth $L$, we can compute the corresponding noise scale
$\sigma^\ast$ in this surrogate model.

\begin{theorem}\label{thm:closed-sigma}
Fix a target $p\in[0,\tfrac12)$, a depth $\ell\in\mathbb N$, and $\omega>0$.
There exists a unique $\sigma^\star=\sigma^\star(p,\ell,\omega)>0$ such that $\pi_\ell(\sigma^\star)=p$, and it is given by
\begin{equation}\label{eq:sigma-star}
\sigma^\star(p,\ell,\omega)\;=\; -\,\frac{\omega}{\Phi^{-1}\!\left(\dfrac{1-(1-2p)^{1/\ell}}{2}\right)}\,.
\end{equation}
\end{theorem}
Figure~\ref{negative_rate_f} shows how the negative rate varies with $L$ and $p$.
As shown in Figure~\ref{chi_value}, for the proposed initialization $\chi_{\ell}$ at the $\sigma^\ast(p,L,\omega)$
stays within a few percent of $1$ for both $p=0.01$ and $p=0.49$ over depths
up to $2\times 10^{5}$, suggesting that the calibration preserves
trainable gradient scales. The relationship between the FFNN-level negative rate and our scalar surrogate, as well as the corresponding validation experiments, is detailed in
Appendix~\ref{app_negative}.

\subsection{Comparative Analysis of Gaussian and Proposed Initializations}\label{section4.3}
Gaussian i.i.d.\ schemes such as Xavier, He, and EOC all draw weights
independently from $\mathcal{N}(0,\sigma_w^2/N_{\ell})$, differing only in the
choice of $\sigma_w$. In this section, we contrast our method
with EOC, examining their forward and backward signal propagation both
theoretically and empirically.

\begin{figure}[t!]
\centering 
\includegraphics[width=1\textwidth]{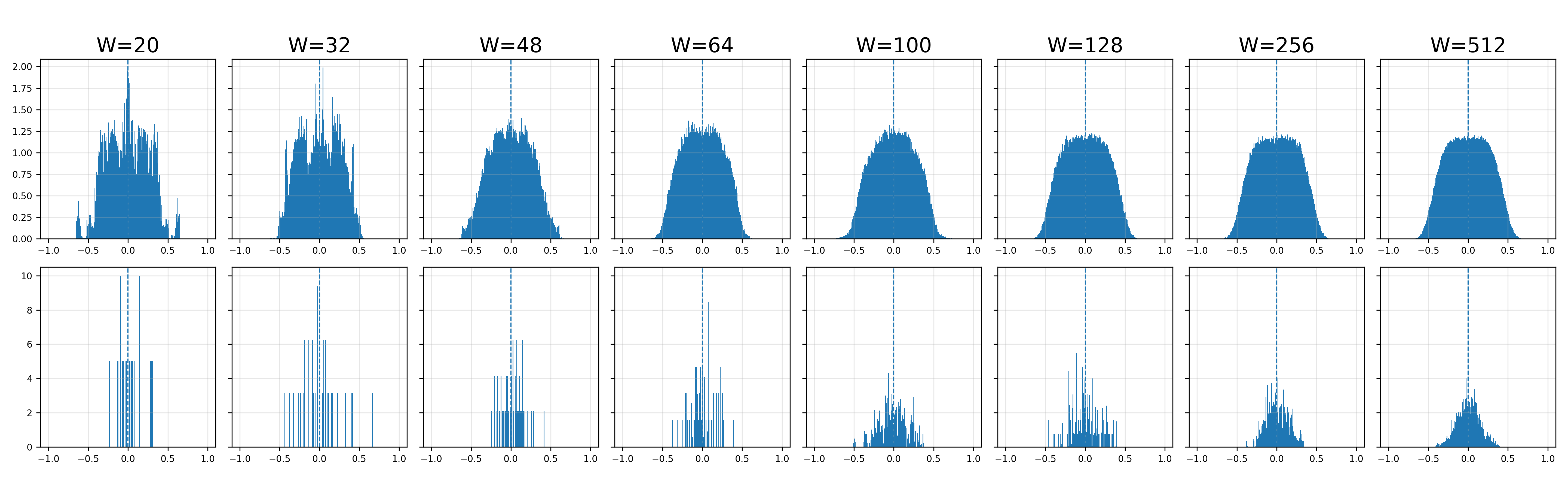}
\caption{Last layer activation histograms for tanh networks with depth $L=1000$ and varying width under the proposed initialization \textbf{(top row)} and the EOC initialization \textbf{(bottom row)}. Each column corresponds to a different hidden width.}
\label{fig_histo_width}
\end{figure}

\paragraph{Forward Pass.}
As shown in Figures~\ref{fig_histo_depth}, \ref{fig_histo_depth0.1}, and \ref{fig_histo_depth10}, our proposed initialization keeps the activations well dispersed even at depth $10{,}000$, whereas the EOC initialization drives them toward saturation near zero.
Moreover, Figures~\ref{fig_histo_width}, \ref{fig_histo_width_0.1}, and \ref{fig_histo_width10} show that, for EOC, signal propagation degrades as the width decreases, reflecting that this initialization is derived under an infinite width (or sufficiently wide) assumption. In contrast, our initialization is obtained by directly
controlling the effective gain $a_i^{\ell}$ and does not rely on any large width
approximation, which leads to stable forward propagation even in deep and
relatively narrow networks. The theoretical motivation for forward signal propagation with odd-sigmoid activations is discussed in Sections~\ref{4.1} and \ref{4.2}.

\paragraph{Backward Pass.}
Although our initialization is derived without mean field assumptions, for the
backward pass we adopt the standard mean field framework as an analytical tool to
study gradient propagation. Let $\mathbf g^\ell = \partial\mathcal L/\partial \mathbf x^\ell$ denote the gradient at depth
$\ell$ for a scalar loss $\mathcal L$.
Under the assumptions, the
backpropagated gradients satisfy the form
$
  \frac{1}{N_\ell}\,\mathbb E\bigl\|\mathbf g^\ell\bigr\|_2^2
  \;\approx\;
  \chi_{\ell+1}\,
  \frac{1}{N_{\ell+1}}\,
  \mathbb E\bigl\|\mathbf g^{\ell+1}\bigr\|_2^2,
$
where $\chi_{\ell+1}$ is the average gradient amplification factor of layer
$\ell+1$.  Values $\chi_{\ell+1}\ll 1$ correspond to vanishing gradients,
whereas $\chi_{\ell+1}\gg 1$ lead to exploding gradients; thus, keeping
$\chi_{\ell+1}$ close to $1$ across layers is essential for stable training
in very deep networks.
For our proposed initialization we show in
Appendix~\ref{chi_app} that
$\chi_{\ell+1}$ can be written as $\chi_{\ell+1}
  \;\approx\;
  (\omega^2 + \sigma_z^2)\,
  \mathbb E\bigl[f'(\mathbf h^{\ell+1})^2\bigr]$.
For comparison, under a standard i.i.d.\ Gaussian initialization, the corresponding factor takes the form $\chi_{\ell+1}^{*}\approx \sigma_w^2\,\mathbb E[f'(\mathbf h^{\ell+1})^2]$.

\begin{wrapfigure}{r}{0.43\textwidth} 
  \vspace{-1pt}
  \centering
  \includegraphics[width=\linewidth]{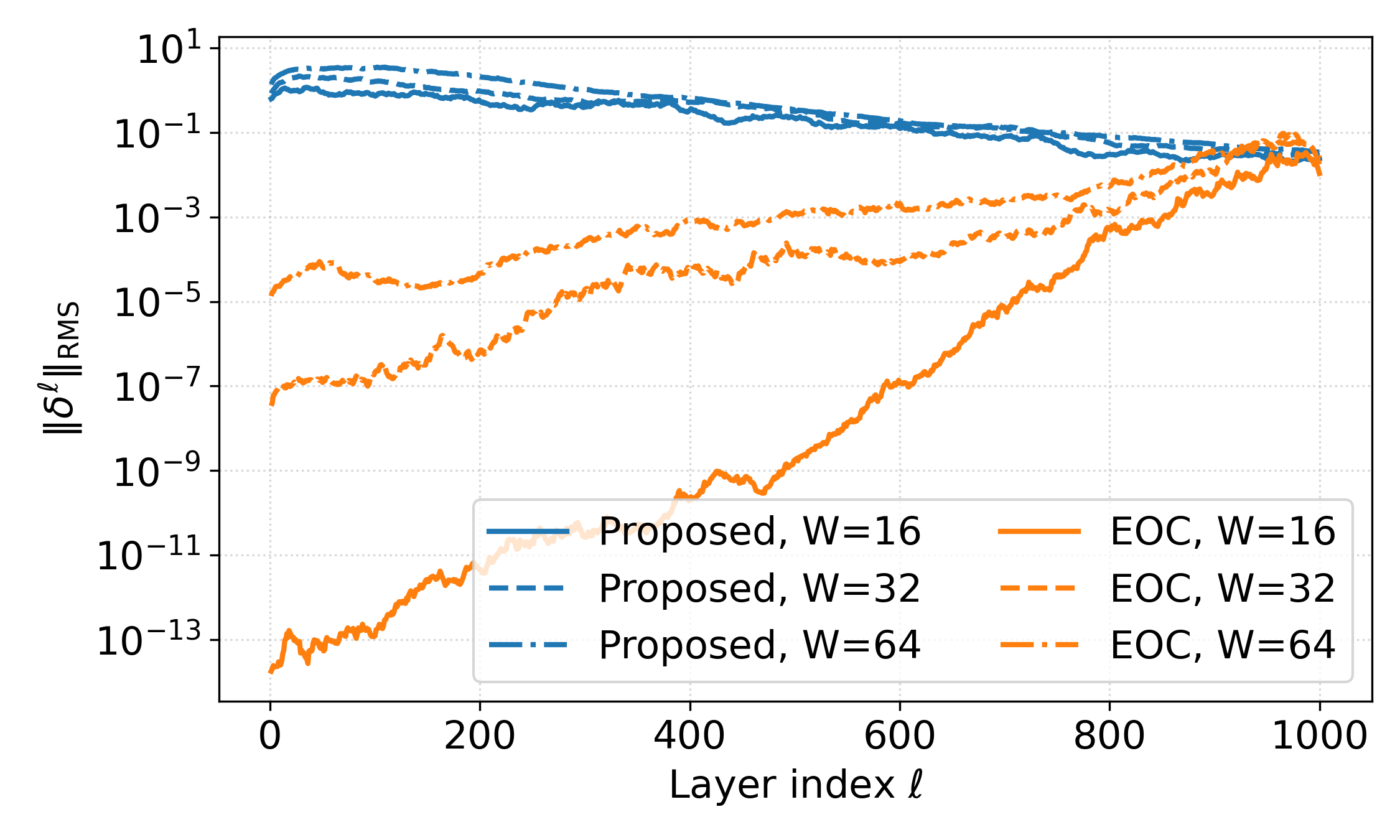}
  \vspace{-1pt}
\caption{
Layerwise gradient norms at initialization for a $\tanh$ FFNN
of depth $L=1000$. We set $\boldsymbol{\delta}^\ell := \partial L / \partial \mathbf{h}^\ell$.}
\label{fig:gradieont_norm}
\end{wrapfigure}
For our proposed initialization, as $\sigma_z$ grows, $\mathbb E[f'(\mathbf h^{\ell+1})^2]$ decreases, so the
product stays close to $1$ over a relatively wide range of $\sigma_z$.  In
contrast, for standard Gaussian i.i.d.\ initialization
$\chi_{\ell+1}^*(\sigma_w)\approx\sigma_w^2\,\mathbb E[f'(\mathbf h^{\ell+1})^2]$
has no analogous $\omega^2$ term and is much more sensitive to perturbations of $\sigma_w$ around its critical value.  
Therefore, $\chi_{\ell+1}$ remains close to $1$ over a wide range of noise scales $\sigma_z$, leading to much more robust training~(Figure~\ref{chi_heatmap_tanh}).
Consistent with this
analysis, Figure~\ref{fig:heatmap_tanh} indicates that our initialization supports accurate training over a much broader region in $(L,\sigma)$ and across activation scales, whereas Gaussian initializations remain effective only in a narrow band of $\sigma$ and at relatively shallow depths. Figure~\ref{fig:gradieont_norm} shows the gradient norm for the quadratic loss $L=\tfrac12\|\mathbf{y}\|_2^2$ as a function of depth. Under our proposed initialization, the gradient norm is effectively preserved
over very deep networks and remains stable across different widths. In contrast, EOC and Gaussian initializations exhibit rapid gradient decay at moderate widths, consistent
with their reliance on infinite width assumptions, whereas our scheme maintains well scaled forward activations and backward gradients
even in finite width regimes.

\begin{figure}[t!]
\centering 
\begin{tabular}{ccc}
\begin{subfigure}[b]{0.26\textwidth}
    \centering
    \includegraphics[width=\textwidth]{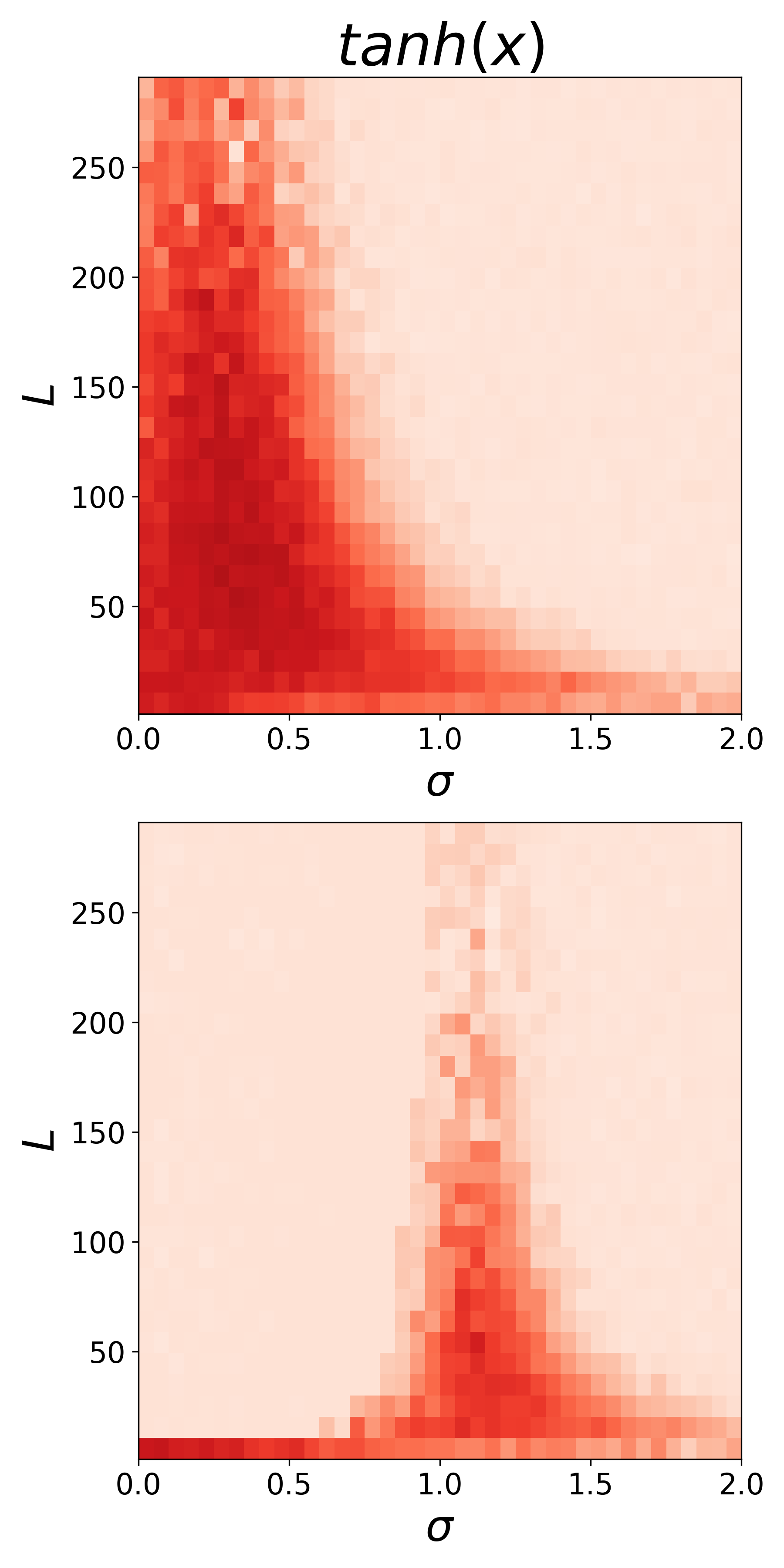}
\end{subfigure} &
\begin{subfigure}[b]{0.26\textwidth}
    \centering
    \includegraphics[width=\textwidth]{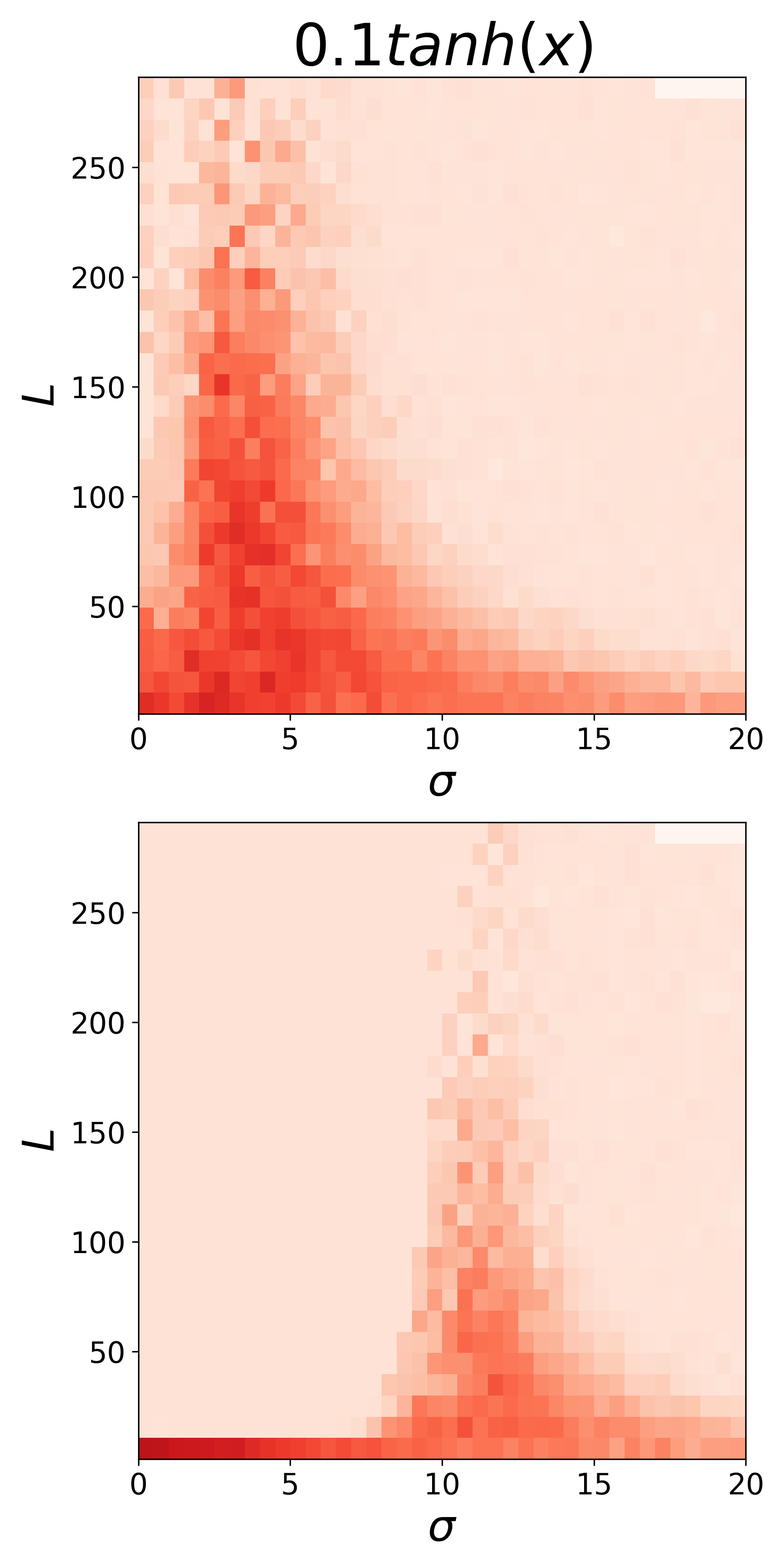}
\end{subfigure} &
\begin{subfigure}[b]{0.26\textwidth}
    \centering
    \includegraphics[width=\textwidth]{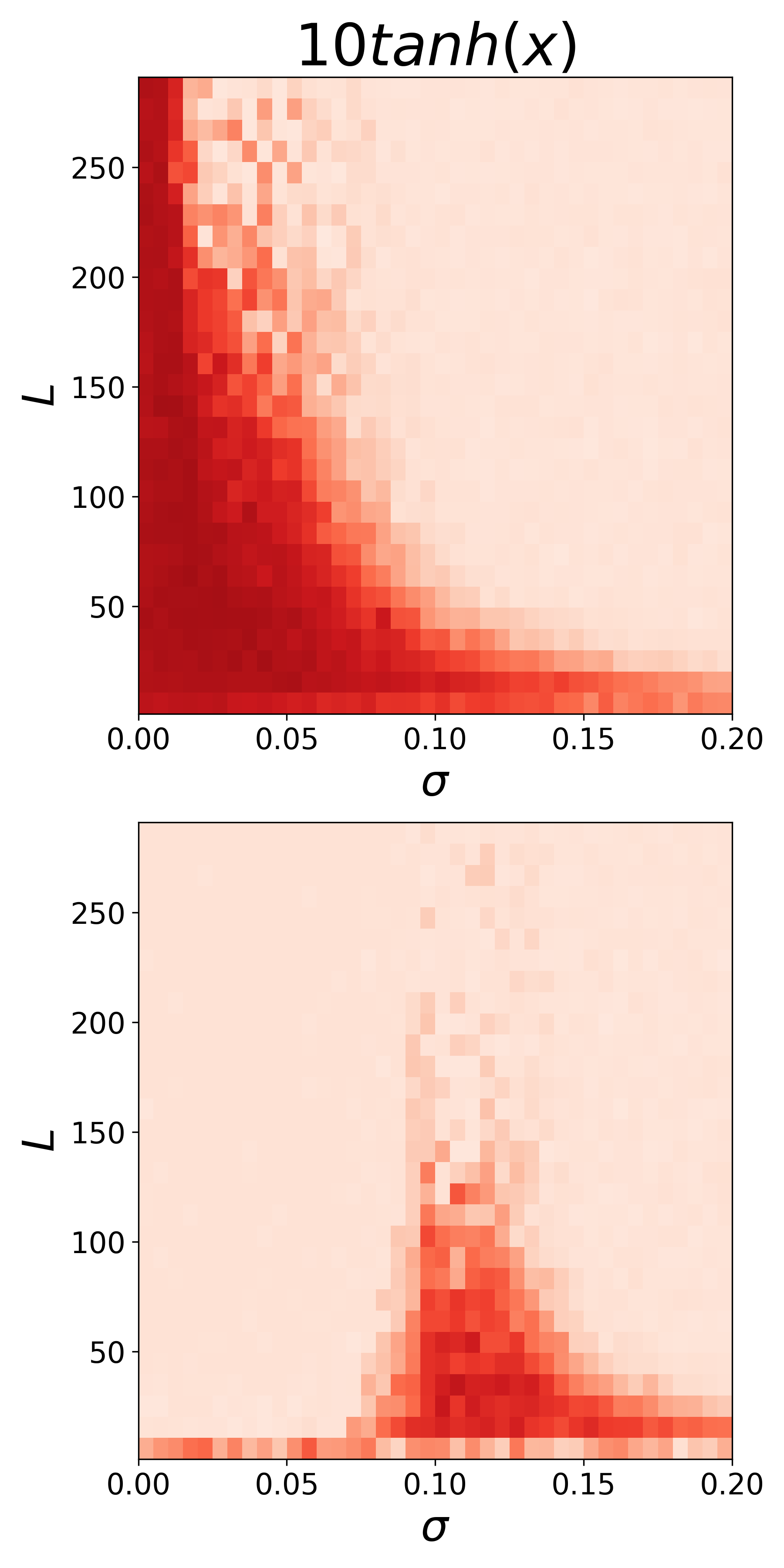}
\end{subfigure} 
\end{tabular}
\caption{Heatmaps of validation accuracy on MNIST as a function of depth $L$ 
and noise scale $\sigma$ for three activations
$\tanh(x)$, $0.1\tanh(x)$, and $10\tanh(x)$.
Each model is a FFNN with width $128$ trained with Adam.
Each cell shows the mean validation accuracy over 3 runs.
\textbf{(Top row)} Proposed.
\textbf{(Bottom row)} EOC.
Darker colors indicate higher validation accuracy.}

\label{fig:heatmap_tanh}
\end{figure}

\section{Experiments}\label{sec:5}
In this section, we evaluate the proposed initialization method.
Section~\ref{5.1} studies widely used activations with $\omega\approx 1$, comparing our method against Gaussian initializations (Xavier, He, EOC) in terms of data efficiency and the dependence on network width and depth. 
Section~\ref{55} then turns to activations with $\omega \not\approx 1$, where
we analyze performance, the ability to train effectively without
batch normalization, learning-rate sensitivity, and the trainability of
networks across a wide range of activation scales, including experiments on
physicsinformed neural networks~(PINNs). The experimental setting is described in Appendix~\ref{app.nn}.

\subsection{Experiments with $\omega \approx 1$ Activations}\label{5.1}

\begin{table*}[t!]
\centering
\caption{Mean of the best validation accuracies within 50 epochs over 10 independent runs for a 50 layer FFNN (512 units) on MNIST and Fashion MNIST, trained on 100 or 500 sample subsets.}

\label{table3}
\resizebox{13cm}{!}{%
\begin{tabular}{ll cc cc cc cc cc cc}
\toprule
\multirow{2}{*}{Dataset} & \multirow{2}{*}{Method} & \multicolumn{2}{c}{$\tanh(x)$} & \multicolumn{2}{c}{$\operatorname{erf}(x)$} & \multicolumn{2}{c}{$\arctan(x)$} & \multicolumn{2}{c}{$\operatorname{gd}(x)$} & \multicolumn{2}{c}{$\operatorname{softsign}_2(x)$} & \multicolumn{2}{c}{$\operatorname{softsign}_1(x)+\operatorname{softsign}_2(x)$} \\
\cmidrule(lr){3-14}
 & & 100 & 500 & 100 & 500 & 100 & 500 & 100 & 500 & 100 & 500 & 100 & 500\\
\midrule
\multirow{4}{*}{MNIST}
& Xavier    & 64.00 & 79.85 & 66.15 & 84.63 & 64.58 & 81.30 & 60.87& 84.18 &60.52 &81.38 &32.63&53.97\\
& He   & 41.65 & 72.55 & 21.05 & 48.37 & 51.40 & 77.68 & 42.50 & 72.75 &48.35&76.18 &11.35&10.02\\
& EOC  & 59.83 & 82.92 & 64.58 & 82.97 & 64.52 & 83.73 & 62.30 & 84.05 &66.57&82.48&59.60&71.57\\
& Proposed  & \textbf{68.23} & \textbf{86.75} & \textbf{66.53} & \textbf{\underline{87.13}} & \textbf{67.63} & \textbf{86.82} & \textbf{66.38} & \textbf{86.23} & \textbf{\underline{69.51}} & \textbf{86.43} &\textbf{65.65}&\textbf{84.02}\\

\midrule
\multirow{4}{*}{FMNIST}
& Xavier    & 67.65 & 74.53 & 68.92 & 76.13 & 66.75 & 73.78 & 67.10 & 72.65 &66.85 &74.47&49.38&69.88\\
& He   & 61.13 & 74.60 & 44.60 & 65.25 & 62.92 & 76.75 & 62.88 & 76.55 &64.97&75.22&11.50&11.70\\
& EOC  & 67.22 & 76.87 & 66.85 & 73.83 & 66.93 & 76.45 & 67.00 & 77.63 &67.08&75.67&65.48&73.07\\
& Proposed  & \textbf{70.67} & \textbf{77.43} & \textbf{70.63} & \textbf{78.17} & \textbf{\underline{71.55}} & \textbf{77.92} & \textbf{68.62} & \textbf{77.80} & \textbf{68.17} & \textbf{\underline{78.33}} &\textbf{67.93}&\textbf{76.43} \\
\bottomrule
\label{table222}
\end{tabular}}
\end{table*}

\paragraph{Dataset Efficiency.}
Table~\ref{table222} presents the best validation accuracy within 50 epochs for a 50 layer FFNN~(512 units), trained on small subsets of size 100 or 500. We compare four initializations across odd–sigmoid activations \(f\in\mathcal{F}\). The proposed initialization attains the top accuracy in all settings, across both datasets and both sample regimes, with the improvements most pronounced at 100 samples. These results indicate that our method is data efficient, achieving higher validation accuracy with limited data. 

\paragraph{Network Size Independence.} 
We assess how independent the proposed initialization is of depth
and width in networks. Figures~\ref{dp1}, \ref{dp2}, \ref{dp3}, and \ref{dp4} show validation accuracy versus depth for FFNNs~(width 64) on MNIST, Fashion MNIST, CIFAR-10, and CIFAR-100, using odd sigmoid activations and four initializations. 
Each panel fixes one activation and shows the best validation
accuracy over 10 epochs for depths $L \in \{20,50,100,150,200\}$. 
Across all datasets, activations, and depths,
the proposed initialization consistently achieves the highest validation
accuracy among all Gaussian initializations. Figures~\ref{wd1} and \ref{wd2} further examine the effect of width in deep networks and show that standard Gaussian initializations struggle when the network is either too narrow or too wide, whereas the
proposed scheme maintains strong performance over a broad range of widths.

\begin{figure}[b!]
\centering 
\begin{tabular}{cccc}
\begin{subfigure}[b]{0.31\textwidth}
    \centering
    \includegraphics[width=\textwidth]{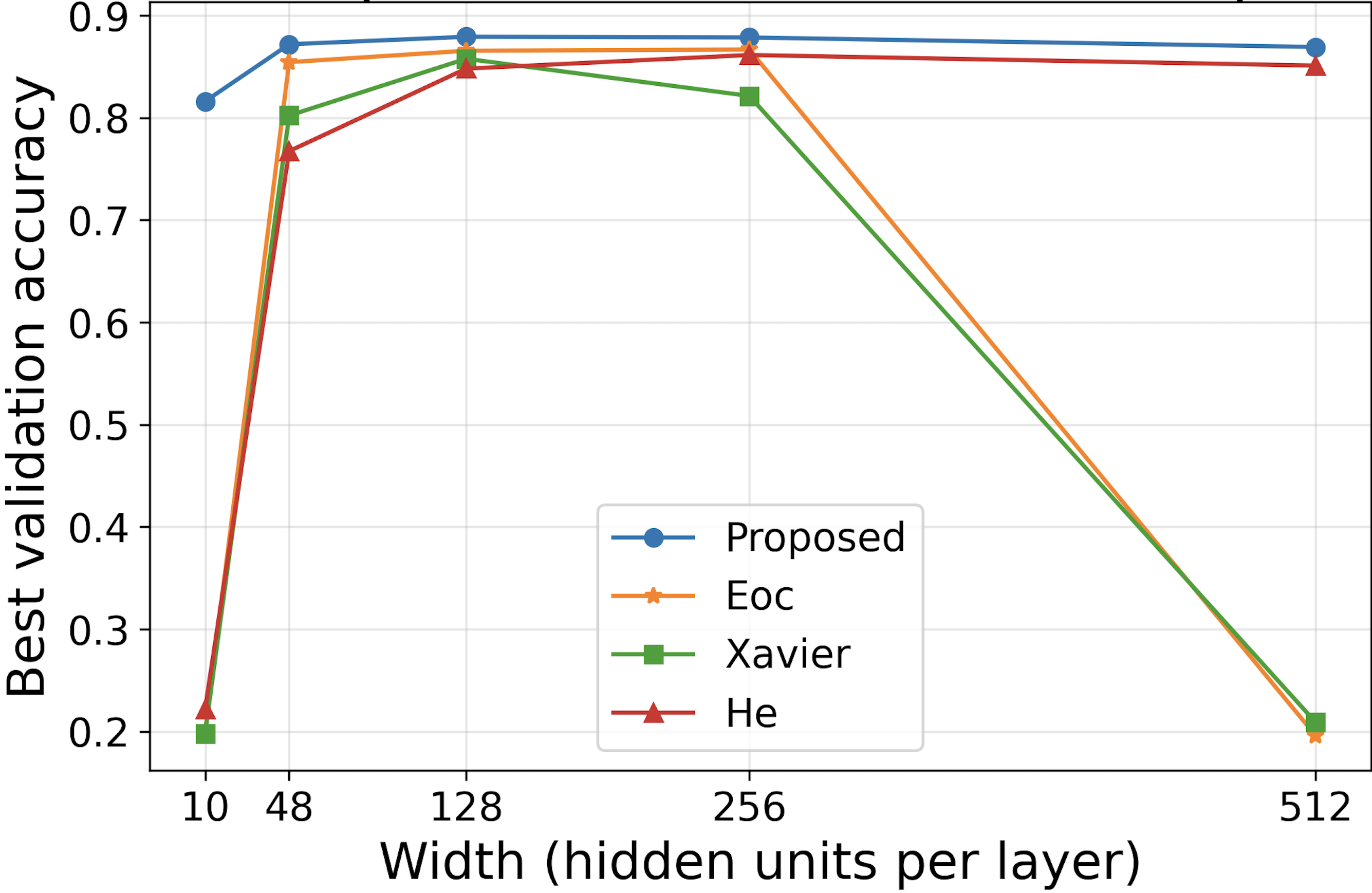}
    \caption{Fashion MNIST}
\end{subfigure} &
\begin{subfigure}[b]{0.31\textwidth}
    \centering
    \includegraphics[width=\textwidth]{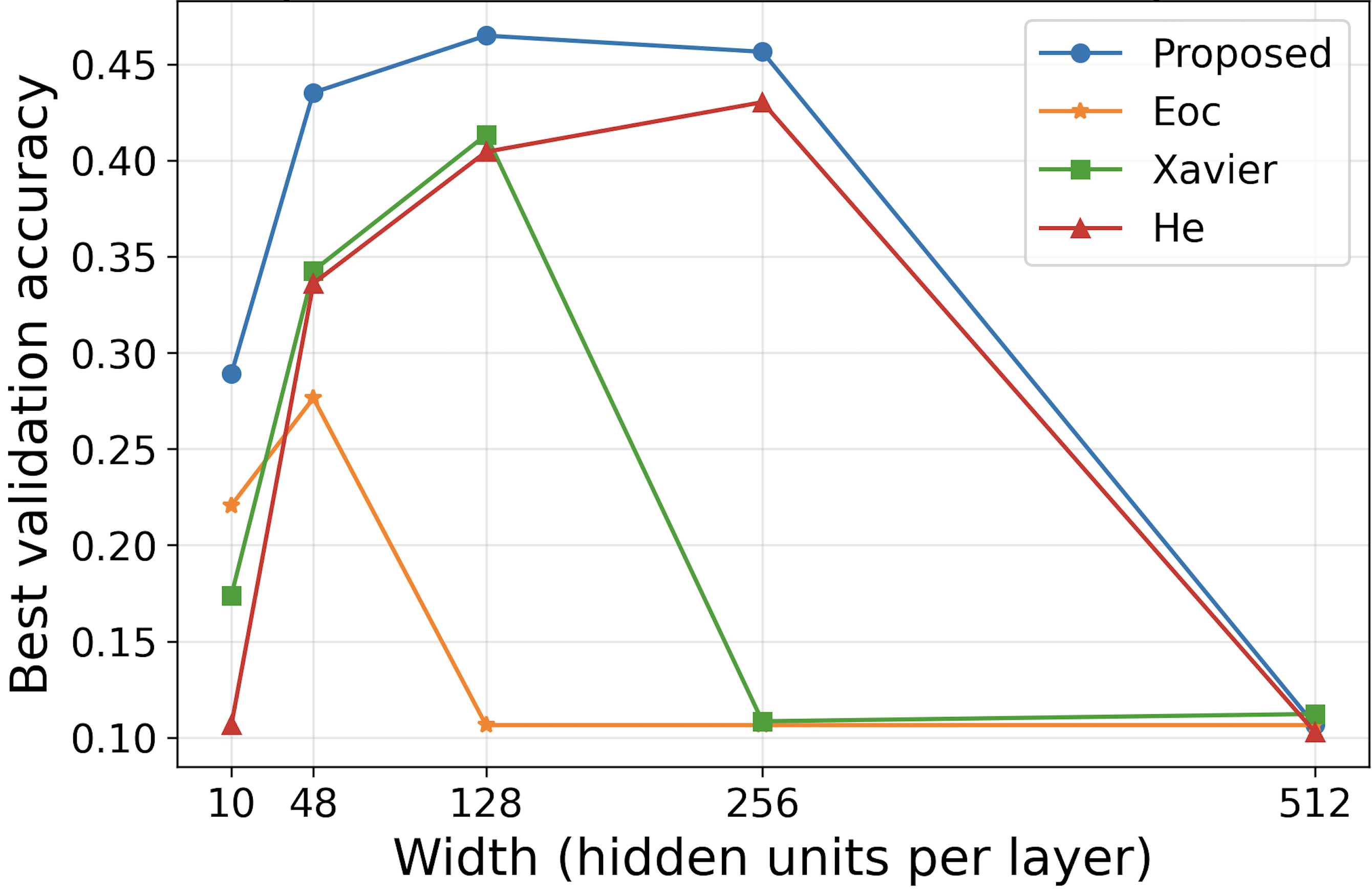}
    \caption{CIFAR 10}
\end{subfigure} &
\begin{subfigure}[b]{0.31\textwidth}
    \centering
    \includegraphics[width=\textwidth]{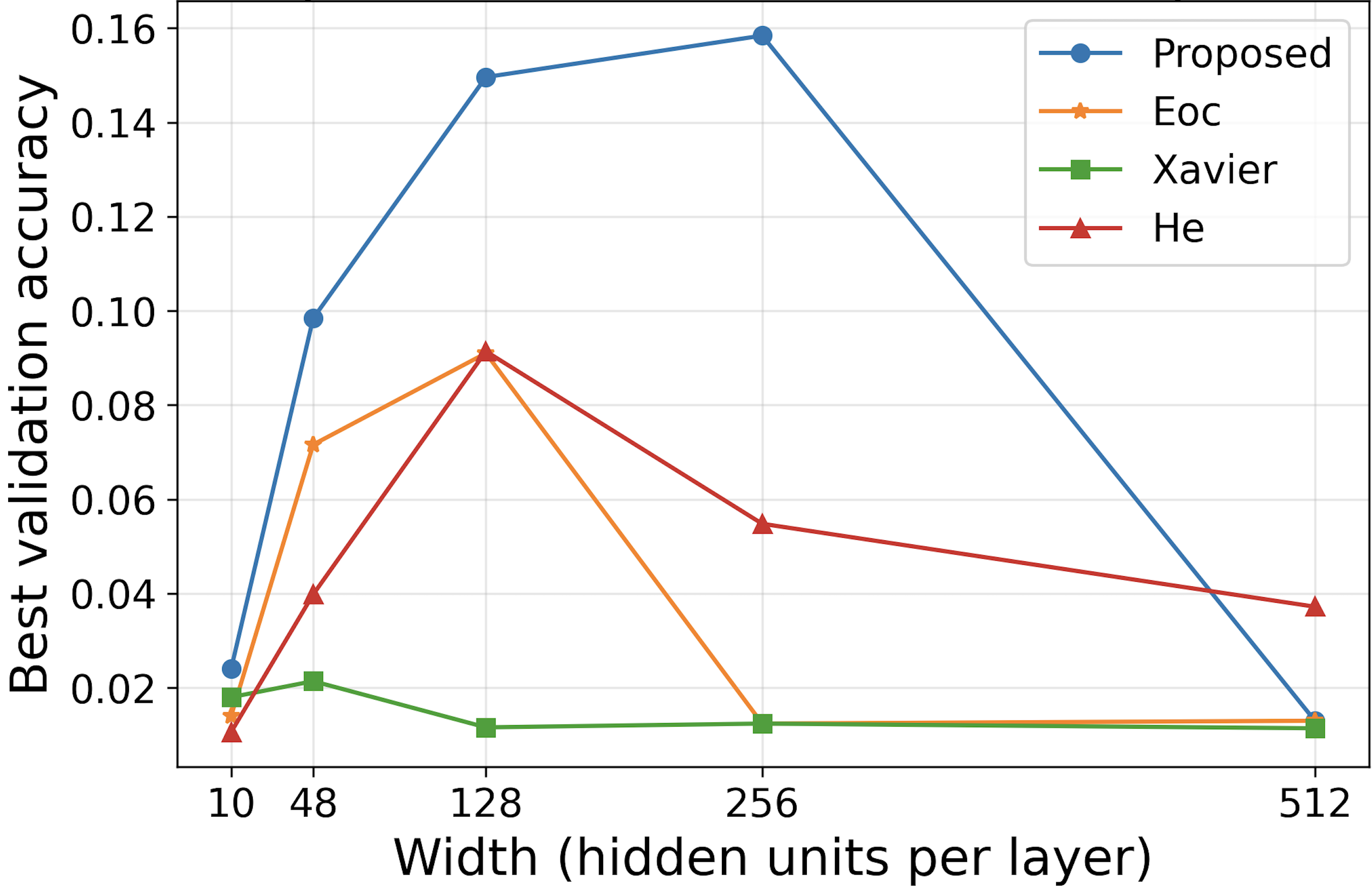}
    \caption{CIFAR 100}
\end{subfigure} 
\end{tabular} 
\caption{Best validation accuracy versus width \(\{10, 48, 128, 150, 200, 512\}\) for
a 100 layer \(\tanh\) FFNN. Each curve
shows the Proposed, EOC, Xavier, and He initialization schemes, and each point
corresponds to the best validation accuracy over 20 training epochs.}
\label{wd1}
\end{figure}

\subsection{Experiments with Activations Far from $\omega \approx 1$}\label{55}
In this section we investigate activations $f,g \in \mathcal{F}$ whose $\omega = 1/f'(0)$ is not close to $1$, including rescaled variants
$\alpha f(x)$, input scaled variants $f(\alpha x)$, and positive linear
combinations $\alpha f(x) + \beta g(x)$ with $\alpha,\beta>0$.

\paragraph{Batch Normalization Free Training.}
For
\begin{equation}\label{eq:combine}
f(x)=\tanh(ax)+\operatorname{erf}(bx)+\frac{x}{1+|cx|}+\operatorname{gd}(dx),
\end{equation}
by Corollary~\ref{cor1} this $f$ is an odd–sigmoid activation. 
Empirical results for various $(a,b,c,d)$ appear in Figure~\ref{batch_norm}. Across all settings, the proposed initialization achieves the highest validation
accuracy even without batch normalization, outperforming Gaussian i.i.d.\
initializations both with and without batch normalization.
We further investigate the effect of applying batch normalization, including
training 800 layer networks, in Figures~\ref{tanh_batch} and \ref{800layers}.



\paragraph{Learnable Learning Rate.} 
To study optimization stability, we plot learning rate versus validation
accuracy curves on MNIST and Fashion MNIST for 20 layer, width 512 FFNNs with
activations $f(x)=\tfrac{2}{\pi}\arctan(a x)$ and $f(x)=\tanh(\alpha x)$, using
scales $\alpha\in\{10^2,10^1,1,10^{-1},10^{-2},10^{-3}\}$~(Figures~\ref{lrst1}, \ref{lrst2}, \ref{lrst3}, and \ref{lrst4}). 
The proposed method consistently yields a learnable, typically wider LR window across a broad range of \(\alpha\), demonstrating that training remains feasible across diverse \(\omega\) scales. If the learning rate is chosen too large or too small relative to the
$\omega$-scaled band, the first parameter update either destroys the calibrated
noise spread or becomes negligible, so that training effectively fails from the
second pass onward. Motivated by this, we use the practical learning rate range
\[
\eta \in [\,10^{-5}\,\omega,\;10^{-3}\,\omega\,].
\]

\paragraph{Scale Preserving Odd–Sigmoid Activations.}
For $f\in\mathcal{F}$, we investigate whether the proposed and EOC
initializations preserve the activation range of the scaled activation
$\alpha f$~(Figure~\ref{activations_range}). Unlike EOC, the proposed
initialization maintains an activation range proportional to $\alpha$ even as
the network depth increases. Networks with $\alpha\tanh$, $\alpha\arctan$, and
$\alpha\operatorname{softsign_2}$ activations are trained on MNIST and
Fashion MNIST for $\alpha$ ranging from $10^{-2}$ to $10^{9}$, and the proposed
initialization enables stable training across all $\alpha$ scales for every
activation~(Table~\ref{scale_table1}--\ref{scale_table6}). 
Since an $\alpha$ scaled activation directly controls the scale of the output
$y_i$, we also evaluate our initialization on regression problems using
physics informed neural networks~(PINNs). As shown in Figures~\ref{black_sch} and~\ref{burgers}, these results indicate
that scaled odd–sigmoid activations are well suited for regression
tasks, where controlling the output range is important. We defer detailed PINN setups and PDE formulations to Appendix~\ref{scale_pinn}.

\begin{figure}[t!]
\centering 
\includegraphics[width=0.87\textwidth]{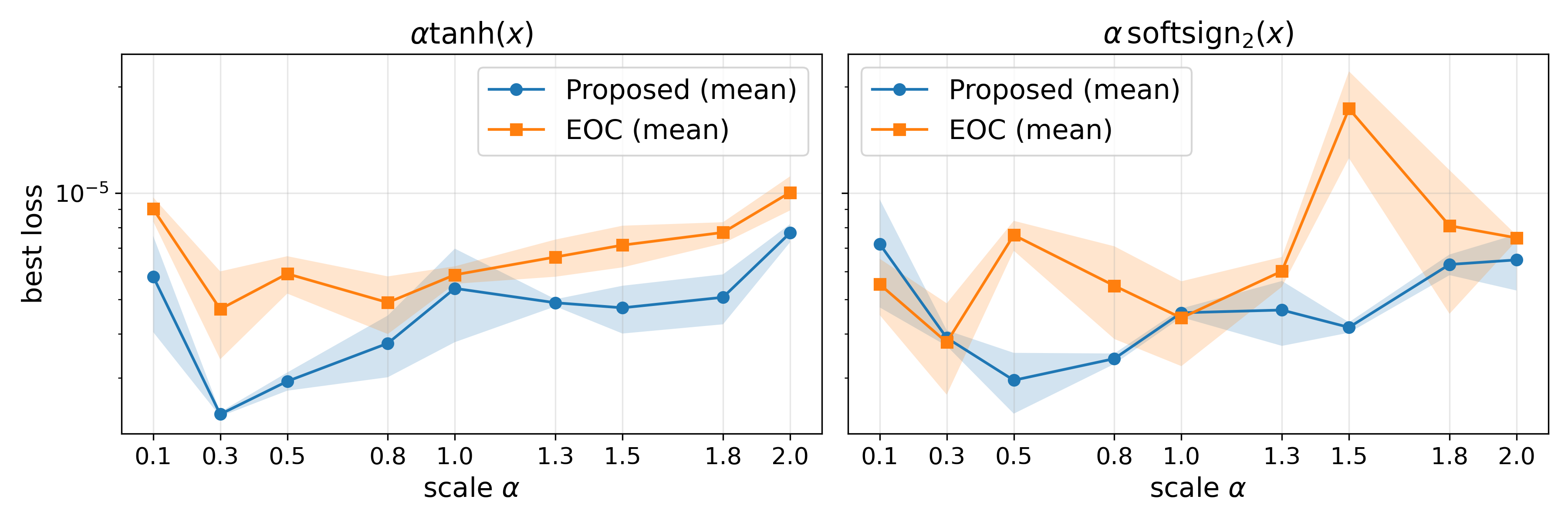}
\caption{
Best PINN loss versus activation scale for the
Black--Scholes PINN (depth $50$, width $64$), comparing the proposed and EOC
initializations. \textbf{(Left)} $a\tanh(x)$ with $a \in \{0.1,0.3,0.5,0.8,1.0,1.3,1.5,1.8,2.0\}$.
\textbf{(Right)} $b\,\operatorname{softsign}_2(x)$ with the same set of scales.
}

\label{black_sch}
\end{figure}

\section{Conclusion}
We introduced an activation aware initialization scheme for FFNNs with odd–sigmoid activations that does not rely on infinite or very wide width assumptions. 
Unlike standard Gaussian i.i.d.\ initializations, the proposed method is more robust to variance choice and maintains stable signal and gradient propagation even in very deep and relatively narrow networks. 
In physics informed neural networks, we further observed that appropriately scaled odd–sigmoid activations, combined with our initialization, can achieve lower PINN losses than standard choices. These results suggest that, with a suitable initialization, scaled odd–sigmoid activations can be used as practical design knobs to improve performance beyond what is possible with conventional Gaussian initializations.

\section*{Ethics Statement}
This research adheres to the ICLR Code of Ethics. 
A large language model was used only to refine the writing and assist in preparing visualizations; all core ideas, methods, and results were developed independently by the authors. We adhere to the highest standards of research integrity and transparency as set out in the ICLR guidelines.

\bibliography{iclr2025_conference}
\bibliographystyle{iclr2025_conference}

\clearpage
\appendix
\section*{Supplementary Material}
The supplementary material is structured as follows.
\begin{itemize}
  \item Appendix A provides the theoretical results and their proofs.
  \item Appendix B contains additional simulation results without neural networks.
  \item Appendix C reports extra experimental details and results with neural networks.
\end{itemize}

\section{Theoretical Results}\label{Appendix A}
In this section, we provide proofs for the statements presented in Section~\ref{4}. 
\subsection{Theoretical Results for Section~\ref{4.1}}\label{appendix a1}
\begin{lemma}\label{lemma:property1}
If $f\in\mathcal{F}$, then 
the following holds.
\begin{itemize}
    \item[(i)] $f(0)=0$.
    \item[(ii)] $0=\arg\max |f'(x)|$.
    \item[(iii)] $\displaystyle\lim_{x\to\infty} f'(x)=0$.
\end{itemize}
\end{lemma}

\begin{proof}
(i) It is trivial that odd symmetry implies $f(0)=0$. 
(ii) Since $f'$ is an even function, by Definition~\ref{def:odd-sig}(v)
$f'(x)$ has the maximum value at $x=0$.
(iii) Since $f'$ is strictly decreasing on $[0,\infty)$ and $f'(x)>0$ for all $x\in \mathbb{R}$,
the Monotone Convergence Theorem for real functions implies that 
$\ell := \lim_{x \to \infty} f'(x)$
exists for some $\ell \in [0,\infty)$.
Suppose that  $\ell>0$.
Then there exists $R > 0$ such that 
\[
f'(x) \ge \frac{\ell}{2} \quad \text{for all } x \ge R.
\]
Integrating from $R$ to $x$, we obtain
$$f(x)-f(R)=\int_R^x f'(t)\,dt \;\ge\; \frac{\ell}{2}(x-R) \quad \text{for all }x>R.$$
Hence $f(x) \to \infty$ as $x \to \infty$, which contradicts the boundedness of $f$.
Therefore $\ell=0$.
\end{proof}

Lemma~\ref{lemma:property1} summarizes the basic regularity and saturation
properties of odd-sigmoid activations in $\mathcal{F}$. In particular, it shows
that every $f\in\mathcal{F}$ has a unique global slope maximum at the origin and
becomes arbitrarily flat in the tails. These simple but structural features will
be crucial for understanding how the gain parameter $a$ reshapes the fixed-point
structure of the scalar map $x\mapsto f(a x)$.

\medskip

\paragraph{Proposition~\ref{prop1}}
Suppose $f\in\mathcal{F}$ with $\omega:=1/f'(0)$, and for a fixed $a>0$ define $\phi_a(x):=f(ax)$. Then
\begin{itemize}
  \item[(i)] If \( 0<a\le \omega\), then $f(ax)$ has a unique fixed point $x^*=0$.  
  \item[(ii)] If \( a>\omega\), then $f(ax)$ has three distinct fixed points: $x^* = -\xi_a,\;0,\;\xi_a$ such that  $\xi_a > 0$.
\end{itemize}

\begin{proof}
For $a>0$, consider $g(x,a):=f(ax)-x$.
We have $g(0,a)=0$ and $g'(x,a)=a\,f'(ax)-1$.

Case (i): \(0<a\le \omega\).
For $x>0$ we have $ax>0$, and $f'(ax)<f'(0)$; hence
\[
g'(x,a)=a\,f'(ax)-1 < a\,f'(0)-1 \;\le\; 0,
\]
for all $x>0$. Thus $g(\cdot,a)$ is strictly decreasing on
$(0,\infty)$, $g(0,a)=0$, and $\lim_{x\to\infty} g(x,a)=L-x=-\infty$,
so $g(x,a)<0$ for all $x>0$ and there is no positive root.
Since $g(\cdot,a)$ is odd, there is no negative root either. Hence $x=0$ is the unique solution.

Case (ii): $a>\omega$.
Then $g'(0,a)=a f'(0)-1>0$. $g'(\cdot,a)$ is strictly decreasing on $[0,\infty)$,
and by the Lemma~\ref{lemma:property1}
\[
\lim_{x\to\infty} g'(x,a)=a\lim_{x\to\infty} f'(ax)-1=-1<0.
\]
By the intermediate value theorem, there exists a unique $\hat{x}>0$ with $g'(\hat{x},a)=0$.
Hence $g(\cdot,a)$ is strictly increasing on $[0,\hat{x}]$ and strictly decreasing on
$[\hat{x},\infty)$. Since $g(0,a)=0$ and $g'(0,a)>0$, we have $g(x,a)>0$ for all
$0<x\le \hat{x}$. Using (3), $\lim_{x\to\infty} g(x,a)=-\infty$; because $g$ is strictly
decreasing on $[\hat{x},\infty)$ and continuous, there exists a unique $\xi_a>\hat{x}$ with
$g(\xi_a,a)=0$. Thus, there is exactly one positive, nonzero root. Oddness of $g$
yields the symmetric negative root $-\xi_a$, so the full set of real solutions is
$\{-\xi_a,\,0,\,\xi_a\}$.
\end{proof}

Proposition~\ref{prop1} provides a precise characterization of how the fixed
points of the scalar map $x\mapsto f(a x)$ bifurcate as the gain $a$ crosses
the critical value $\omega = 1/f'(0)$. For $a\le\omega$ the origin is the only
fixed point, while for $a>\omega$ a symmetric nonzero pair $\pm\xi_a$ emerges,
exhibiting a classical pitchfork structure. We next lift this
static picture to the dynamical setting, showing that the iterates of
$x_{n+1} = f(a x_n)$ converge to the corresponding fixed point for every
positive initial condition.

\medskip

\paragraph{Theorem~\ref{thm1}}
Suppose $f\in\mathcal{F}$ with $\omega:=1/f'(0)$, and for a fixed $a>0$ define 
\[
 x_0>0, \qquad x_{n+1}=\phi_a(x_n),\qquad n=0,1,2,\dots.
\]
Then the sequence $\{x_n\}$ converges for every $x_0>0$.
Furthermore, 
\begin{itemize}
    \item[$(1)$] if $0 < a \leq \omega$, then $x_n\to 0$ as $n\to \infty$.
    \item[$(2)$] if $a > \omega$, then $x_n\to \xi_a$ as $n\to \infty$.
\end{itemize}

\begin{proof}
(1) Since $f$ is odd and $f'$ is strictly decreasing on $[0,\infty)$, we have $0<f'(x)<f'(0)$ for all $x\neq 0$; hence, for any $a\in\bigl(0,\omega\bigr)$ and any $x_n>0$, it follows that $x_{n+1}=f(a x_n)<x_n$ for all $n\in\mathbb{N}$. By the monotone convergence theorem, it converges to the fixed point $x^*=0$. \\
(2)  Let $x_0<\xi_a$. Since $\phi'(x)$ is decreasing for $x\geq 0$, with $\xi_a$ is the unique fixed point for $x>0$, it holds that $x_n<x_{n+1} <\xi_a$ for all $n\in\mathbb{N}$. 
Thus, by the monotone convergence theorem, the sequence converges to the fixed point $\xi_a$. 
The proof is similar when $x_0 >\xi_a$. By the monotone convergence theorem, the sequence also converges to the fixed point $\xi_a$.
\end{proof}

\medskip
\paragraph{Corollary~\ref{cor1}}
Let $f_1,f_2\in\mathcal{F}$ and let $c_1,c_2\ge0$ with $(c_1,c_2)\neq(0,0)$. If $g=c_1 f_1+c_2 f_2$, then $g\in\mathcal{F}$.
Furthermore, it holds that 
\[
\frac{1}{\omega_g}=\frac{c_1}{\omega_{f_1}}+\frac{c_2}{\omega_{f_2}}.
\]
\begin{proof}
Let $g = c_1 f_1 + c_2 f_2$. Since $c_1,c_2\ge0$, the linear combination
preserves oddness and bounded saturation by linearity. Strict monotonicity on
$\mathbb{R}$ follows from $g'(x) = c_1 f_1'(x) + c_2 f_2'(x) > 0$, and slope
decay on $[0,\infty)$ is preserved because a positive linear combination of
strictly decreasing functions is strictly decreasing. Thus $g\in\mathcal{F}$.
Evaluating at the origin gives
\[
  g'(0) = c_1 f_1'(0) + c_2 f_2'(0)
        = \frac{c_1}{\omega_{f_1}} + \frac{c_2}{\omega_{f_2}},
\]
so by definition $1/\omega_g = g'(0)$, which yields the claimed identity.
\end{proof}

In practice this means that we can build
richer odd-sigmoid activations by mixing simpler ones without losing the
pitchfork structure described above. We now move from the constant–gain setting to the more general case where the layerwise gains $(a_n)$ are allowed to vary with depth.

\clearpage 

\begin{proposition}\label{prop2}
Let$f\in \mathcal{F}$ and $\{a_n\}_{n=1}^\infty$ be a positive real sequence, i.e., $a_n>0$ for all $n\in \mathbb{N}$, such that only finitely many elements are greater than $\omega=1/f'(0)$. 
For a positive sequence $\{a_n\}_{n\ge1}$, set
\[
\Phi^m \;:=\; \phi_{a_m}\circ \phi_{a_{m-1}}\circ \cdots \circ \phi_{a_1}.
\]
Then for any $x\in \mathbb{R}$ 
$$
\lim_{m \to \infty} \Phi^m(x) = 0.
$$
\end{proposition}
\begin{proof}
Let $N:=\max\{n:\ a_n>\omega\}$ (take $N=0$ if the set is empty) and define
\[
b_n=\begin{cases}a_n,& n\le N,\\ 0,& n>N,\end{cases}
\qquad
c_n=\begin{cases}a_n,& n\le N,\\ \omega,& n>N.\end{cases}
\]
Let $\hat\Phi^m:=\phi_{b_m}\circ\cdots\circ\phi_{b_1}$ and
$\tilde\Phi^m:=\phi_{c_m}\circ\cdots\circ\phi_{c_1}$.
For $x\ge0$ and $n>N$, by definition of $\mathcal{F}$, we obtain
\[
\hat\Phi^m(x)\ \le\ \Phi^m(x)\ \le\ \tilde\Phi^m(x).
\]
Oddness yields the same in absolute value for all $x\in\mathbb R$:
\[
|\hat\Phi^m(x)|\ \le\ |\Phi^m(x)|\ \le\ |\tilde\Phi^m(x)|.
\]
By Proposition~\ref{thm1}, $\hat\Phi^m(x)\to 0$ and $\tilde\Phi^m(x)\to 0$. The squeeze theorem gives $\Phi^m(x)\to 0$.
\end{proof}

\medskip

\begin{corollary}\label{cor2}
Let $\epsilon>0$ be given and set $\omega:=1/f'(0)$.
Suppose that 
$\{a_n\}_{n=1}^\infty$ be a positive real sequence such that only finitely many elements are lower than $\omega+\epsilon$. Then for any $x\in\mathbb{R}\setminus \{0\}$
\[
\liminf_{m\to\infty}\,|\Phi^m(x)|\ \ge\ \xi_{\omega+\epsilon}.
\]
\end{corollary}

\begin{proof}
Let $N:=\max\{n:\ a_n<\omega+\epsilon\}$ (take $N=0$ if the set is empty) and define
\[
b_n=\begin{cases}
a_n,& n\le N,\\[2pt]
\omega+\epsilon,& n>N.
\end{cases}
\qquad
\hat{\Phi}_m:=\phi_{b_m}\circ\cdots\circ\phi_{b_1},\qquad
\Phi^m:=\phi_{a_m}\circ\cdots\circ\phi_{a_1},
\]
where $\phi_a(x):=f(ax)$.
By definition of $\mathcal{F}$, 
\[
|\hat{\Phi}^m(x)|\ \le\ |\Phi^m(x)|\qquad(\forall x\in\mathbb R,\ \forall m).
\]
Taking $\liminf$ in the inequality yields
\[
\liminf_{m\to\infty} |\Phi^m(x)|\ \ge\ \xi_{\omega+\epsilon}.
\]
\end{proof}

Proposition~\ref{prop2} and Corollary~\ref{cor2} together describe two opposite extremes of how layerwise gains affect the one-dimensional dynamics. Roughly speaking, Proposition~\ref{prop2} says that if, after some finite depth, all gains $a_n$ stay below the critical value $\omega$, then the composed map
$\Phi^m$ always drives the signal back to zero, no matter what happened in the
earlier layers. In contrast, Corollary~\ref{cor2} shows that if the gains are
eventually bounded away from $\omega$ by a fixed margin $\varepsilon>0$, then
the compositions cannot collapse to zero: for any nonzero input, the iterates
stay at least as large (in absolute value) as the positive fixed point $\xi_{\omega+\varepsilon}$.

\clearpage
\subsection{Theoretical Results for Section~\ref{4.2}}

\paragraph{Lemma~\ref{lem1}}
Using the elementwise formulation in \eqref{eq:prelim-signal} and employing the proposed weight initialization, fix an arbitrary layer $\ell$ and index $i$ such that $x_i^\ell\neq 0$. Then, conditionally on $x^\ell$,
\[
a_i^{\ell+1}\ \sim\ \mathcal N\!\Biggl(\,\omega,\ \frac{\sigma_z^2}{N_\ell}\Bigl(1+\sum_{j\ne i}\Bigl(\frac{x_j^\ell}{x_i^\ell}\Bigr)^2\Bigr)\Biggr).
\]
Moreover, if $|x_j^\ell|\le M$ for all $j$ and $|x_i^\ell|\ge \varepsilon>0$, then
\[
  \frac{\sigma_z^2}{N_\ell}
  \;\le\;
  \Var\!\bigl(a_i^{\ell+1}\,\big|\,x^\ell\bigr)
  \;\le\;
  \sigma_z^2\,\frac{M^2}{\varepsilon^2}.
\]

\begin{proof}
From \eqref{eq:prelim-signal} and the proposed initialization, we write
\[
  W^{\ell+1} \;=\; D^{\ell+1} + Z^{\ell+1},
\]
where the diagonal of $D^{\ell+1}$ equals $\omega$ and $Z^{\ell+1}$ has independent entries
\(
  (Z^{\ell+1})_{ij} \sim \mathcal N(0,\sigma_z^2/N_\ell).
\)
The pre-activation at coordinate $i$ reads
\[
  s_i^{\ell+1}
  \;=\;
  \sum_{j=1}^{N_\ell}\bigl((D^{\ell+1})_{ij} + (Z^{\ell+1})_{ij}\bigr)\,x_j^\ell
  \;=\;
  \omega\,x_i^\ell + (Z^{\ell+1})_{ii}\,x_i^\ell
  + \sum_{j\ne i}(Z^{\ell+1})_{ij}\,x_j^\ell.
\]
On the event $x_i^\ell\neq 0$, we define the effective gain
\[
  a_i^{\ell+1} := \frac{s_i^{\ell+1}}{x_i^\ell}
  \;=\;
  \omega + (Z^{\ell+1})_{ii}
  + \sum_{j\ne i}(Z^{\ell+1})_{ij}\,\frac{x_j^\ell}{x_i^\ell}.
\]
Conditionally on $x^\ell$, the coefficients $\{x_j^\ell/x_i^\ell\}$ are deterministic, whereas the random variables $\{(Z^{\ell+1})_{ij}\}_{j=1}^{N_\ell}$ are independent, centered Gaussians with variance $\sigma_z^2/N_\ell$.
Therefore $a_i^{\ell+1}\mid x^\ell$ is a linear combination of independent Gaussians, hence Gaussian:
\[
  a_i^{\ell+1}\,\big|\,x^\ell \;\sim\;
  \mathcal{N}\!\Biggl(
    \omega,\;
    \frac{\sigma_z^2}{N_\ell}
    \Bigl(
      1 + \sum_{j\neq i}\Bigl(\frac{x_j^\ell}{x_i^\ell}\Bigr)^2
    \Bigr)
  \Biggr).
\]
This yields the conditional variance
\[
  \Var\!\bigl(a_i^{\ell+1}\mid x^\ell\bigr)
  = \frac{\sigma_z^2}{N_\ell}\Bigl(
      1 + \sum_{j\neq i}\Bigl(\frac{x_j^\ell}{x_i^\ell}\Bigr)^2
    \Bigr).
\]
Since the summation term is nonnegative, we immediately obtain the conditional lower bound
\(
  \Var(a_i^{\ell+1}\mid x^\ell) \ge \sigma_z^2/N_\ell
\)
and hence the unconditional lower bound
\(
  \Var(a_i^{\ell+1}) \ge \sigma_z^2/N_\ell.
\)

For the upper bound, note that
\[
  1 + \sum_{j\neq i}\Bigl(\frac{x_j^\ell}{x_i^\ell}\Bigr)^2
  \;=\;
  \frac{(x_i^\ell)^2 + \sum_{j\neq i}(x_j^\ell)^2}{(x_i^\ell)^2}
  \;=\;
  \frac{\|x^\ell\|_2^2}{(x_i^\ell)^2},
\]
whence
\[
  \Var\!\bigl(a_i^{\ell+1}\mid x^\ell\bigr)
  = \frac{\sigma_z^2}{N_\ell}\,\frac{\|x^\ell\|_2^2}{(x_i^\ell)^2}.
\]
If $|x_j^\ell|\le M$ for all $j$, then
\(
  \|x^\ell\|_2^2 = \sum_{j=1}^{N_\ell}(x_j^\ell)^2 \le N_\ell M^2,
\)
and if $|x_i^\ell|\ge\varepsilon>0$, we obtain
\[
  \frac{\|x^\ell\|_2^2}{(x_i^\ell)^2}
  \;\le\;
  \frac{N_\ell M^2}{\varepsilon^2},
\]
which implies
\[
  \Var\!\bigl(a_i^{\ell+1}\mid x^\ell\bigr)
  \;\le\;
  \frac{\sigma_z^2}{N_\ell}\,\frac{N_\ell M^2}{\varepsilon^2}
  \;=\;
  \sigma_z^2\,\frac{M^2}{\varepsilon^2}.
\]
The unconditional upper bound follows from
\(
  \Var(a_i^{\ell+1})
  = \E[\Var(a_i^{\ell+1}\mid x^\ell)]
    + \Var(\E[a_i^{\ell+1}\mid x^\ell]),
\)
and the fact that $\E[a_i^{\ell+1}\mid x^\ell]=\omega$ is constant and the first term is bounded by the conditional upper bound.
\end{proof}

\noindent
Lemma~\ref{lem1} shows that, under our structured initialization, the effective gain $a_i^{\ell+1}$ is approximately Gaussian with mean $\omega$ and a variance that scales like
$\sigma_z^2$ times a data dependent factor. In particular, the lower bound
$\Var(a_i^{\ell+1}\mid x^\ell)\ge\sigma_z^2/N_\ell$ guarantees a nontrivial amount of gain noise at every layer, while the upper bound prevents the variance from blowing up when the
activations remain bounded away from zero. These properties will be crucial when we study
how the gain variance behaves as depth and noise scale vary.

\medskip

\paragraph{Theorem~\ref{thm4.6}}
Let $f\in\mathcal F$ be an odd–sigmoid activation with $\omega := 1/f'(0)$, and fix any $\varepsilon>0$.
Consider the feedforward network and proposed initialization,
except that the diagonal element is set to
$a_0 := \omega + \varepsilon,$
and let $a_i^{\ell+1}$ be the effective gain defined in~\eqref{eq:prelim-signal}. Fix a tolerance $\gamma\in(0,1)$ and a finite depth
$L\in\mathbb{N}$. Then there exist a threshold depth
$L_{0}\le L$ and a noise threshold $\sigma_0>0$ such that,
for all $0<\sigma_z\le\sigma_0$,
\begin{equation}\label{eq:gain-var-supercritical}
  \mathbb{P}\Bigl(
    (1-\gamma)\,\sigma_z^2
    \;\le\;
    \Var(a_i^{\ell+1}\mid x^\ell)
    \;\le\;
    (1+\gamma)\,\sigma_z^2
    \quad\text{for all } L_{0}\le \ell<L,\ 1\le i\le N_\ell
  \Bigr)
  \;\ge\; 1-\gamma.
\end{equation}

\begin{proof}
We first consider $\sigma_z=0$.
In this case we set $\mathbf{W}^{\ell} = \mathbf{D}^{\ell} \in \mathbb{R}^{N_{\ell}\times N_{\ell-1}}$ with $(\mathbf{D}^{\ell})_{ij} = a_0$ if $i \equiv j \ (\mathrm{mod}\ N_{\ell-1})$ and $0$ otherwise, so that the layerwise update reduces to
\[
  x^{\ell+1} = f(a_0 x^\ell).
\]
Hence each coordinate $x_i^\ell$ evolves independently according to the scalar recurrence $x_{n+1} = f(a_0 x_n)$.
Since $a_0 f'(0) > 1$, Theorem~\ref{thm1} implies that
this map has three fixed points $\{0,\pm\xi_{a_0}\}$ and, for any $x_0\neq 0$,
\[
  x_n \to \pm\xi_{a_0}\quad\text{as }n\to\infty,
\]
with the sign determined by $\mathrm{sign}(x_0)$. For every $\delta>0$ and every initial value $x_0\neq 0$
there exists an integer $N_i(\delta)$ such that
\[
  |x_n(x_0)| \in [\xi_{a_0}-\delta,\ \xi_{a_0}+\delta]
  \qquad\text{for all } n\ge N_i(\delta).
\]
Since each layer contains only finitely many neurons, we can take
\[
  L_0(\delta)
  \;:=\;
  \max_{1\le i\le N_\ell} N_i(\delta),
\]
so that
\begin{equation}\label{eq:det-band}
  |x_i^\ell(0)|
  \;\in\;
  [\xi_{a_0}-\delta,\ \xi_{a_0}+\delta]
  \qquad\text{for all } 1\le i\le N_\ell,\ \ell\ge L_0(\delta),
\end{equation}
where $x^\ell(0)$ denotes the activations at depth $\ell$ in the
noiseless case $\sigma_z=0$.

From these bounds we obtain, for all $\ell\ge L_{0}(\delta)$,
\[
  N_\ell (\xi_{a_0}-\delta)^2
  \;\le\;
  \|x^\ell(0)\|_2^2
  \;\le\;
  N_\ell (\xi_{a_0}+\delta)^2,
  \qquad
  (x_i^\ell(0))^2 \in [(\xi_{a_0}-\delta)^2,(\xi_{a_0}+\delta)^2].
\]
Therefore the deterministic ratio defined in Lemma~\ref{lem1}
\[
  R_\ell(0;i)
  := \frac{\|x^\ell(0)\|_2^2}{N_\ell (x_i^\ell(0))^2}
\]
satisfies
\[
  R_\ell(0;i)
  \in
  \Biggl[
    \frac{(\xi_{a_0}-\delta)^2}{(\xi_{a_0}+\delta)^2},\;
    \frac{(\xi_{a_0}+\delta)^2}{(\xi_{a_0}-\delta)^2}
  \Biggr]
  \quad\text{for all } \ell\ge L_{\mathrm{0}}(\delta), i.
\]
As $\delta\rightarrow 0$, this interval shrinks to $1$.
Hence, given any $\gamma\in(0,1)$ we can choose $\delta>0$ such that
\[
  \bigl|R_\ell(0;i) - 1\bigr| \le \frac{\gamma}{2}
  \quad\text{for all } \ell\ge L_{\mathrm{0}}(\delta),\ i.
\]

Now we consider $\sigma_z>0$.
Recall that in our initialization we write
\[
  \mathbf Z^{(\ell)} = \frac{\sigma_z}{\sqrt{N_{\ell-1}}}\,\mathbf G^{(\ell)},
  \qquad
  (\mathbf G^{(\ell)})_{ij}\sim\mathcal N(0,1)\ \text{i.i.d.},
\]
so that the randomness is entirely carried by the Gaussian matrices
$\mathbf G^{(1)},\dots,\mathbf G^{(L)}$.
For any deterministic family of matrices
$\widehat{\mathbf G}^{(1)},\dots,\widehat{\mathbf G}^{(L)}$ we define,
for $\sigma_z\ge 0$, the corresponding activations
$\widehat{\mathbf x}^\ell(\sigma_z)$ recursively by
\begin{align*}
  \widehat{\mathbf x}^0(\sigma_z) &:= \mathbf x^0, \\
  \widehat{\mathbf x}^{\ell}(\sigma_z)
  &:= f\Bigl(
        \bigl(a_0 \mathbf D^{\ell} +
              \tfrac{\sigma_z}{\sqrt{N_{\ell-1}}}\widehat{\mathbf G}^{(\ell)}\bigr)
        \widehat{\mathbf x}^{\ell-1}(\sigma_z)
      \Bigr),
  \qquad \ell=1,2,\dots,L,
\end{align*}
where $f$ is applied coordinatewise.
For each fixed choice of $(\widehat{\mathbf G}^{(1)},\dots,\widehat{\mathbf G}^{(L)})$,
each layer $\ell$ and each coordinate $i$, this defines a map
\[
  \sigma_z \;\longmapsto\; \widehat{x}_i^\ell(\sigma_z).
\]
Since $f$ is continuous and, for fixed $\widehat{\mathbf G}^{(\ell)}$, the map
\(
  \sigma_z \mapsto
  \bigl(a_0 \mathbf D^{(\ell)} +
        \tfrac{\sigma_z}{\sqrt{N_{\ell-1}}}\widehat{\mathbf G}^{(\ell)}\bigr)
        \widehat{\mathbf x}^{\ell-1}(\sigma_z)
\)
is affine in $\sigma_z$, it follows by induction on $\ell$ that
\begin{equation}\label{eq:cont-x-sigma}
  \sigma_z \;\longmapsto\; \widehat{x}_i^\ell(\sigma_z)
  \quad\text{is continuous on $[0,\sigma_1]$ for any finite $\sigma_1>0$.}
\end{equation}

For each $\ell$ and $i$ with $x_i^\ell(\sigma_z)\neq 0$, define
\[
  R_\ell(\sigma_z;i)
  := \frac{\|\mathbf x^\ell(\sigma_z)\|_2^2}{N_\ell\,(x_i^\ell(\sigma_z))^2},
  \qquad
  R_\ell(0;i)
  := \frac{\|\mathbf x^\ell(0)\|_2^2}{N_\ell\,(x_i^\ell(0))^2},
\]
so that $R_\ell(0;i)$ is exactly the deterministic ratio considered above.
On the event that $x_i^\ell(\sigma_z)\neq 0$ for all $\ell\le L$ and all
$0\le\sigma_z\le\sigma_1$, \eqref{eq:cont-x-sigma} implies that
$\sigma_z\mapsto R_\ell(\sigma_z;i)$ is continuous.
Hence, for each fixed pair $(\ell,i)$ and any $\gamma\in(0,1)$ there exists
$\sigma_0(\ell,i)>0$ such that
\begin{equation}\label{eq:cont-R}
  \bigl|R_\ell(\sigma_z;i) - R_\ell(0;i)\bigr|
  \;\le\; \frac{\gamma}{2}
  \qquad\text{for all } 0\le\sigma_z\le\sigma_0(\ell,i).
\end{equation}
Combining \eqref{eq:cont-R} with the deterministic bound
$\bigl|R_\ell(0;i)-1\bigr|\le\gamma/2$ (valid for all $\ell\ge L_0(\delta)$
from~\eqref{eq:det-band}), we obtain
\[
  \bigl|R_\ell(\sigma_z;i) - 1\bigr|
  \;\le\;
  \bigl|R_\ell(\sigma_z;i) - R_\ell(0;i)\bigr|
  +
  \bigl|R_\ell(0;i) - 1\bigr|
  \;\le\; \gamma
\]
for all $0\le\sigma_z\le\sigma_0(\ell,i)$ and all $\ell\ge L_0(\delta)$.

Since we only consider a finite set of layers $\ell<L$ and indices
$1\le i\le N_\ell$, we may define, for each fixed
$(\widehat{\mathbf G}^{(1)},\dots,\widehat{\mathbf G}^{(L)})$,
\[
  \sigma_0(\widehat{\mathbf G}^{(1)},\dots,\widehat{\mathbf G}^{(L)})
  := \min_{L_0(\delta)\le \ell < L}
     \min_{1\le i\le N_\ell} \sigma_0(\ell,i) \;>\; 0,
\]
so that the bound
\begin{equation}\label{eq:R-close-1}
  \bigl|R_\ell(\sigma_z;i) - 1\bigr|
  \;\le\; \gamma
\end{equation}
holds simultaneously for all $L_0(\delta)\le\ell<L$ and all $1\le i\le N_\ell$
whenever $0\le\sigma_z\le\sigma_0(\widehat{\mathbf G}^{(1)},\dots,\widehat{\mathbf G}^{(L)})$.

Now regard $(\mathbf G^{(1)},\dots,\mathbf G^{(L)})$ as random Gaussian
matrices and set
\[
  S := \sigma_0(\mathbf G^{(1)},\dots,\mathbf G^{(L)}).
\]
By the construction above we have $S>0$ almost surely. Hence, for a given
$\gamma\in(0,1)$ we can choose a deterministic constant $\sigma_0>0$ such that
$\mathbb P(S\ge\sigma_0)\ge 1-\gamma$ (for example, take $\sigma_0$ to be the
$(1-\gamma)$–quantile of $S$). On the event $\{S\ge\sigma_0\}$ the estimate
\eqref{eq:R-close-1} therefore holds for all $0\le\sigma_z\le\sigma_0$,
all $L_0(\delta)\le\ell<L$ and all $1\le i\le N_\ell$.

Finally, Lemma~4.5 (with $a_0$ in place of $\omega$) gives the conditional
variance formula
\[
  \Var(a_i^{\ell+1}\mid x^\ell)
  = \frac{\sigma_z^2}{N_\ell}\sum_{j=1}^{N_\ell}
    \Bigl(\frac{x_j^\ell(\sigma_z)}{x_i^\ell(\sigma_z)}\Bigr)^2
  = \sigma_z^2\,R_\ell(\sigma_z;i).
\]
Combining this with \eqref{eq:R-close-1} and the choice of~$\sigma_0$ yields, for all
$0<\sigma_z\le\sigma_0$ and all $L_0(\delta)\le\ell<L$, $1\le i\le N_\ell$,
\[
  (1-\gamma)\,\sigma_z^2
  \;\le\;
  \Var(a_i^{\ell+1}\mid x^\ell)
  \;\le\;
  (1+\gamma)\,\sigma_z^2,
\]
on an event of probability at least $1-\gamma$.
\end{proof}

\clearpage

\noindent
The previous results describe how the magnitude of the effective gains behaves across layers~(Figures~\ref{fig6} and \ref{negative_rate_f}).
To control the sign dynamics, we now turn to a scalar surrogate model in which the gains are
i.i.d.\ Gaussian with mean $\omega$ and variance $\sigma^2$. In this simplified setting, the only
quantity that matters is the probability that the scalar iterate becomes negative at a given depth.
The next lemma provides a closed-form recursion for this negative rate. 

\medskip

\begin{lemma}\label{lem:sign-recursion}
Let $f\in\mathcal{F}$ and $x_0>0$, for every $j\ge1$,
\begin{equation}\label{eq:pi-closed}
\pi_j=\tfrac12\Big(1-(1-2p_-)^{\,j}\Big),\qquad 
p_-=\mathbb P(A_1<0)=\Phi\!\Big(-\frac{k}{\sigma}\Big).
\end{equation}
\end{lemma}

\begin{proof}
Since $f$ is odd and strictly increasing, $\operatorname{sign}(f(u))=\operatorname{sign}(u)$ for all $u$, hence
\[
\{X_j<0\}=\{A_jX_{j-1}<0\}=\big(\{A_j<0\}\cap\{X_{j-1}>0\}\big)\,\cup\,\big(\{A_j>0\}\cap\{X_{j-1}<0\}\big),
\]
up to null sets (because $\mathbb P(A_j=0)=0$ for Gaussian $A_j$).
Independence of $A_j$ and $X_{j-1}$ yields
\[
\pi_j=\mathbb P(A_j<0)\,\mathbb P(X_{j-1}>0)+\mathbb P(A_j>0)\,\mathbb P(X_{j-1}<0)
=p_-(1-\pi_{j-1})+(1-p_-)\pi_{j-1}.
\]
Thus $\pi_j=(1-2p_-)\pi_{j-1}+p_-$. With $\pi_0=\mathbb P(X_0<0)=0$, the first-order linear recursion solves to
\[
\pi_j-\tfrac12=(1-2p_-)\big(\pi_{j-1}-\tfrac12\big)\;\Longrightarrow\;
\pi_j-\tfrac12=(1-2p_-)^{\,j}\big(\pi_0-\tfrac12\big)= -\tfrac12(1-2p_-)^{\,j}.
\]
The value of $p_-$ follows from $A_1\sim\mathcal N(\omega,\sigma^2)$.
\end{proof}

\medskip

\paragraph{Theorem~\ref{thm:closed-sigma}}
Fix a target $p\in[0,\tfrac12)$, a depth $\ell\in\mathbb N$, and $\omega>0$.
There exists a unique $\sigma^\star=\sigma^\star(p,\ell,\omega)>0$ such that $\pi_\ell(\sigma^\star)=p$, and it is given by
\begin{equation}
\sigma^\star(p,\ell,\omega)\;=\; -\,\frac{\omega}{\Phi^{-1}\!\left(\dfrac{1-(1-2p)^{1/\ell}}{2}\right)}\,.
\end{equation}

\begin{proof}
$\pi_\ell(\cdot)$ is continuous and strictly increasing from $0$ (at $\sigma\downarrow 0$) to $1/2$ (as $\sigma\uparrow\infty$).
Hence for any $p\in[0,1/2)$ there exists a unique $\sigma^\star>0$ with $\pi_\ell(\sigma^\star)=p$.

To obtain the explicit form, set $q:=p_-(\sigma^\star)=\Phi(-\omega/\sigma^\star)\in(0,1/2)$.
From Lemma~\ref{lem:sign-recursion} we have
\[
p=\pi_\ell(\sigma^\star)=\tfrac12\Big(1-(1-2q)^\ell\Big)\quad\Longrightarrow\quad
q=\frac{1-(1-2p)^{1/\ell}}{2}.
\]
Applying the inverse CDF $\Phi^{-1}$ to $q=\Phi(-\omega/\sigma^\star)$ yields
\[
-\frac{k}{\sigma^\star}=\Phi^{-1}(q)\quad\Longrightarrow\quad
\sigma^\star=-\,\frac{k}{\Phi^{-1}(q)}.
\]
Since $q\in(0,1/2)$, the quantile $\Phi^{-1}(q)<0$, so the right-hand side is positive.
\end{proof}

\clearpage
\subsection{Theoretical Results for Section~\ref{section4.3}}\label{chi_app}

Before turning to the empirical comparisons in Section~\ref{section4.3}, we
connect our gain calibration to gradient propagation. In the main text we
argued that the trainability of very deep networks is governed by how the
norm of the backpropagated gradient evolves with depth. In this appendix we
make this precise by computing, under standard mean-field assumptions, the
layerwise gradient amplification factor $\chi_\ell$ for our structured
initialization and contrasting it with the Gaussian i.i.d.\ case. This will
justify the claims in Section~\ref{section4.3} about the robustness of the
proposed scheme to the choice of variance.

\medskip

Let $\mathbf{g}^\ell := \partial \mathcal{L} / \partial \mathbf{x}^\ell \in
\mathbb{R}^{N_\ell}$ denote the gradient at layer $\ell$ for a loss
$\mathcal{L}$. By the chain rule and the layerwise relation
$\mathbf{h}^{\ell+1} = \mathbf{W}^{\ell+1}\mathbf{x}^{\ell} + \mathbf{b}^{\ell+1}$,
we can write
\[
  \mathbf{g}^\ell
  = \bigl(\mathbf{W}^{\ell+1}\bigr)^\top
    \bigl(f'(\mathbf{h}^{\ell+1}) \odot \mathbf{g}^{\ell+1}\bigr)
  = \bigl(\mathbf{J}^{\ell+1}\bigr)^\top \mathbf{g}^{\ell+1},
\]
where $\odot$ denotes the Hadamard product and
\[
  \mathbf{J}^{\ell+1}
  := \operatorname{diag}\bigl(f'(\mathbf{h}^{\ell+1})\bigr)\,\mathbf{W}^{\ell+1}
\]
is the Jacobian matrix of layer $\ell+1$.

In the wide layer regime ($N_\ell\to\infty$) and under standard mean field
assumptions given $\mathbf{h}^{\ell+1}$), the squared gradient norms satisfy the approximate recursion
\begin{equation}\label{eq:grad-recursion-chi}
  \frac{1}{N_\ell}\,\mathbb{E}\bigl\|\mathbf{g}^\ell\bigr\|_2^2
  \;\approx\;
  \chi_{\ell+1}\,
  \frac{1}{N_{\ell+1}}\,\mathbb{E}\bigl\|\mathbf{g}^{\ell+1}\bigr\|_2^2,
\end{equation}
where the layerwise amplification factor $\chi_{\ell+1}$ is defined by
\[
  \chi_{\ell+1}
  :=
  \frac{1}{N_{\ell+1}}\,
  \mathbb{E}\bigl\|\mathbf{J}^{\ell+1}\mathbf{u}\bigr\|_2^2,
\]
with $\mathbf{u}\sim\mathcal{N}(\mathbf{0},\mathbf{I}_{N_{\ell+1}})$ independent
of $\mathbf{J}^{\ell+1}$. Iterating~\eqref{eq:grad-recursion-chi} over
$\ell = 0,\dots,L-1$ gives
\[
  \frac{1}{N_0}\,\mathbb{E}\bigl\|\mathbf{g}^0\bigr\|_2^2
  \;\approx\;
  \Bigl(\prod_{\ell=0}^{L-1}\chi_{\ell+1}\Bigr)
  \frac{1}{N_L}\,\mathbb{E}\bigl\|\mathbf{g}^L\bigr\|_2^2.
\]

We now compute $\chi_{\ell+1}$ for the proposed structured initialization
$\mathbf{W}^\ell = \omega \mathbf{D}^\ell + \mathbf{Z}^\ell$. By
definition of $\chi_{\ell+1}$,
\[
  \chi_{\ell+1}
  = \frac{1}{N_{\ell+1}}\,\mathbb{E}\bigl\|\mathbf{J}^{\ell+1}\mathbf{u}\bigr\|_2^2
  = \frac{1}{N_{\ell+1}}\,
    \mathbb{E}\sum_{i=1}^{N_{\ell+1}}
    \biggl(
      \sum_{j=1}^{N_\ell}
        w_{ij}^{\ell+1} f'\bigl(h_i^{\ell+1}\bigr) u_j
    \biggr)^2.
\]
Conditioning on $\mathbf{W}^{\ell+1}$ and $\mathbf{h}^{\ell+1}$, the inner
sum is a centred Gaussian in $\mathbf{u}$ with variance
\[
  \sum_{j=1}^{N_\ell} \bigl(w_{ij}^{\ell+1}\bigr)^2\,f'\bigl(h_i^{\ell+1}\bigr)^2.
\]
Taking the expectation over $\mathbf{u}$ and then over the weights and
preactivations yields
\begin{align*}
  \chi_{\ell+1}
  &= \frac{1}{N_{\ell+1}}\,
     \mathbb{E}\sum_{i=1}^{N_{\ell+1}}
       f'\bigl(h_i^{\ell+1}\bigr)^2
       \sum_{j=1}^{N_\ell}\bigl(w_{ij}^{\ell+1}\bigr)^2 \\[0.5ex]
  &\approx
     \mathbb{E}\Bigl[
       f'\bigl(h_1^{\ell+1}\bigr)^2\,
       \sum_{j=1}^{N_\ell}\bigl(w_{1j}^{\ell+1}\bigr)^2
     \Bigr],
\end{align*}
where we used exchangeability of the rows of $\mathbf{W}^{\ell+1}$ and of the
coordinates of $\mathbf{h}^{\ell+1}$.

For the proposed initialization,
$(\mathbf{Z}^{\ell+1}_{ij})\sim\mathcal{N}(0,\sigma_z^2/N_\ell)$, so
\[
  \sum_{j=1}^{N_\ell}\bigl(w_{1j}^{\ell+1}\bigr)^2
  = \omega^2 + \sum_{j=1}^{N_\ell}\bigl(Z_{1j}^{\ell+1}\bigr)^2.
\]
Taking expectation over $\mathbf{Z}^{\ell+1}$ and using
$\mathbb{E}\bigl[(Z_{1j}^{\ell+1})^2\bigr]=\sigma_z^2/N_\ell$ gives
\[
  \mathbb{E}\Bigl[\sum_{j=1}^{N_\ell}\bigl(w_{1j}^{\ell+1}\bigr)^2\Bigr]
  =
  \omega^2
  + \sum_{j=1}^{N_\ell}\mathbb{E}\bigl[(Z_{1j}^{\ell+1})^2\bigr]
  = \omega^2 + \sigma_z^2.
\]
Moreover, by construction of the initialization, $\mathbf{W}^{\ell+1}$ and
$\mathbf{h}^{\ell+1}$ are independent, and the coordinates $h_i^{\ell+1}$ are
exchangeable.  Hence
\[
  \mathbb{E}\Bigl[
    f'\bigl(h_1^{\ell+1}\bigr)^2\,
    \sum_{j=1}^{N_\ell}\bigl(w_{1j}^{\ell+1}\bigr)^2
  \Bigr]
  \;\approx\;
  \mathbb{E}\bigl[f'(h^{\ell+1})^2\bigr]\,
  \mathbb{E}\Bigl[\sum_{j=1}^{N_\ell}\bigl(w_{1j}^{\ell+1}\bigr)^2\Bigr]
  =
  \bigl(\omega^2+\sigma_z^2\bigr)\,
  \mathbb{E}\bigl[f'(h^{\ell+1})^2\bigr].
\]
Thus,
\begin{equation}\label{eq:chi-proposed}
  \chi_{\ell+1}
  \;\approx\;
  \bigl(\omega^2 + \sigma_z^2\bigr)\,
  \mathbb{E}\bigl[f'(h^{\ell+1})^2\bigr].
\end{equation}

\medskip

\noindent
Equation~\eqref{eq:chi-proposed} shows that, for the proposed initialization,
the mean-field amplification factor $\chi_{\ell+1}$ depends on the combined
scale $\omega^2+\sigma_z^2$ rather than on a bare variance parameter alone.
In particular, increasing the noise level $\sigma_z$ decreases
$\mathbb{E}[f'(h^{\ell+1})^2]$ while increasing the prefactor
$\omega^2+\sigma_z^2$, so that $\chi_{\ell+1}$ remains close to one over a
much wider range of $\sigma_z$ than in the Gaussian i.i.d.\ case, where
$\chi_{\ell+1}^\star \approx \sigma_w^2\,\mathbb{E}[f'(h^{\ell+1})^2]$ has no
analogous $\omega^2$ term. This analytic behavior underlies the empirical
observations in Section~\ref{section4.3} that our initialization preserves
gradient norms more effectively across depth and is substantially more robust to variance misspecification.

\clearpage 

\section{Additional Experimental Results without Trained Networks}\label{Appendix B}
\subsection{Definitions of Activation Functions}\label{activation_define}
This section introduces the activation functions considered in this paper. The following functions belong to the odd-sigmoid function class.

\paragraph{Gudermannian function}
The Gudermannian function $\operatorname{gd}:\mathbb{R}\to\mathbb{R}$ is defined by
\[
\operatorname{gd}(x) = \int_{0}^{x}\frac{dt}{\cosh t} 
= 2\arctan\!\left(\tanh\!\left(\tfrac{x}{2}\right)\right).
\]

\paragraph{Error Function.}
The error function $\operatorname{erf}:\mathbb{R}\to\mathbb{R}$ is defined as
\[
\operatorname{erf}(x) = \frac{2}{\sqrt{\pi}} \int_{0}^{x} e^{-t^{2}}\,dt.
\]

\paragraph{Softsign-Type Functions.}
For $k \geq 1$, we define the generalized softsign function
\[
\operatorname{softsign}_k(x) \;=\; \frac{x}{\bigl(1+|x|^k\bigr)^{1/k}}, 
\qquad x\in\mathbb{R}.
\]
This family interpolates between several commonly used smooth odd activations:
\[
\operatorname{softsign}_1(x) = \frac{x}{1+|x|}, 
\quad
\operatorname{softsign}_2(x) = \frac{x}{\sqrt{1+x^2}}, 
\quad
\operatorname{softsign}_3(x) = \frac{x}{\bigl(1+|x|^3\bigr)^{1/3}}.
\]
We also use the combined variant
\[
\operatorname{softsign}_{1+3}(x)
:= \operatorname{softsign}_1(x)+\operatorname{softsign}_3(x).
\]
\paragraph{Hyperbolic Tangent.}
The hyperbolic tangent function $\tanh:\mathbb{R}\to\mathbb{R}$ is defined by
\[
\tanh(x) = \frac{e^x - e^{-x}}{e^x + e^{-x}}.
\]

\paragraph{Arctangent.}
The (scaled) arctangent function $\arctan:\mathbb{R}\to\mathbb{R}$ is given by
\[
\arctan(x) = \int_{0}^{x} \frac{dt}{1+t^2}.
\]
In practice, we often use the normalized form $\tfrac{2}{\pi}\arctan(x)$ so that its range matches that of $\tanh(x)$.

\clearpage
\subsection{Properties of the Odd-Sigmoid Functions}
This section empirically investigates the properties of the iteration \(x_{n+1}=f(a x_n)\). 
When the gain $a$ is fixed across all iterations, Figure~\ref{fig6} shows how the dynamics depend on the initial data: for any nonzero input, the iterates converge to the same nonzero fixed point $\xi_a$. Figure~\ref{fig111} further demonstrates that this limit depends only on the gain, not on the input, by comparing distinct gains $a\neq a'>0$ and observing convergence to different fixed points $\xi_a$ and $\xi_{a'}$. These results indicate that, by appropriately choosing $a$, one can control how long the signal remains informative across layers and thus preserve information up to a desired depth.

\begin{figure}[h!]
\centering 
\includegraphics[width=0.8\textwidth]{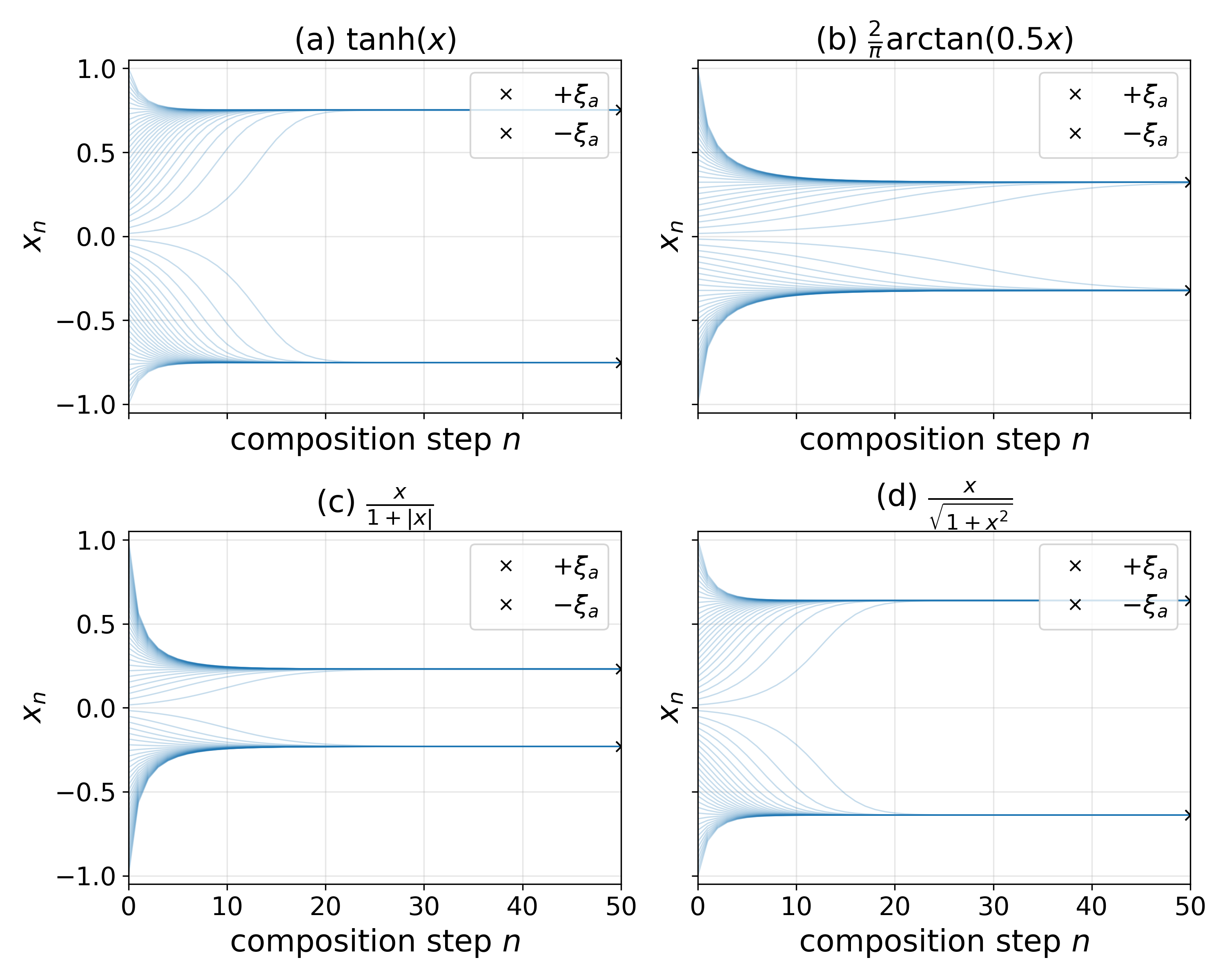}
\caption{
Iterated scalar dynamics for four odd-sigmoid activations under a fixed supercritical gain.
For each activation $f$, we compute $\omega = 1/f'(0)$ and set $a = \omega + 0.3$.
We then iterate the one--dimensional map $x_{n+1} = f(a x_n)$ for $n=0,\dots,50$
starting from 60 initial values $x_0 \in [-1,1]$.
The limiting fixed points $\pm\xi_a$ satisfying $f(a \xi_a) = \xi_a$ are approximated
by iterating from $\pm 1$ and are marked at $n=50$ with ``$\times$''.
}
\label{fig6}
\end{figure}

\medskip

\begin{figure}[h!]
\centering 
\includegraphics[width=0.8\textwidth]{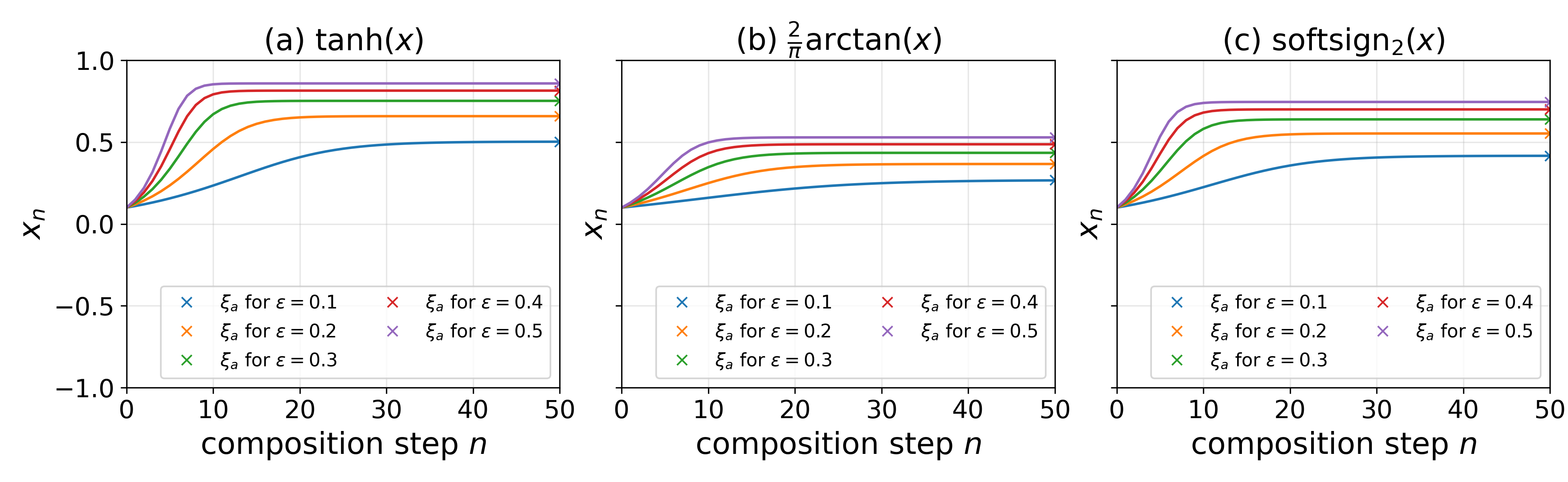}
\caption{Convergence of the iteration \(x_{n+1}=f(a x_n)\) for odd–sigmoid \(f\) with \(a=\omega+\epsilon\) (\(\epsilon\in\{0.1,0.2,0.3,0.4,0.5\}\)); curves show \(x_n\) up to \(n=50\) from \(x_0=0.1\), and ‘\(\times\)’ marks the positive fixed point \(\xi_a\) solving \(f(a\xi_a)=\xi_a\).}
\label{fig111}
\end{figure}

\clearpage
\subsection{Negative Rate Function}
We calibrate the noise scale $\sigma_z$ in the surrogate model by targeting a desired negative rate at a specified depth. For the scalar recursion with gains
$A_j\sim\mathcal N(\omega,\sigma^2)$ and depth $L$, let $\pi_L(\sigma)$ denote the
probability that the iterate is negative at layer $L$. Given a target
$p\in[0,\tfrac12)$, Theorem~\ref{thm:closed-sigma} yields a unique
$\sigma^\ast(p,L,\omega)$ such that $\pi_L(\sigma^\ast)=p$. Figure~\ref{negative_rate_f} illustrates
this calibration: panel (a) shows the closed-form scale $\sigma^\ast(p,L,\omega)$
as a function of $p$ for $L=100$ and $\omega=1$, while panel (b) plots
$\pi_L(\sigma)$ over $(L,\sigma)$ for $\omega=1$. As expected, $\pi_L(\sigma)\to 0$
as $\sigma\to 0$ (no sign flips) and $\pi_L(\sigma)\to\tfrac12$ as
$\sigma\to\infty$ (full symmetry), so the negative rate traces a narrow band
between these two extremes.

Figure~\ref{chi_value} evaluates the resulting calibration from a mean-field perspective.
For $f(x)=\tanh(x)$ and $\omega=1$, we choose $\sigma_z=\sigma^\ast(p,L,\omega)$
for target negative rates $p=0.01$ and $p=0.49$, and compute the corresponding
gradient amplification factor $\chi_\ell \approx (\omega^2+\sigma_z^2)\,\mathbb{E}[f'(h^\ell)^2]$
across depths. The curves show that $\chi_\ell$ remains very close to $1$ over a
wide range of depths, up to $L=2\times 10^5$, for both target negative rates,
indicating that the negative-rate calibration keeps gradients in a trainable
regime even in extremely deep networks.

\begin{figure}[h!]
\centering 
\begin{tabular}{ccc}
\begin{subfigure}[b]{0.486\textwidth}
    \centering
    \includegraphics[width=\textwidth]{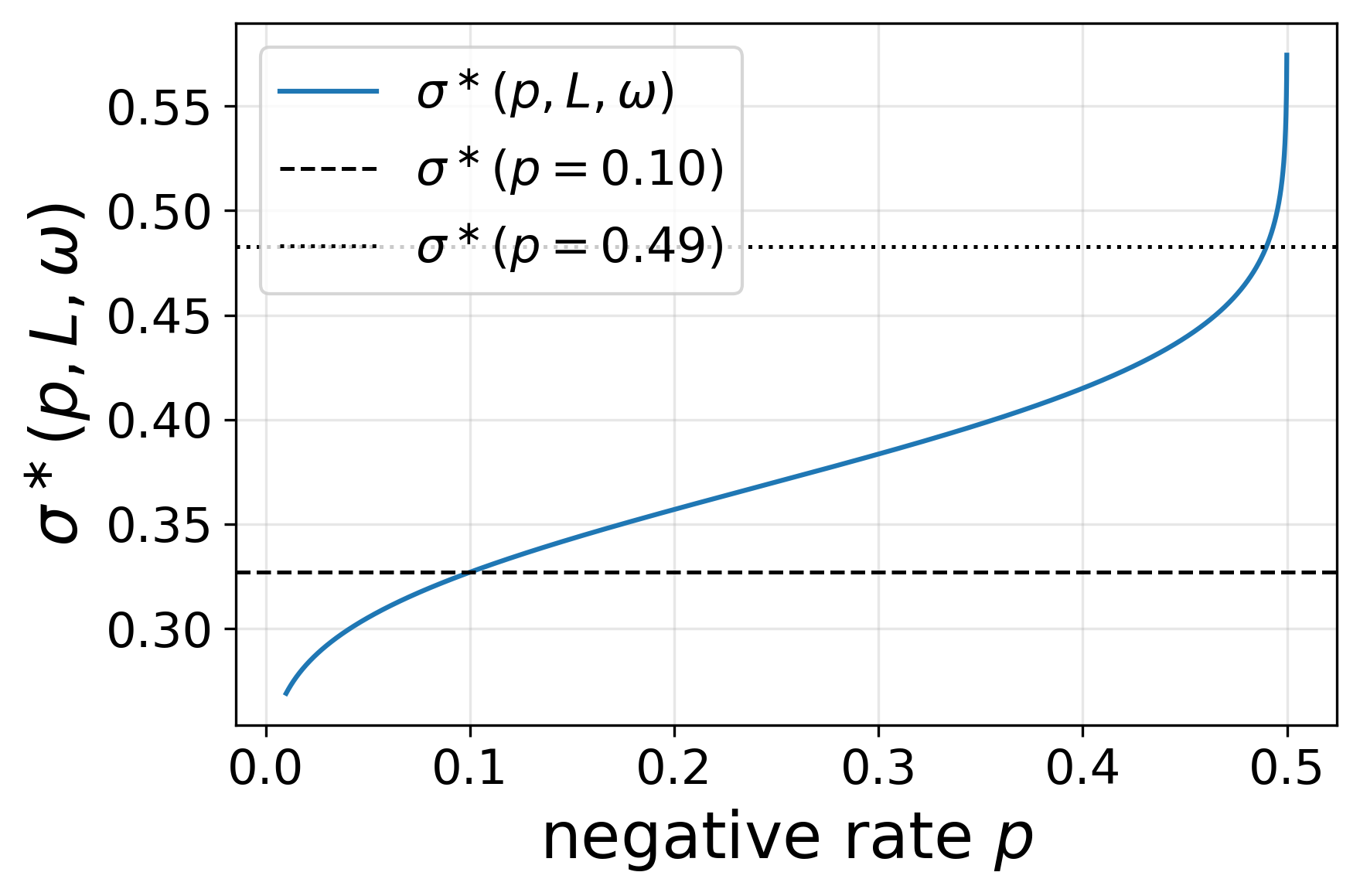}
    \caption{$L=100$}
\end{subfigure} &
\begin{subfigure}[b]{0.45\textwidth}
    \centering
    \includegraphics[width=\textwidth]{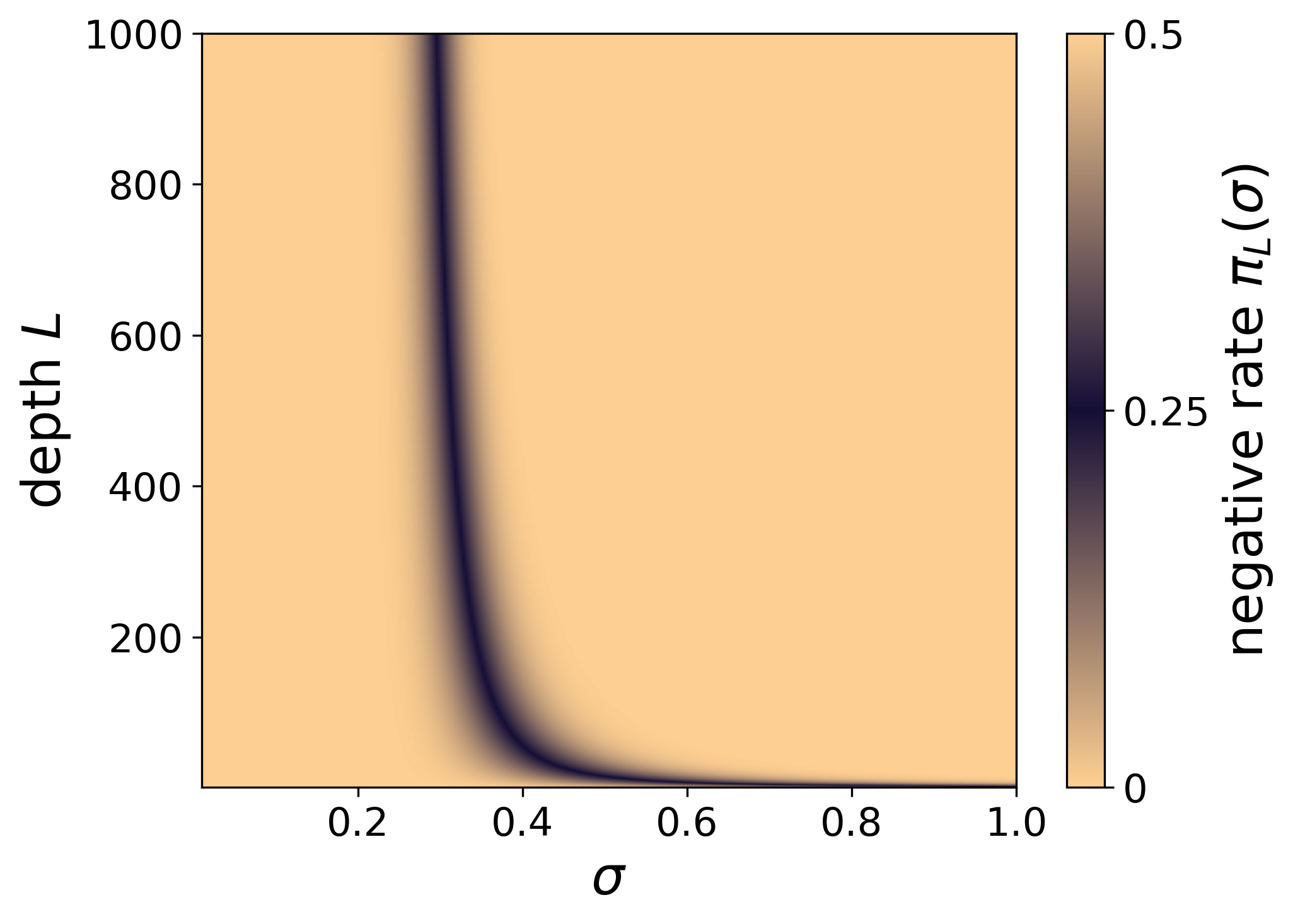}
    \caption{$L = 1,2,\dots,1000$}
\end{subfigure} 
\end{tabular}
\caption{%
 Closed-form characterization of the scalar surrogate.
    \textbf{(a)} Closed form scale $\sigma^\ast(p,L,\omega)$ as a function of the target negative rate $p$ for $L=100$ and $\omega=1$, with reference lines at $p=0.10$ and $p=0.49$.
    \textbf{(b)} Closed form scalar surrogate negative rate $\pi_L(\sigma)$ for $\omega=1$, network depths $L = 1,2,\dots,1000$, and $\sigma\in[0.01,1.0]$.}
\label{negative_rate_f}
\end{figure}

\begin{figure}[h!]
\centering 
\includegraphics[width=0.6\textwidth]{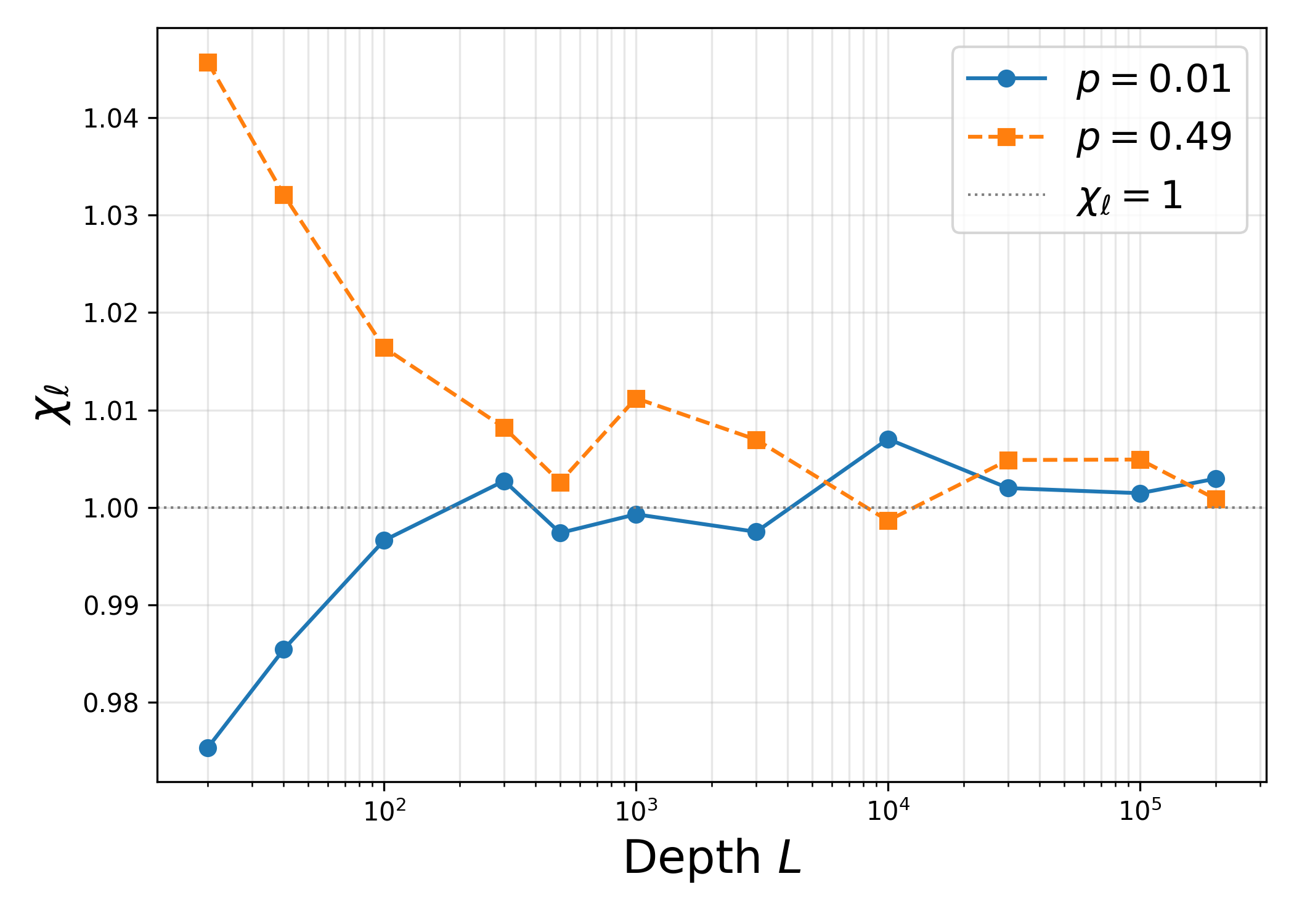}
\caption{$\chi_\ell$ for the proposed initialization with $f(x)=\tanh(x)$, evaluated at the closed form noise scales
$\sigma^\ast(p,L,\omega)$ corresponding to target negative rates
$p=0.01$ and $p=0.49$.
For each depth $\ell$, we estimate
$\chi_{\ell}(\sigma_z) \approx (\omega^2+\sigma_z^2)\,\mathbb{E}[f'(h^{\ell})^2]$.}
\label{chi_value}
\end{figure}

\clearpage 

\subsection{Negative Rate Surrogate Analysis}\label{app_negative}

We begin from the elementwise representation of a fully connected network with
odd–sigmoid activation $f\in\mathcal{F}$. For each layer $\ell$ and neuron
index $i$, the forward update can be written as
\begin{equation}\label{eq:app-prelim-signal}
x_i^{\ell+1}
= f\!\left(a_i^{\ell+1}\,x_i^{\ell}\right),
\qquad
a_i^{\ell+1}
= w_{ii}^{\ell+1}
+ \sum_{j\neq i}\frac{w_{ij}^{\ell+1}x_j^{\ell}}{x_i^{\ell}},
\end{equation}
so that $a_i^{\ell+1}$ plays the role of an effective scalar gain for neuron
$i$ in layer $\ell+1$. Under our proposed initialization
$W^{\ell+1} = \omega D^{\ell+1} + Z^{\ell+1}$ with $(Z^{\ell+1}_{ij})
\sim \mathcal N(0,\sigma_z^2/N_\ell)$, Lemma~\ref{lem1} shows that, conditional
on $x^\ell$,
\begin{equation}\label{eq:app-var-a}
a_i^{\ell+1}\ \sim\ \mathcal N\!\Biggl(\,\omega,\ \frac{\sigma_z^2}{N_\ell}
\Bigl(1+\sum_{j\ne i}\Bigl(\frac{x_j^\ell}{x_i^\ell}\Bigr)^2\Bigr)\Biggr).
\end{equation}
To handle this dependence on the current activations, we approximate the gain
distribution by a scalar surrogate model $a_i^{\ell+1} \sim
\mathcal N(\omega,\sigma_z^2)$ and estimate sign statistics from this Gaussian
approximation.

We define the negative rate at depth $L$ as
\[
  \pi_L(\sigma)
  \;:=\;
  \mathbb{P}\bigl(x_L < 0 \,\big|\, x_0 > 0\bigr),
\]
that is, the probability that a neuron whose initial activation is positive
ends up with a negative sign at layer $L$. Because every $f\in\mathcal{F}$ is
odd and strictly increasing, the sign of $x_i^\ell$ is entirely controlled by
the product of the effective gains along the path, and the probability that a
positive entry becomes negative at depth $L$ equals the probability that a
negative entry becomes positive. It therefore suffices to track a single scalar
chain starting from $x_0>0$.

Our initialization has a nonzero mean gain $\omega$, which makes sign changes
along a given coordinate well defined across layers. We interpret frequent sign
flips during the forward pass as a form of information loss, and we therefore
use the surrogate calibration to control the negative rate at the final layer.
In the scalar surrogate, for a given depth $L$ and gain variance $\sigma^2$ we can compute $\pi_L(\sigma)$ in closed form. We then define, for a target
$p\in(0,\tfrac12)$, the calibrated scale
$\sigma^\ast(p,L,\omega)$ as the unique solution to $\pi_L(\sigma^\ast)=p$
(see Theorem~\ref{thm:closed-sigma}). In practice, we choose a desired
“real” negative rate $p_{\mathrm{real}}=0.4$ and select $\sigma_z$ so that the empirical FFNN-driven negative rate at depth $L$ is close to $p_{\mathrm{real}}$,
which preserves most sign information while still leaving enough randomness for learning.

\begin{figure}[h!]
\centering 
\includegraphics[width=0.63\textwidth]{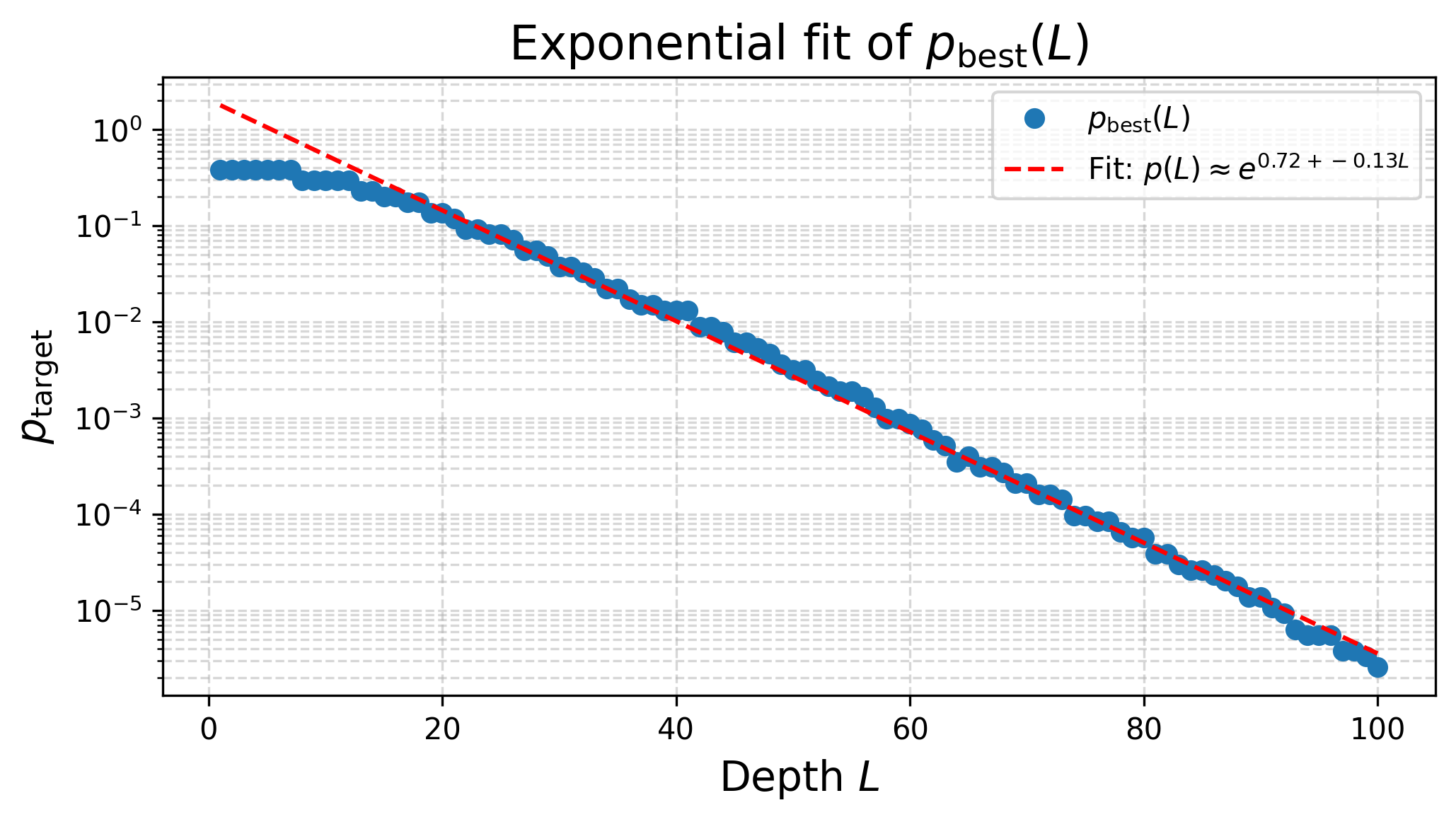}
\caption{Best surrogate target $p_{\mathrm{target}}(L)$ producing an FFNN-driven negative rate
$\tilde\pi_L \approx 0.4$ as a function of depth $L$ for a $\tanh$ network with width
$64$. The curve shows that $p_{\mathrm{target}}(L)$ stays near $0.3$–$0.4$ for shallow
depths and then decays approximately exponentially with $L$, providing an empirical
calibration rule for choosing the surrogate negative rate in deep networks.
}
\label{surro_real}
\end{figure}

To understand how the scalar surrogate relates to the actual network, we
compare $\pi_L(\sigma^\ast)$ with the empirical negative rate obtained from an
initialized FFNN. For each $p_{\mathrm{target}}$ and depth $L$, we first compute
$\sigma^\ast(p_{\mathrm{target}},L,\omega)$, initialize an $L$-layer $\tanh$
network with weights $W^\ell = \omega I + (\sigma^\ast/\sqrt{N})G^\ell$, extract
the effective gains $a_i^\ell$, and form FFNN-driven scalar chains
$y_0 = 1$, $y_{\ell+1} = \tanh(a^\ell y_\ell)$. Measuring
$\tilde\pi_L = \mathbb{P}(y_L<0)$ and comparing it to the analytic
$\pi_L(\sigma^\ast)$ reveals that the discrepancy grows with depth (see
Figure~\ref{surrogate_real}), as expected from the increasing
influence of higher-order correlations. Nevertheless, the relationship between
the surrogate target $p_{\mathrm{target}}$ and the FFNN-driven negative rate
$\tilde\pi_L$ remains systematic.

\begin{figure}[t!]
\centering 
\begin{tabular}{ccc}
\begin{subfigure}[b]{0.30\textwidth}
    \centering
    \includegraphics[width=\textwidth]{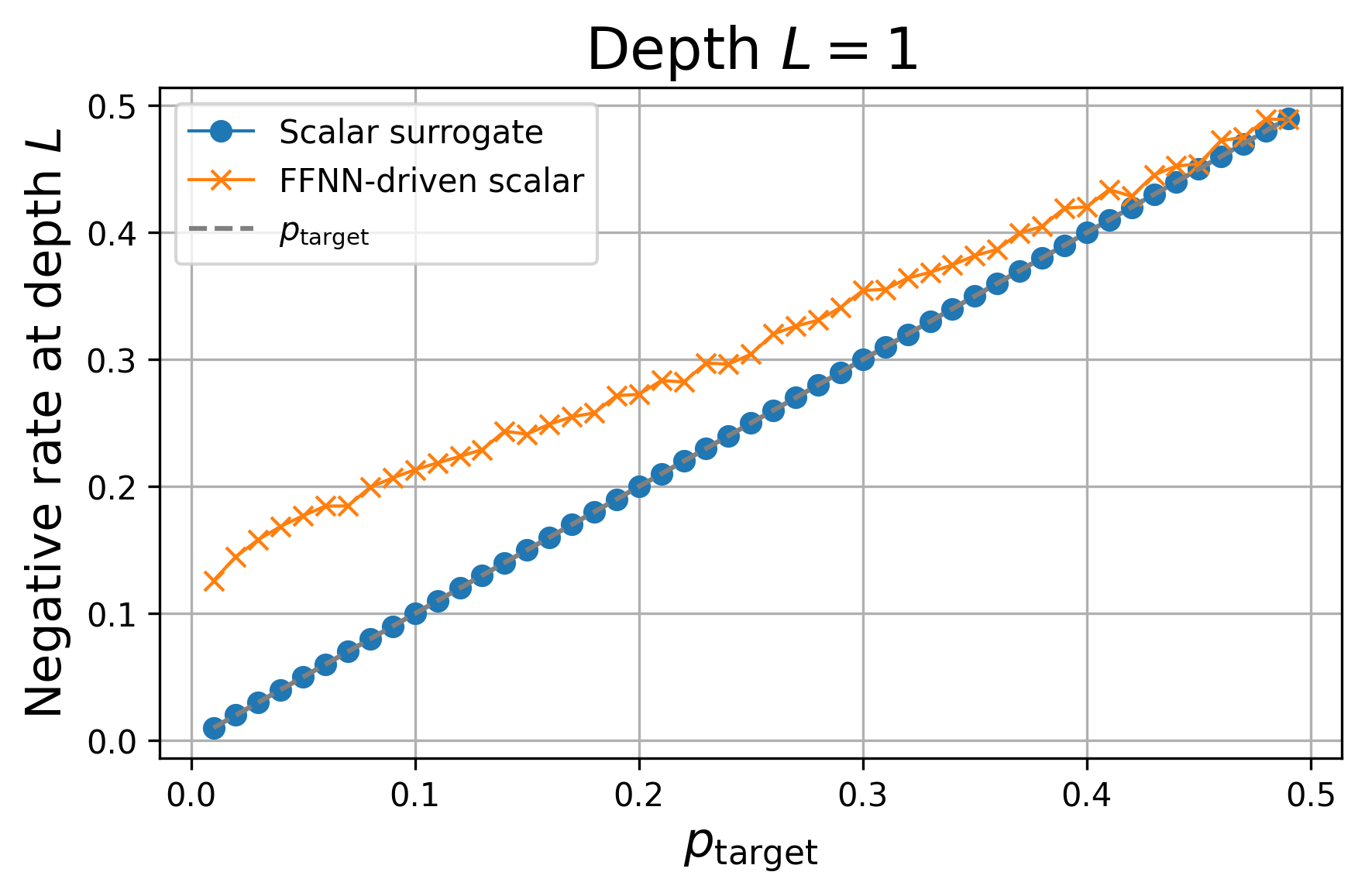}
    \caption{$L=1$}
\end{subfigure} &
\begin{subfigure}[b]{0.30\textwidth}
    \centering
    \includegraphics[width=\textwidth]{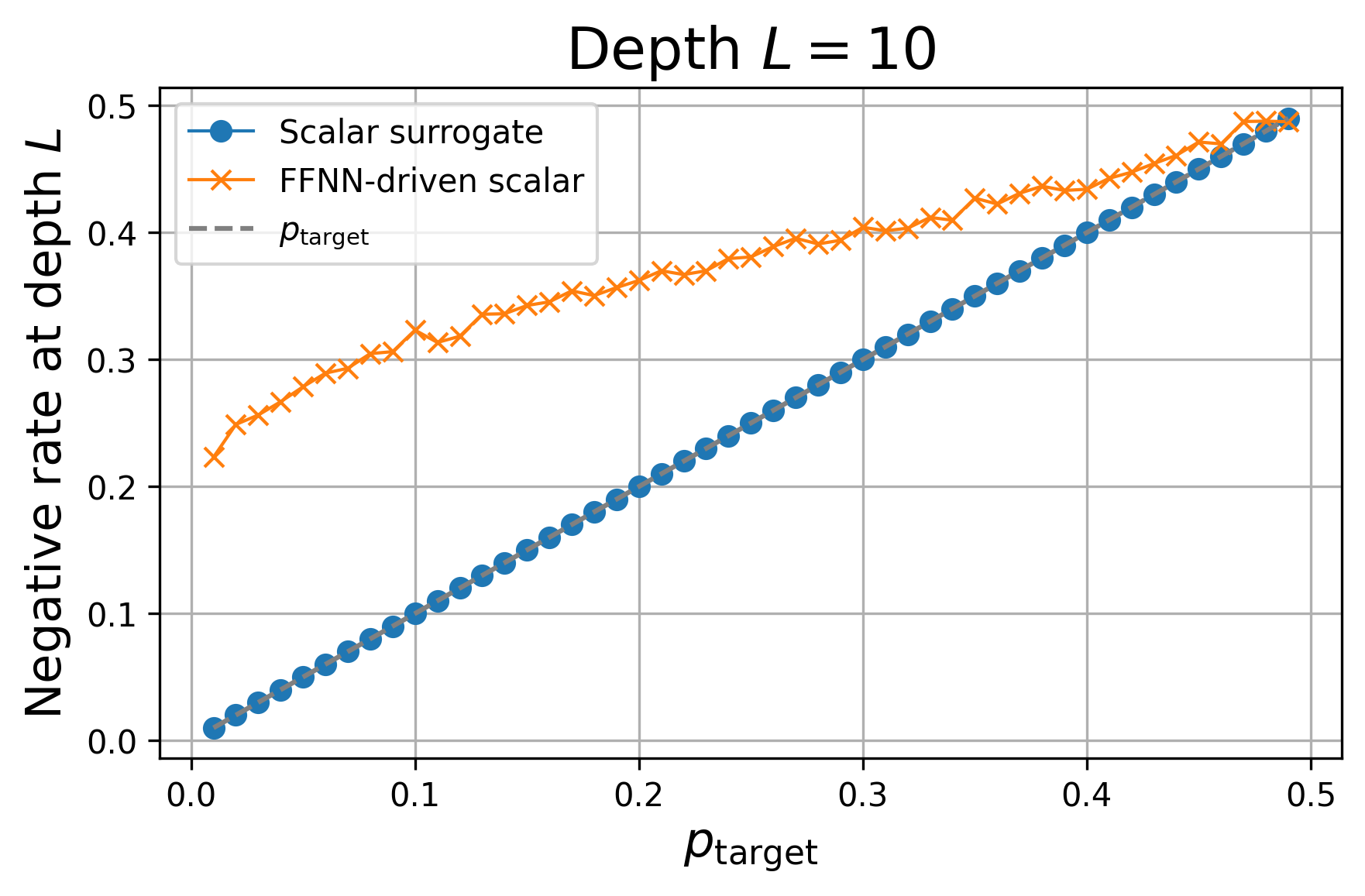}
    \caption{$L=10$}
\end{subfigure} &
\begin{subfigure}[b]{0.30\textwidth}
    \centering
    \includegraphics[width=\textwidth]{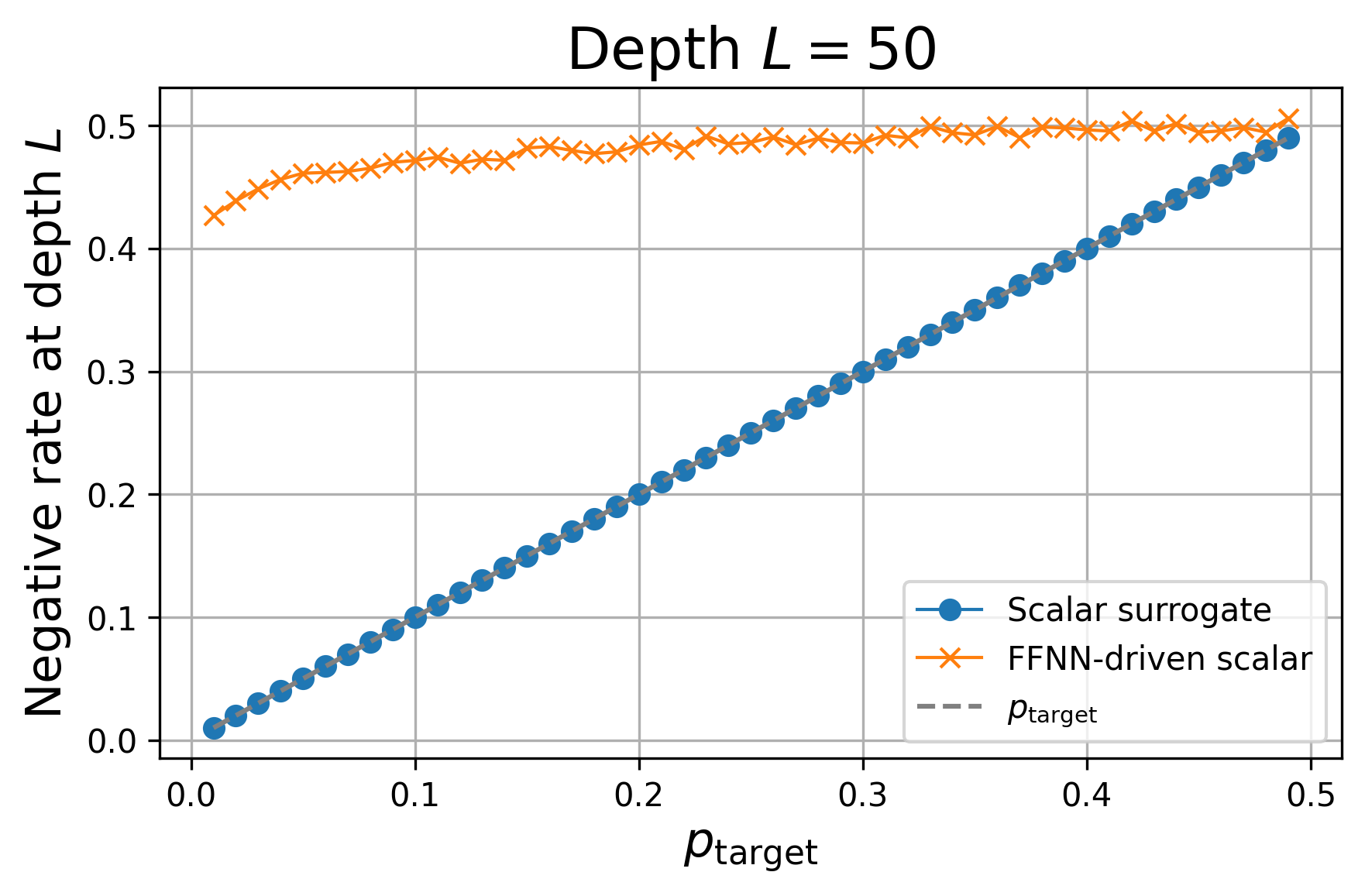}
    \caption{$L=50$}
\end{subfigure} 
\end{tabular}
\caption{%
 Negative rate at depth $L=1, 10, 50$ as a function of the target negative rate
$p_{\mathrm{target}}$. The curve $\pi_1(\sigma^\star)$ denotes the scalar
surrogate prediction at the calibrated scale $\sigma^\star(p_{\mathrm{target}})$,
$\tilde\pi_1$ is the empirical negative rate obtained from the FFNN-driven scalar
chains, and the dashed line shows the identity $p_{\mathrm{target}}$. }
\label{surrogate_real}
\end{figure}

\begin{figure}[b!]
\centering 
\begin{tabular}{ccc}
\begin{subfigure}[b]{0.486\textwidth}
    \centering
    \includegraphics[width=\textwidth]{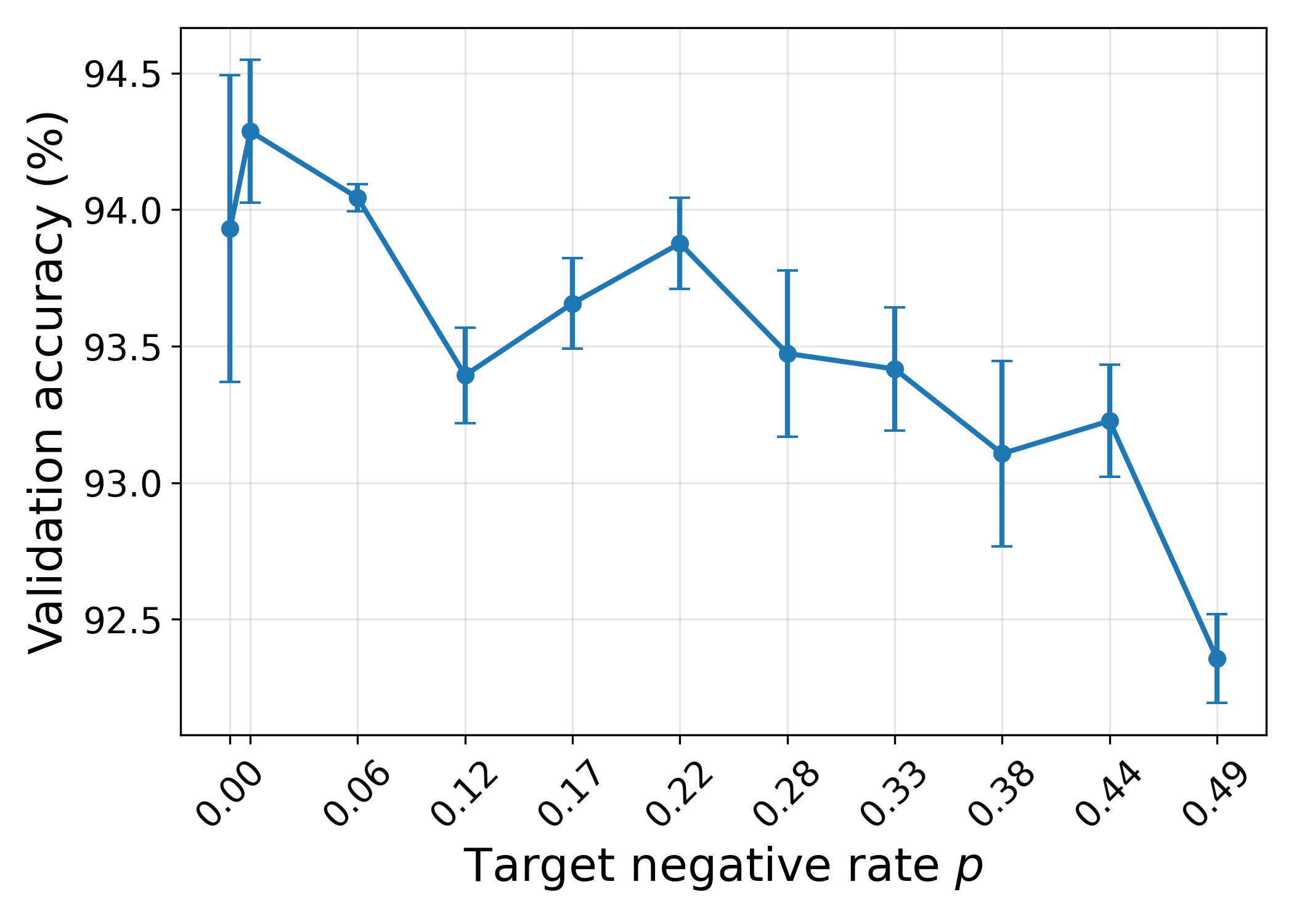}
    \caption{MNIST}
\end{subfigure} &
\begin{subfigure}[b]{0.486\textwidth}
    \centering
    \includegraphics[width=\textwidth]{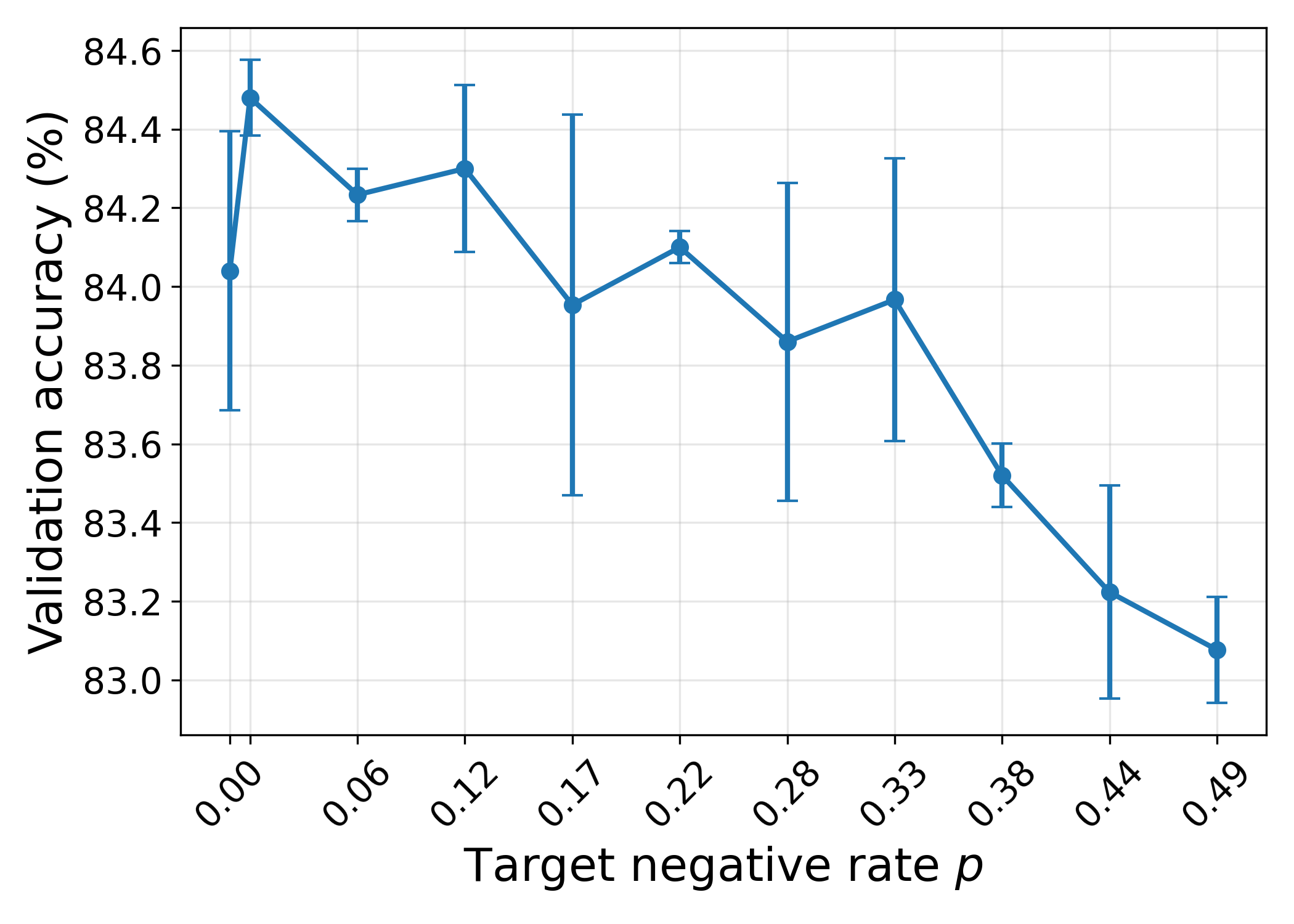}
    \caption{MNIST}
\end{subfigure} 
\end{tabular}
\caption{%
 Validation accuracy as a function of the target negative rate $p$
for a 50-layer fully connected network (width $64$) with activation
$f(x)=\tanh(x)$ under the proposed initialization. Each point shows the
mean $\pm$ standard deviation over $5$ runs, where each run is trained
for 600 iterations, with
$\sigma^\ast(p,L,\omega)$ computed from the scalar surrogate calibration.
}
\label{fsweep}
\end{figure}

In particular, for each depth $L$ we can invert this relationship numerically:
we search over $p_{\mathrm{target}}$ and find the value $p_{\mathrm{sur}}(L)$
such that the FFNN-driven negative rate $\tilde\pi_L$ is as close as possible
to $p_{\mathrm{real}}=0.4$~(see Section~\ref{b.7} ).

The resulting sequence $p_{\mathrm{sur}}(L)$ is
reported in Figure~\ref{surro_real}. For small depths $L$, the surrogate
target remains near $0.3$–$0.4$, but for larger $L$ it decays approximately
exponentially. Fitting a simple model
$\log p_{\mathrm{sur}}(L) \approx c_0 + c_1 L$ on the tail (e.g., $L\ge 10$)
yields
\[
  p_{\mathrm{sur}}(L) \;\approx\; C e^{-\alpha L},
\]
for some $C>0$ and $\alpha>0$. This empirical law expresses how aggressively the
surrogate target negative rate must shrink with depth in order for the actual
network to maintain a fixed, moderate negative rate at its final layer.

These calibration insights are consistent with our training experiments.
Figures~\ref{fsweep} and~\ref{fsweep1} plot validation
accuracy as a function of the surrogate target $p$ for deep ($L=50$) and
shallow ($L=3$) networks. For the deep case, performance is maximized when
$p_{\mathrm{sur}}$ is on the order of $0.01$, in line with the exponentially
decayed target predicted by the surrogate analysis. In contrast, for shallow
networks the best performance is attained near $p_{\mathrm{sur}}\approx 0.4$,
matching the regime where the fitted decay has not yet taken effect. Finally,
Figures~\ref{fsweep50} and~\ref{fsweep100} show that, for depths $L=50$ and $L=100$, the learning curves are optimized near the
$\sigma^\ast$ values obtained from the calibrated surrogate, with accuracy
degrading as we move away from these scales. Together, these results validate
that the negative-rate surrogate provides a useful, quantitatively accurate
guideline for choosing the noise level $\sigma_z$ across depths, and that the
resulting diagonal–plus–noise initialization indeed preserves signal sign
statistics in deep networks.

\begin{figure}[h!]
\centering 
\begin{tabular}{ccc}
\begin{subfigure}[b]{0.486\textwidth}
    \centering
    \includegraphics[width=\textwidth]{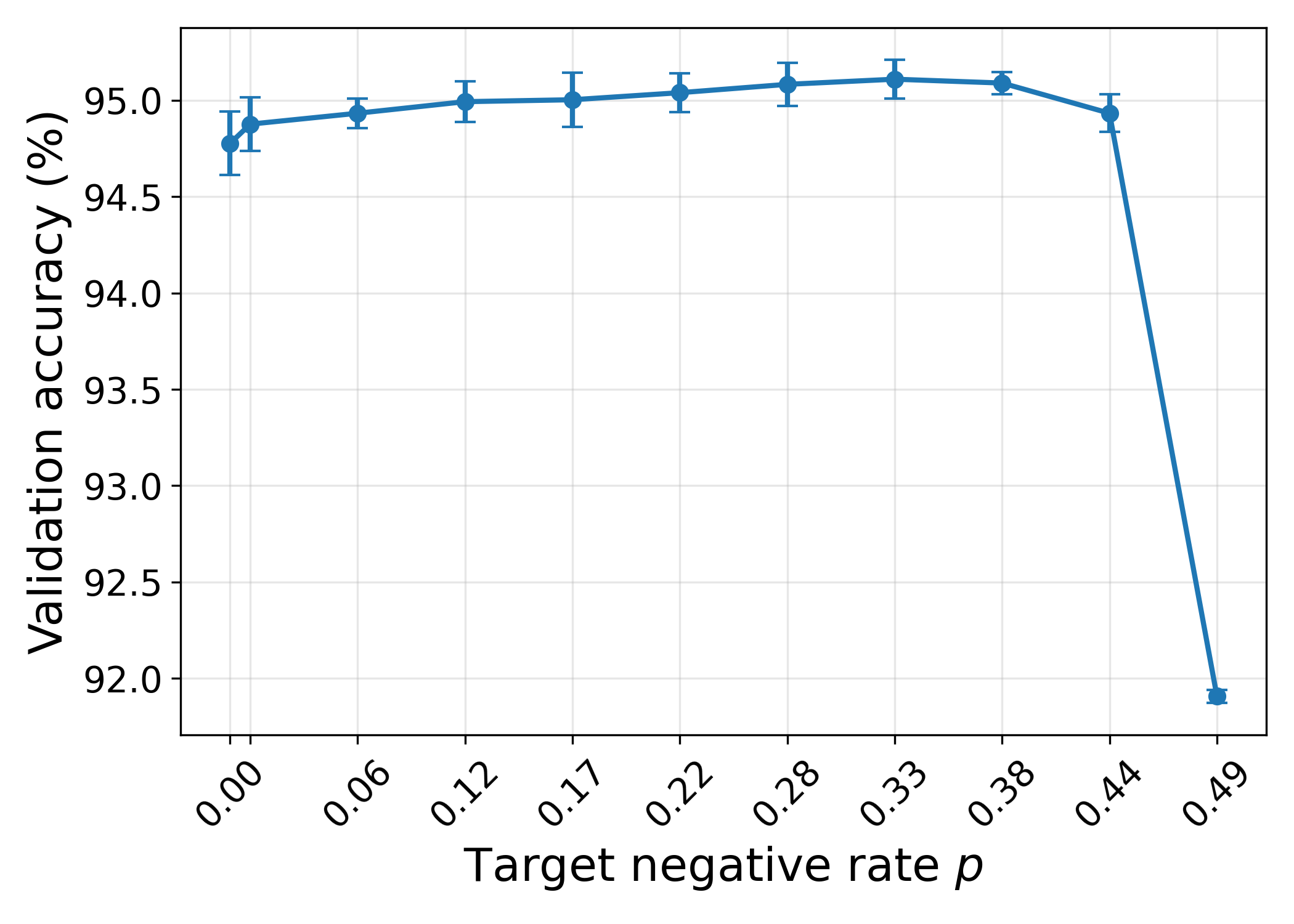}
    \caption{MNIST}
\end{subfigure} &
\begin{subfigure}[b]{0.486\textwidth}
    \centering
    \includegraphics[width=\textwidth]{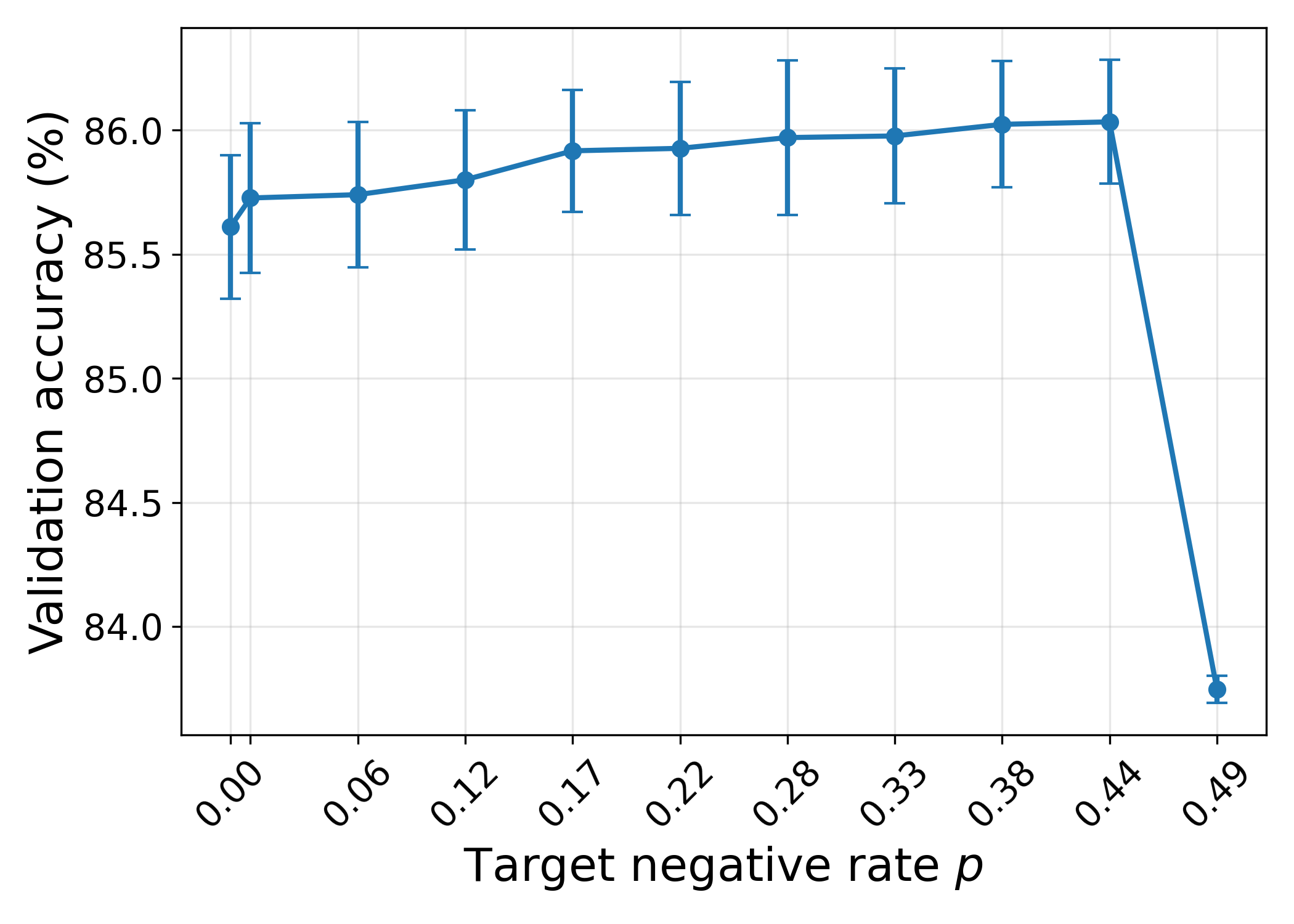}
    \caption{FMNIST}
\end{subfigure} 
\end{tabular}
\caption{%
 Validation accuracy as a function of the target negative rate $p$
for a 3-layer fully connected network (width $512$) with activation
$f(x)=\tanh(x)$ under the proposed initialization. Each point shows the
mean $\pm$ standard deviation over $5$ runs, where each run is trained
for 600 iterations, with
$\sigma^\ast(p,L,\omega)$ computed from the scalar surrogate calibration.
}
\label{fsweep1}
\end{figure}

\begin{figure}[h!]
\centering 
\begin{tabular}{ccc}
\begin{subfigure}[b]{0.32\textwidth}
    \centering
    \includegraphics[width=\textwidth]{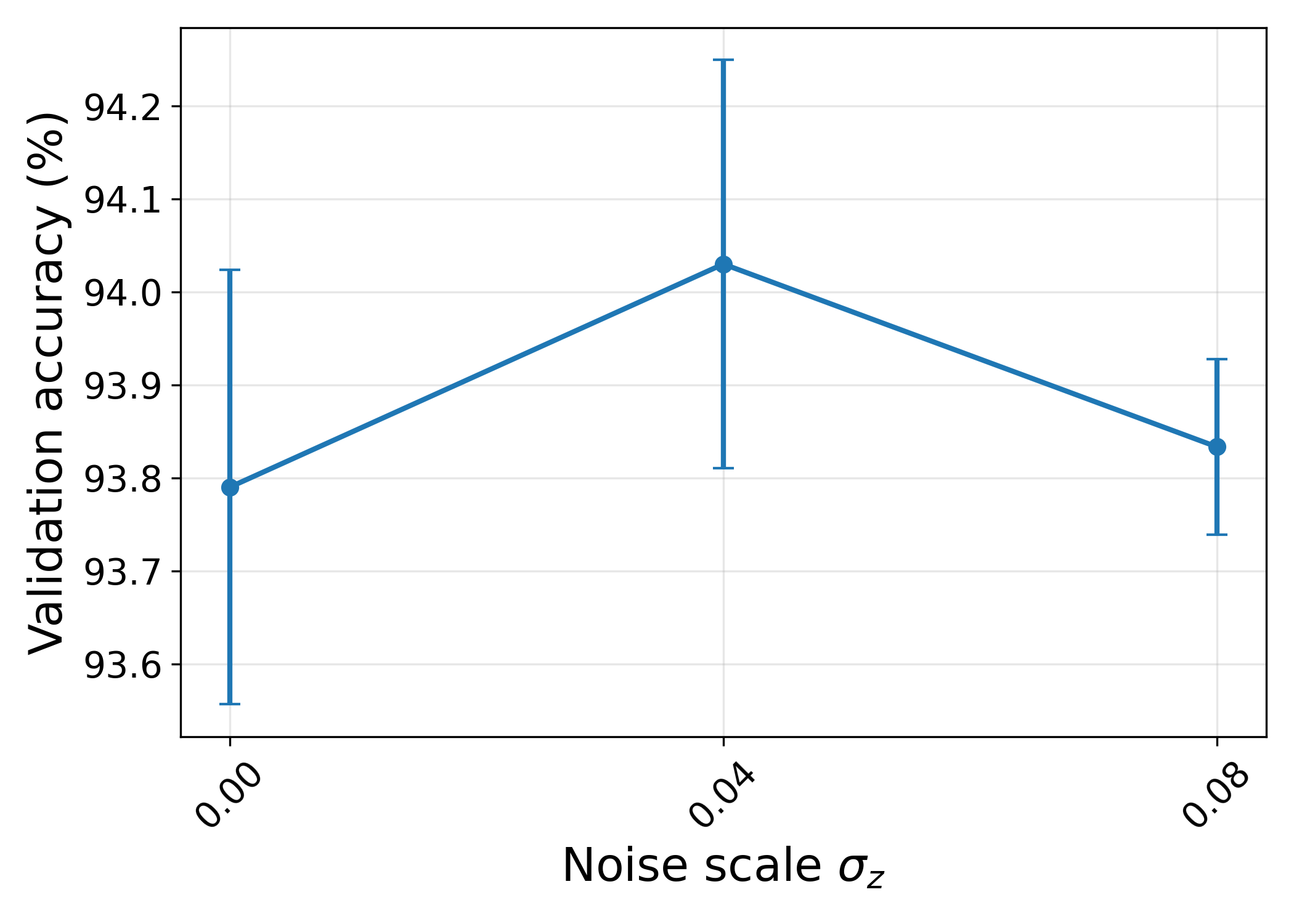}
    \caption{MNIST}
\end{subfigure} &
\begin{subfigure}[b]{0.32\textwidth}
    \centering
    \includegraphics[width=\textwidth]{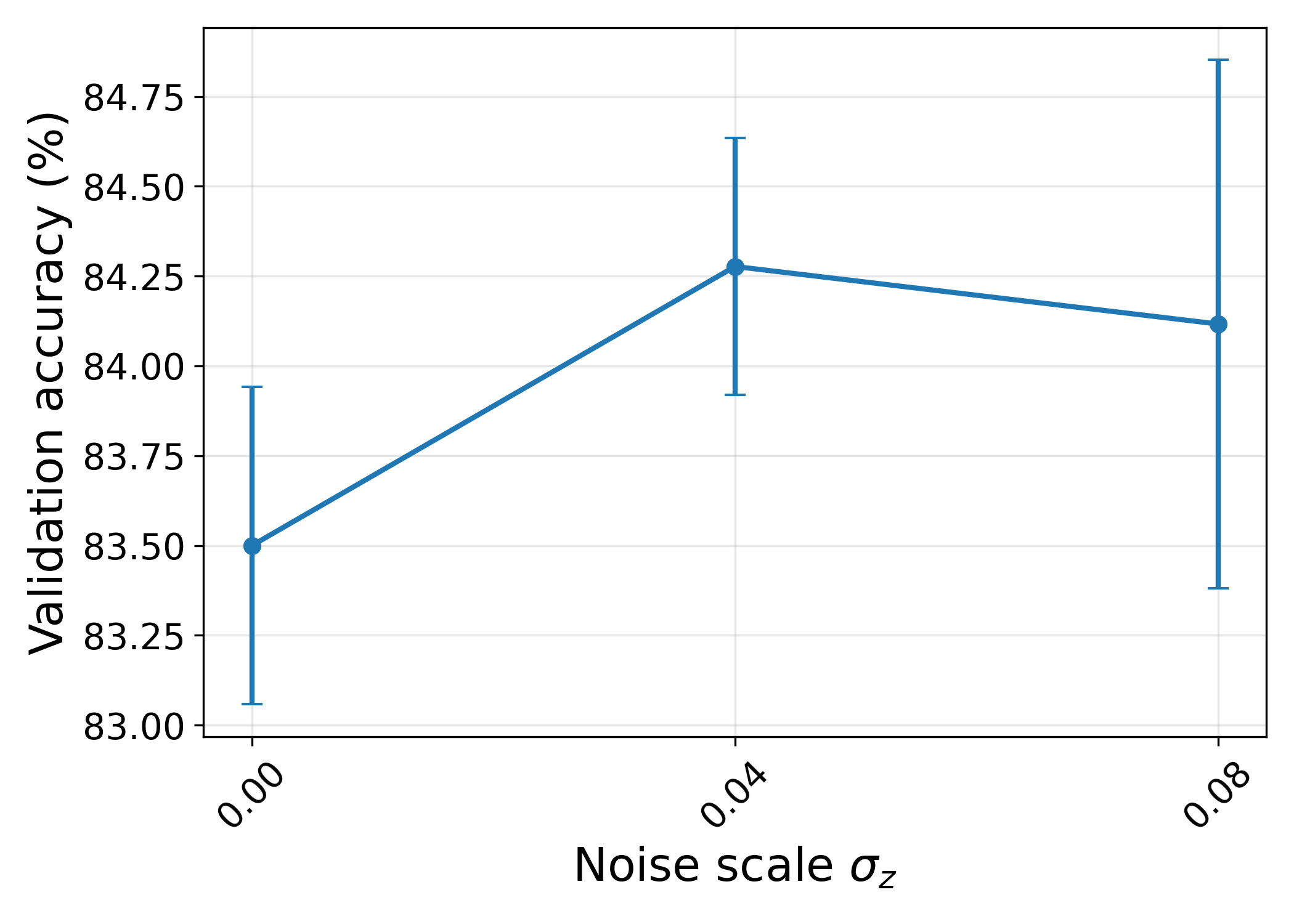}
    \caption{FMNIST}
\end{subfigure} &
\begin{subfigure}[b]{0.32\textwidth}
    \centering
    \includegraphics[width=\textwidth]{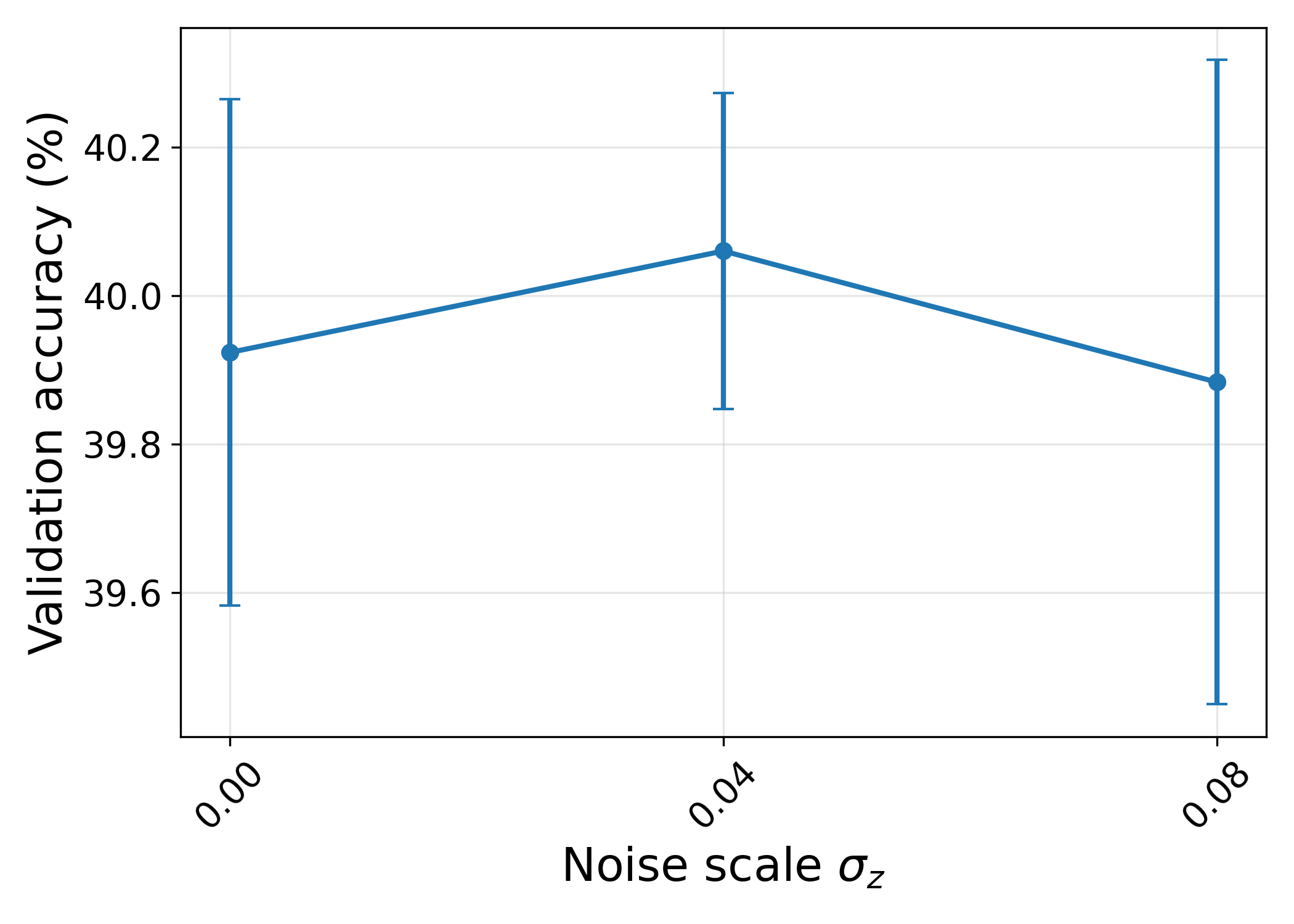}
    \caption{CIFAR-10}
\end{subfigure} 
\end{tabular}
\caption{%
 Validation accuracy as a function of the noise scale $\sigma_z$ for 50 layer fully connected networks~(width $64$) with activation $f(x)=\tanh(x)$ under the
proposed initialization, on MNIST, Fashion MNIST, and CIFAR-10. For each
$\sigma_z\in\{0,0.04,0.08\}$, we report the mean $\pm$ standard deviation over 5 training runs.
}
\label{fsweep50}
\end{figure}

\begin{figure}[h!]
\centering 
\begin{tabular}{ccc}
\begin{subfigure}[b]{0.32\textwidth}
    \centering
    \includegraphics[width=\textwidth]{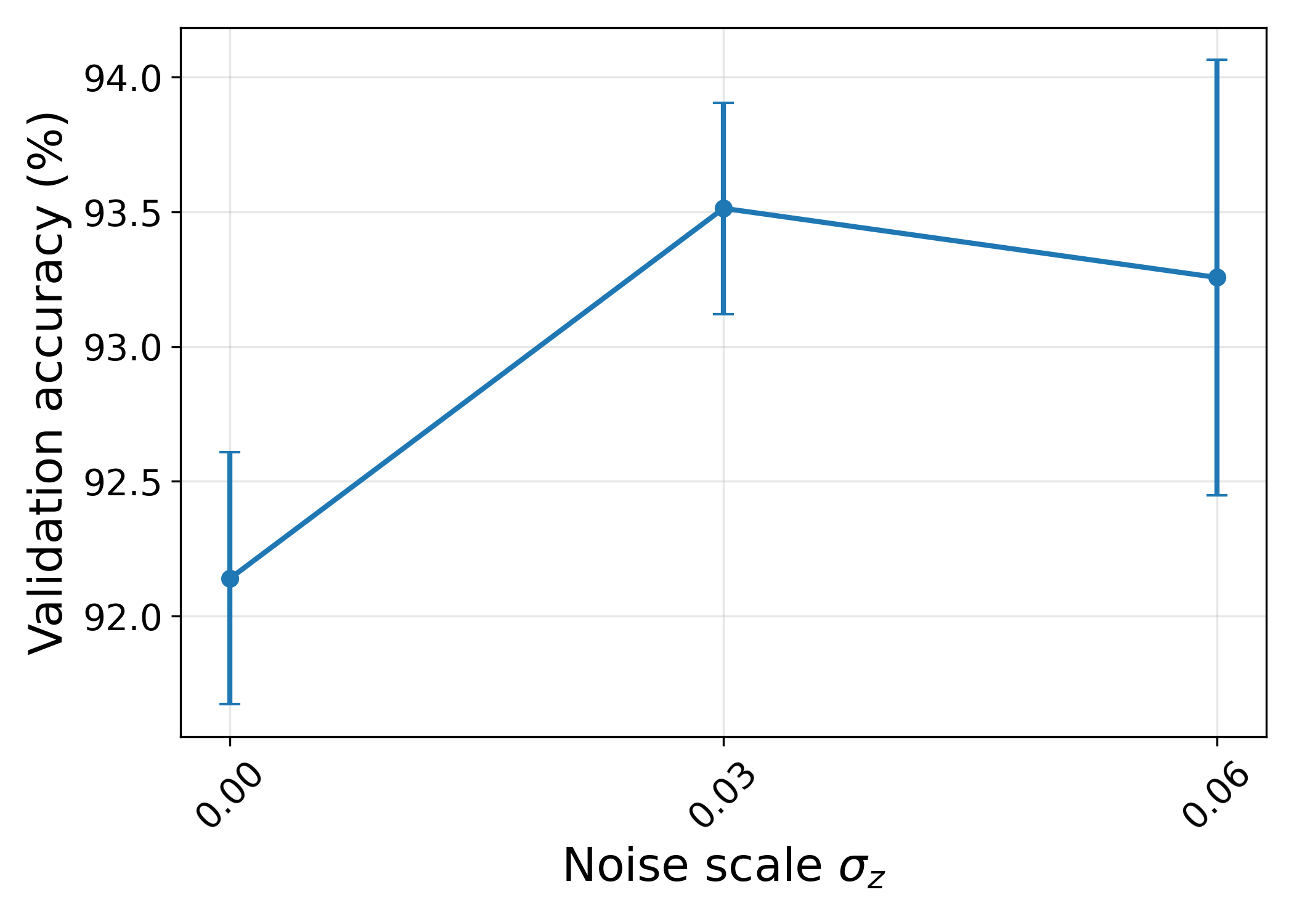}
    \caption{MNIST}
\end{subfigure} &
\begin{subfigure}[b]{0.32\textwidth}
    \centering
    \includegraphics[width=\textwidth]{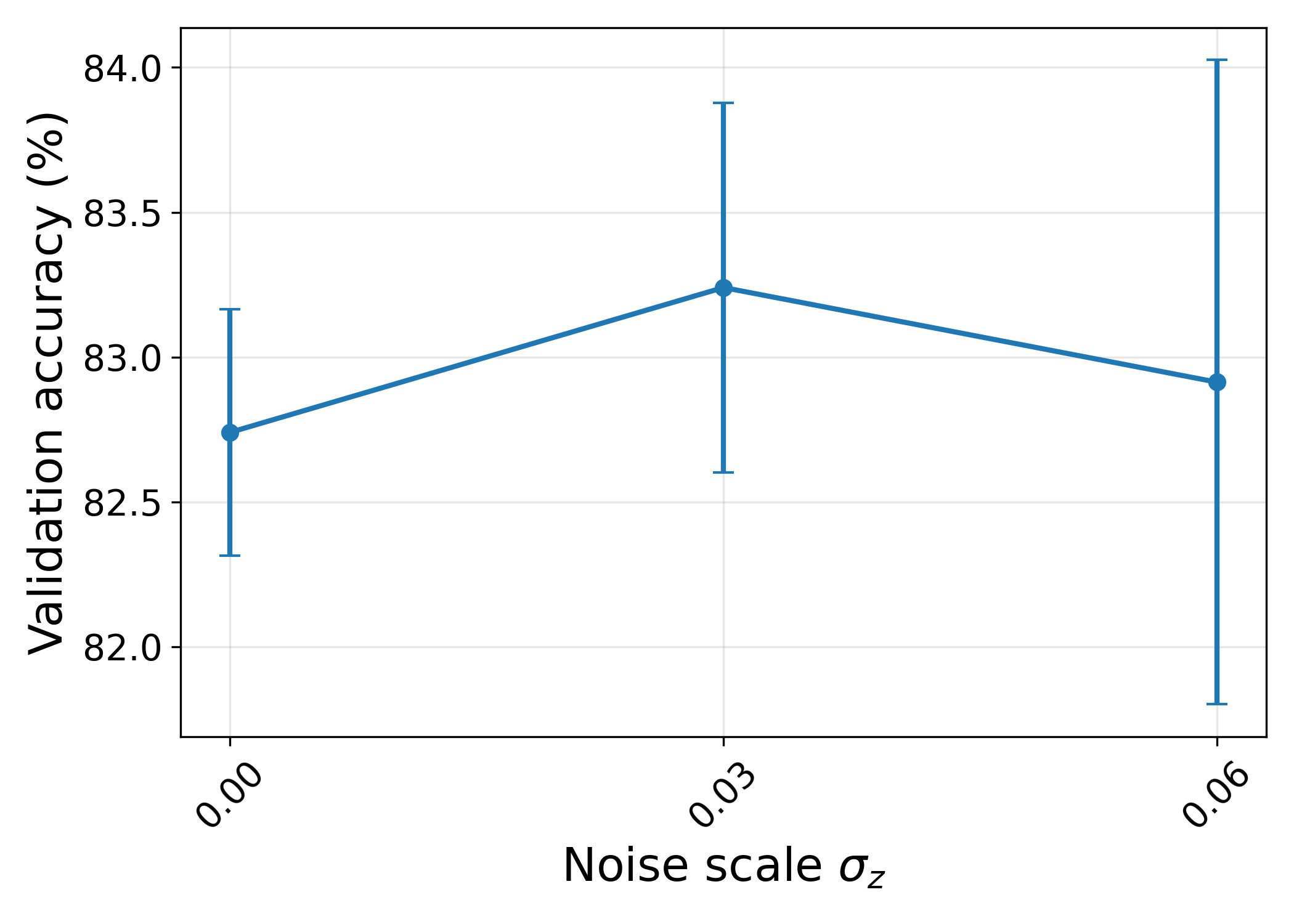}
    \caption{FMNIST}
\end{subfigure} &
\begin{subfigure}[b]{0.32\textwidth}
    \centering
    \includegraphics[width=\textwidth]{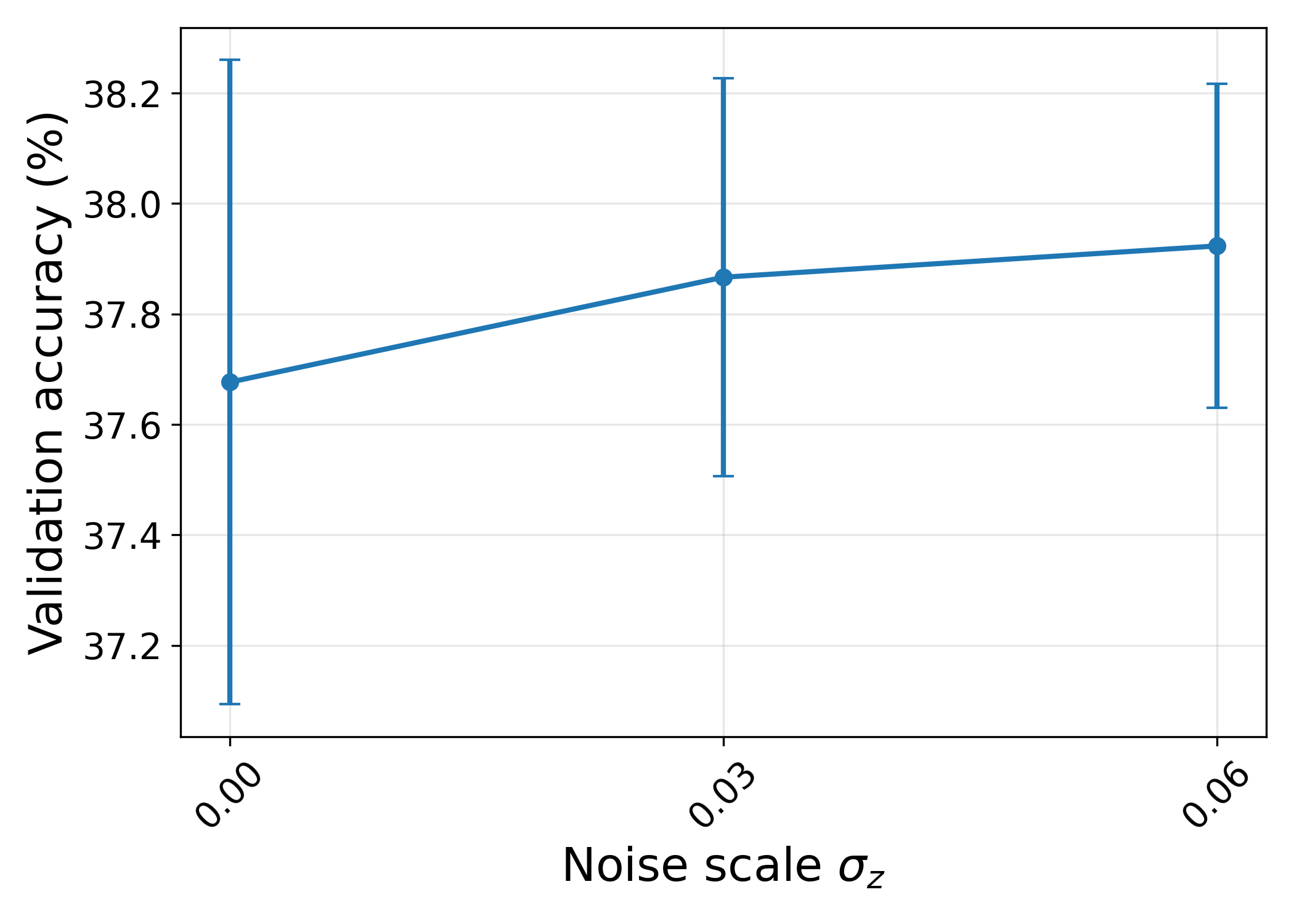}
    \caption{CIFAR-10}
\end{subfigure} 
\end{tabular}
\caption{%
 Validation accuracy as a function of the noise scale $\sigma_z$ for 100 layer fully connected networks~(width $64$) with activation $f(x)=\tanh(x)$ under the
proposed initialization, on MNIST, Fashion MNIST, and CIFAR-10. For each
$\sigma_z\in\{0,0.03,0.06\}$, we report the mean $\pm$ standard deviation over 5 training runs.
}
\label{fsweep100}
\end{figure}

\clearpage

\subsection{forward signal propagation}
In Section~\ref{section4.3} we study forward signal propagation under the proposed
initialization. Figure~\ref{negative_rate_f} (a) reports how well the
last-layer activation distribution is preserved as the depth $L$ increases (with
fixed width $W=64$), while panel (b) varies the width $N_\ell$ at fixed depth
$L=1000$ to test whether the distribution remains well spread even in relatively
narrow networks. In both settings, the proposed initialization maintains a
dispersed last-layer distribution up to depth $L=10{,}000$ and down to width
$N_\ell=20$, whereas the EOC (Gaussian i.i.d.) initialization quickly collapses
toward a more concentrated distribution.

To summarize how well the last-layer distribution is dispersed, we employ a
normalized histogram-entropy score
\begin{equation}\label{spread}
\mathrm{Spread}(x)
=\frac{-\sum_{i=1}^{B} p_i \log p_i}{\log B},
\qquad
p_i=\int_{\text{bin }i} \!\! \hat{f}_x(t)\,dt,\;\; \sum_{i=1}^{B} p_i=1,
\end{equation}
where $\hat{f}_x$ is the empirical density over $[-1,1]$ using $B$ bins.
Values close to $1$ indicate a highly dispersed (near-uniform) last-layer
distribution, whereas values near $0$ correspond to a highly concentrated
distribution.

\medskip

\begin{figure}[h!]
\centering 
\begin{tabular}{ccc}
\begin{subfigure}[b]{0.45\textwidth}
    \centering
    \includegraphics[width=\textwidth]{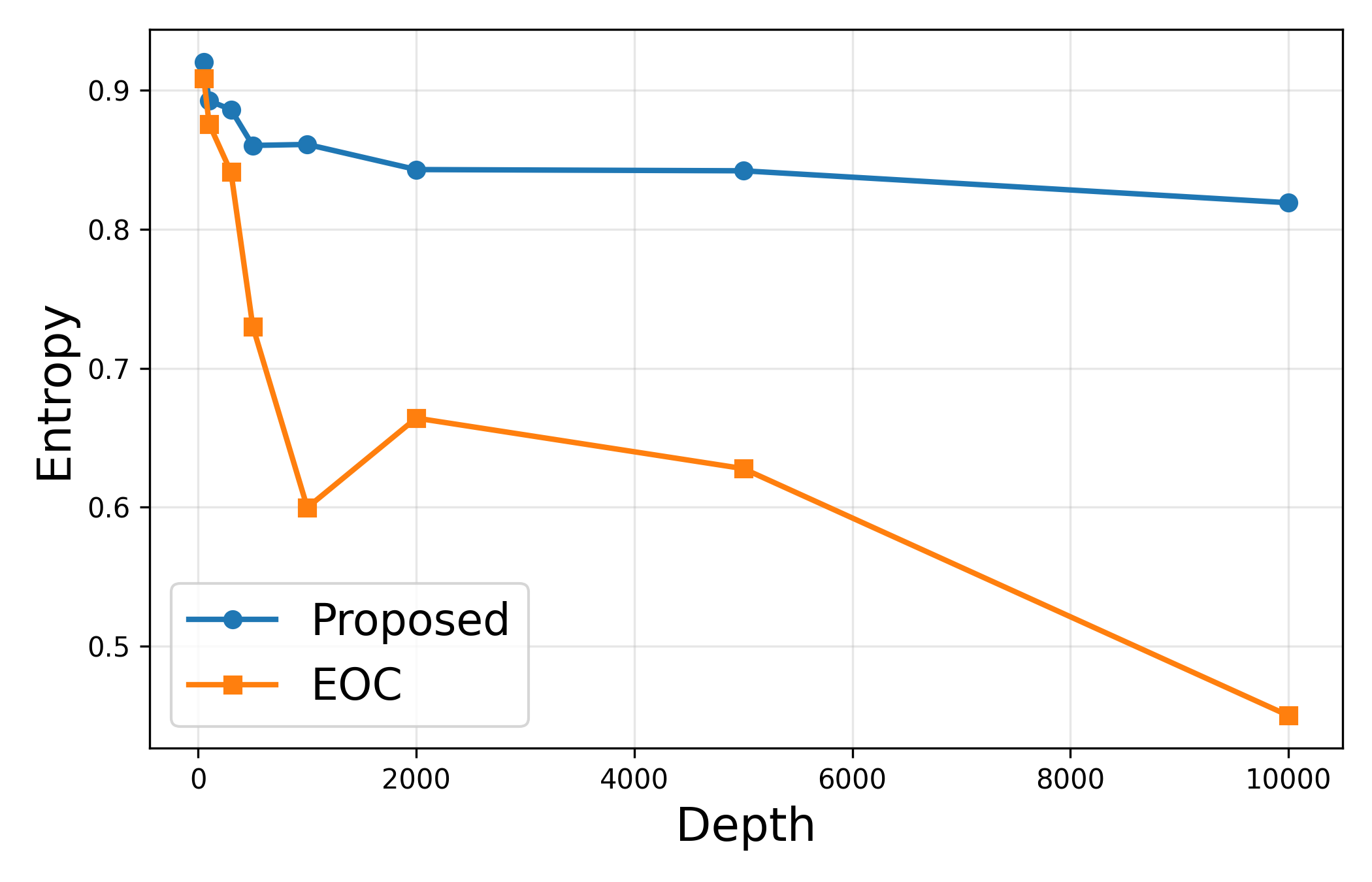}
    \caption{Depth}
\end{subfigure} &
\begin{subfigure}[b]{0.45\textwidth}
    \centering
    \includegraphics[width=\textwidth]{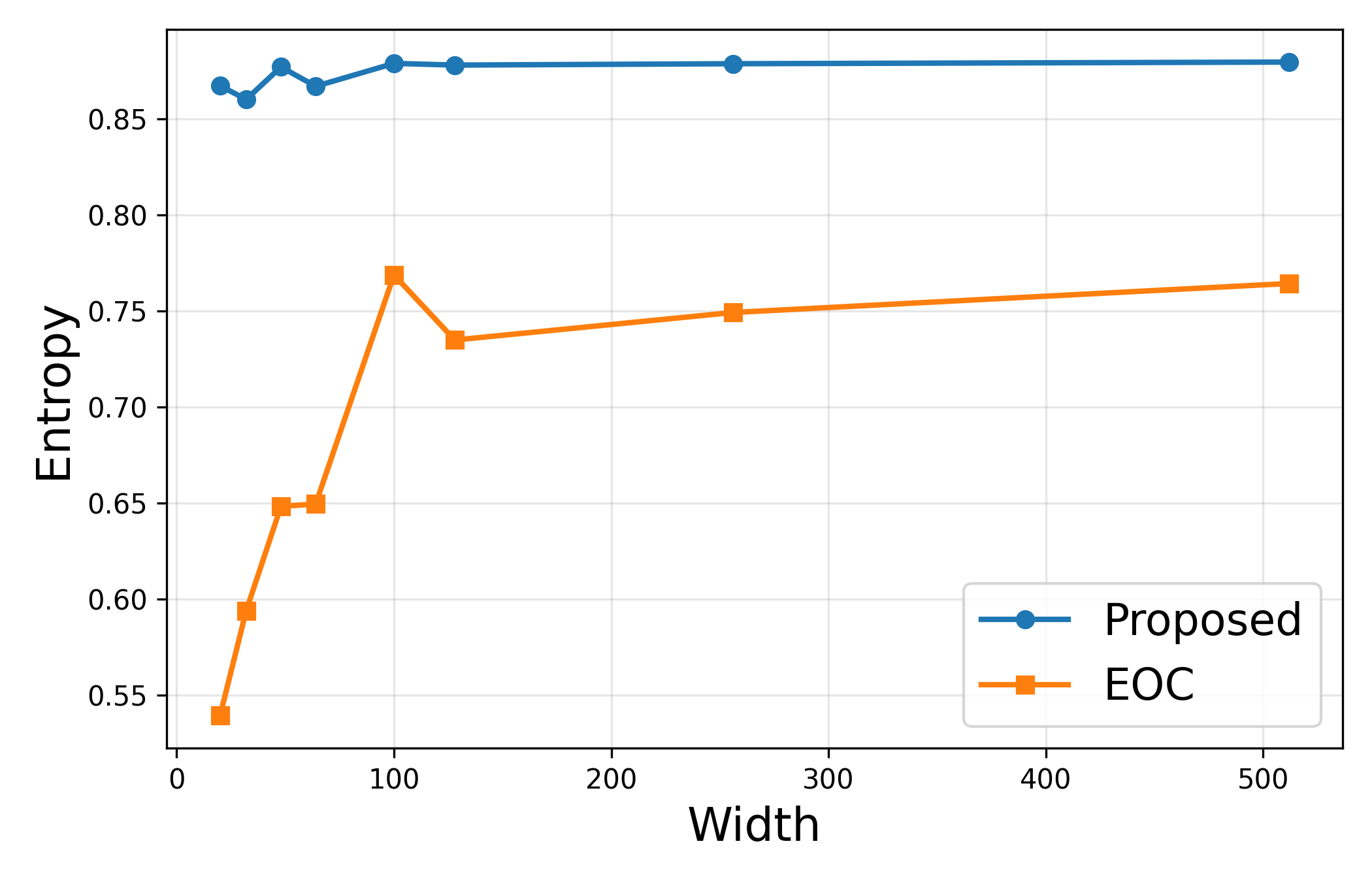}
    \caption{Width}
\end{subfigure} 
\end{tabular}
\caption{
 Entropy of the last-layer activation distribution for tanh networks under the
proposed initialization and the EOC initialization. \textbf{(a)} Entropy as a function
of depth $L$ with fixed width $W=64$. \textbf{(b)} Entropy as a function of width
$N_{\ell}$ with fixed depth $L=1000$, using widths
$W \in \{20, 32, 48, 64, 100, 128, 256, 512\}$ as shown in the panels. The
entropy is defined in Equation~\ref{spread}.
}
\label{negative_rate_f1}
\end{figure}

Figures~\ref{fig_histo_depth}, \ref{fig_histo_depth0.1}, and \ref{fig_histo_depth10} visualize the last-layer activation histograms obtained after the initial forward pass in 1000-layer FFNNs of width $64$ with activations $0.1\tanh(x)$, $\tanh(x)$, and $10\tanh(x)$, respectively. Under the proposed initialization, the activation distribution
remains well dispersed in all three cases, whereas Gaussian i.i.d.\
initializations quickly collapse toward a narrow band around zero. Figures~\ref{fig_histo_width_0.1} and~\ref{fig_histo_width10} show a similar comparison as the width is reduced: even for very narrow networks, the proposed scheme preserves a spread-out last layer distribution, while Gaussian initializations drive the last layer activations to saturate near zero.

\clearpage

\begin{figure}[h!]
\centering 
\includegraphics[width=1\textwidth]{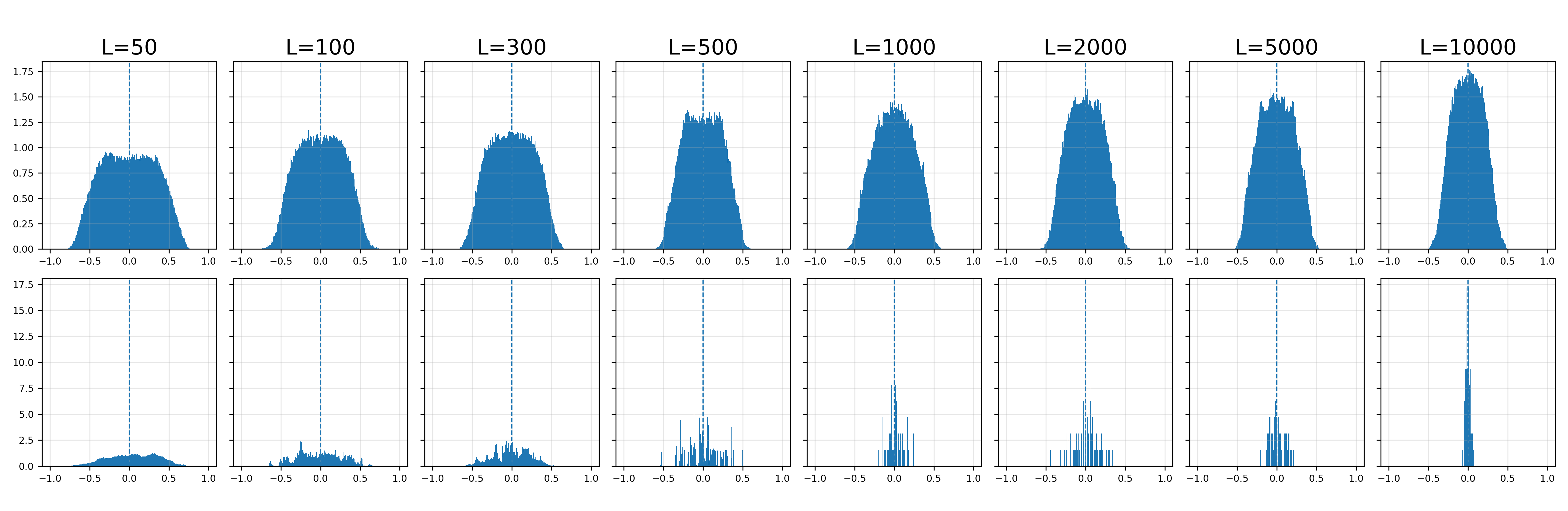}
\caption{Last layer activation histograms for $\tanh$ networks with width $W=64$ and varying depth under the proposed initialization (top row) and the EOC initialization (bottom row). Each column corresponds to a different depth $L$.}
\label{fig_histo_depth}
\end{figure}

\begin{figure}[h!]
\centering 
\includegraphics[width=1\textwidth]{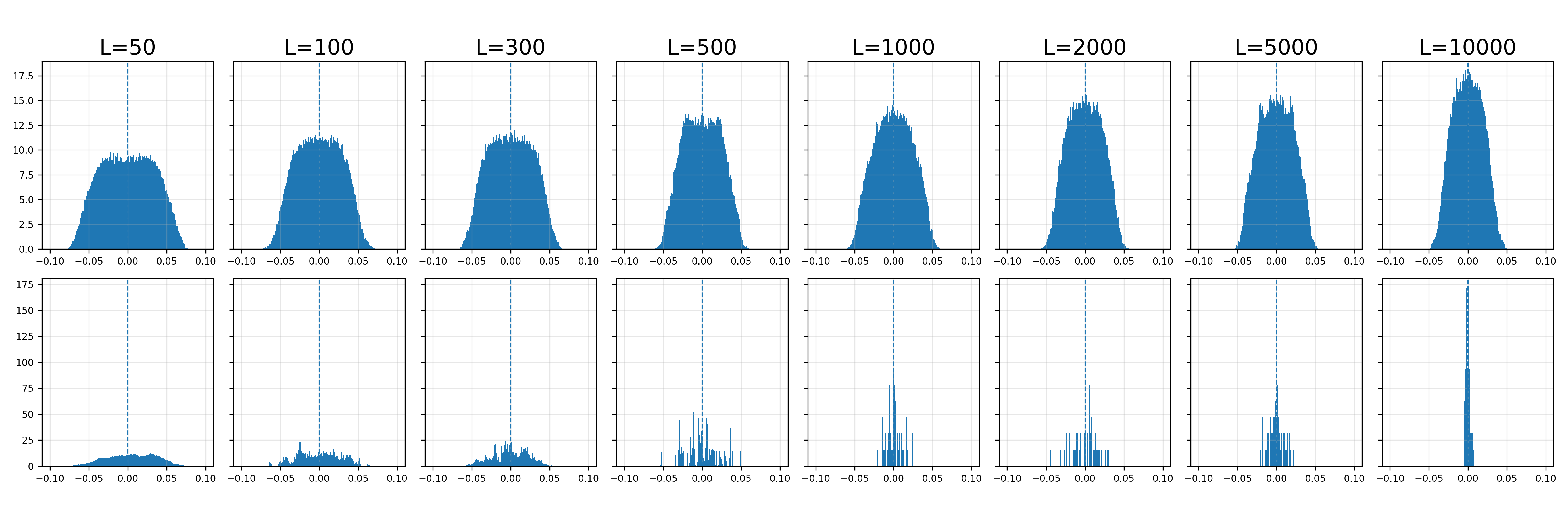}
\caption{Last layer activation histograms for $0.1\tanh$ networks with width $W=64$ and varying depth under the proposed initialization (top row) and the EOC initialization (bottom row). Each column corresponds to a different depth $L$.}
\label{fig_histo_depth0.1}
\end{figure}

\begin{figure}[h!]
\centering 
\includegraphics[width=1\textwidth]{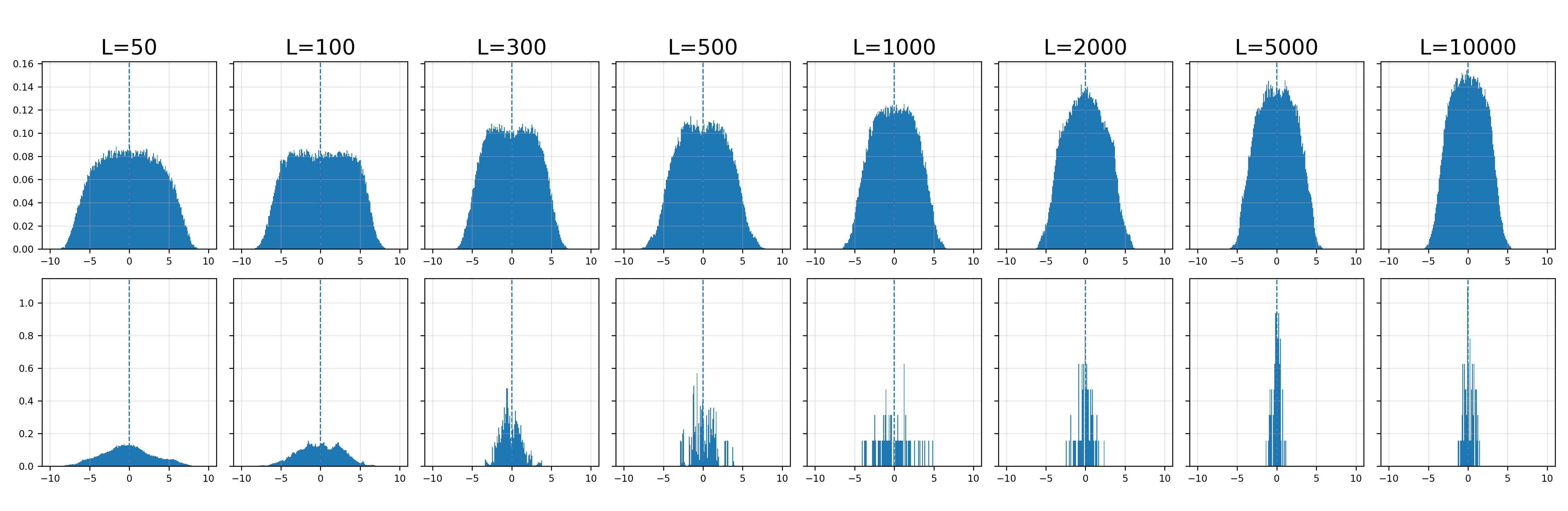}
\caption{Last layer activation histograms for $10\tanh$ networks with width $W=64$ and varying depth under the proposed initialization (top row) and the EOC initialization (bottom row). Each column corresponds to a different depth $L$.}
\label{fig_histo_depth10}
\end{figure}

\clearpage

\begin{figure}[h!]
\centering 
\includegraphics[width=1\textwidth]{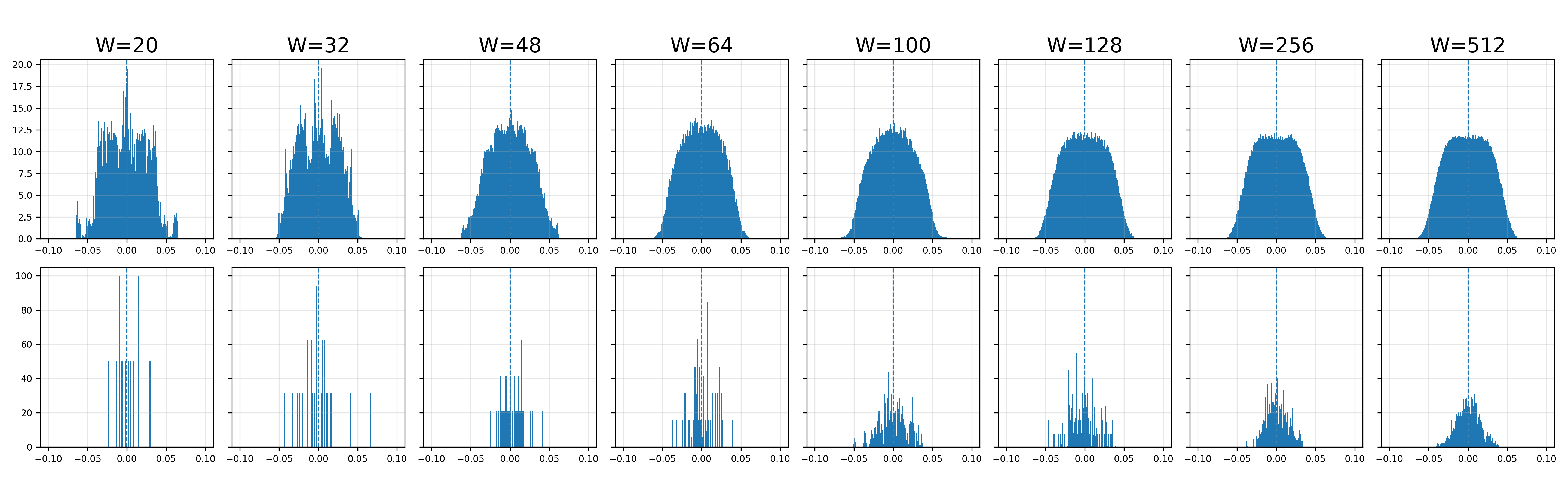}
\caption{Last layer activation histograms for $0.1\tanh$ networks with depth $L=1000$ and varying width under the proposed initialization (top row) and the EOC initialization (bottom row). }
\label{fig_histo_width_0.1}
\end{figure}

\begin{figure}[h!]
\centering 
\includegraphics[width=1\textwidth]{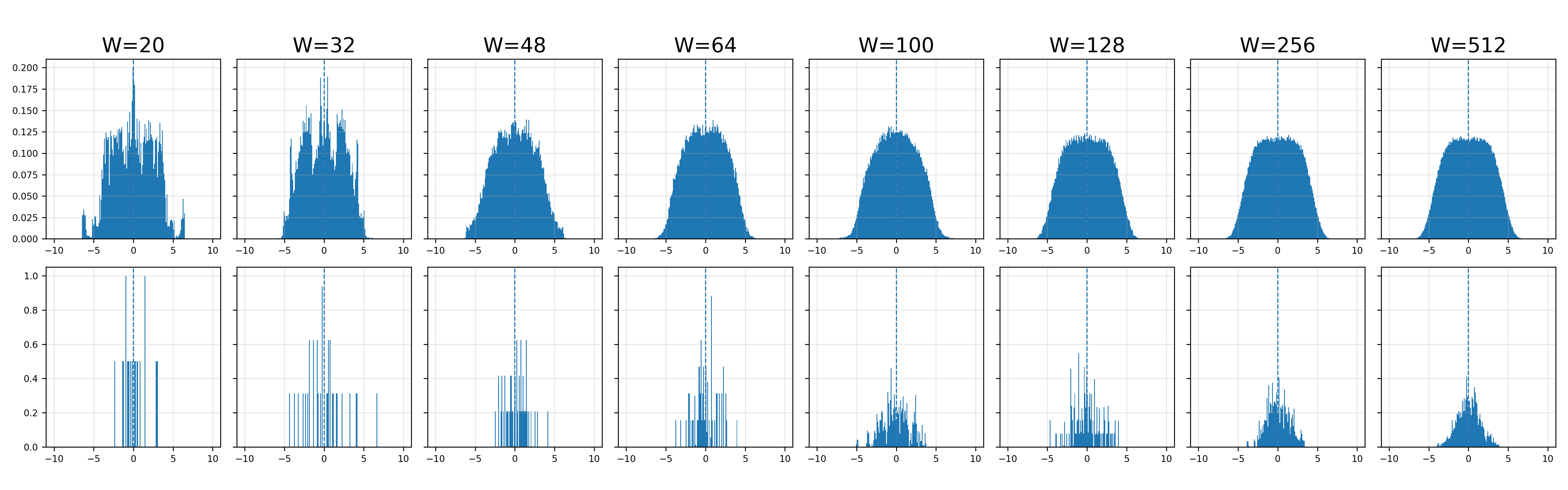}
\caption{Last layer activation histograms for $10\tanh$ networks with depth $L=1000$ and varying width under the proposed initialization (top row) and the EOC initialization (bottom row).}
\label{fig_histo_width10}
\end{figure}

\subsection{backward signal propagation}
Figure~\ref{chi_heatmap_tanh} visualizes $|\chi_L(\sigma)-1|$ for the proposed
initialization and a Gaussian initialization with $\tanh$ activations, as a
function of depth $L$ and scale $\sigma$. For the proposed scheme, the region
where $\chi_L(\sigma)\approx 1$ occupies a broad band in the $(L,\sigma)$ plane,
whereas for the Gaussian initialization it is confined to a narrow strip around
a single variance. This indicates that our initialization keeps $\chi_L$ close
to $1$ over a much wider range of noise scales and depths, and is therefore more
robust and better suited for stable training in deep, narrow networks.

\begin{figure}[h!]
\centering 
\includegraphics[width=1\textwidth]{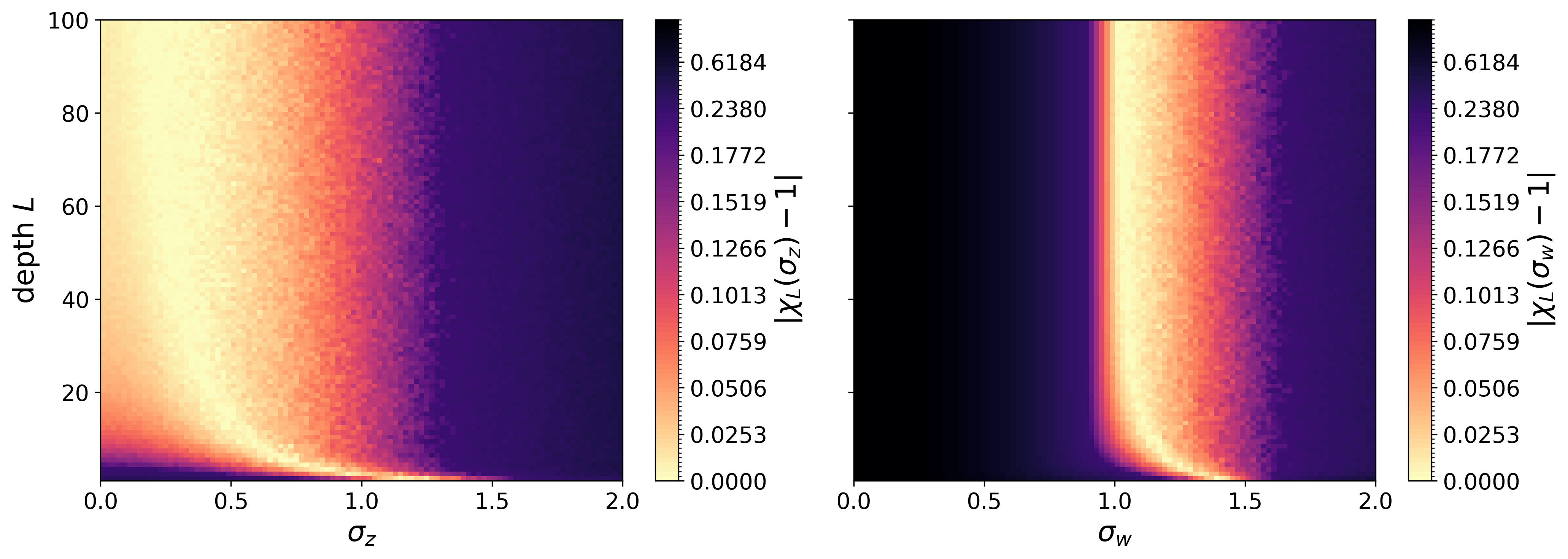}
\caption{Heatmaps of the deviation $\lvert\chi_L(\sigma)-1\rvert$ for the proposed
initialization \textbf{(left)} and a Gaussian i.i.d.\ initialization
\textbf{(right)} with $\tanh$ activations, as a function of depth $L$ and scale $\sigma$.
Brighter bands indicate near critical regimes where forward and backward signals
are approximately preserved.}
\label{chi_heatmap_tanh}
\end{figure}

\subsection{Choice of the target negative rate $p_{\mathrm{real}} = 0.4$.}\label{b.7}
Because $f \in \mathcal{F}$ is odd and strictly increasing, the sign of each coordinate $x_i^\ell$ is entirely determined by the product of the effective gains along that coordinate. We therefore use the coordinate-wise sign flip probability
\[
  \tilde\pi_L := \mathbb{P}(x_L < 0 \mid x_0 > 0)
\]
as a simple indicator for how much ``sign information'' about the input is preserved at depth $L$. Two extreme regimes are undesirable. If $\tilde\pi_L \approx 0$, almost all coordinates preserve their initial sign, so the network behaves nearly like an identity map at initialization; this preserves information but yields very limited expressiveness and weak exploration of the odd--sigmoid nonlinearity. At the other extreme, if $\tilde\pi_L \approx 0.5$, the final sign is essentially a fair coin flip regardless of the initial sign, meaning that the directional information carried by the input has been almost completely randomized and we interpret this as a form of information loss.

In practice, we therefore target an intermediate regime in which most coordinates keep their initial sign, but a non-negligible fraction flip so that the representation can change meaningfully. Concretely, we fix a desired ``real'' negative rate $p_{\mathrm{real}} = 0.4$ and choose $\sigma_z$ so that the empirical FFNN-driven negative rate at depth $L$ satisfies $\tilde\pi_L \approx p_{\mathrm{real}}$. This calibration preserves the sign of the majority of coordinates ($\approx 60\%$) while still allowing a substantial minority ($\approx 40\%$) to flip, which provides enough randomness for learning without completely destroying the initial sign structure.

We do not claim that $p_{\mathrm{real}} = 0.4$ is an information-theoretically optimal value. Rather, it is an empirically grounded target: across a wide range of depths, widths, datasets, and activation scales, we consistently observe that the $\sigma_z$ obtained from the $p_{\mathrm{real}} = 0.4$ calibration yields the best or near-best validation performance, with accuracy degrading as we move to significantly smaller or larger negative rates (see Figures~\ref{fsweep}, \ref{fsweep1}, \ref{fsweep50}, and~\ref{fsweep100}). This suggests that the proposed negative rate criterion provides a practically useful operating point for odd--sigmoid networks.

\clearpage
\section{Additional Experimental Results with Neural Networks}\label{app.nn}
\paragraph{Experimental Setting.}
We evaluate the proposed initialization using the Adam optimizer with a batch size of 128 and reserve 15\% of the training data for validation. All
experiments are implemented in PyTorch without skip connections and without
learning rate decay. We set the learning rate to
$10^{-4},\omega$ and we use the same learning rate for all Gaussian i.i.d.\ baseline initializations~(Xavier, He, and EOC) for a fair comparison.

\subsection{Network Size Independence}

\begin{figure}[h!]
\centering 
\begin{tabular}{ccc}
\begin{subfigure}[b]{0.30\textwidth}
    \centering
    \includegraphics[width=\textwidth]{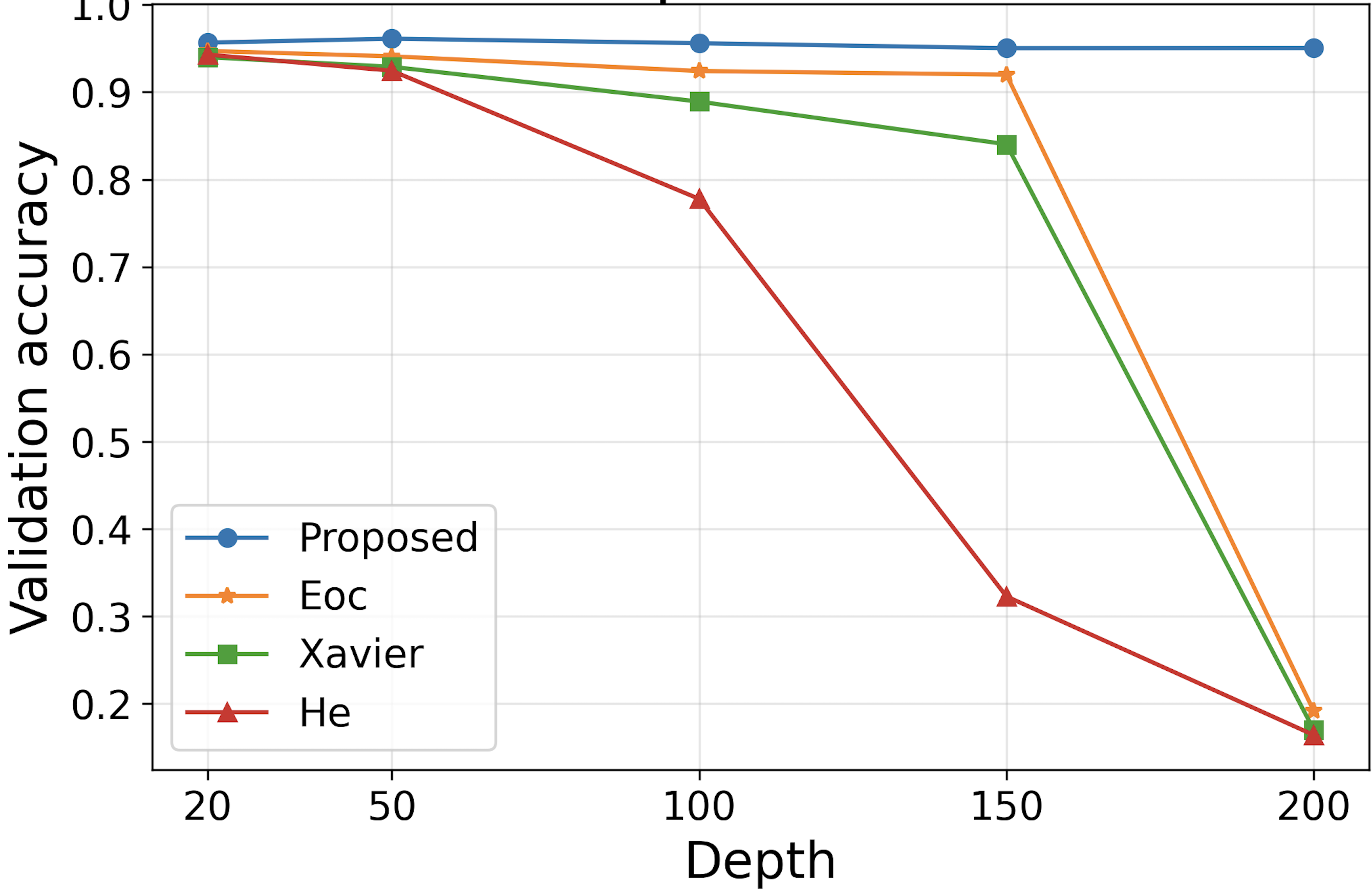}
    \caption{$\tanh$}
\end{subfigure} &
\begin{subfigure}[b]{0.30\textwidth}
    \centering
    \includegraphics[width=\textwidth]{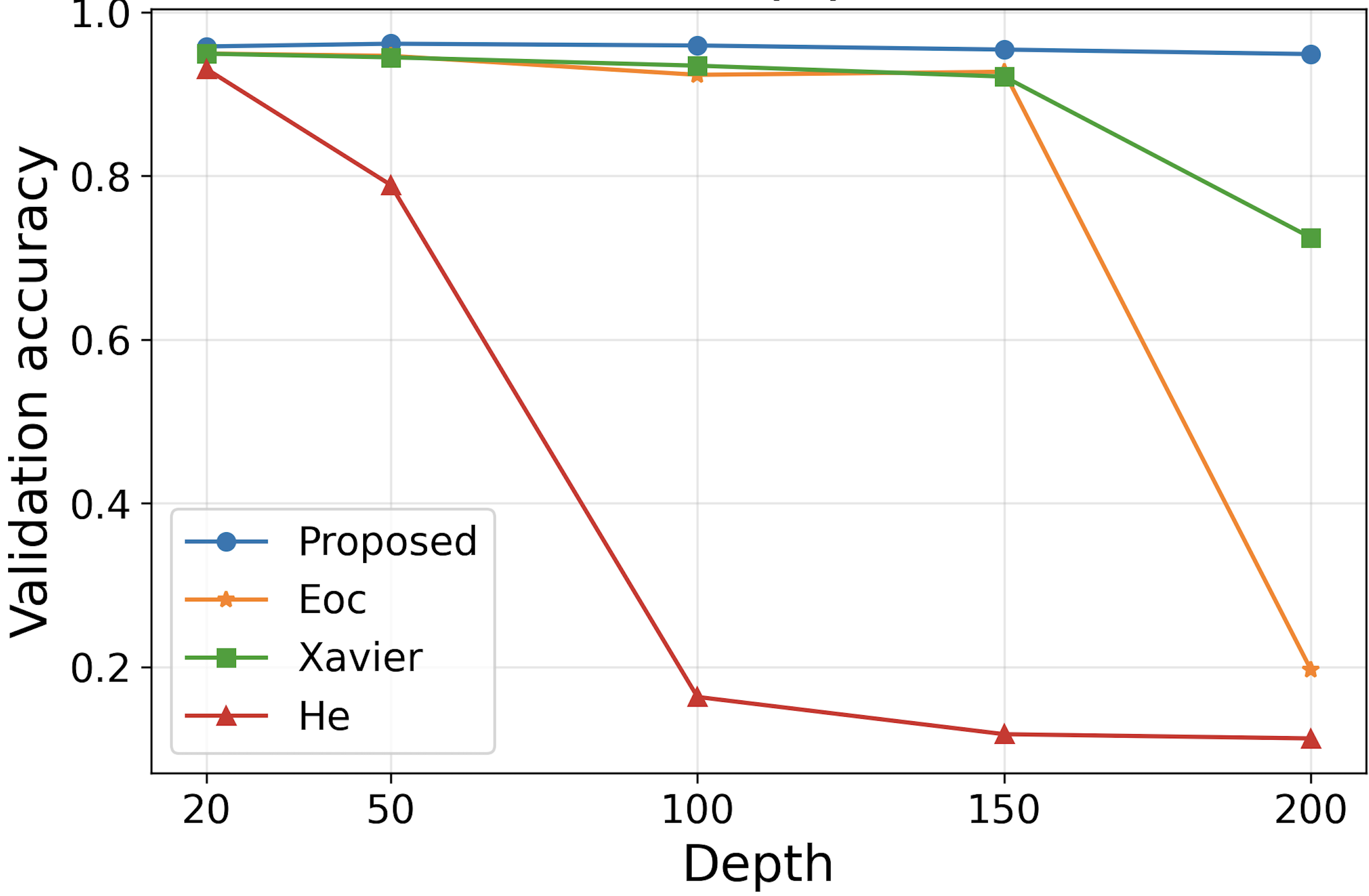}
    \caption{$\operatorname{erf}$}
\end{subfigure} &
\begin{subfigure}[b]{0.30\textwidth}
    \centering
    \includegraphics[width=\textwidth]{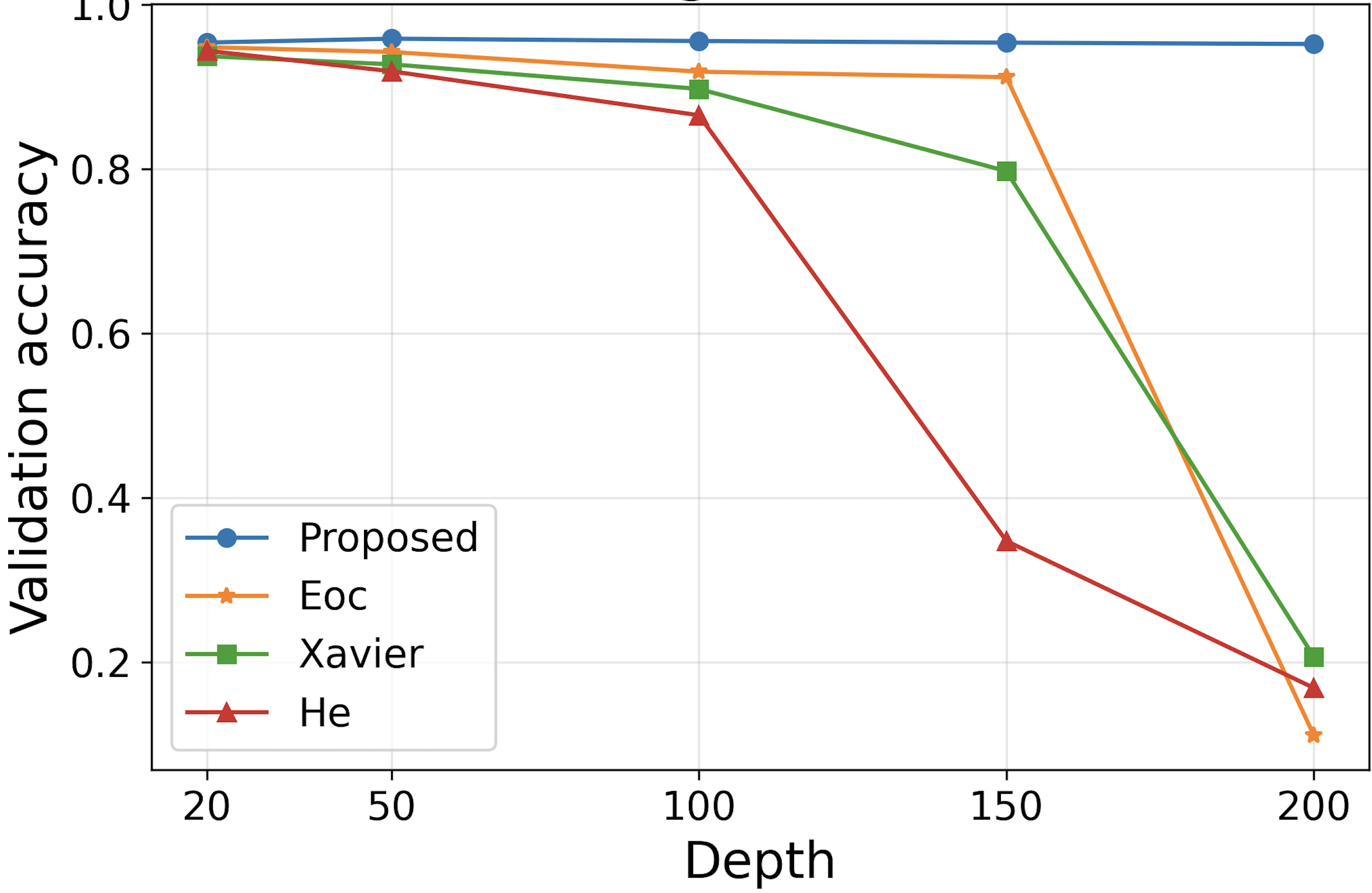}
    \caption{$\operatorname{gd}$}
\end{subfigure} \\
\begin{subfigure}[b]{0.30\textwidth}
    \centering
    \includegraphics[width=\textwidth]{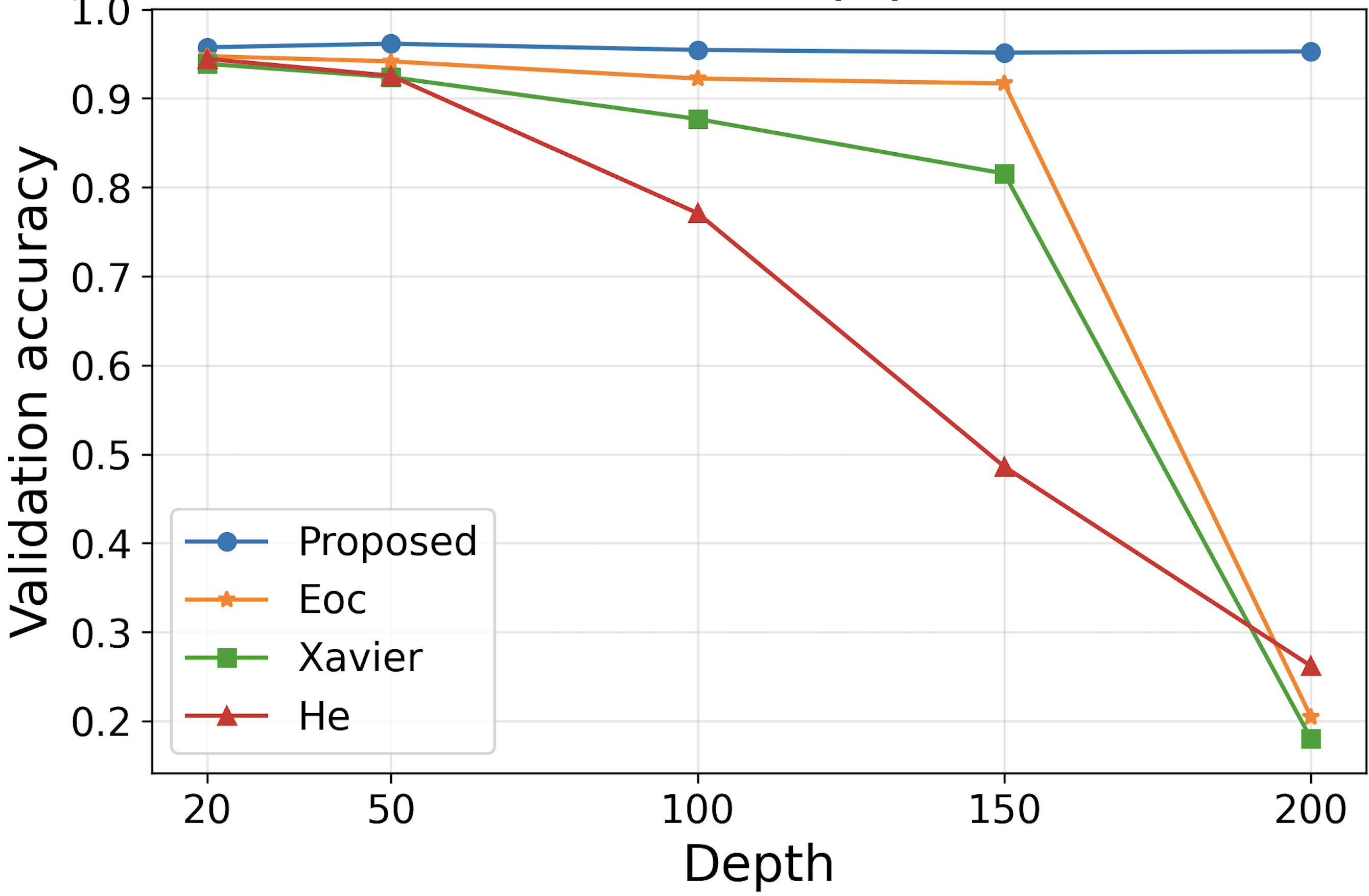}
    \caption{$\operatorname{arctan}$}
\end{subfigure} &
\begin{subfigure}[b]{0.30\textwidth}
    \centering
    \includegraphics[width=\textwidth]{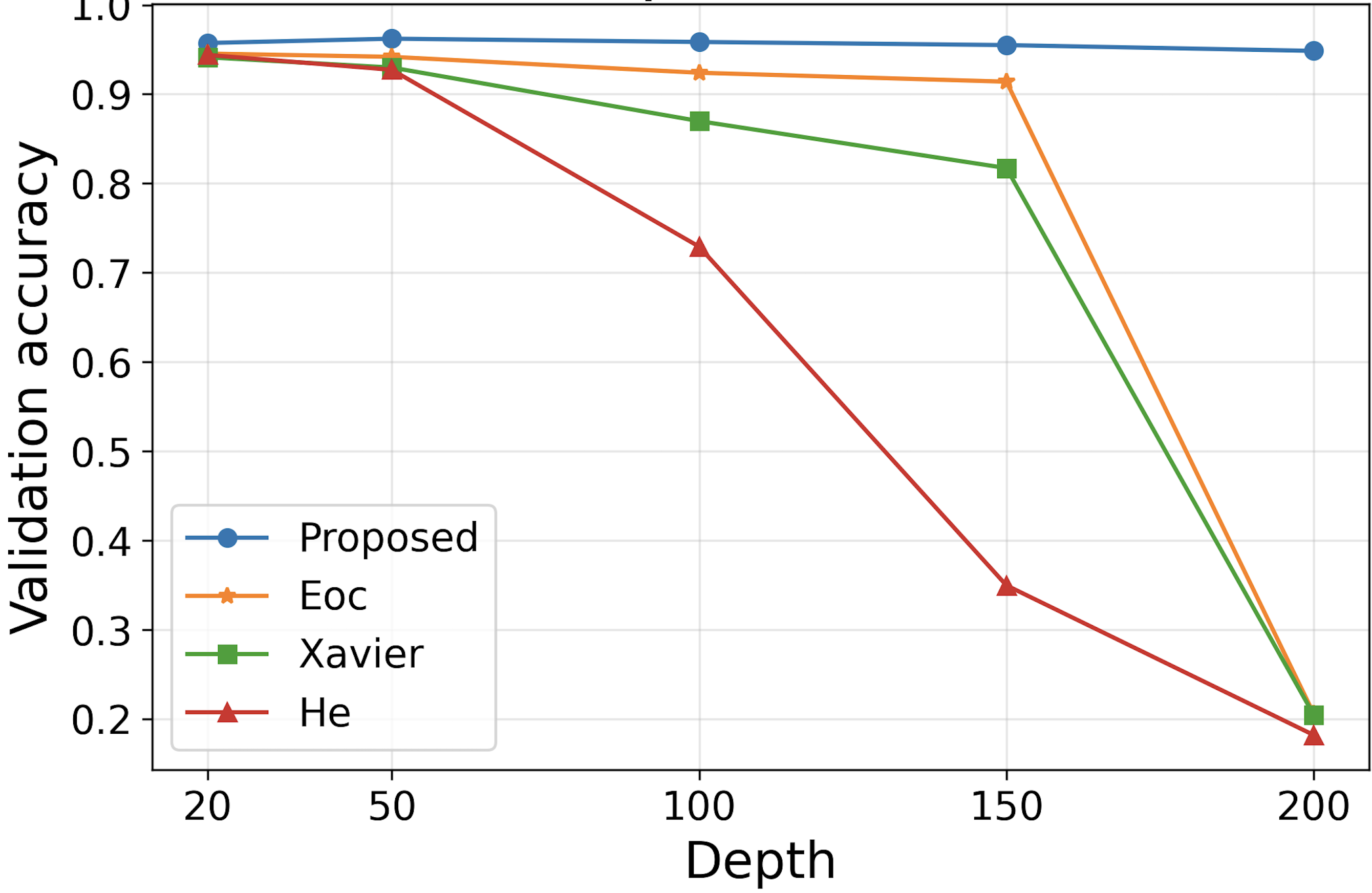}
    \caption{$\operatorname{softsign_2}$}
\end{subfigure} &
\begin{subfigure}[b]{0.30\textwidth}
    \centering
    \includegraphics[width=\textwidth]{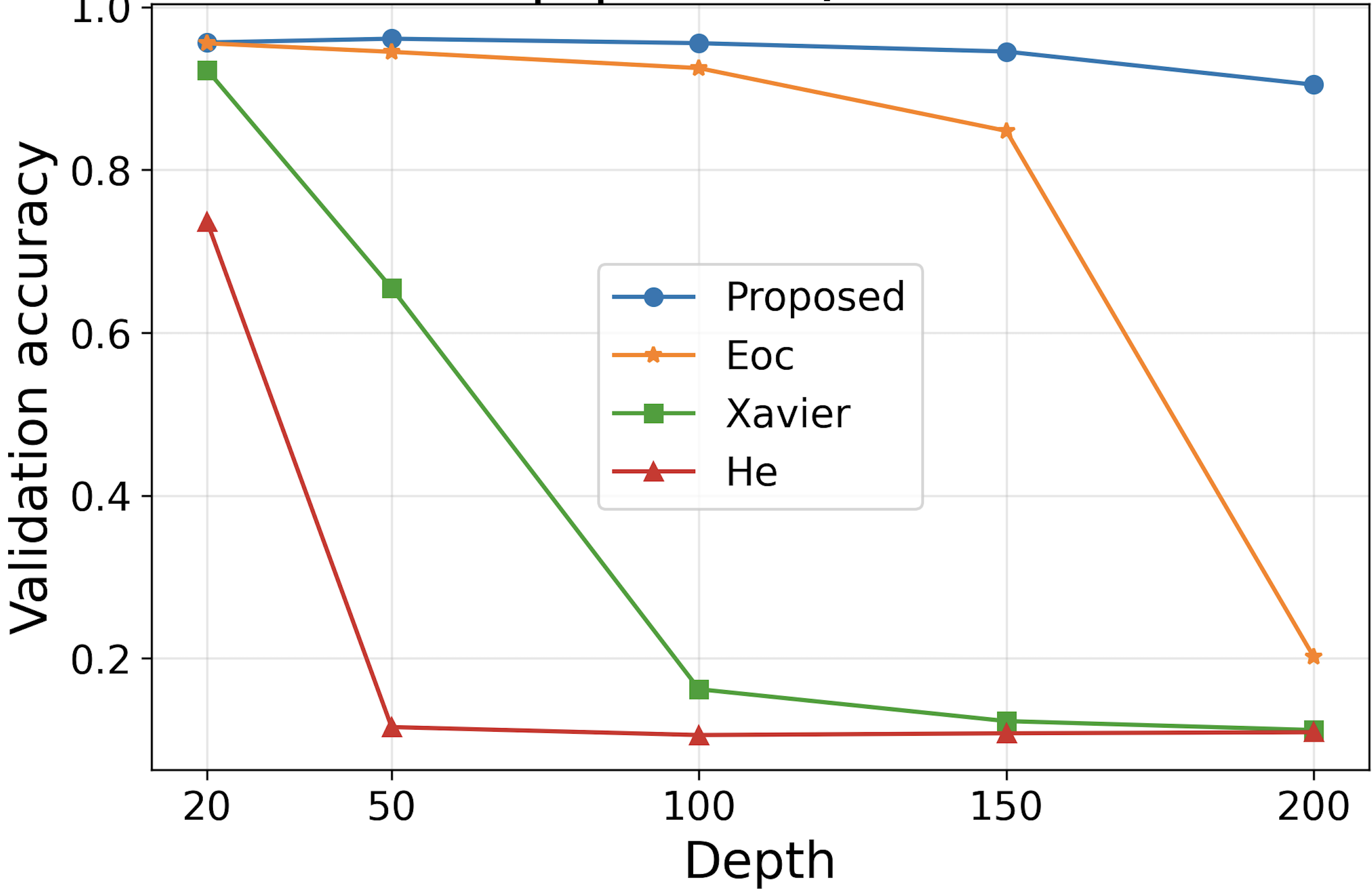}
    \caption{$\operatorname{softsign_1+softsign_2}$}
\end{subfigure} 
\end{tabular} 
\caption{MNIST validation accuracy versus depth for FFNNs~(width 64) with odd-sigmoid activations and four initializations~(Proposed, EOC, Xavier, He). Each panel fixes one activation and shows the best validation accuracy over 10 epochs for depths \(L \in \{20,50,100,150,200\}\).}
\label{dp1}
\end{figure}

\begin{figure}[h!]
\centering 
\begin{tabular}{ccc}
\begin{subfigure}[b]{0.30\textwidth}
    \centering
    \includegraphics[width=\textwidth]{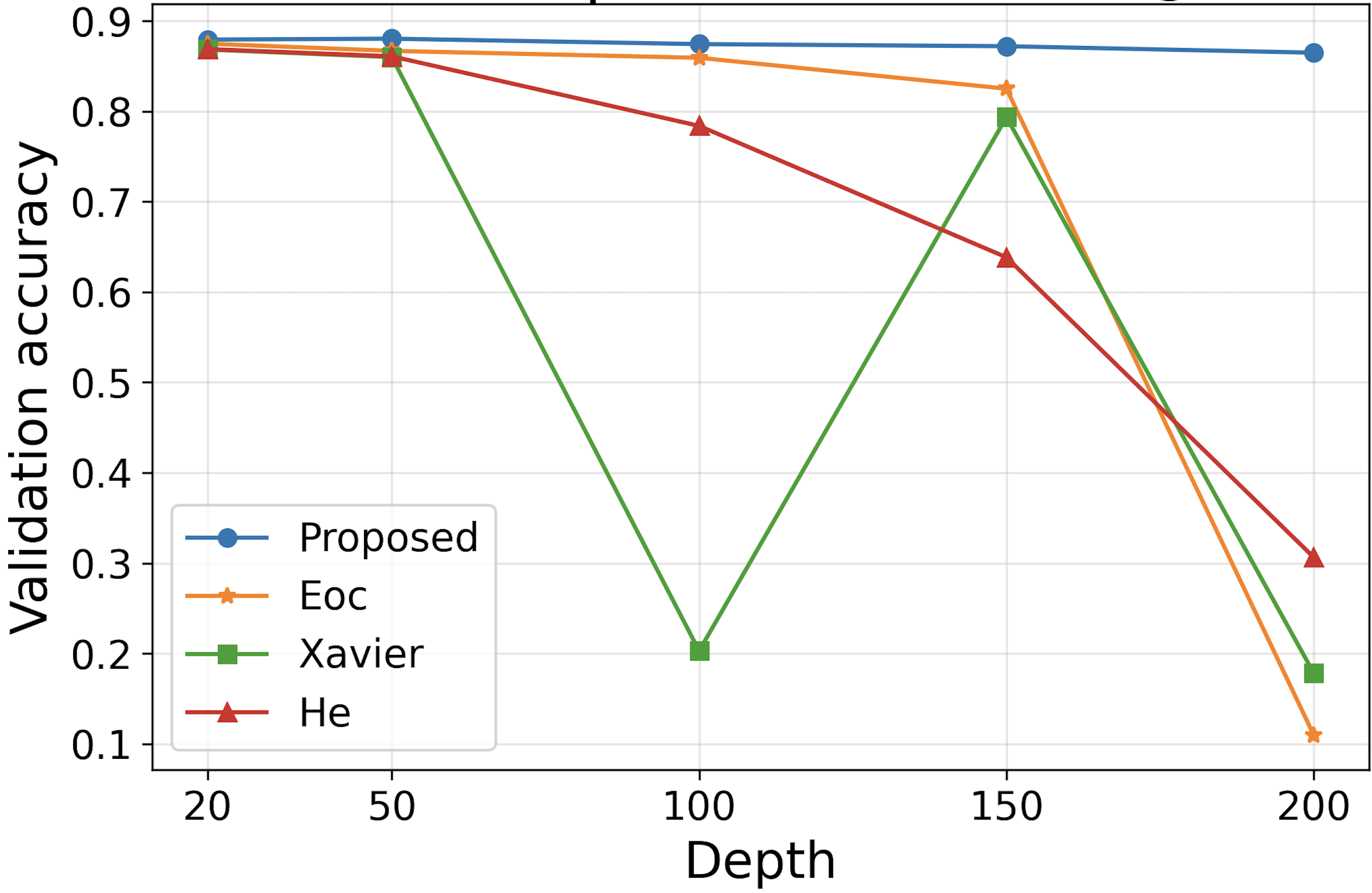}
    \caption{$\tanh$}
\end{subfigure} &
\begin{subfigure}[b]{0.30\textwidth}
    \centering
    \includegraphics[width=\textwidth]{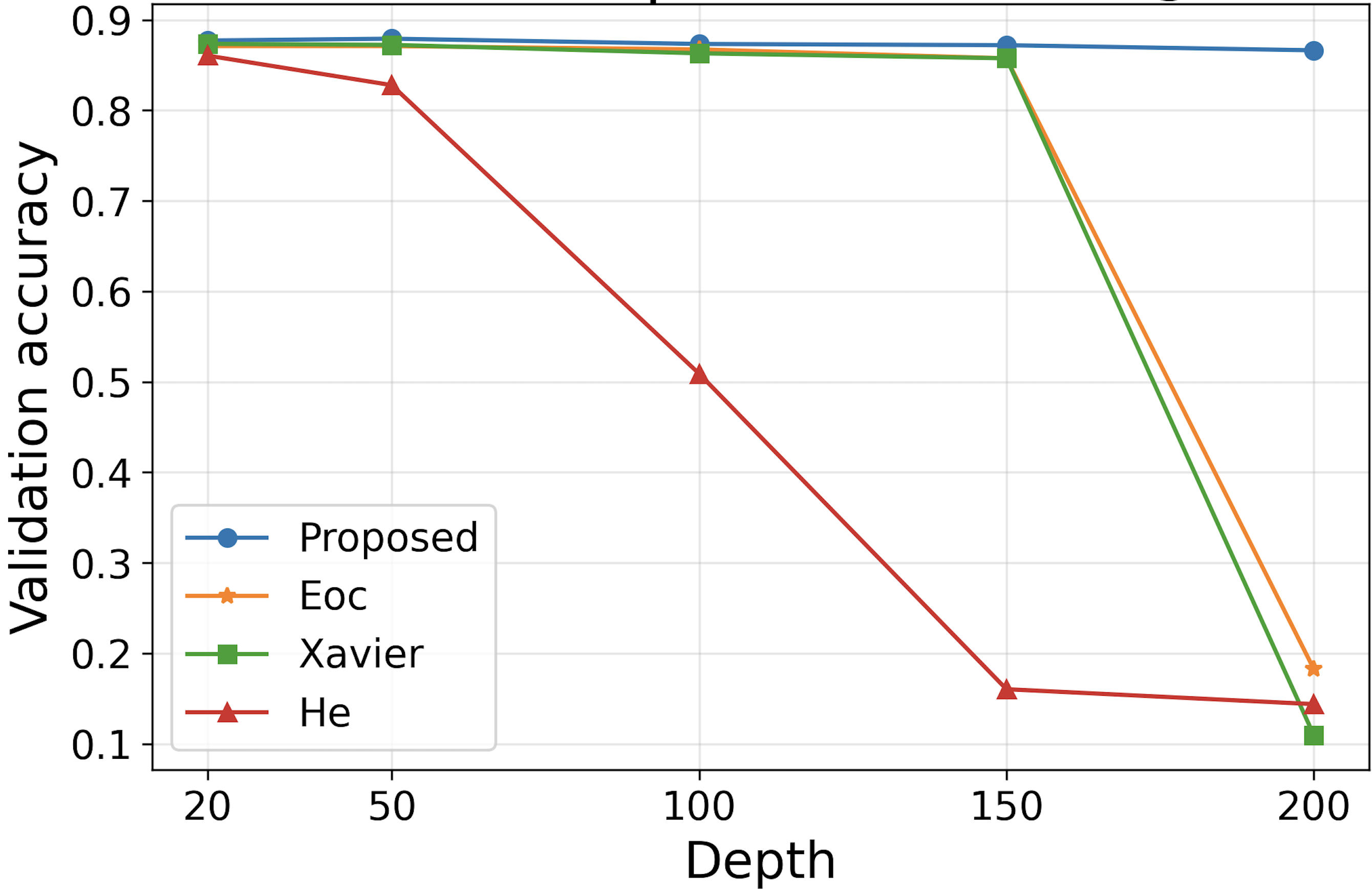}
    \caption{$\operatorname{erf}$}
\end{subfigure} &
\begin{subfigure}[b]{0.30\textwidth}
    \centering
    \includegraphics[width=\textwidth]{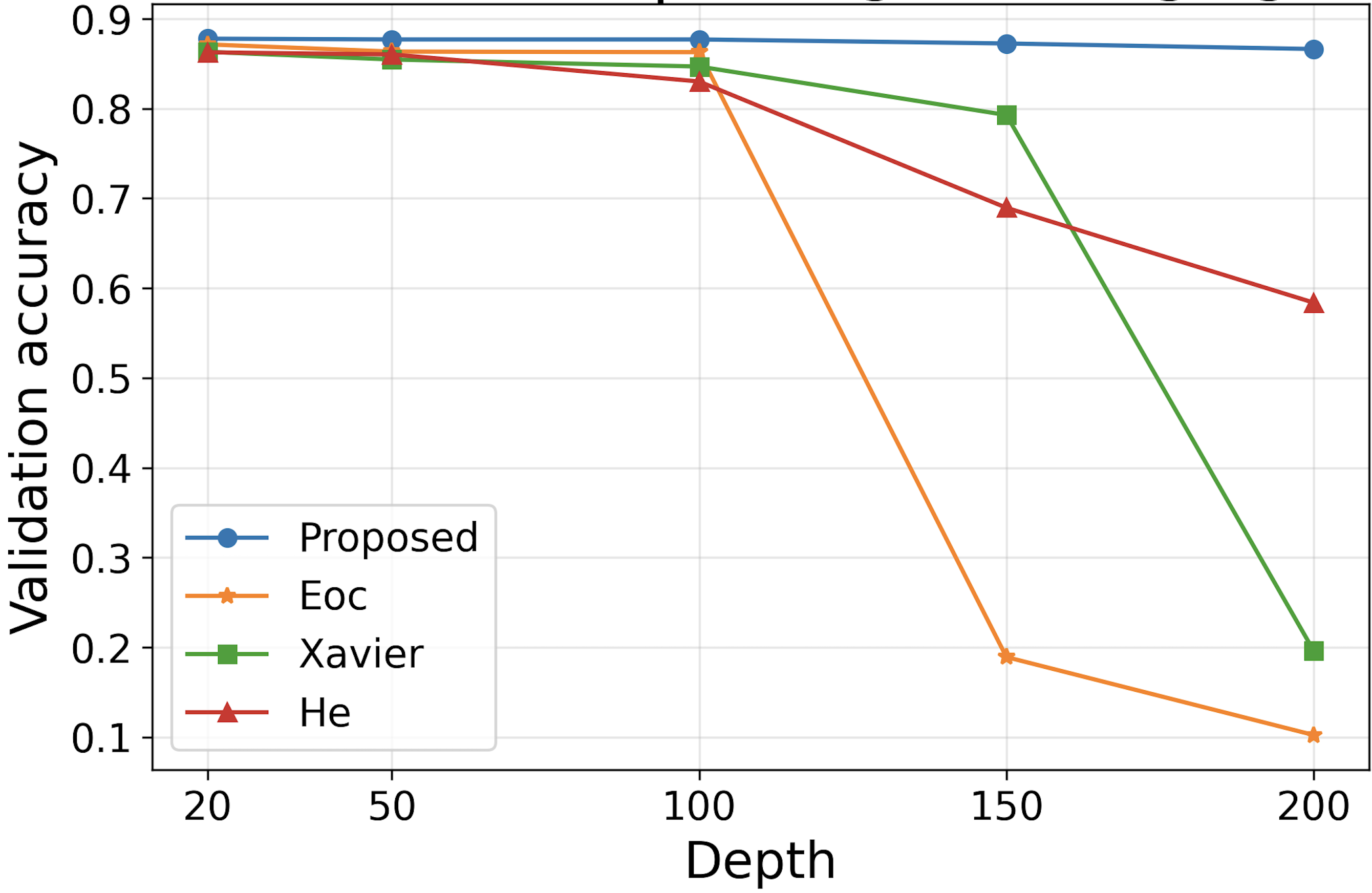}
    \caption{$\operatorname{gd}$}
\end{subfigure} \\
\begin{subfigure}[b]{0.30\textwidth}
    \centering
    \includegraphics[width=\textwidth]{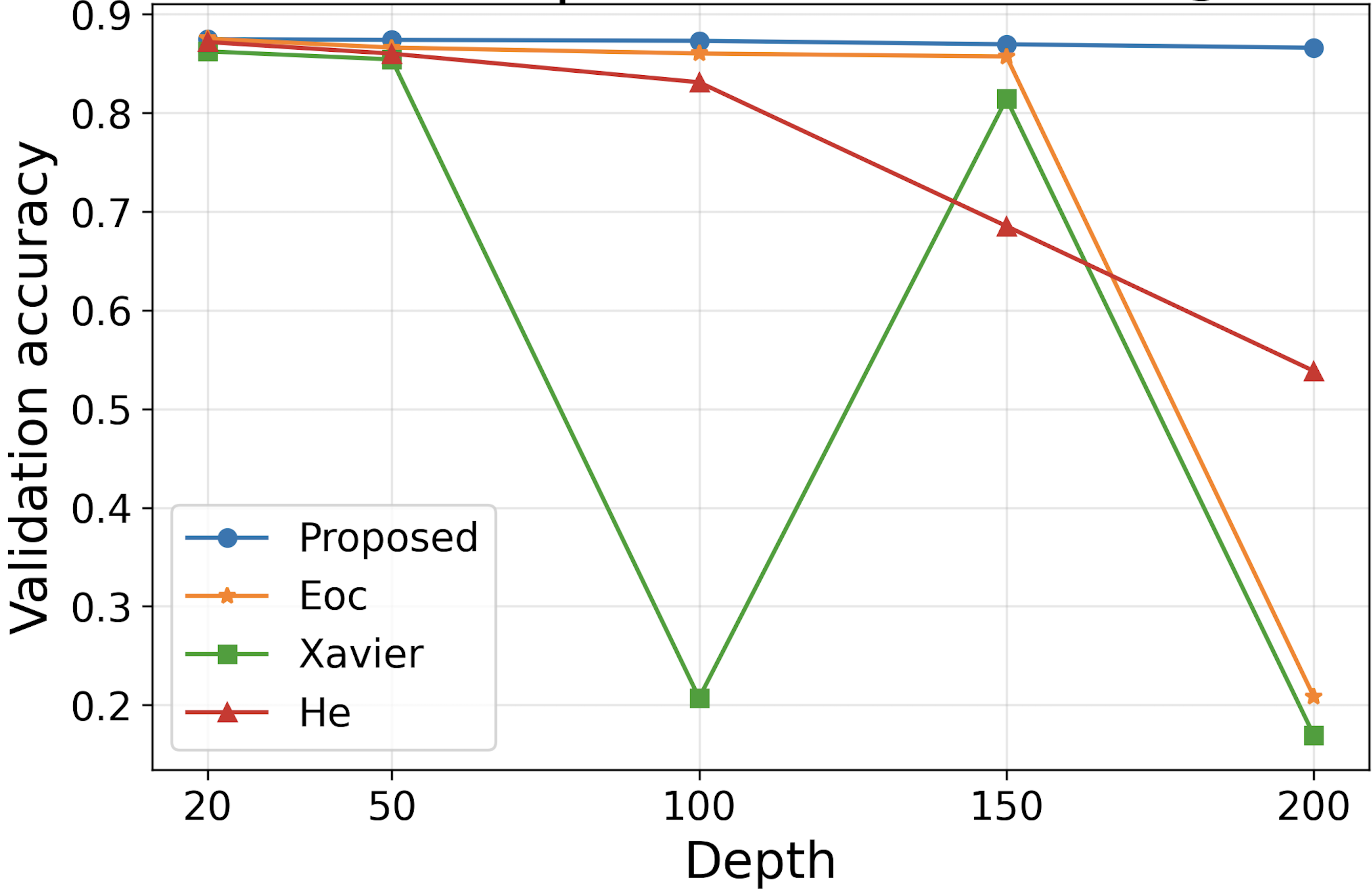}
    \caption{$\operatorname{arctan}$}
\end{subfigure} &
\begin{subfigure}[b]{0.30\textwidth}
    \centering
    \includegraphics[width=\textwidth]{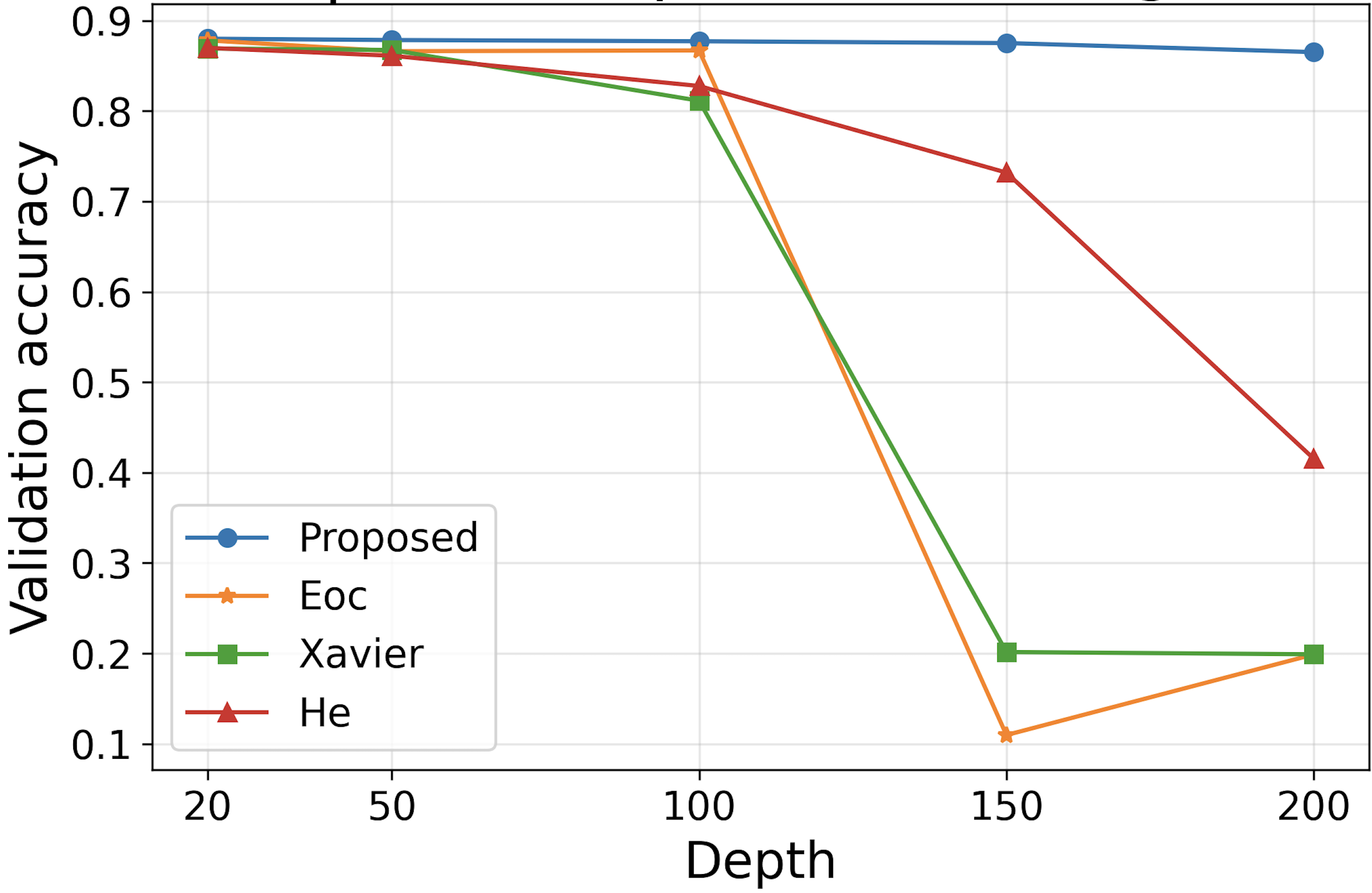}
    \caption{$\operatorname{softsign_2}$}
\end{subfigure} &
\begin{subfigure}[b]{0.30\textwidth}
    \centering
    \includegraphics[width=\textwidth]{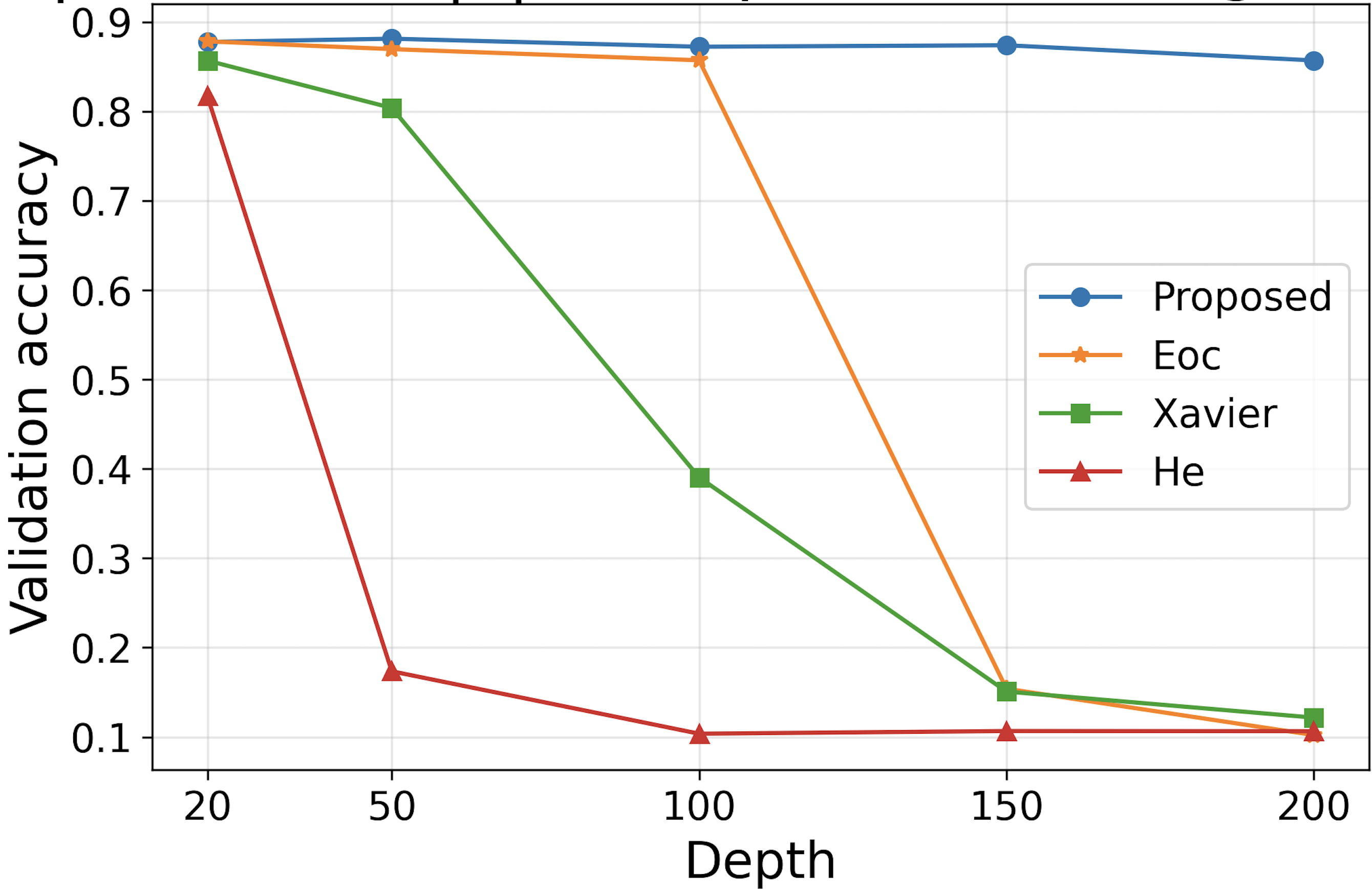}
    \caption{$\operatorname{softsign_1+softsign_2}$}
\end{subfigure} 
\end{tabular} 
\caption{Fashion MNIST validation accuracy versus depth for FFNNs~(width 64) with odd-sigmoid activations and four initializations~(Proposed, EOC, Xavier, He). Each panel fixes one activation and shows the best validation accuracy over 10 epochs for depths \(L \in \{20,50,100,150,200\}\).}
\label{dp2}
\end{figure}
\clearpage
\begin{figure}[h!]
\centering 
\begin{tabular}{ccc}
\begin{subfigure}[b]{0.30\textwidth}
    \centering
    \includegraphics[width=\textwidth]{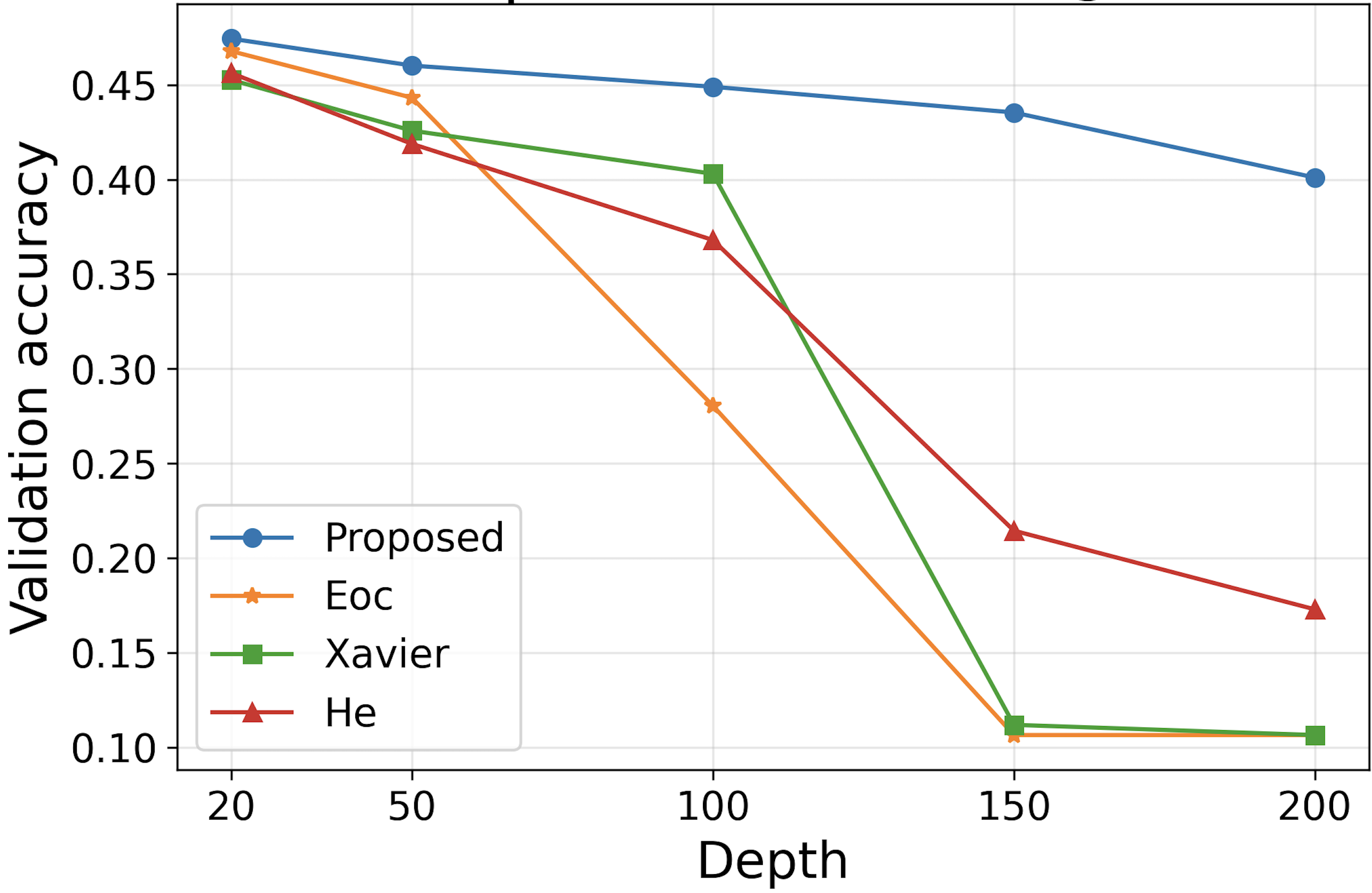}
    \caption{$\tanh$}
\end{subfigure} &
\begin{subfigure}[b]{0.30\textwidth}
    \centering
    \includegraphics[width=\textwidth]{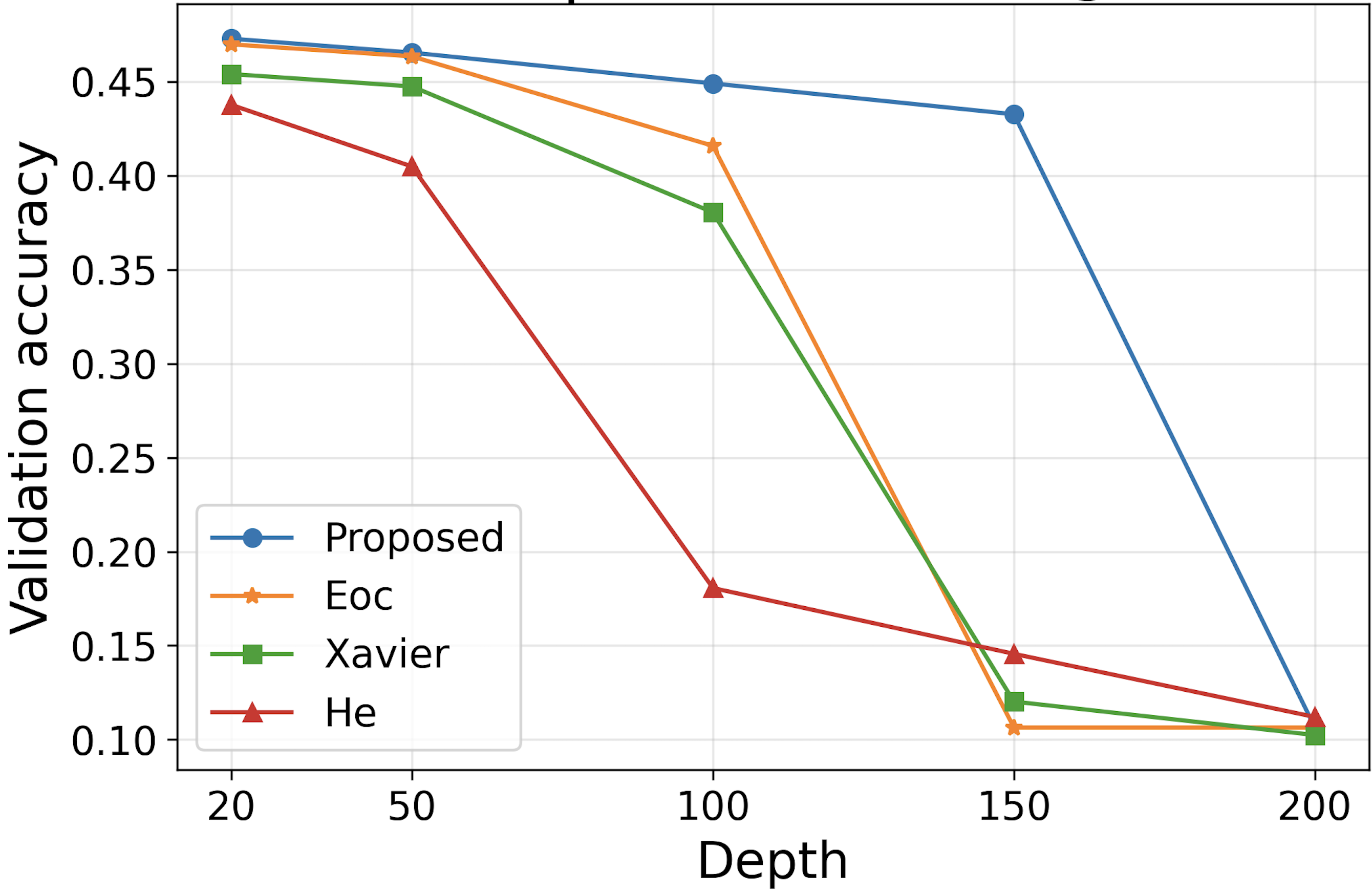}
    \caption{$\operatorname{erf}$}
\end{subfigure} &
\begin{subfigure}[b]{0.30\textwidth}
    \centering
    \includegraphics[width=\textwidth]{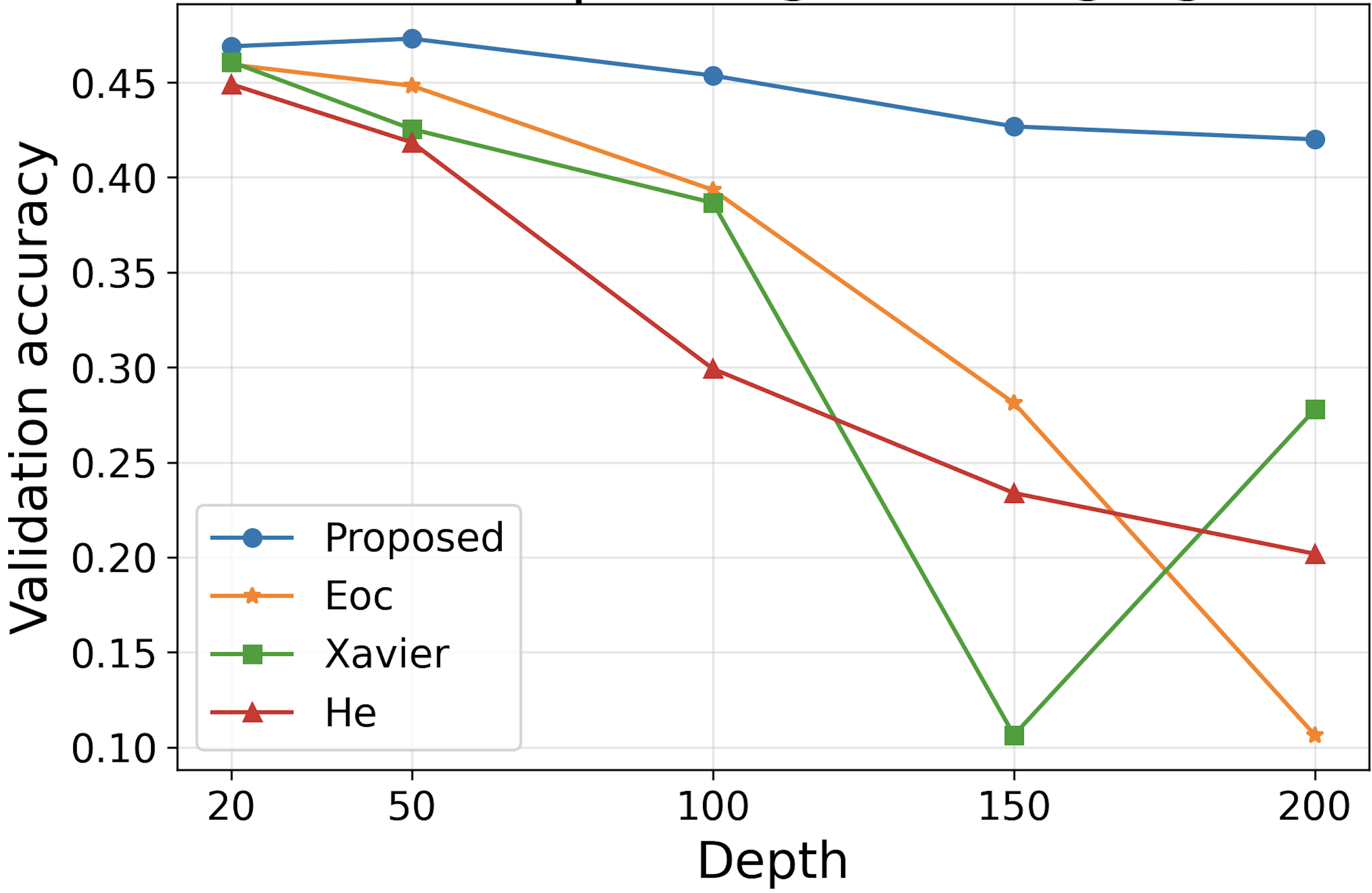}
    \caption{$\operatorname{gd}$}
\end{subfigure} \\
\begin{subfigure}[b]{0.30\textwidth}
    \centering
    \includegraphics[width=\textwidth]{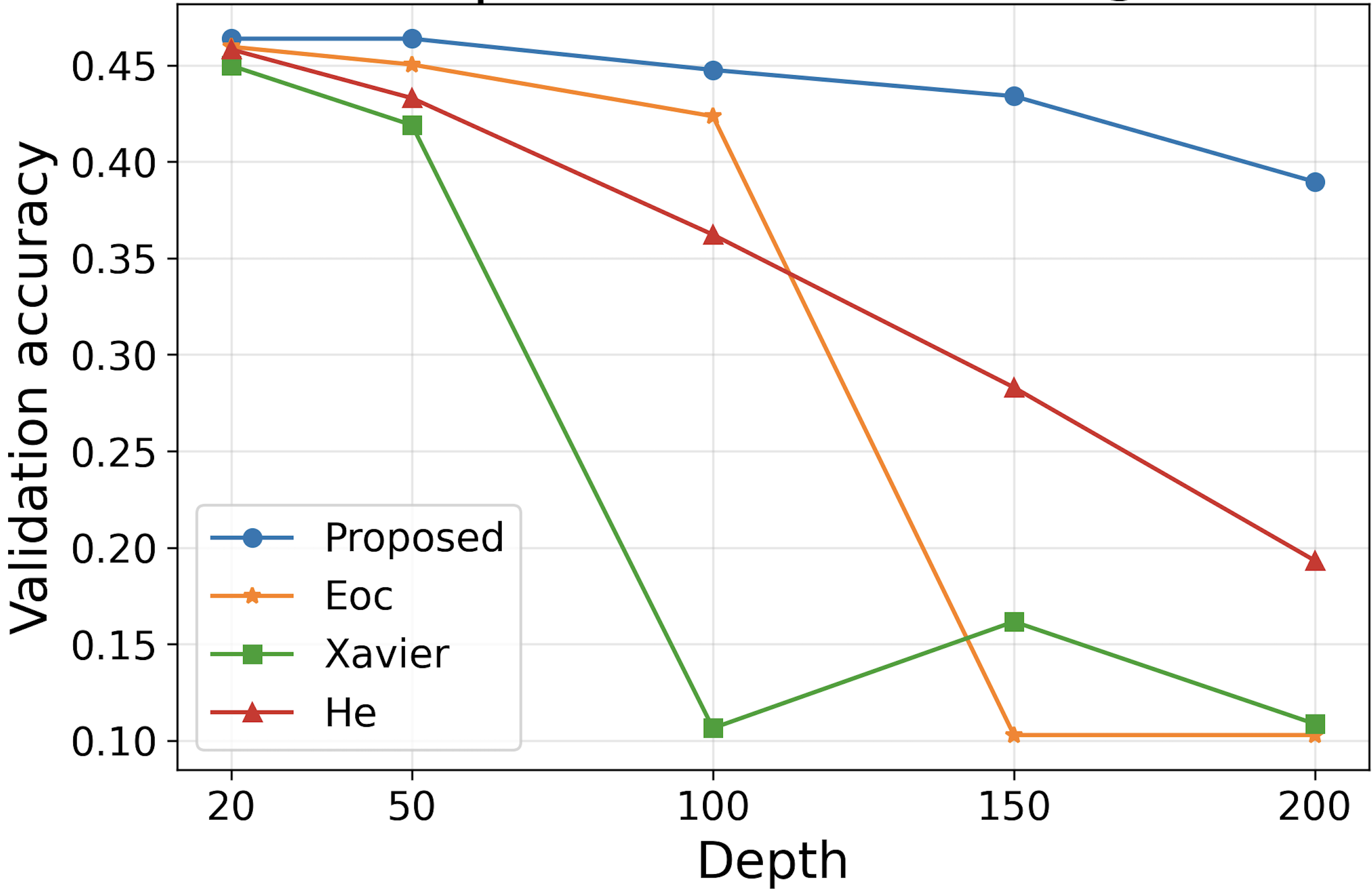}
    \caption{$\operatorname{arctan}$}
\end{subfigure} &
\begin{subfigure}[b]{0.30\textwidth}
    \centering
    \includegraphics[width=\textwidth]{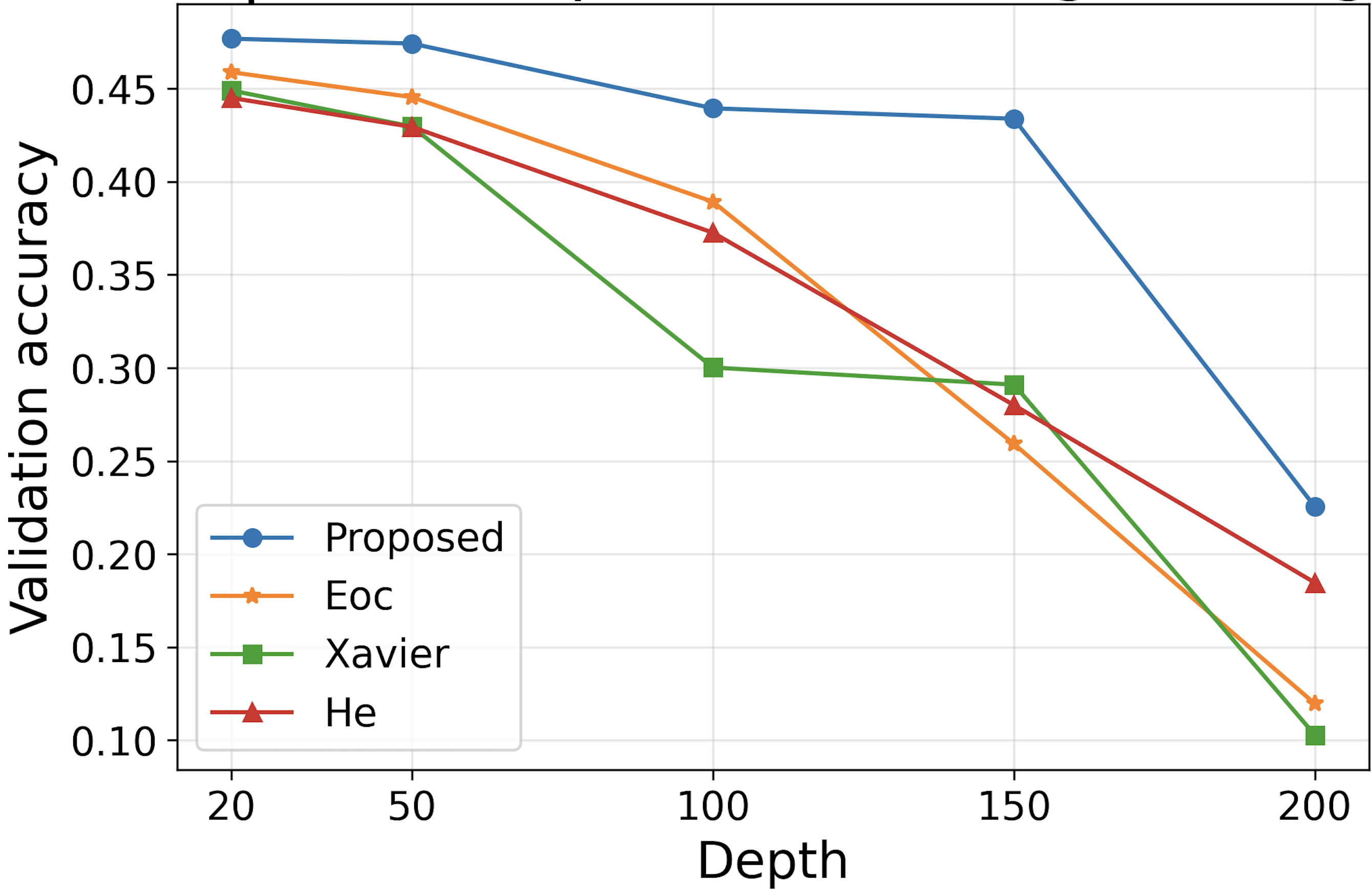}
    \caption{$\operatorname{softsign_2}$}
\end{subfigure} &
\begin{subfigure}[b]{0.30\textwidth}
    \centering
    \includegraphics[width=\textwidth]{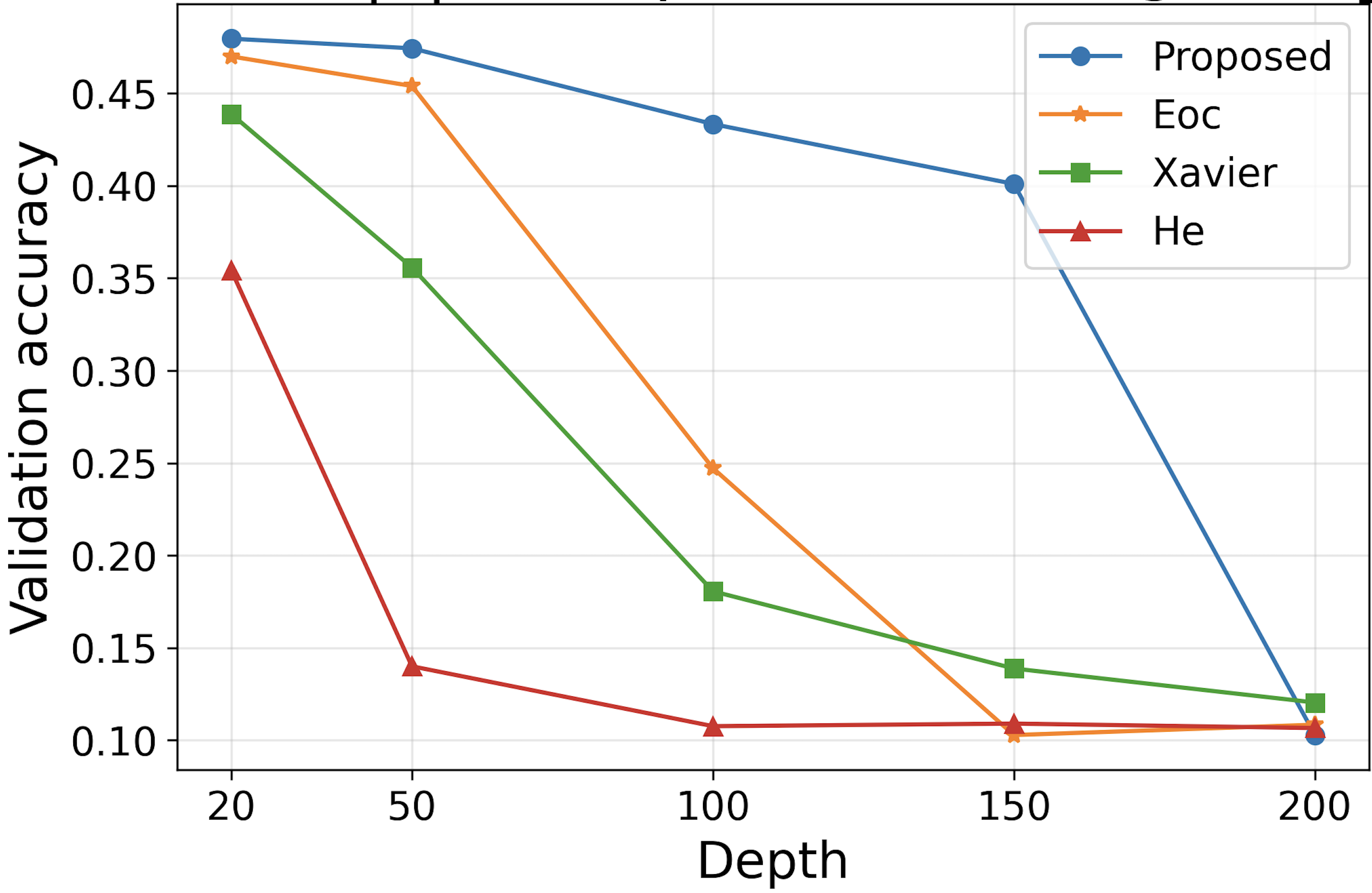}
    \caption{$\operatorname{softsign_1+softsign_2}$}
\end{subfigure} 
\end{tabular} 
\caption{CIFAR 10 validation accuracy versus depth for FFNNs~(width 64) with odd-sigmoid activations and four initializations~(Proposed, EOC, Xavier, He). Each panel fixes one activation and shows the best validation accuracy over 10 epochs for depths \(L \in \{20,50,100,150,200\}\).}
\label{dp3}
\end{figure}

\begin{figure}[h!]
\centering 
\begin{tabular}{ccc}
\begin{subfigure}[b]{0.30\textwidth}
    \centering
    \includegraphics[width=\textwidth]{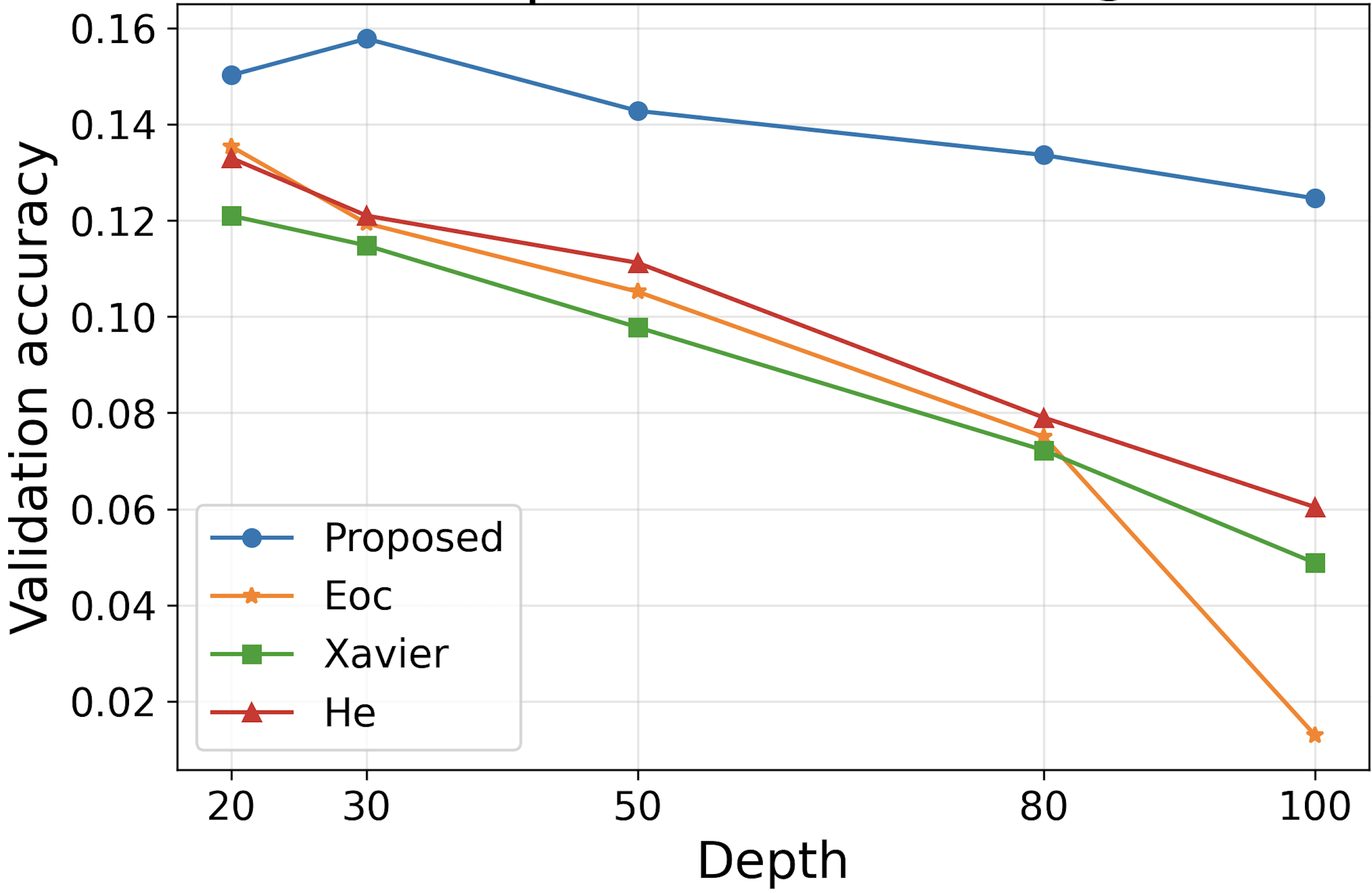}
    \caption{$\tanh$}
\end{subfigure} &
\begin{subfigure}[b]{0.30\textwidth}
    \centering
    \includegraphics[width=\textwidth]{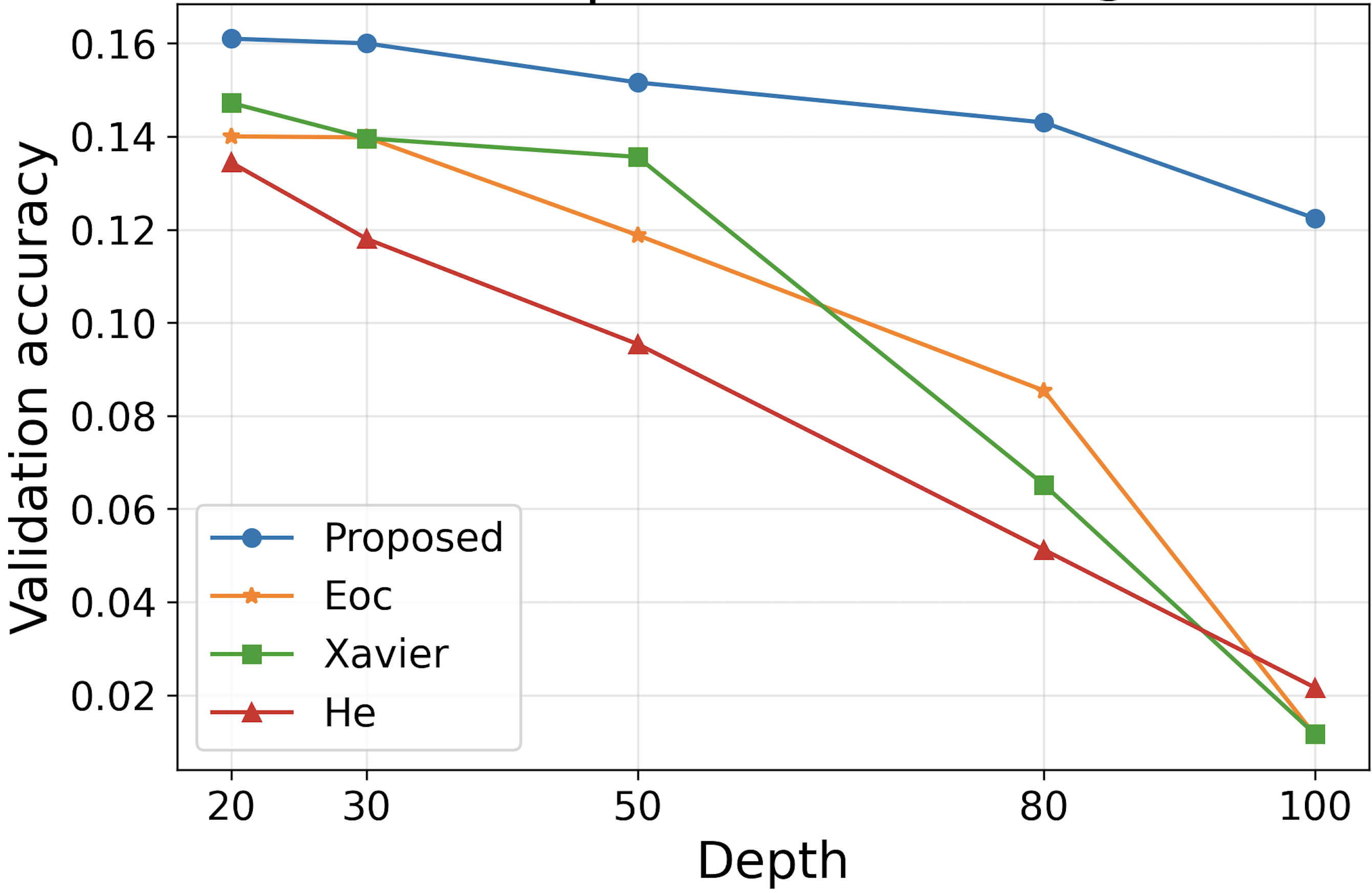}
    \caption{$\operatorname{erf}$}
\end{subfigure} &
\begin{subfigure}[b]{0.30\textwidth}
    \centering
    \includegraphics[width=\textwidth]{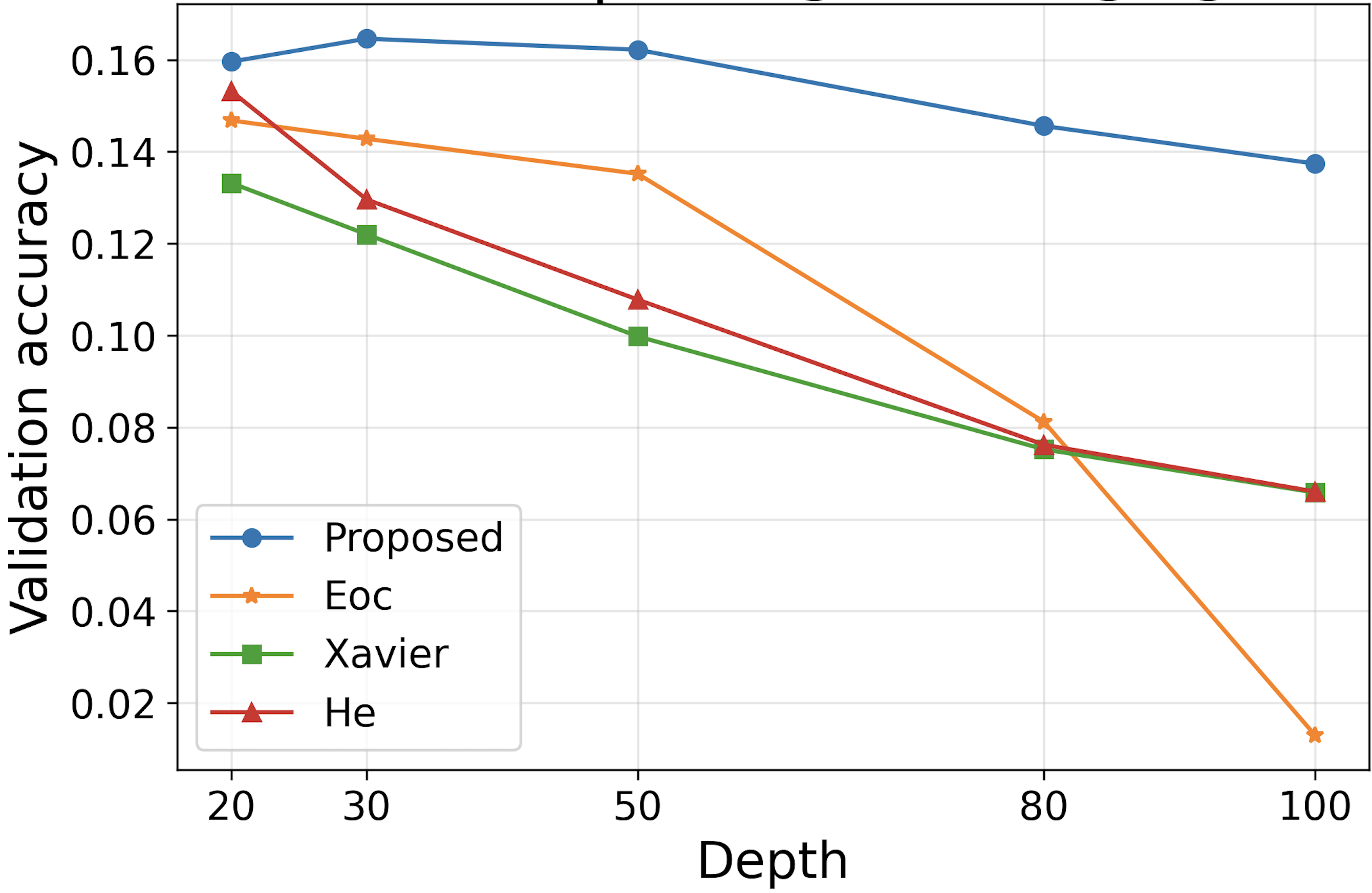}
    \caption{$\operatorname{gd}$}
\end{subfigure} \\
\begin{subfigure}[b]{0.30\textwidth}
    \centering
    \includegraphics[width=\textwidth]{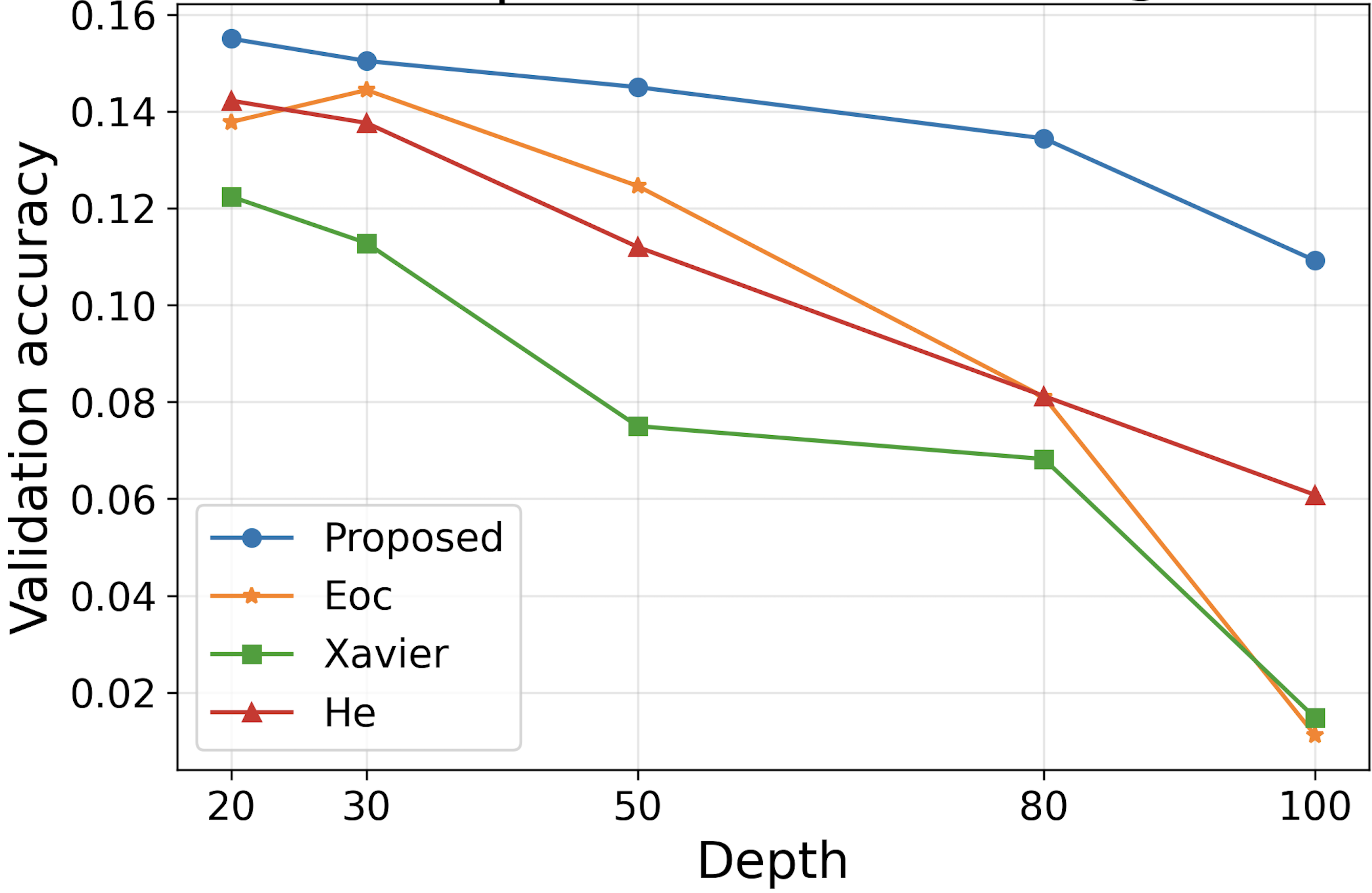}
    \caption{$\operatorname{arctan}$}
\end{subfigure} &
\begin{subfigure}[b]{0.30\textwidth}
    \centering
    \includegraphics[width=\textwidth]{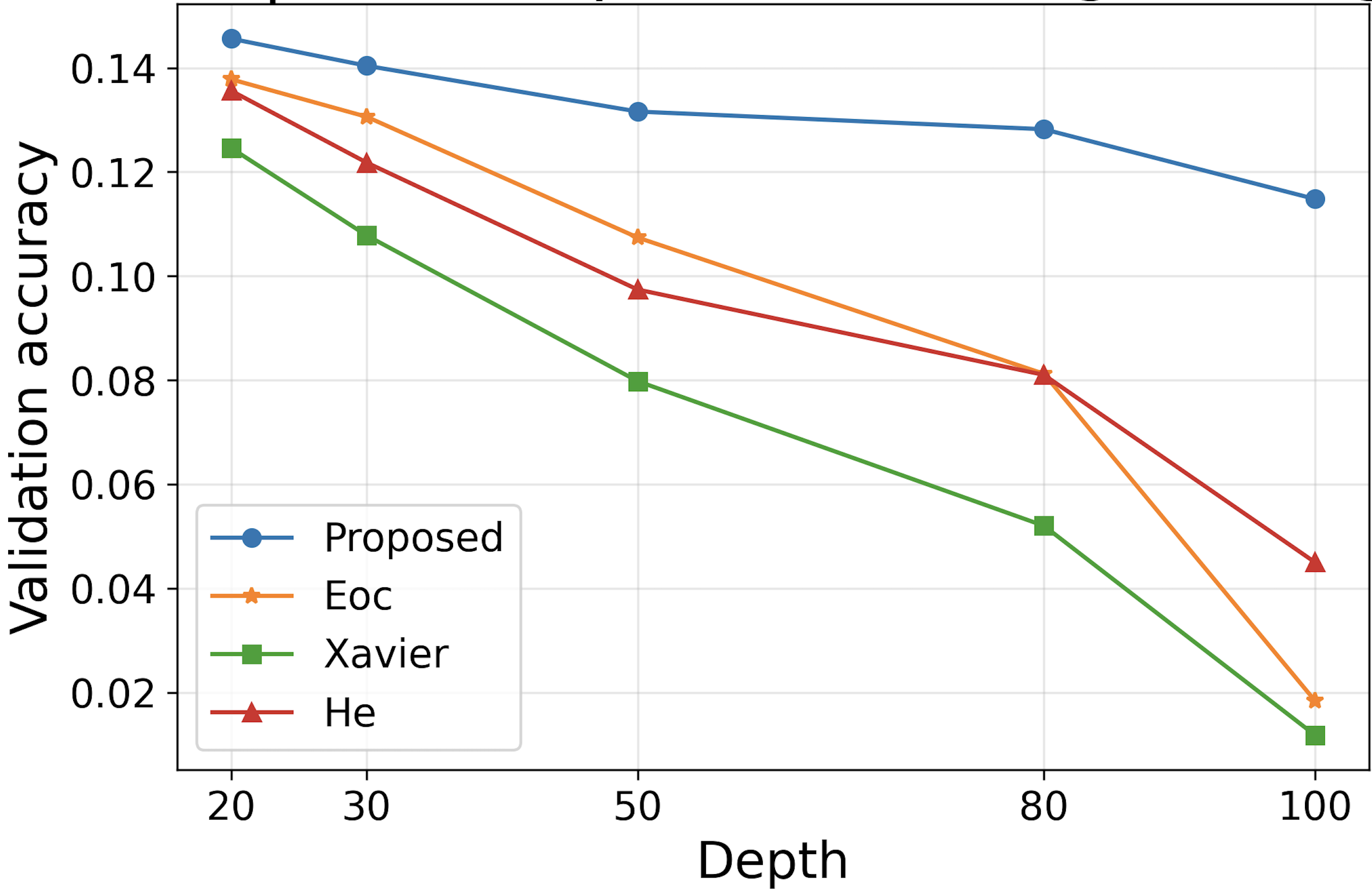}
    \caption{$\operatorname{softsign_2}$}
\end{subfigure} &
\begin{subfigure}[b]{0.30\textwidth}
    \centering
    \includegraphics[width=\textwidth]{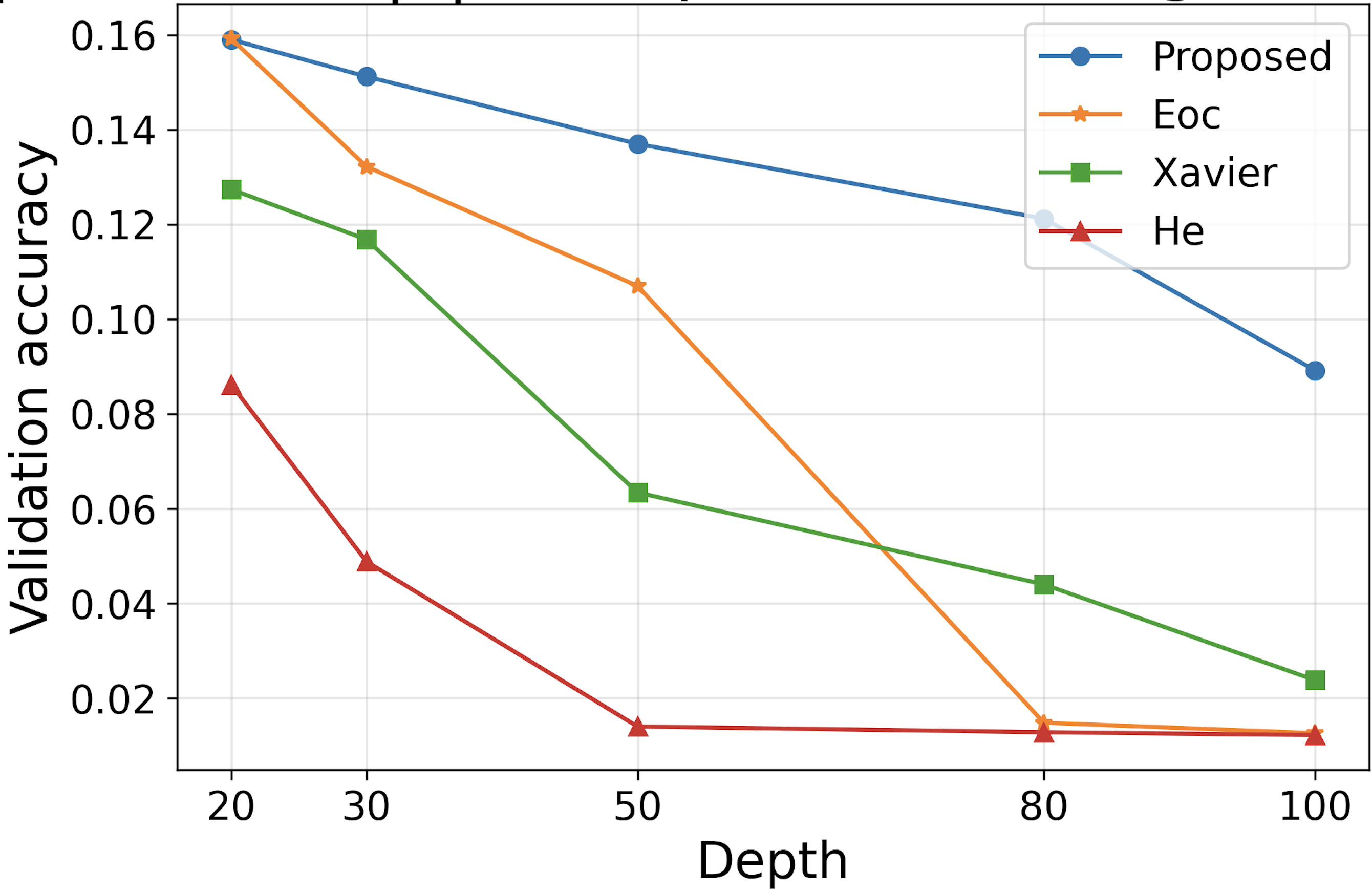}
    \caption{$\operatorname{softsign_1+softsign_2}$}
\end{subfigure} 
\end{tabular} 
\caption{CIFAR 100 validation accuracy versus depth for FFNNs~(width 64) with odd-sigmoid activations and four initializations~(Proposed, EOC, Xavier, He). Each panel fixes one activation and shows the best validation accuracy over 10 epochs for depths \(L \in \{20,50,100,150,200\}\).}
\label{dp4}
\end{figure}

\medskip

Figures~\ref{dp1}, \ref{dp2}, \ref{dp3}, and \ref{dp4} report the validation accuracy as a function of depth for networks with six different activation functions. On the more complex CIFAR-10 and CIFAR-100 datasets, even the proposed initialization exhibits some change in performance once the depth becomes sufficiently large, whereas on MNIST and Fashion-MNIST the validation accuracy is essentially preserved across all depths considered. These results indicate that, across diverse datasets and for odd–sigmoid activations with effective gain parameter $\omega$ close to one, the proposed initialization behaves substantially more depth invariant than standard Gaussian initializations.

\clearpage

\subsection{Batch Normalization Free Training}
Batch normalization~(BN) has made training deep networks substantially easier, but it also incurs a significant computational overhead. This raises the question of whether our initialization can match or surpass BN-based methods while avoiding this overhead. Figure~\ref{tanh_batch} evaluates $\tanh$ networks and shows that, across all datasets, the proposed initialization without BN outperforms or matches competing schemes that rely on BN. Figures~\ref{batch_norm} and~\ref{800layers} extend this comparison to activations with $\omega = 1/f'(0)$ far from $1$, including an 800 layer network with width $32$. In this challenging deep and narrow regime, our method still trains reliably without BN, whereas Gaussian initializations~(with or without BN) either fail to converge or achieve worse accuracy. These results suggest that the proposed initialization can substitute for BN in many settings, enabling stable training of very deep and narrow networks without the cost of normalization layers.

\medskip

\begin{figure}[h!]
\centering 
\begin{tabular}{ccc}
\begin{subfigure}[b]{0.31\textwidth}
    \centering
    \includegraphics[width=\textwidth]{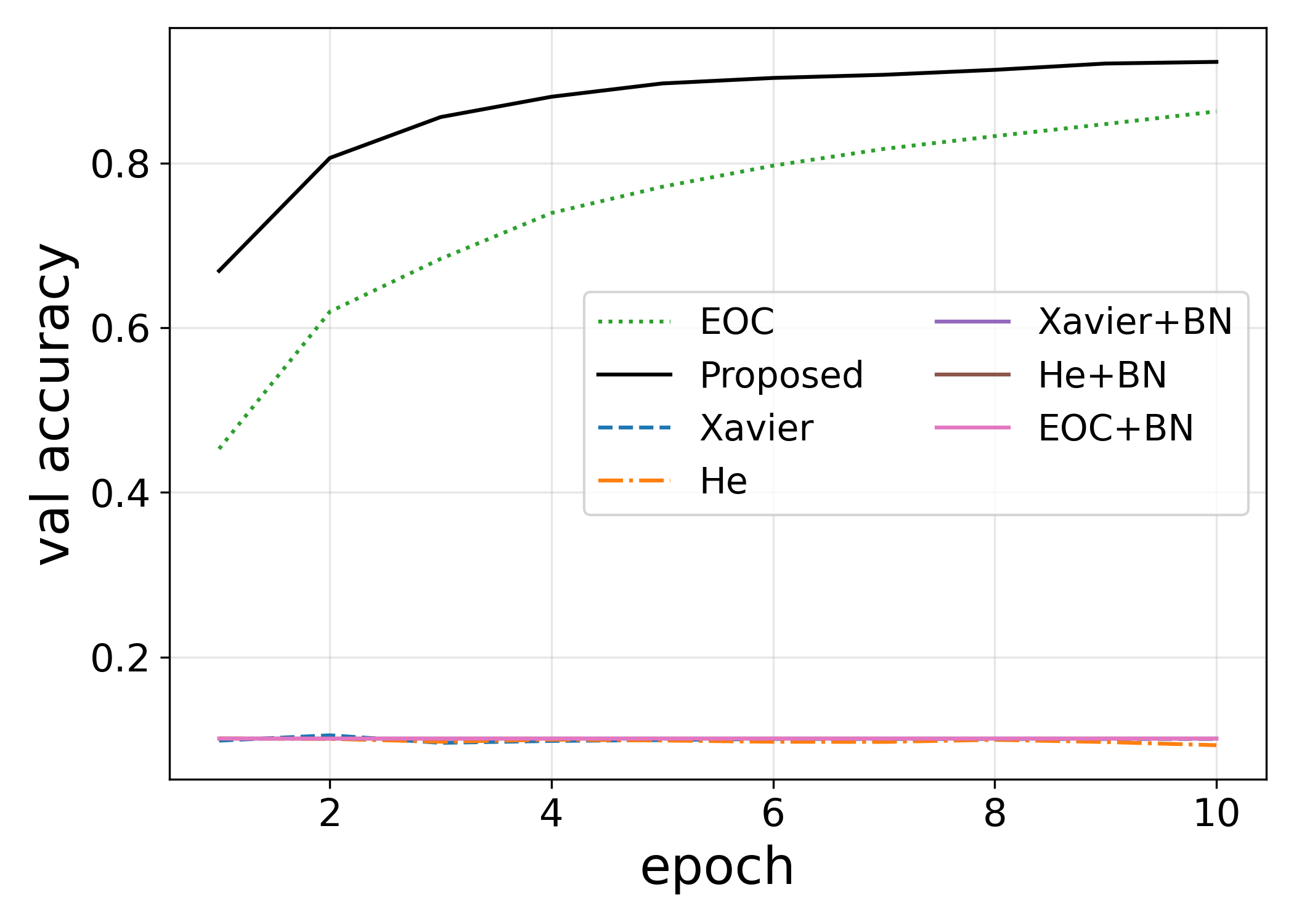}
    \caption{MNIST}
\end{subfigure} &
\begin{subfigure}[b]{0.31\textwidth}
    \centering
    \includegraphics[width=\textwidth]{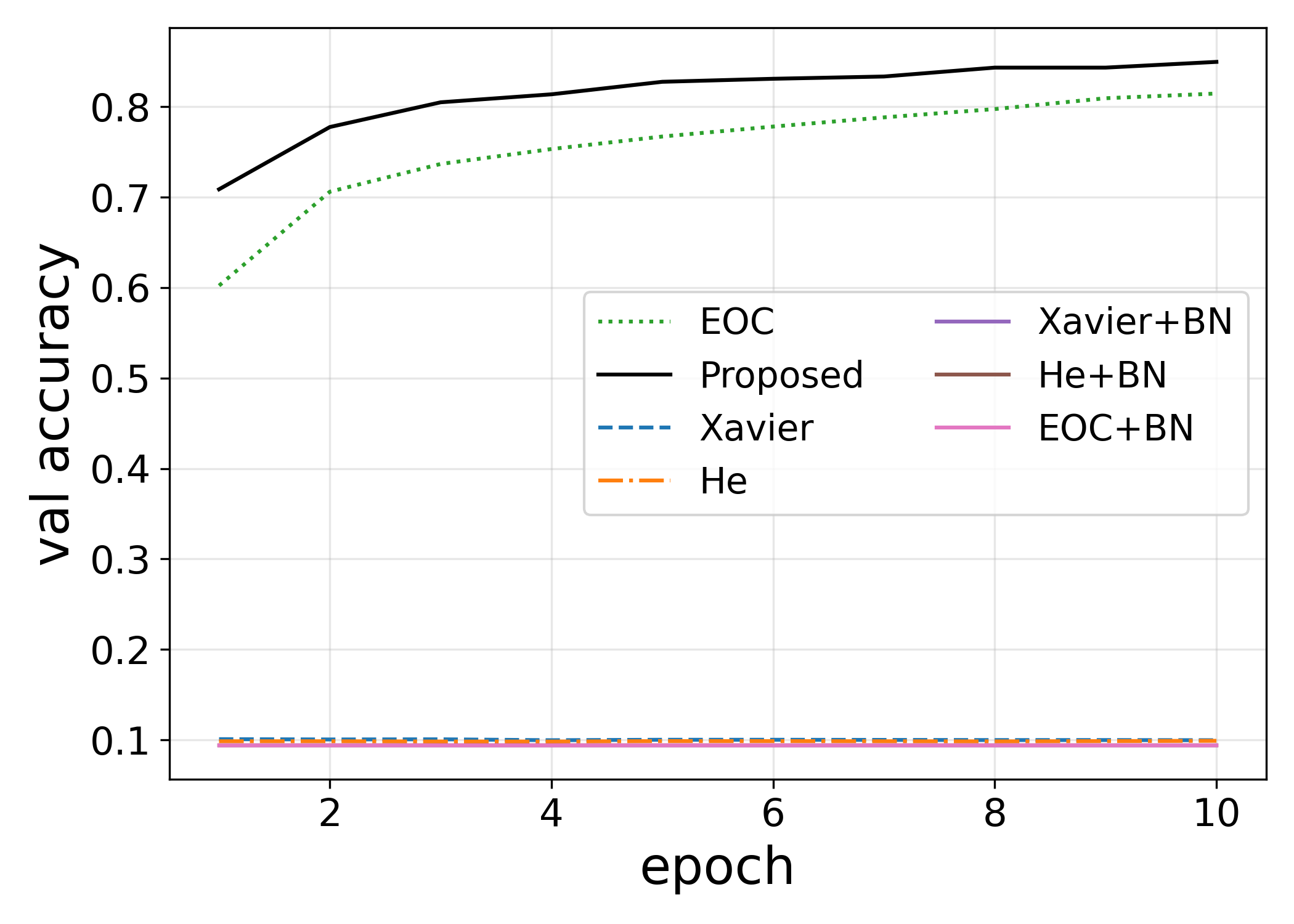}
    \caption{FMNIST}
\end{subfigure} &
\begin{subfigure}[b]{0.31\textwidth}
    \centering
    \includegraphics[width=\textwidth]{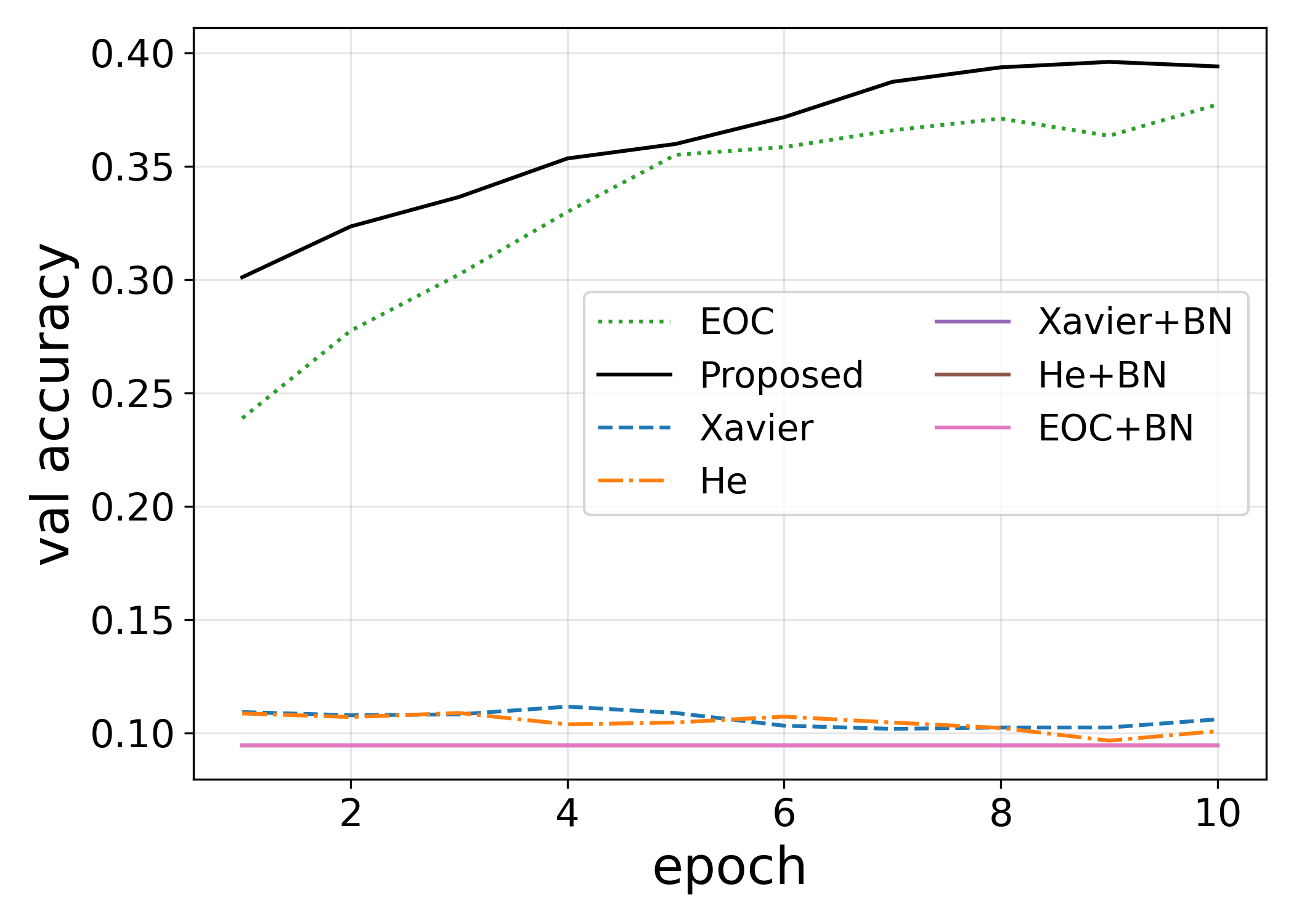}
    \caption{CIFAR10}
\end{subfigure} 
\end{tabular}
\caption{Validation accuracy for the
\(f(x) = \tanh(a x) + \mathrm{erf}(b x) + x/(1+|c x|) + \mathrm{gd}(d x)\)
in FFNN with 100 hidden layers of width 64.
We compare seven strategies: Proposed, Xavier, He, EOC, and their BN variants.
Each panel uses a different dataset and a different choice of coefficients \((a,b,c,d)\):
\textbf{(a)} (10,1000,10,1), \textbf{(b)} (100,1000,10,1000)), \textbf{(c)}(1000,100,0.1,0.01).
}

\label{batch_norm}
\end{figure}

\medskip

\begin{figure}[h!]
\centering 
\begin{tabular}{cccc}
\begin{subfigure}[b]{0.31\textwidth}
    \centering
    \includegraphics[width=\textwidth]{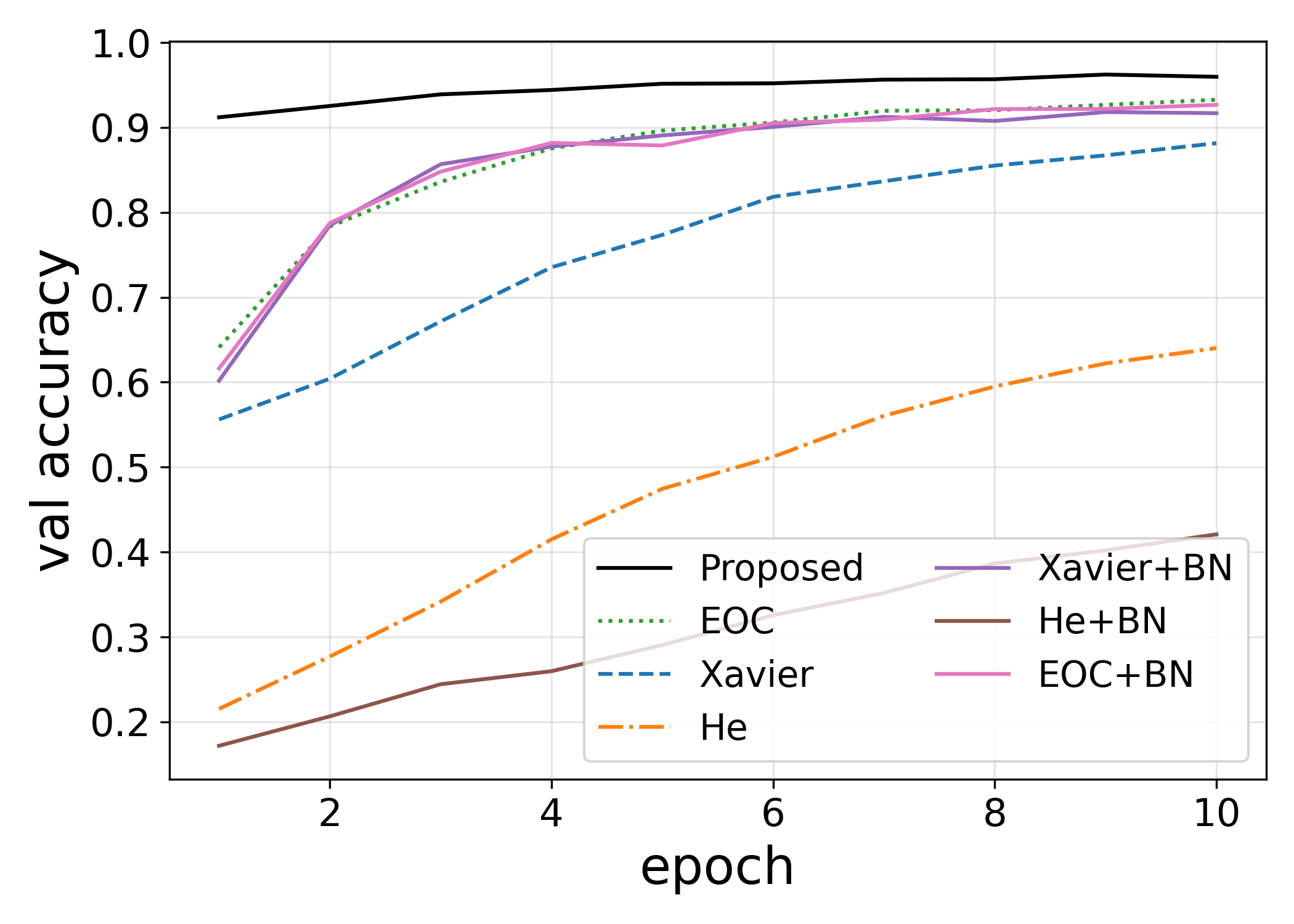}
    \caption{MNIST}
\end{subfigure} &
\begin{subfigure}[b]{0.31\textwidth}
    \centering
    \includegraphics[width=\textwidth]{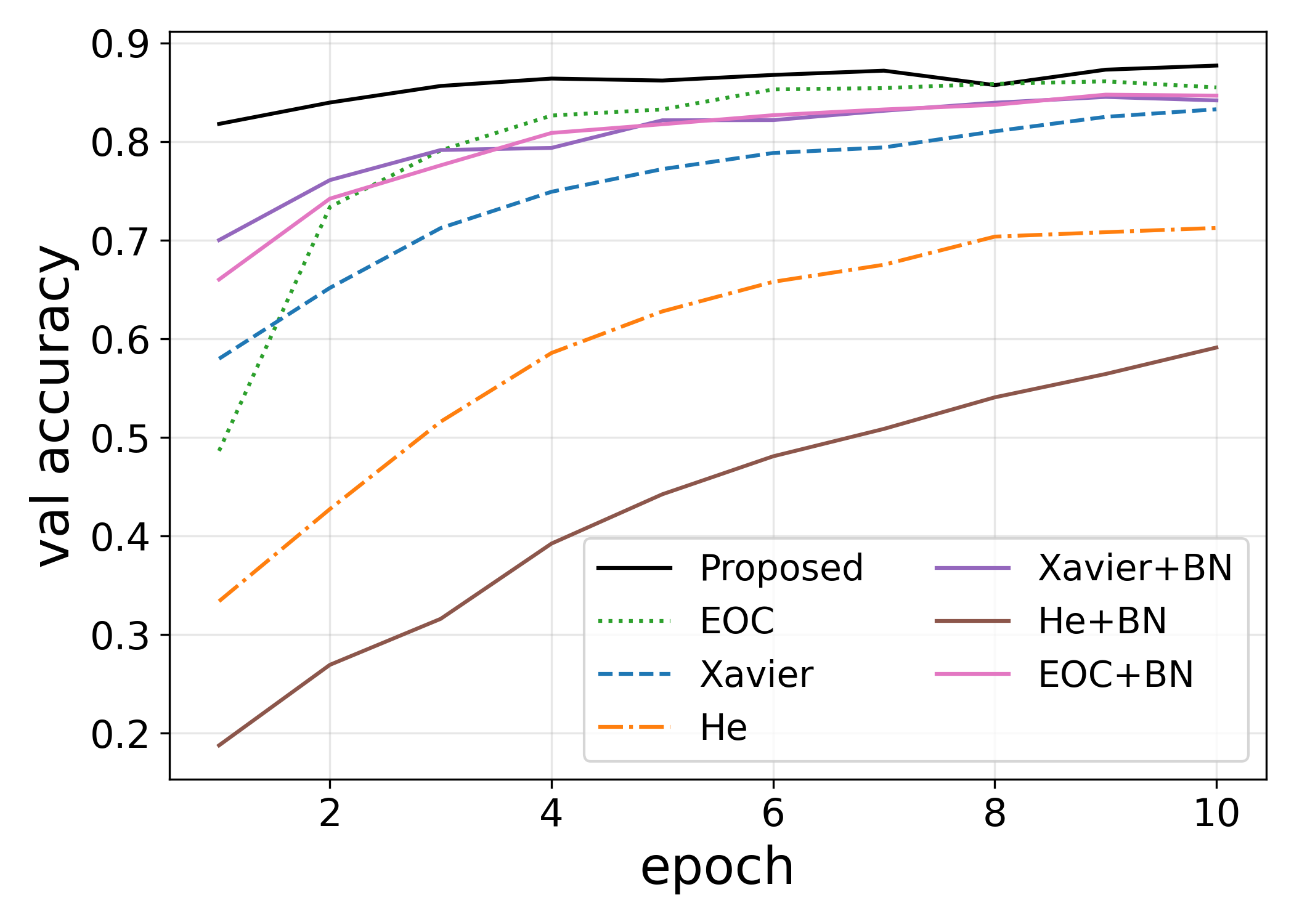}
    \caption{FMNIST}
\end{subfigure} &
\begin{subfigure}[b]{0.31\textwidth}
    \centering
    \includegraphics[width=\textwidth]{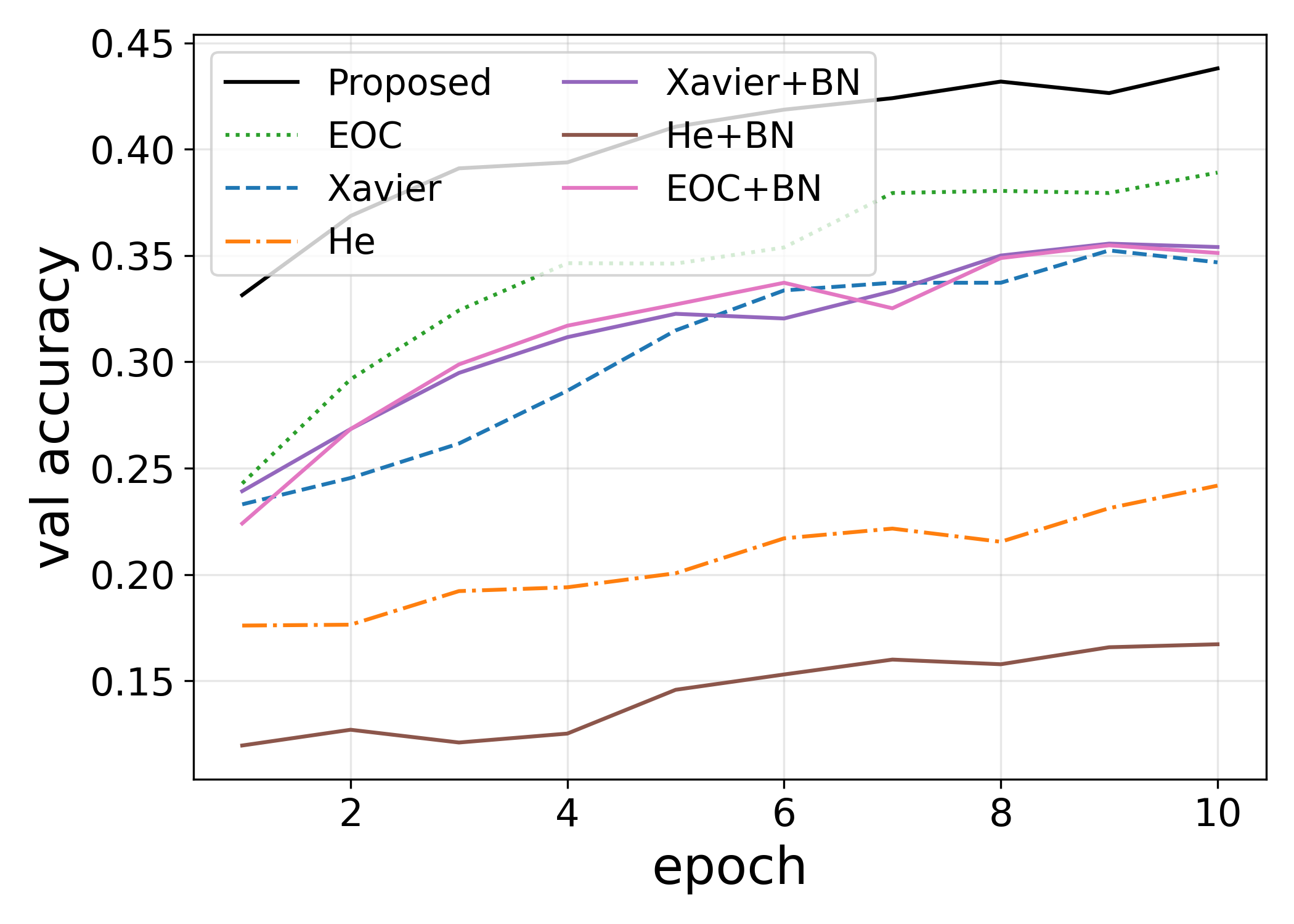}
    \caption{CIFAR 10}
\end{subfigure} 
\end{tabular} 
\caption{Validation accuracy over 10 epochs for fully connected 100 layer
networks with $tanh$ activations and width $64$. The curves compare the
Proposed, EOC, Xavier, and He initializations, with and without batch
normalization.}
\label{tanh_batch}
\end{figure}

\clearpage

\begin{figure}[h!]
\centering 
\includegraphics[width=0.5\textwidth]{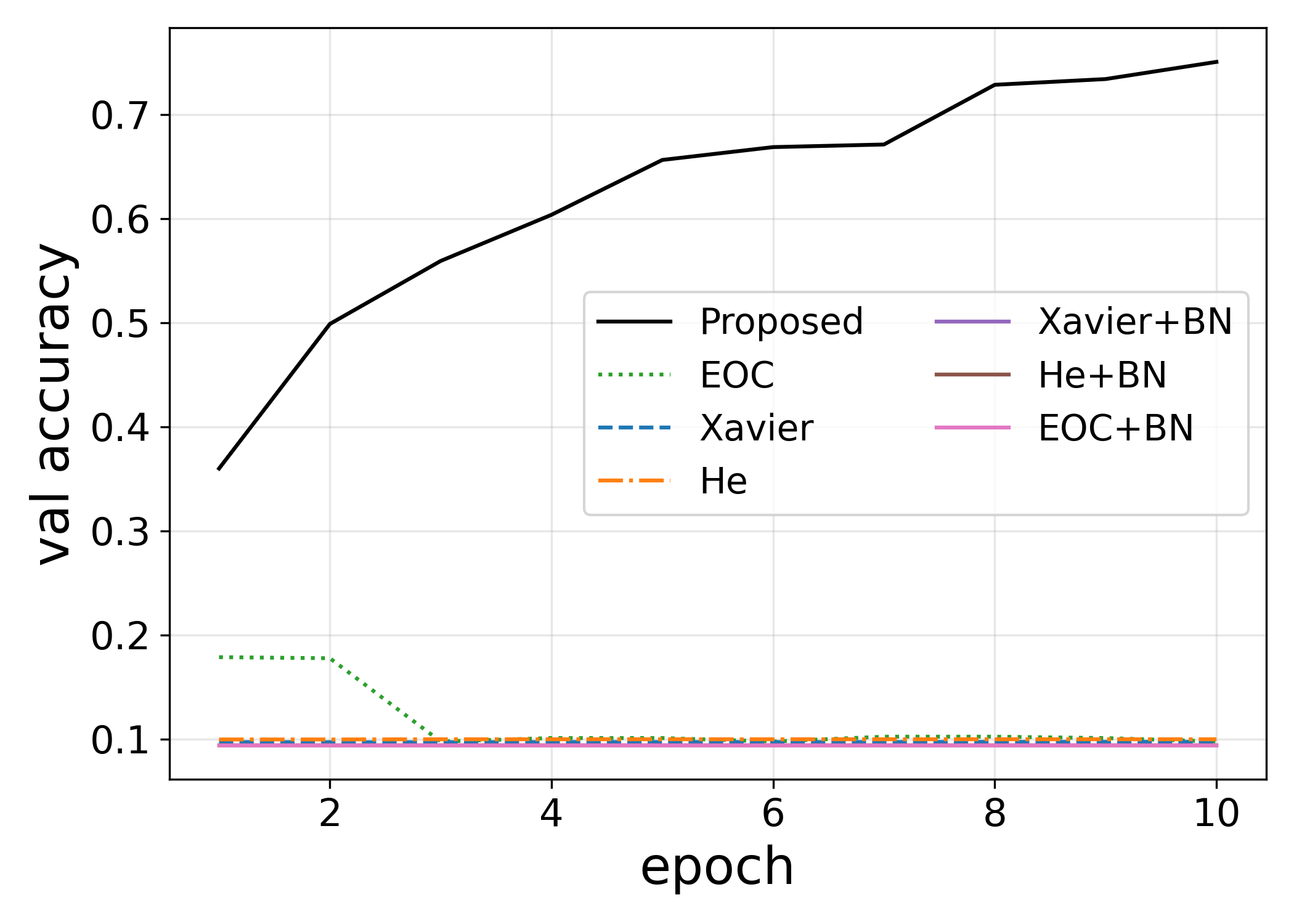}
\caption{Validation accuracy on Fashion MNIST for a fully connected network with
800 hidden layers of width 32 and
\(f(x) = \tanh(a x) + \mathrm{erf}(b x) + x/(1+|c x|) + \mathrm{gd}(d x)\),
using coefficients \((a,b,c,d) = (100,1000,10,1000)\).
}
\label{800layers}
\end{figure}

\begin{figure}[h!]
\centering 
\includegraphics[width=0.5\textwidth]{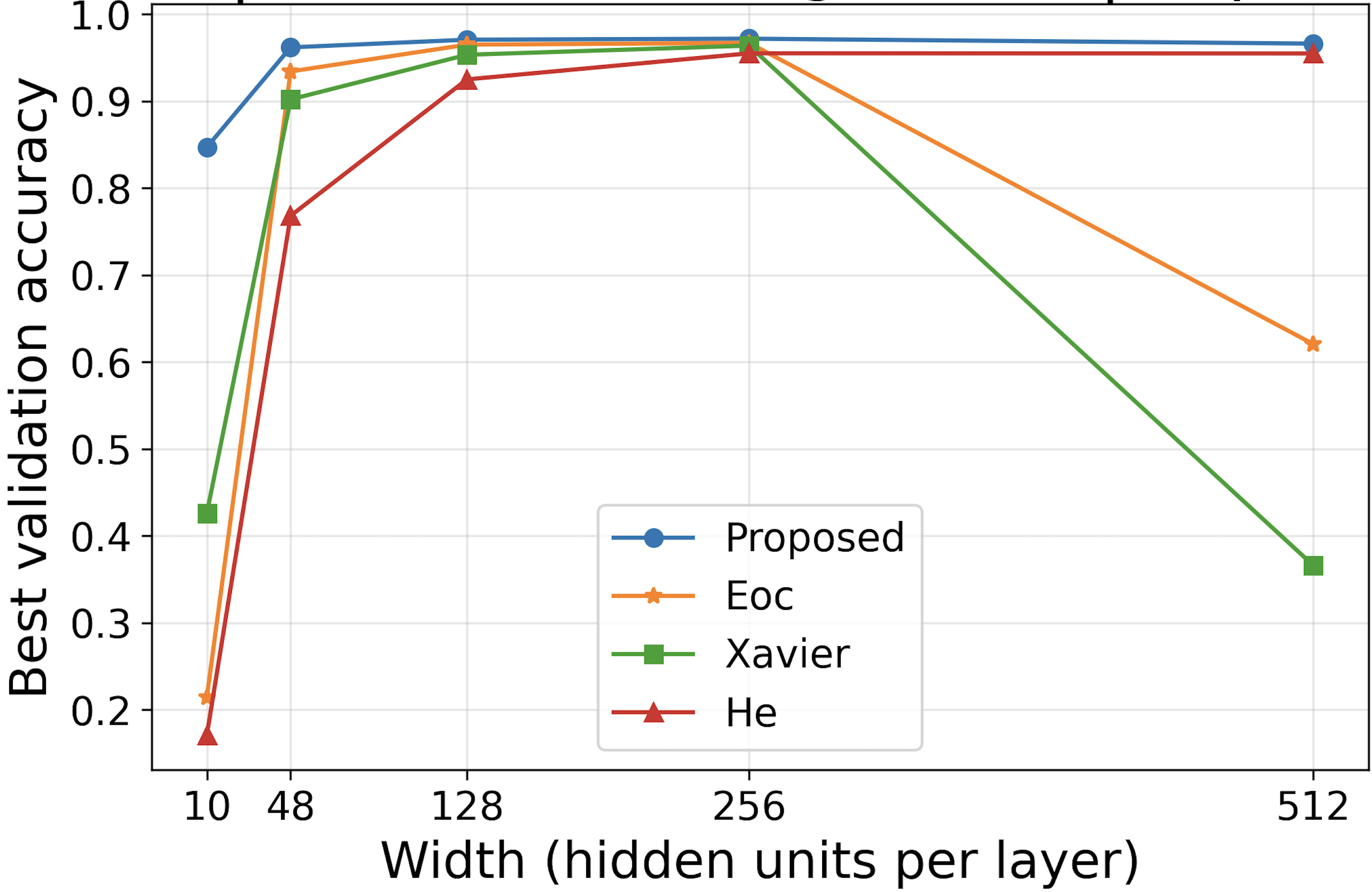}
\caption{Best validation accuracy on MNIST versus width $\{10, 48, 128, 150, 200, 512\}$
for 100 layer $\tanh$ FFNN.
Each curve compares the Proposed, EOC, Xavier, and He initialization schemes,
and each point reports the best validation accuracy over 20 training epochs.
}
\label{wd2}
\end{figure}

\clearpage
\subsection{Learnable Learning Rate}
Our proposed initialization depends explicitly on the
$\omega = 1/f'(0)$, so rescaling the activation also rescales the initialized weights. In particular, when we replace $f$ by a scaled activation
$\alpha f(x)$ or $f(\alpha x)$, the corresponding value of
$\omega$ changes and the scale of the initial weight matrix is adjusted
accordingly. As a consequence, the learning rate that yields comparable
gradient updates should also depend on $\omega$. In our experiments we
therefore choose $\eta$ from an $\omega$–scaled band (e.g.,
$\eta \in [10^{-5}\omega,10^{-3}\omega]$) when using the proposed
initialization.  Across all settings, the proposed initializer remains trainable over a wider interval of $\eta$ than Gaussian initializations, indicating that it is more
robust to the choice of learning rate even when the activation scale varies
significantly.

\begin{figure}[h!]\label{fig2}
\centering 
\begin{tabular}{ccc}
\begin{subfigure}[b]{0.30\textwidth}
    \centering
    \includegraphics[width=\textwidth]{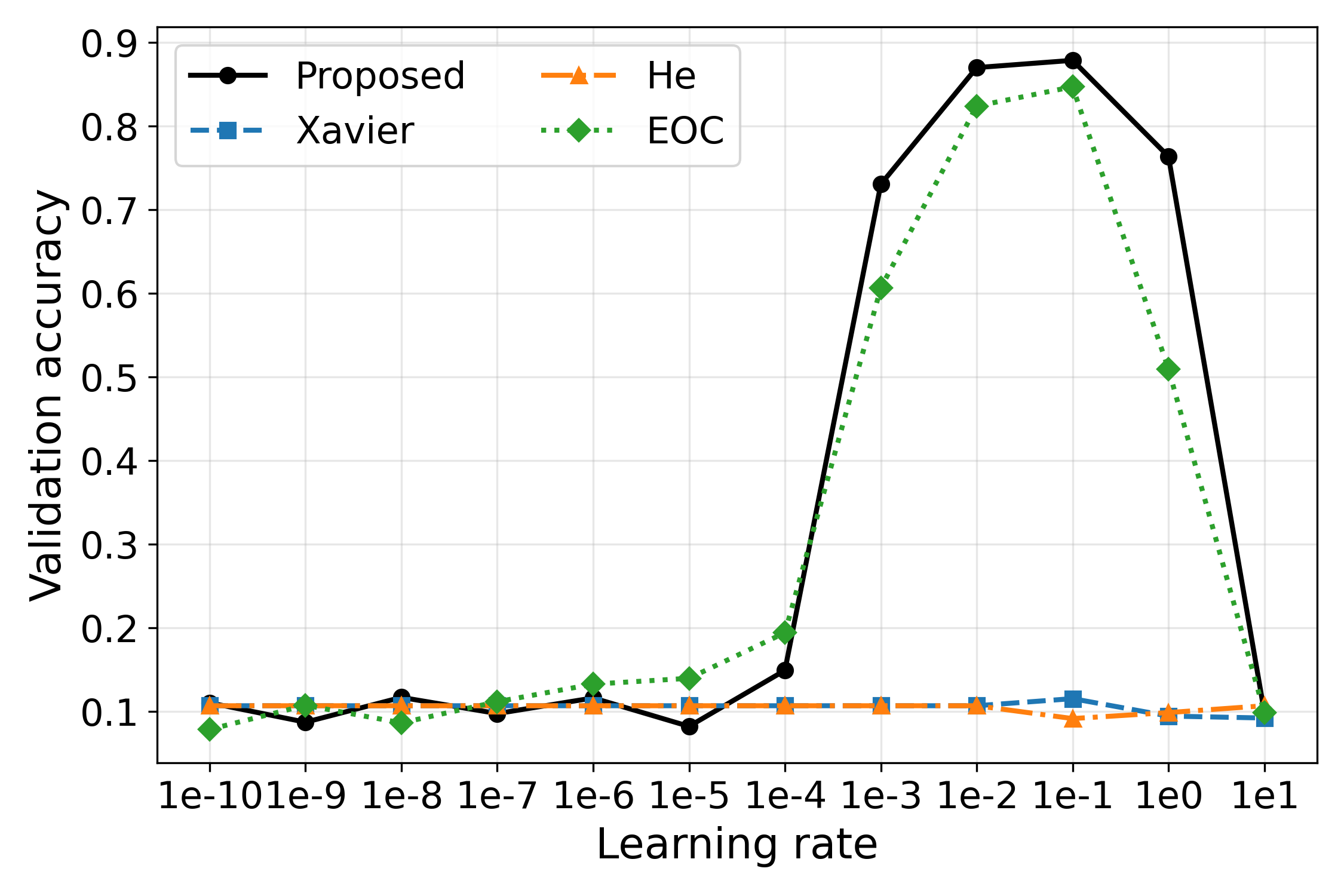}
    \caption{$2/\pi\arctan(0.001x)$}
\end{subfigure} &
\begin{subfigure}[b]{0.30\textwidth}
    \centering
    \includegraphics[width=\textwidth]{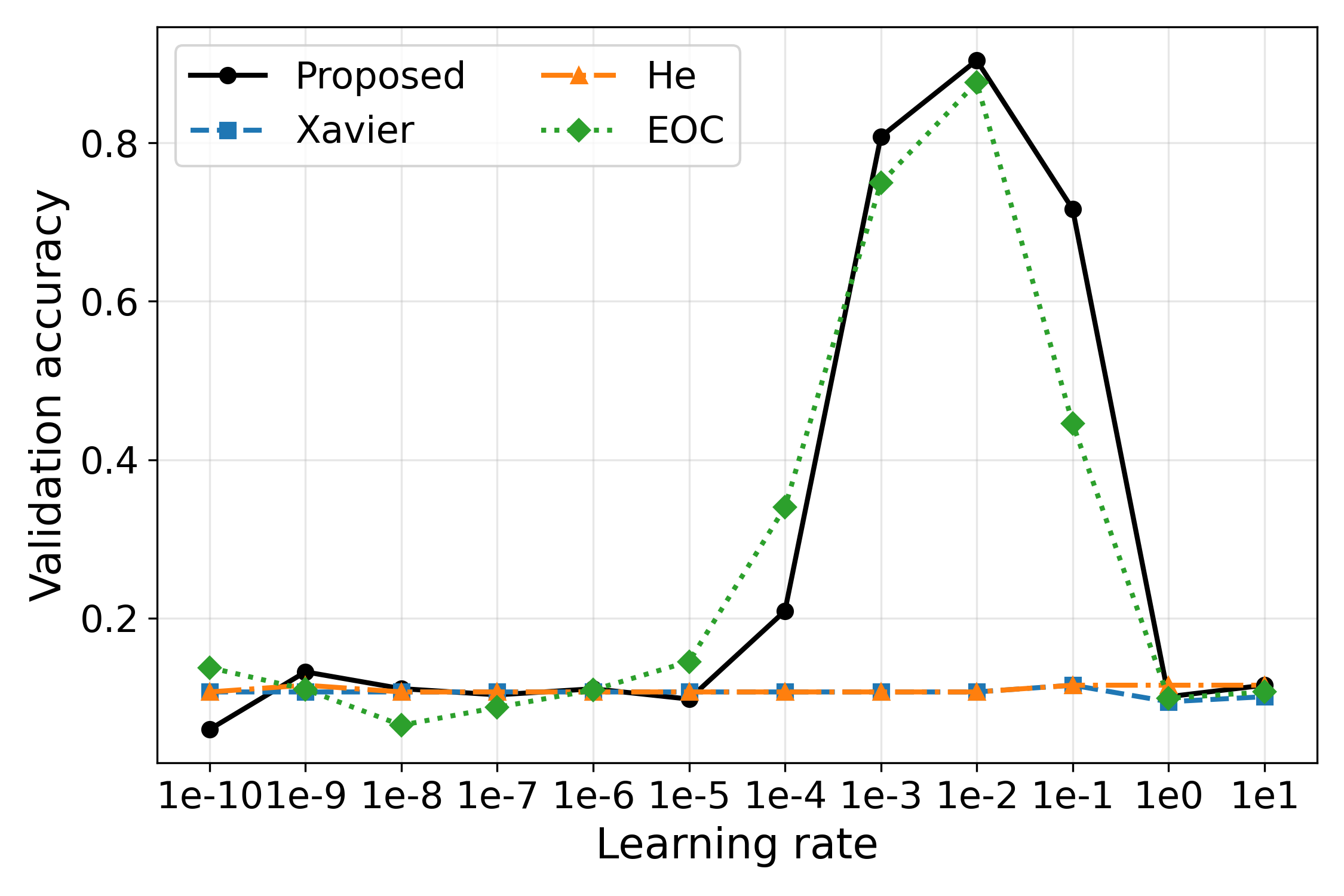}
    \caption{$2/\pi\arctan(0.01x)$}
\end{subfigure} &
\begin{subfigure}[b]{0.30\textwidth}
    \centering
    \includegraphics[width=\textwidth]{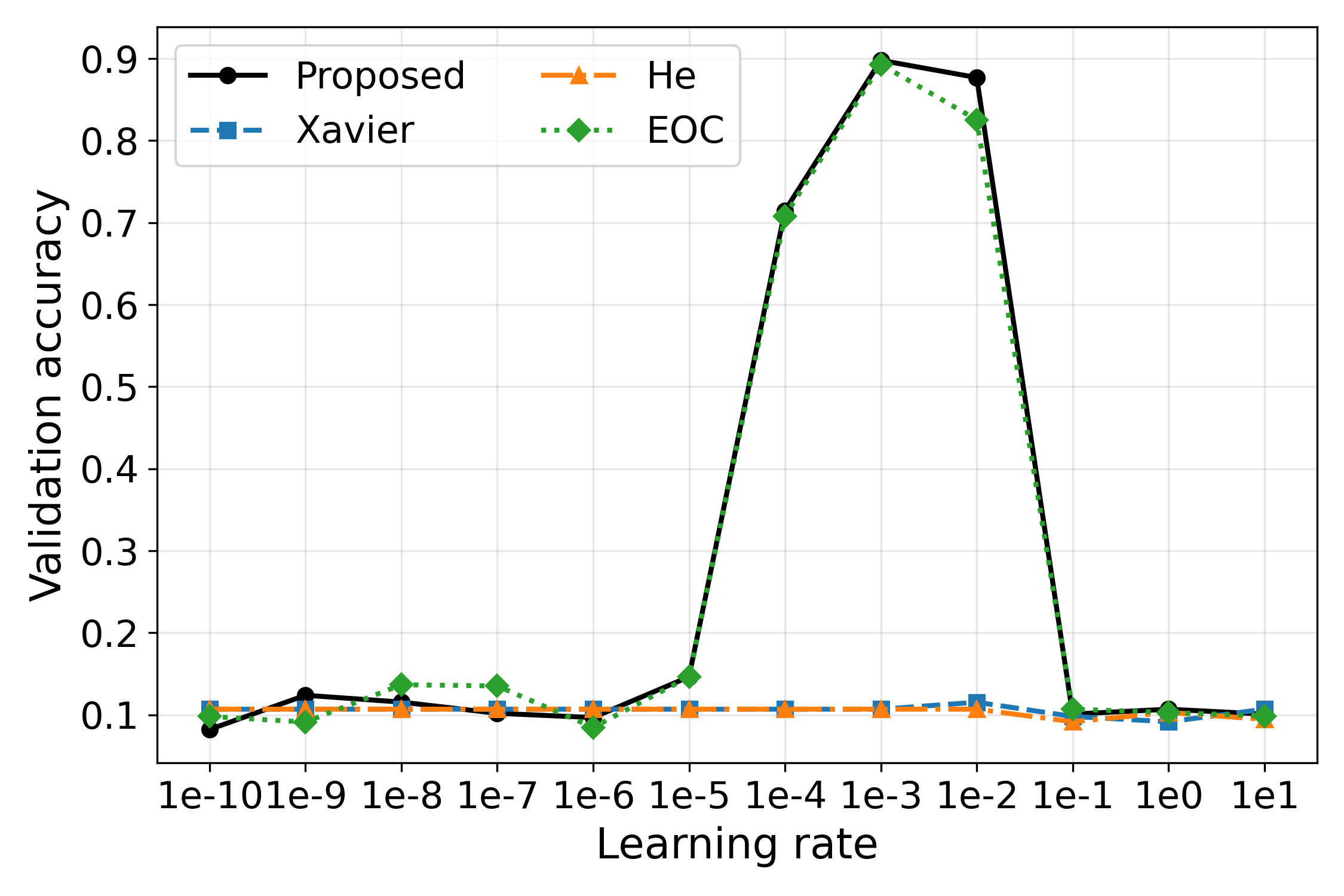}
    \caption{$2/\pi\arctan(0.1x)$}
\end{subfigure} \\
\begin{subfigure}[b]{0.30\textwidth}
    \centering
    \includegraphics[width=\textwidth]{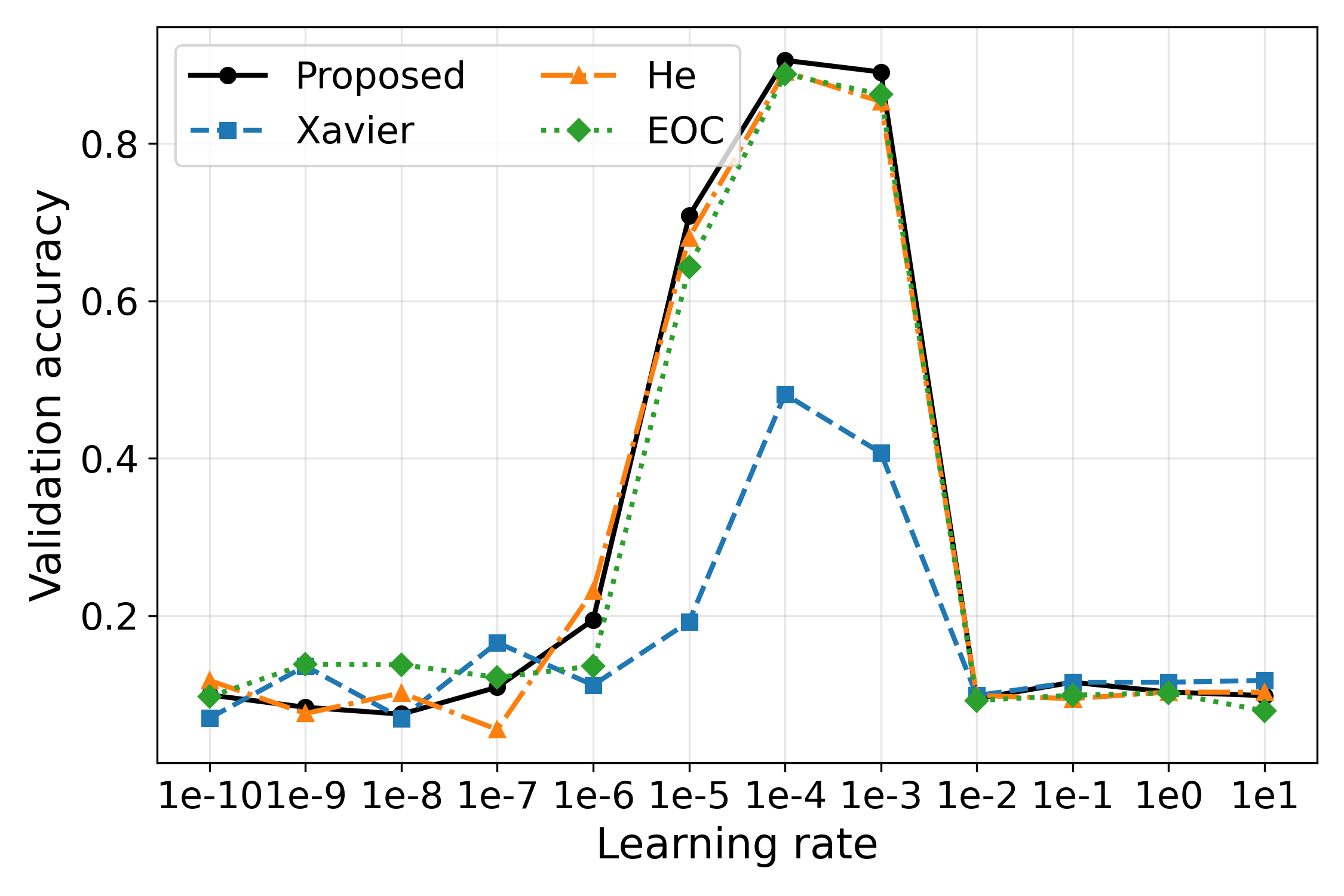}
    \caption{$2/\pi\arctan(1x)$}
\end{subfigure} &
\begin{subfigure}[b]{0.30\textwidth}
    \centering
    \includegraphics[width=\textwidth]{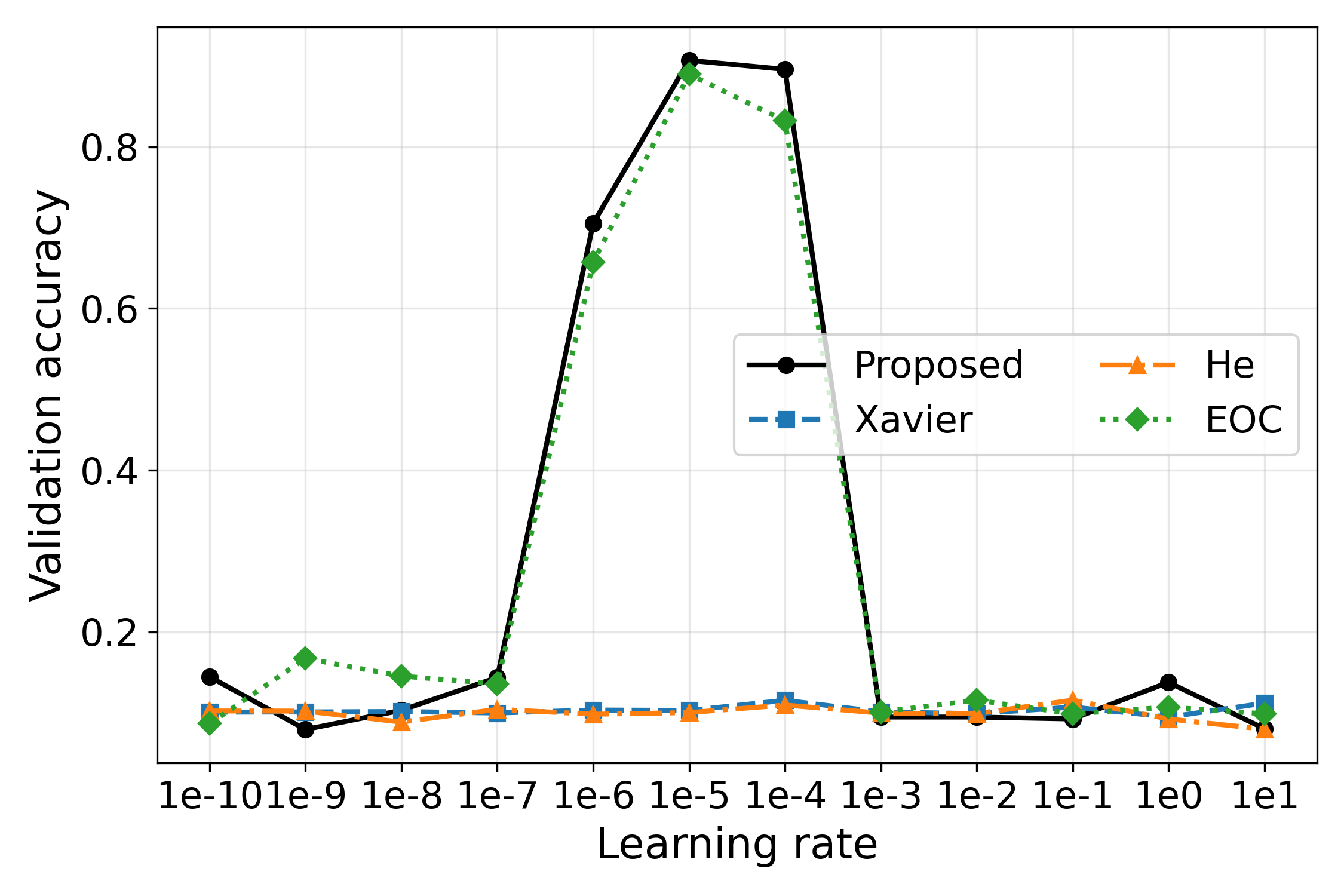}
    \caption{$2/\pi\arctan(10x)$}
\end{subfigure} &
\begin{subfigure}[b]{0.30\textwidth}
    \centering
    \includegraphics[width=\textwidth]{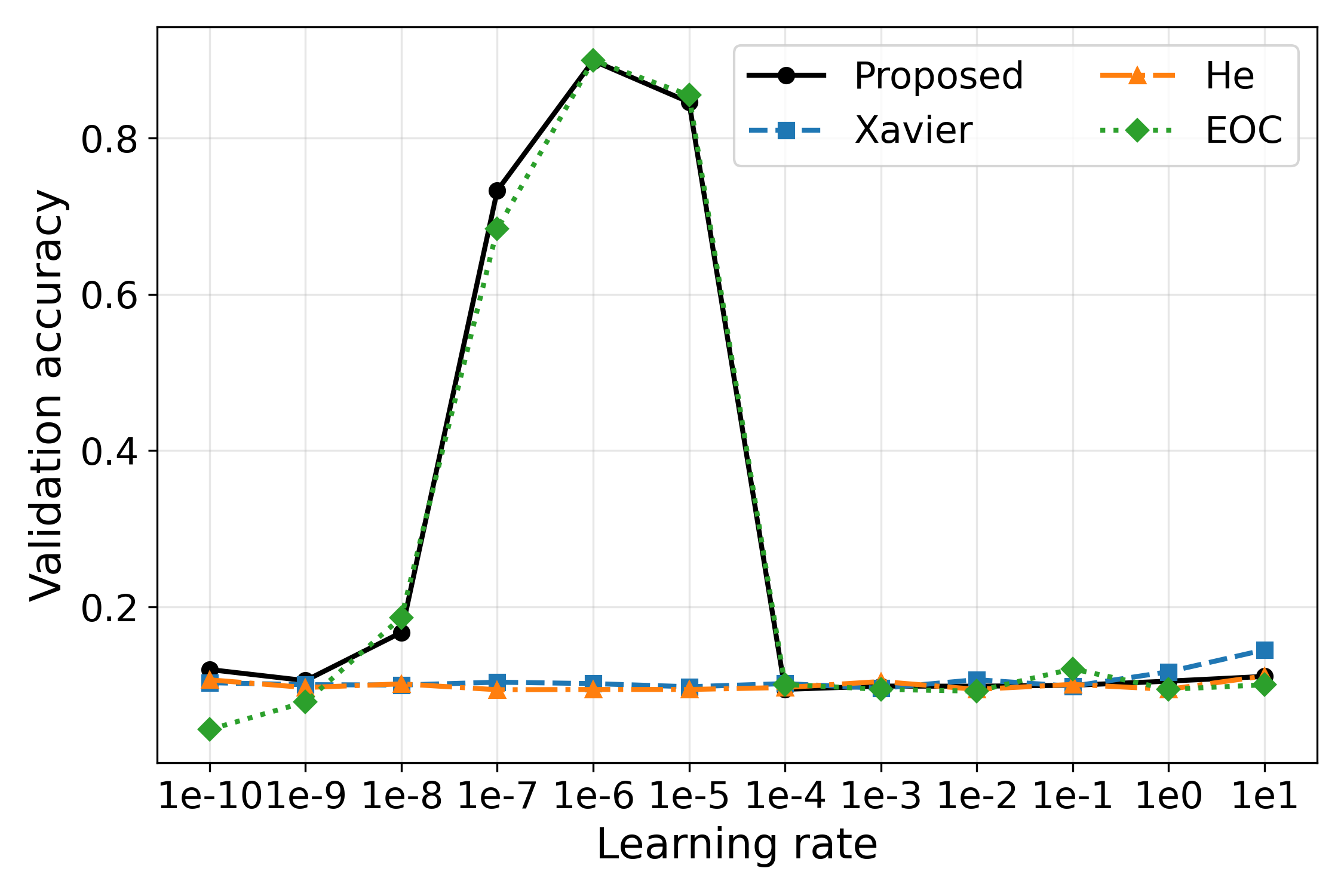}
    \caption{$2/\pi\arctan(100x)$}
\end{subfigure} 
\end{tabular} 
  \caption{Learning rate accuracy curves on MNIST for a 20 layer, width 512 feedforward network
    with activation $f(x) = \tfrac{2}{\pi}\arctan(\alpha x)$.
    Each panel corresponds to a different activation scale
    $\alpha \in \{10^2, 10^1, 1, 10^{-1}, 10^{-2}, 10^{-3}\}$.
    For each learning rate, we train for 200 iterations on a 10k training subset
    and report the validation accuracy.}
\label{lrst1}
\end{figure}

\begin{figure}[h!]\label{fig2}
\centering 
\begin{tabular}{ccc}
\begin{subfigure}[b]{0.30\textwidth}
    \centering
    \includegraphics[width=\textwidth]{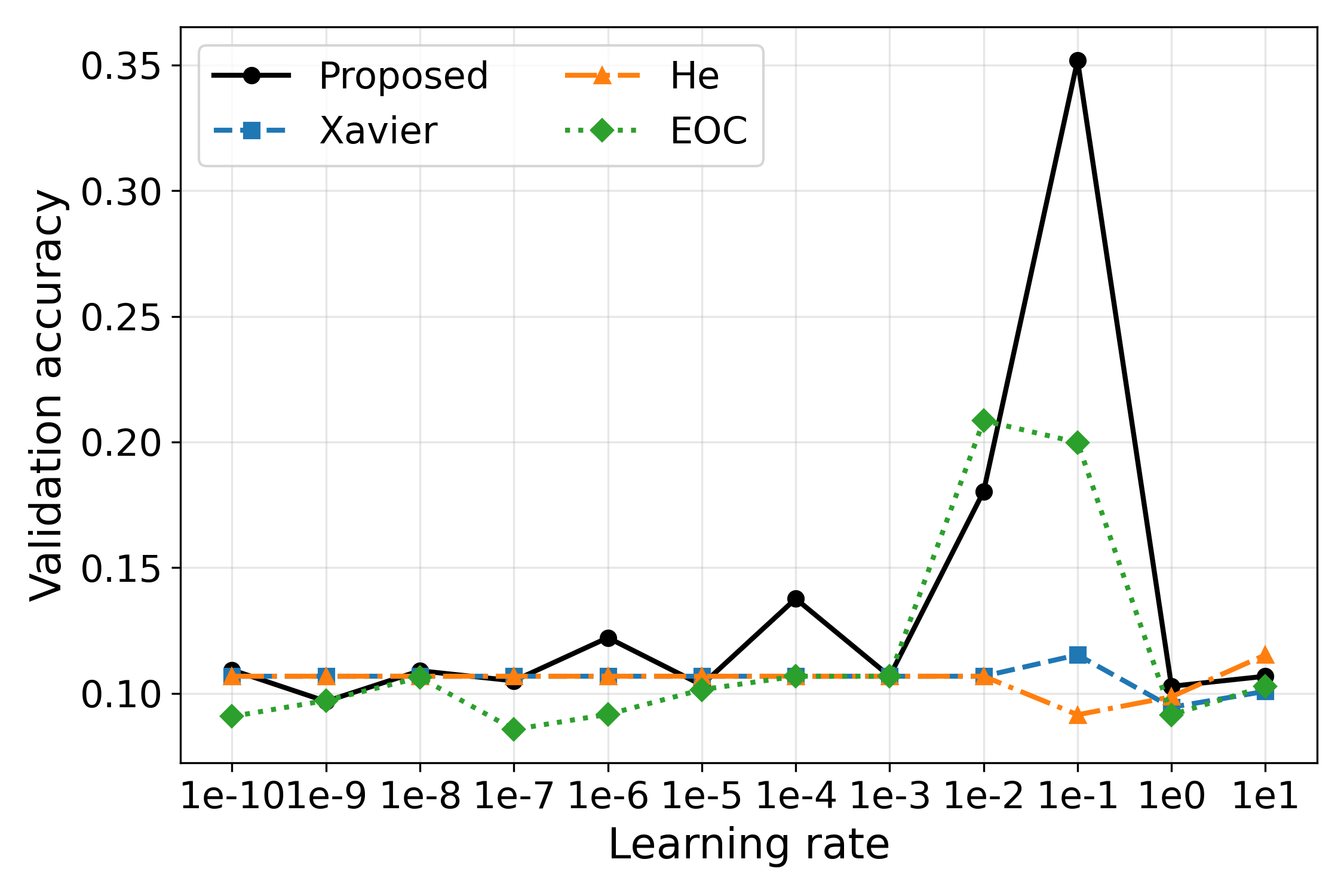}
    \caption{$\tanh(0.001x)$}
\end{subfigure} &
\begin{subfigure}[b]{0.30\textwidth}
    \centering
    \includegraphics[width=\textwidth]{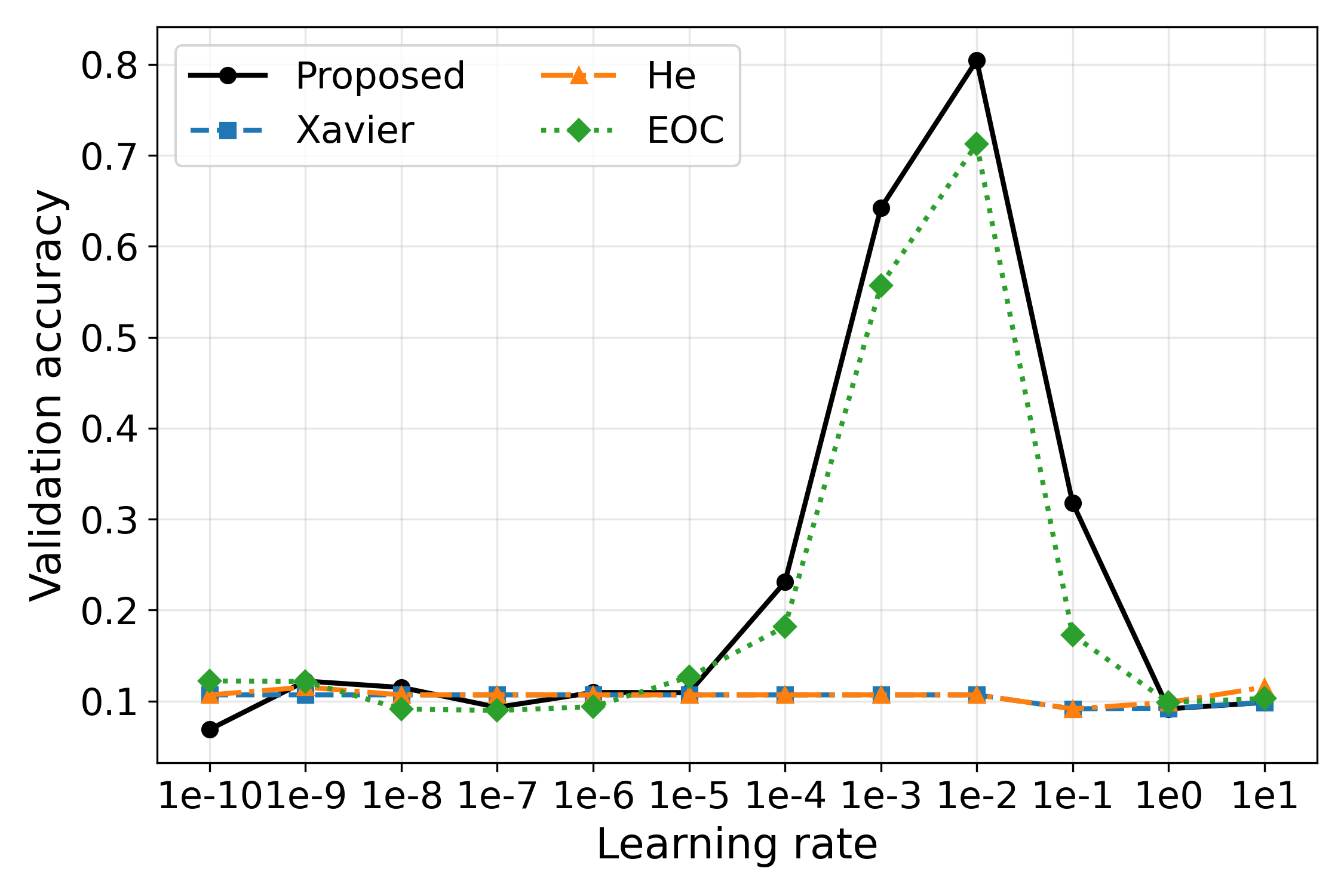}
    \caption{$\tanh(0.01x)$}
\end{subfigure} &
\begin{subfigure}[b]{0.30\textwidth}
    \centering
    \includegraphics[width=\textwidth]{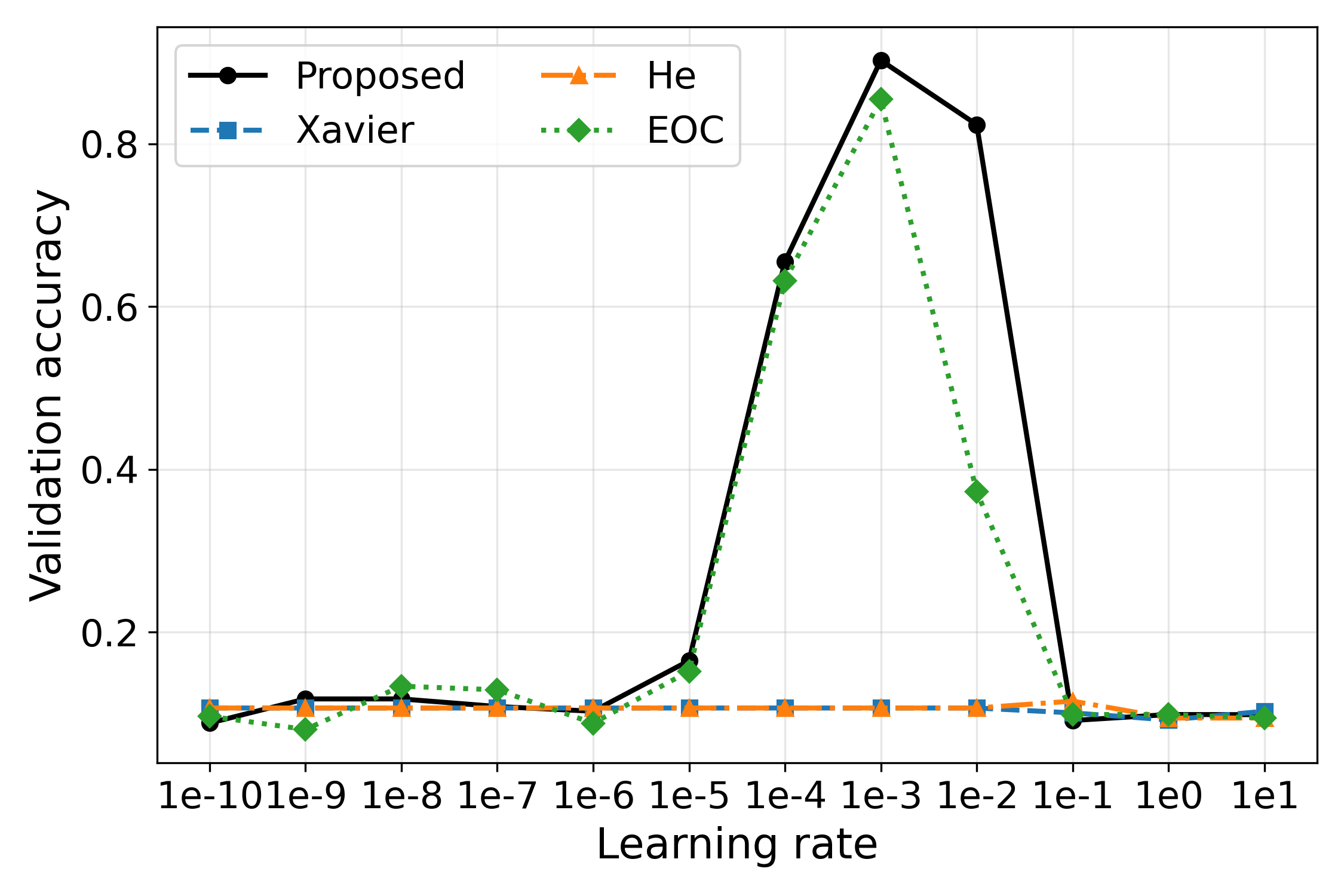}
    \caption{$\tanh(0.1x)$}
\end{subfigure} \\
\begin{subfigure}[b]{0.30\textwidth}
    \centering
    \includegraphics[width=\textwidth]{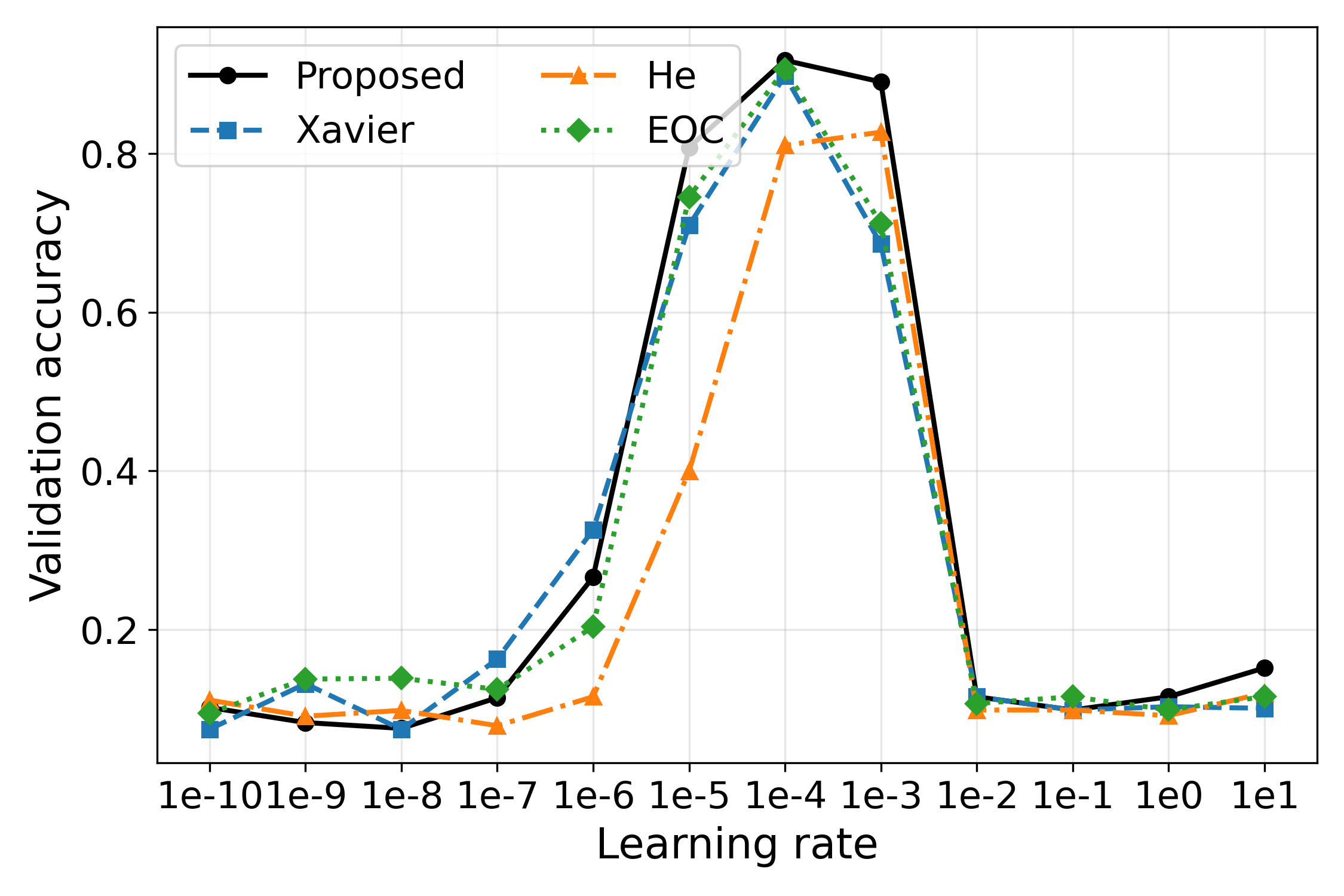}
    \caption{$\tanh(1x)$}
\end{subfigure} &
\begin{subfigure}[b]{0.30\textwidth}
    \centering
    \includegraphics[width=\textwidth]{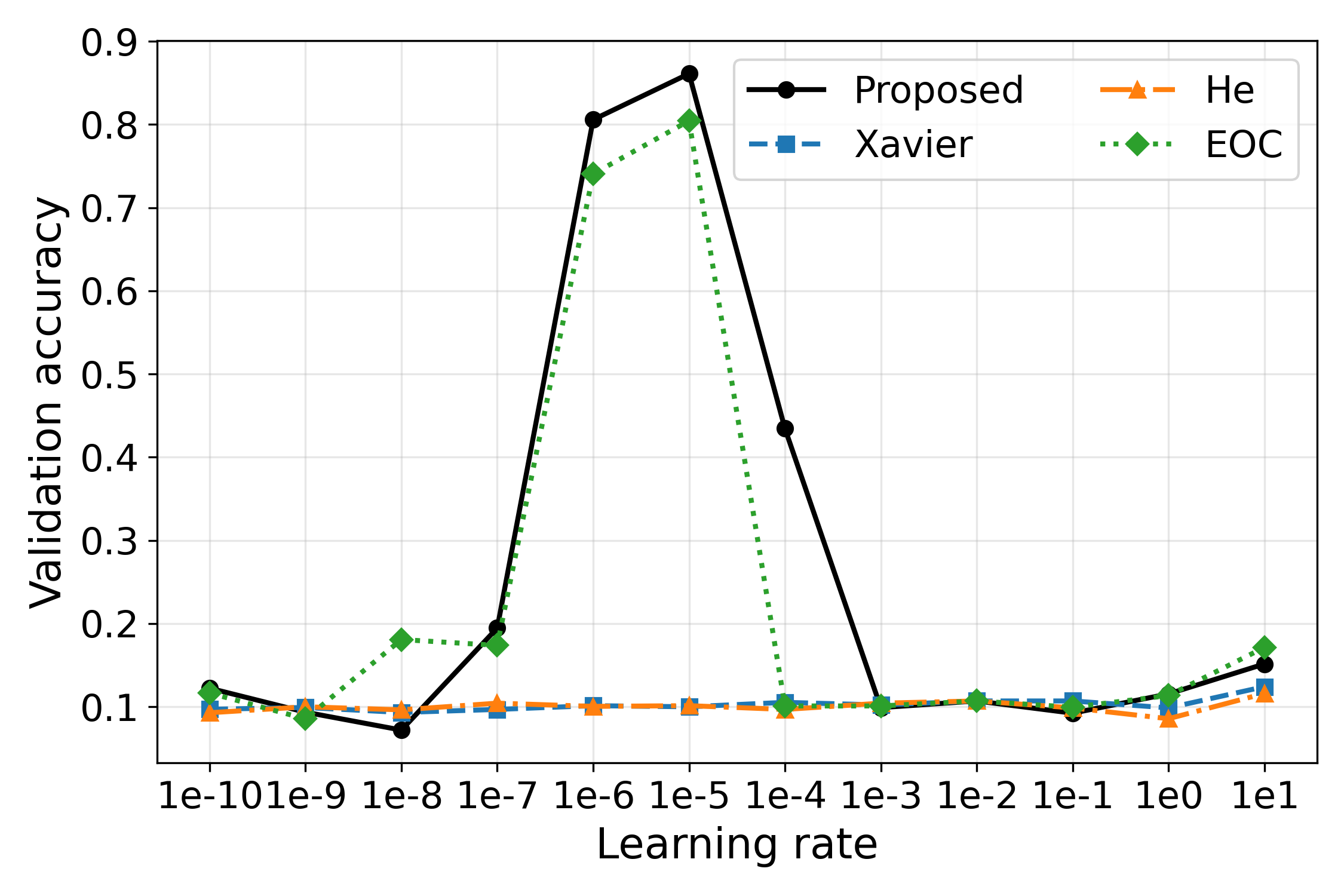}
    \caption{$\tanh(10x)$}
\end{subfigure} &
\begin{subfigure}[b]{0.30\textwidth}
    \centering
    \includegraphics[width=\textwidth]{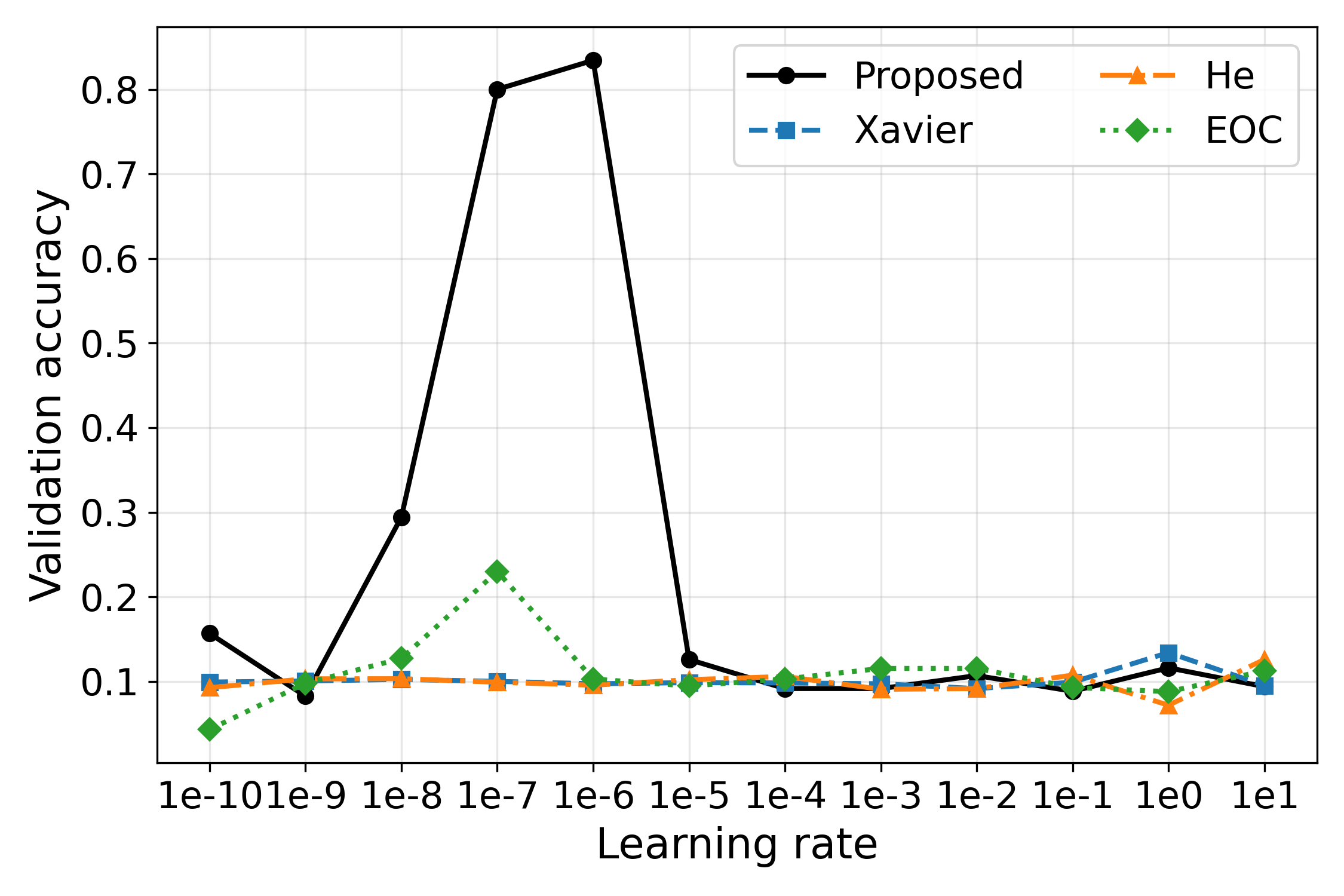}
    \caption{$\tanh(100x)$}
\end{subfigure} 
\end{tabular} 
  \caption{Learning rate accuracy curves on MNIST for a 20 layer, width 512 feedforward network
    with activation $f(x) =\tanh(\alpha x)$.
    Each panel corresponds to a different activation scale
    $\alpha \in \{10^2, 10^1, 1, 10^{-1}, 10^{-2}, 10^{-3}\}$.
    For each learning rate, we train for 200 iterations on a 1k training subset
    and report the validation accuracy.}
\label{lrst2}
\end{figure}

\begin{figure}[h!]\label{fig2}
\centering 
\begin{tabular}{ccc}
\begin{subfigure}[b]{0.30\textwidth}
    \centering
    \includegraphics[width=\textwidth]{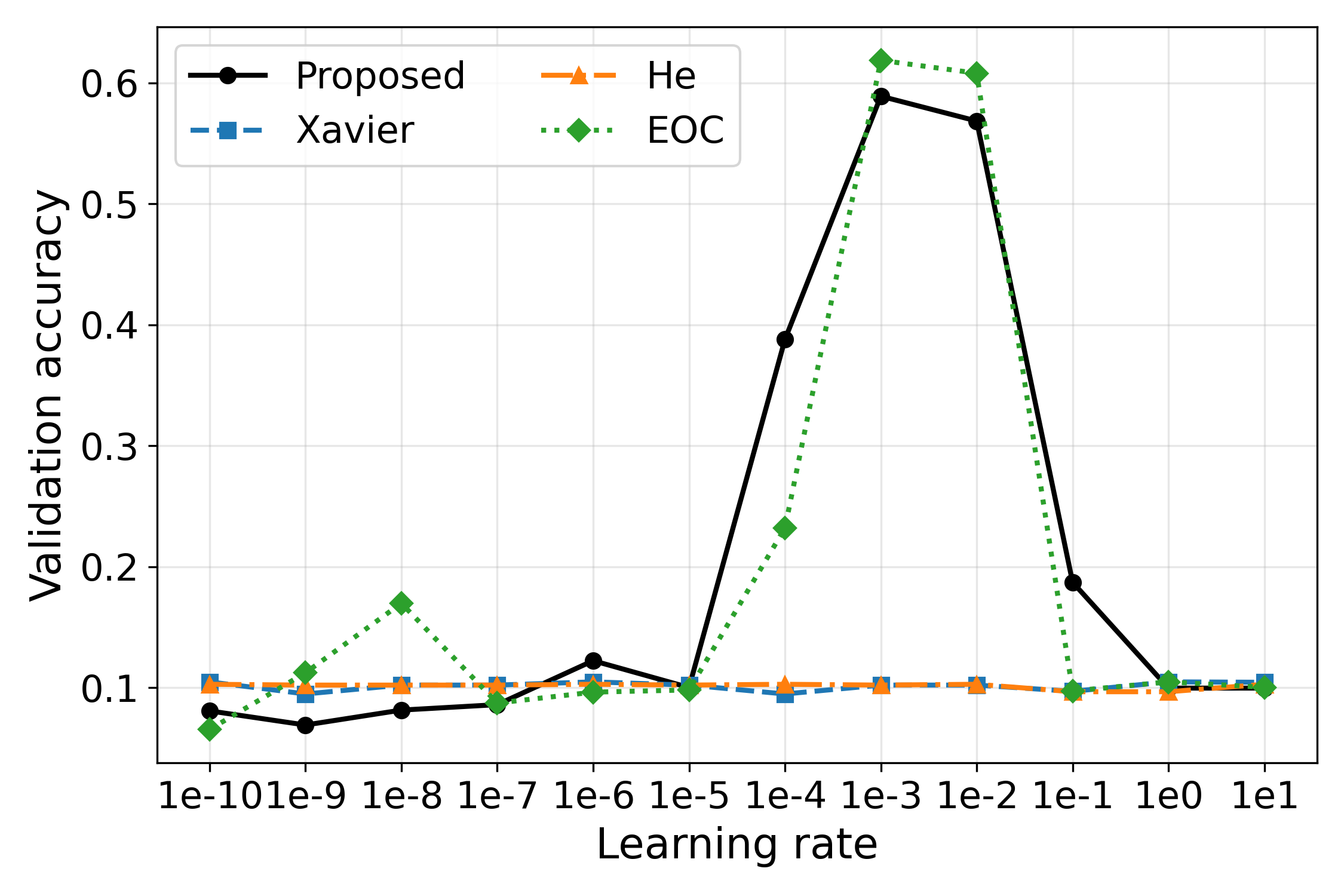}
    \caption{$\tanh(0.01x)$}
\end{subfigure} &
\begin{subfigure}[b]{0.30\textwidth}
    \centering
    \includegraphics[width=\textwidth]{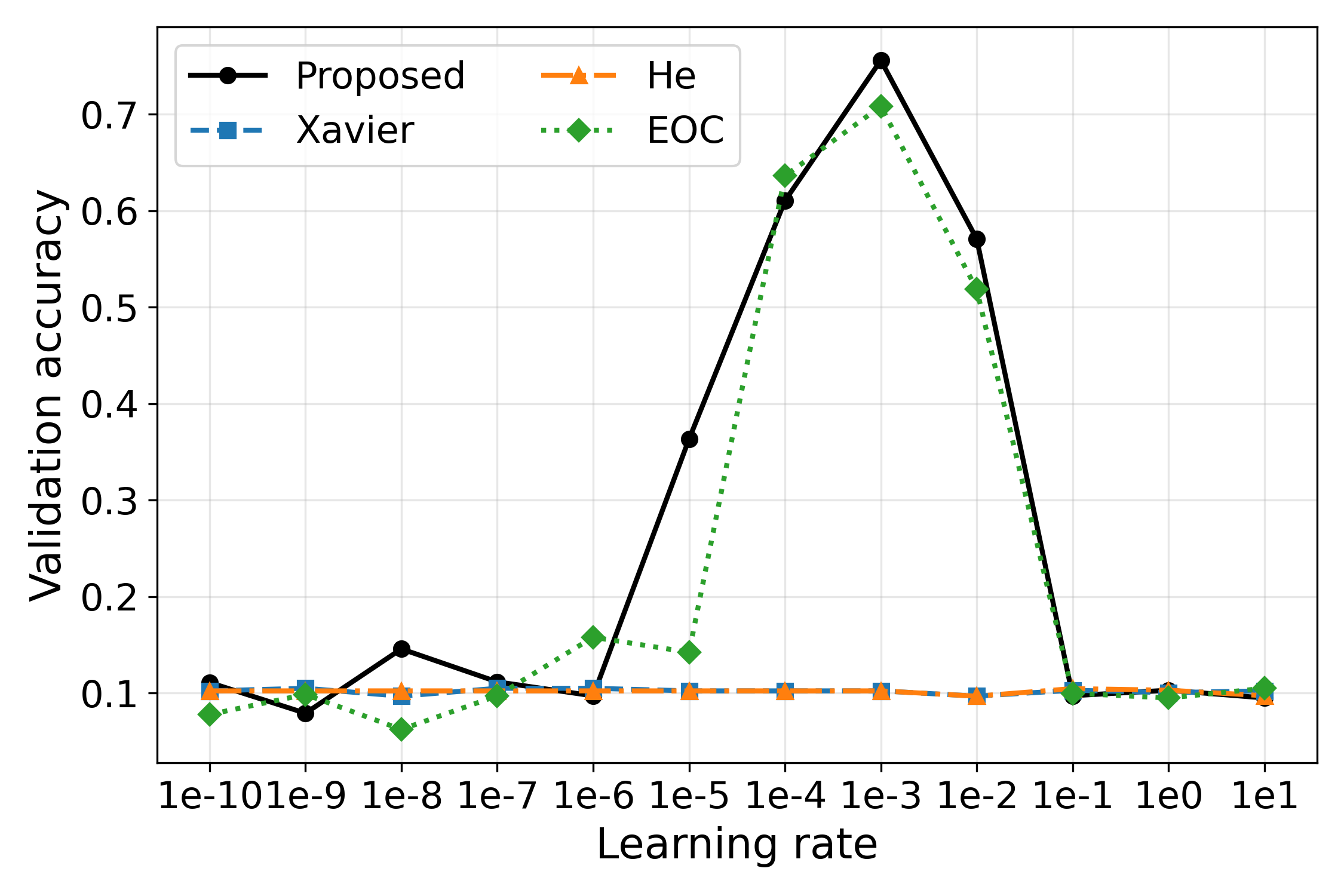}
    \caption{$\tanh(0.1x)$}
\end{subfigure} &
\begin{subfigure}[b]{0.30\textwidth}
    \centering
    \includegraphics[width=\textwidth]{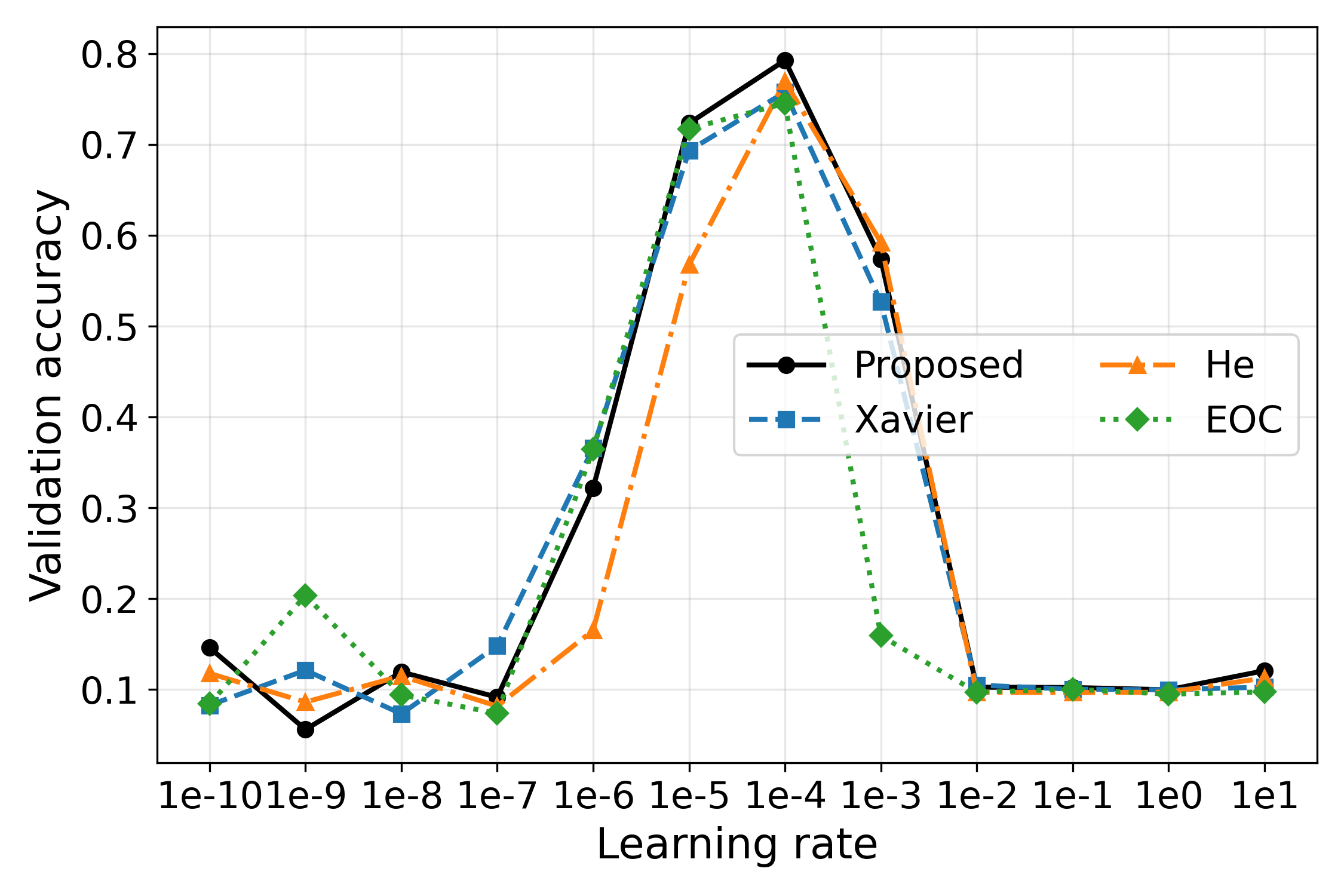}
    \caption{$\tanh(x)$}
\end{subfigure} \\
\begin{subfigure}[b]{0.30\textwidth}
    \centering
    \includegraphics[width=\textwidth]{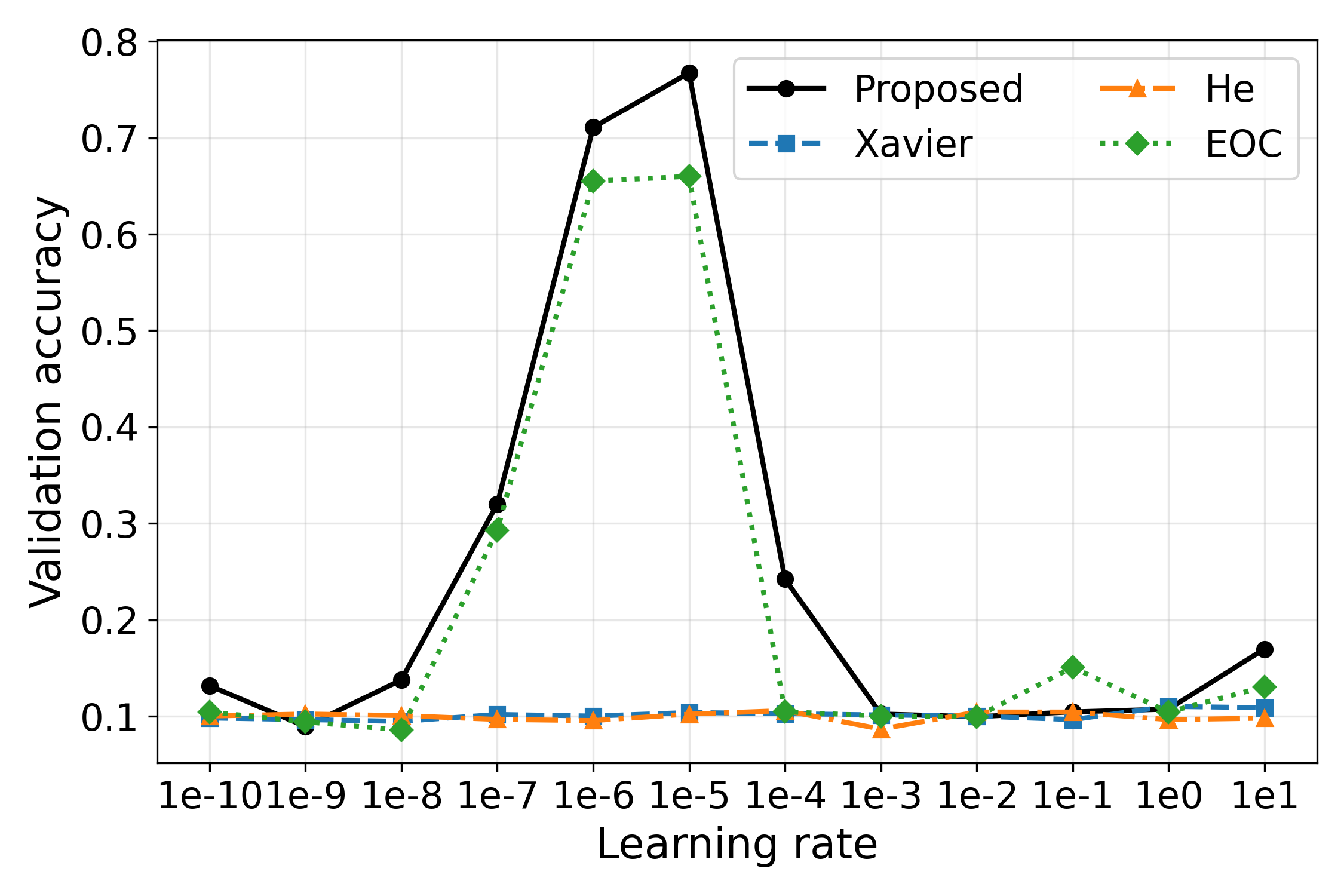}
    \caption{$\tanh(10x)$}
\end{subfigure} &
\begin{subfigure}[b]{0.30\textwidth}
    \centering
    \includegraphics[width=\textwidth]{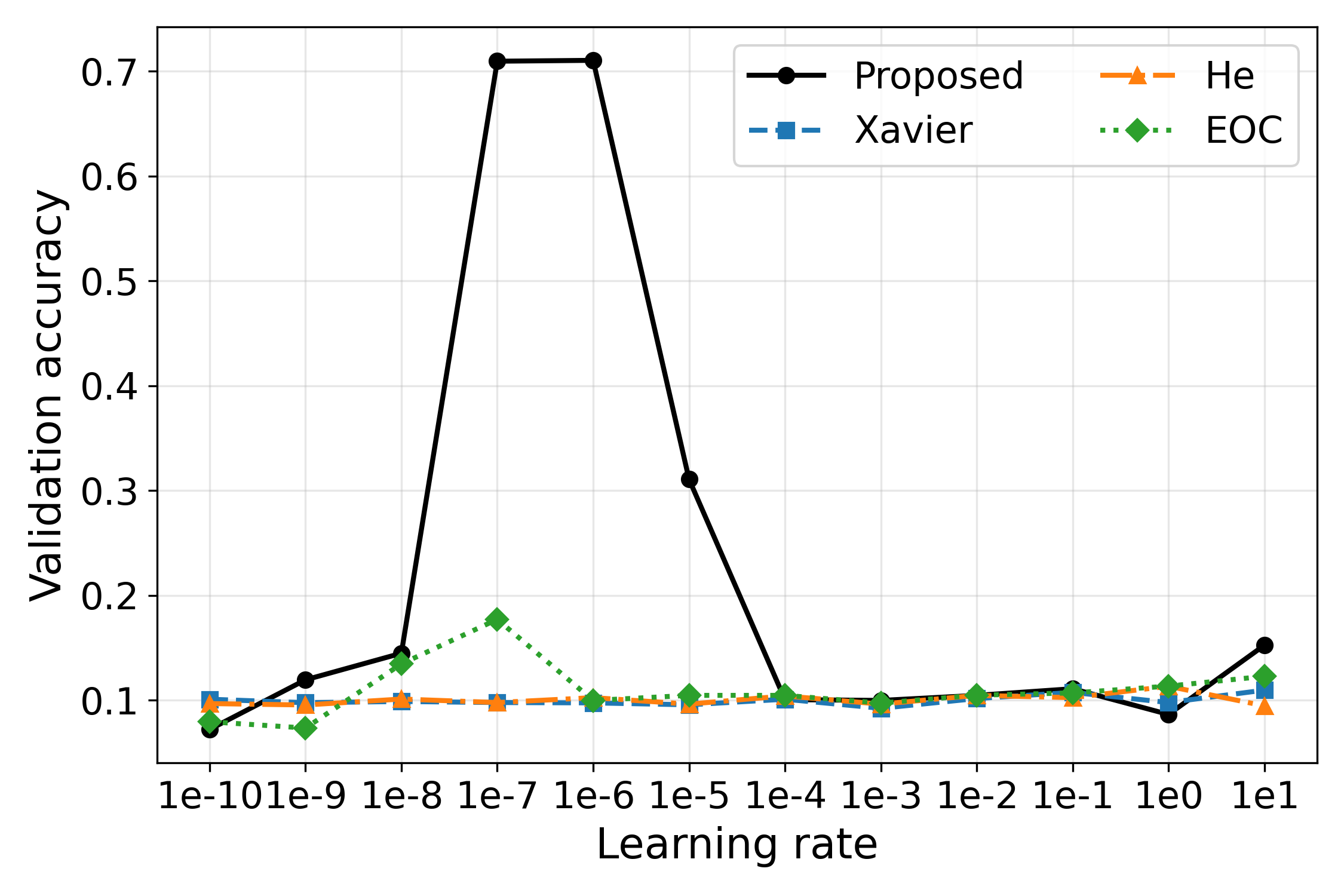}
    \caption{$\tanh(100x)$}
\end{subfigure} &
\begin{subfigure}[b]{0.30\textwidth}
    \centering
    \includegraphics[width=\textwidth]{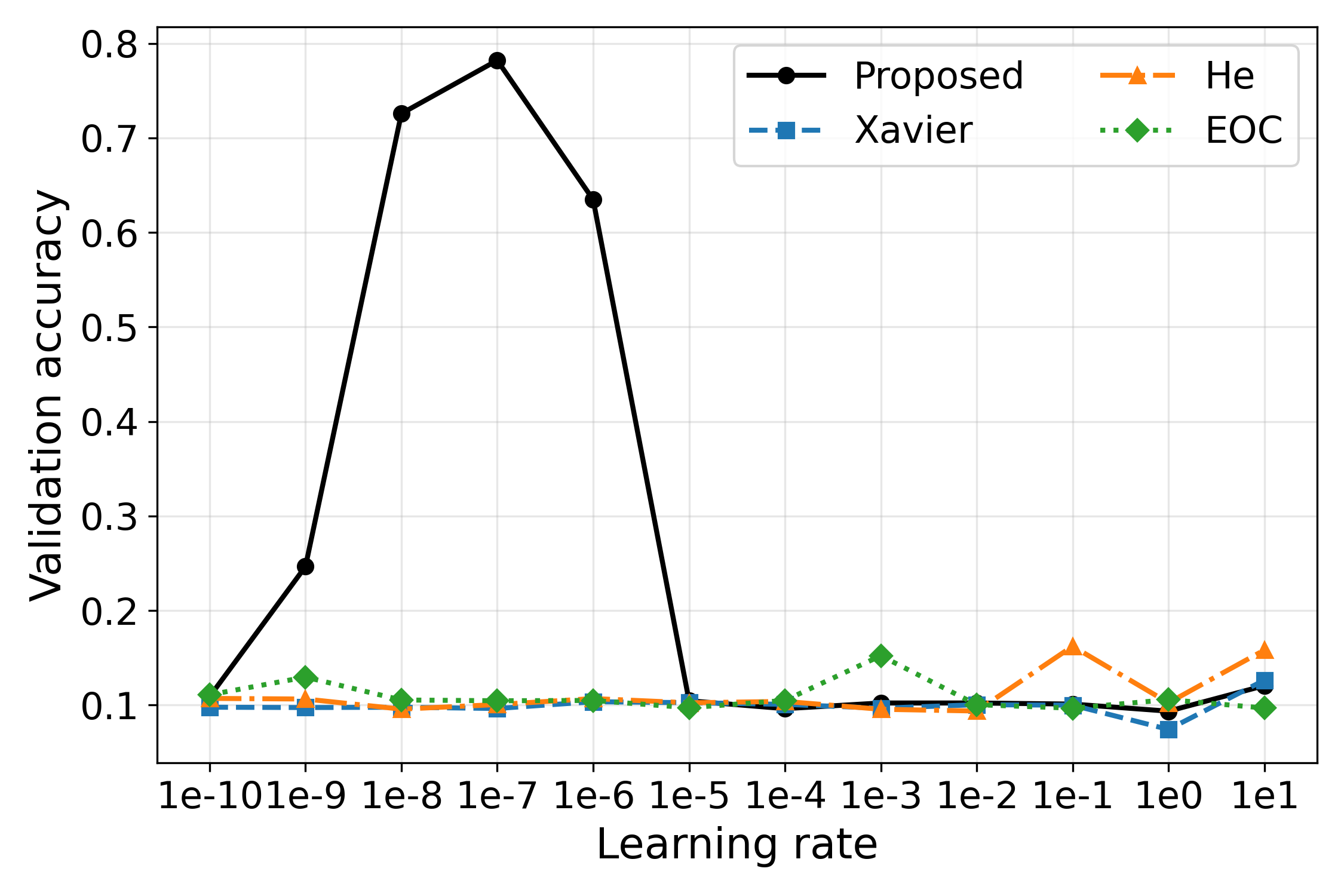}
    \caption{$\tanh(1000x)$}
\end{subfigure} 
\end{tabular} 
  \caption{Learning rate accuracy curves on Fashion MNIST for a 20 layer, width 512 feedforward network
    with activation $f(x) =\tanh(\alpha x)$.
    Each panel corresponds to a different activation scale
    $\alpha \in \{10^3, 10^2, 10^1, 1, 10^{-1}, 10^{-2}\}$.
    For each learning rate, we train for 200 iterations on a 10k training subset
    and report the validation accuracy.
    Curves compare four initializations: Proposed, Xavier, He, and EOC.
  }
\label{lrst3}
\end{figure}

\begin{figure}[h!]\label{fig2}
\centering 
\begin{tabular}{ccc}
\begin{subfigure}[b]{0.30\textwidth}
    \centering
    \includegraphics[width=\textwidth]{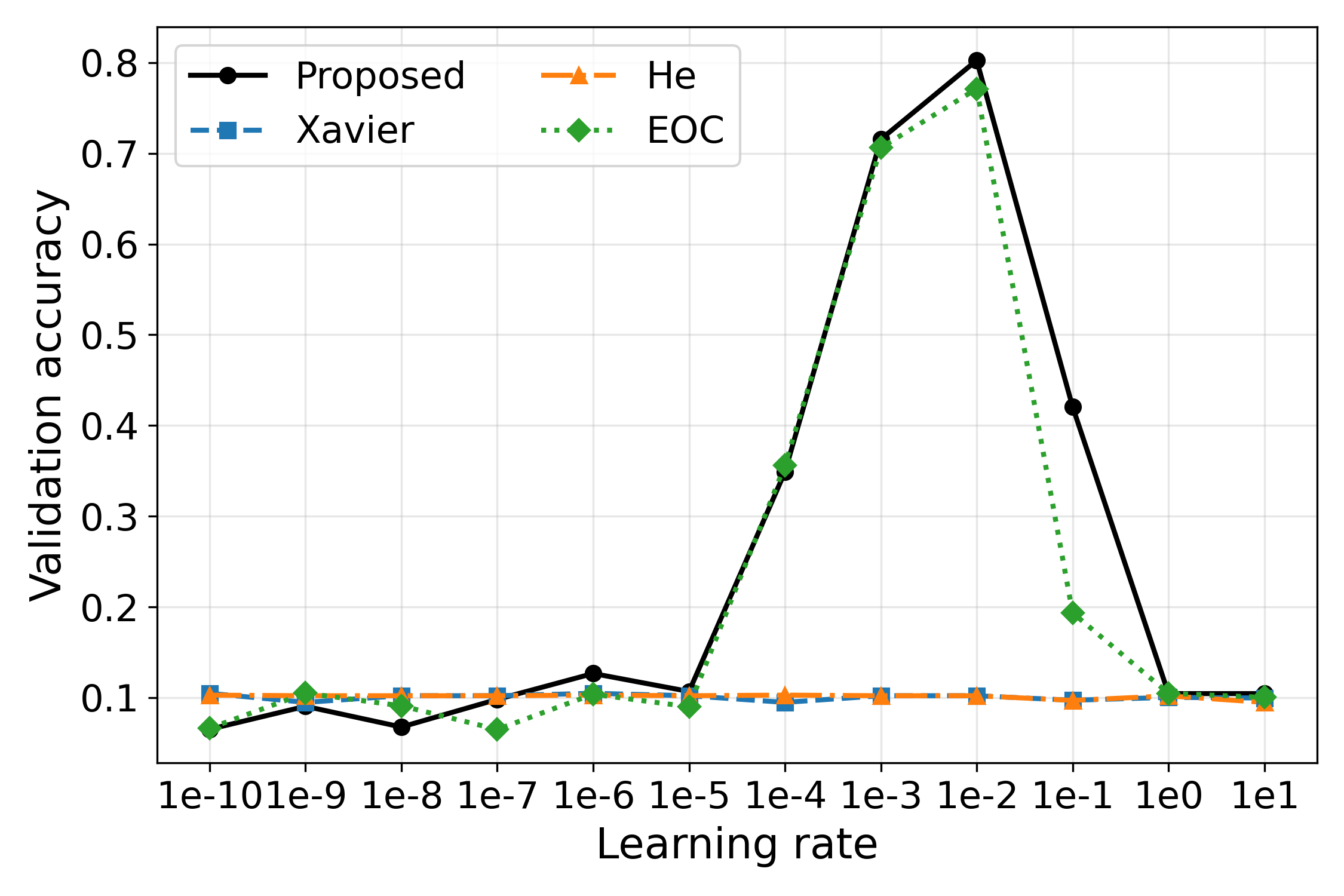}
    \caption{$2/\pi\arctan(0.01x)$}
\end{subfigure} &
\begin{subfigure}[b]{0.30\textwidth}
    \centering
    \includegraphics[width=\textwidth]{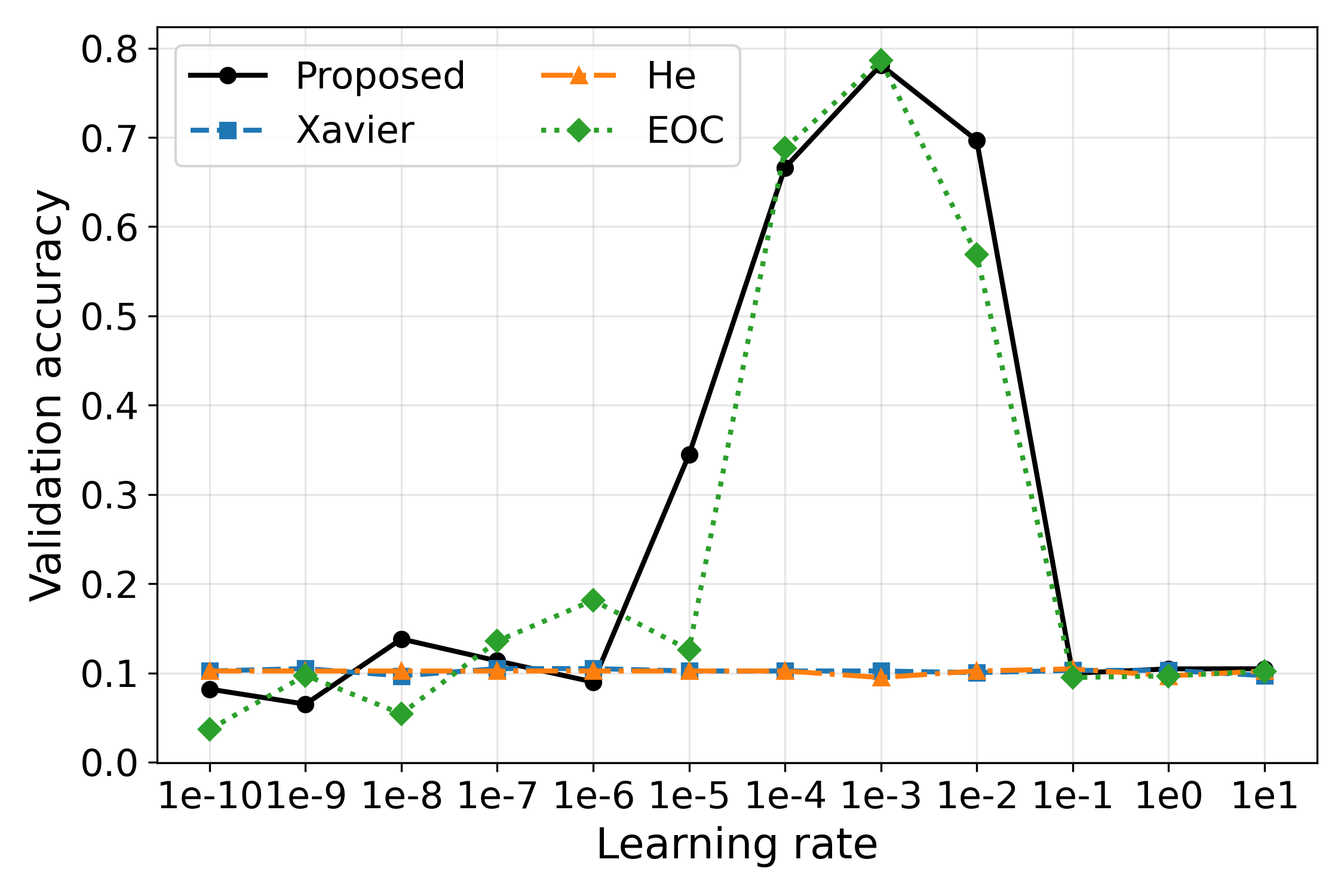}
    \caption{$2/\pi\arctan(0.1x)$}
\end{subfigure} &
\begin{subfigure}[b]{0.30\textwidth}
    \centering
    \includegraphics[width=\textwidth]{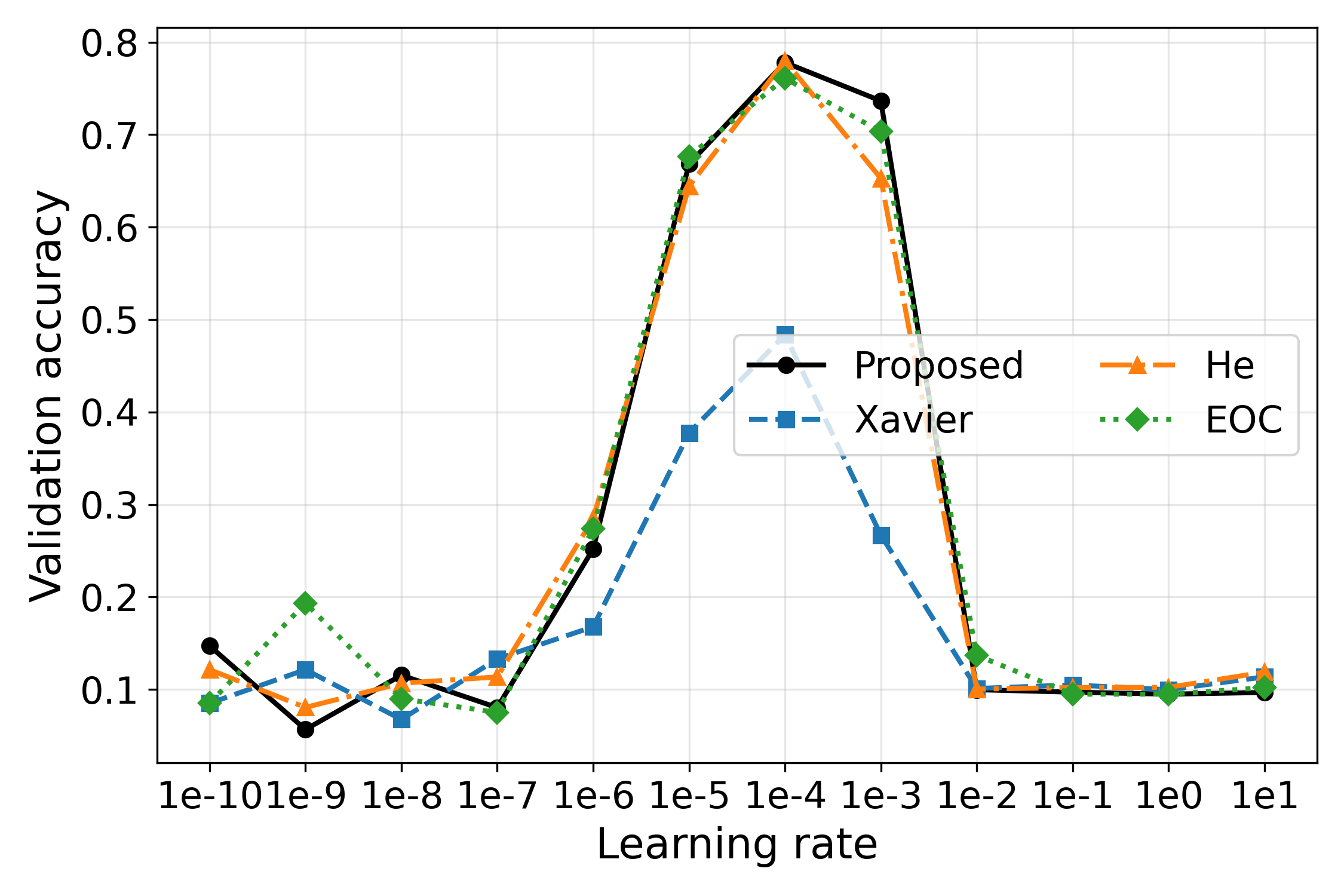}
    \caption{$2/\pi\arctan(x)$}
\end{subfigure} \\
\begin{subfigure}[b]{0.30\textwidth}
    \centering
    \includegraphics[width=\textwidth]{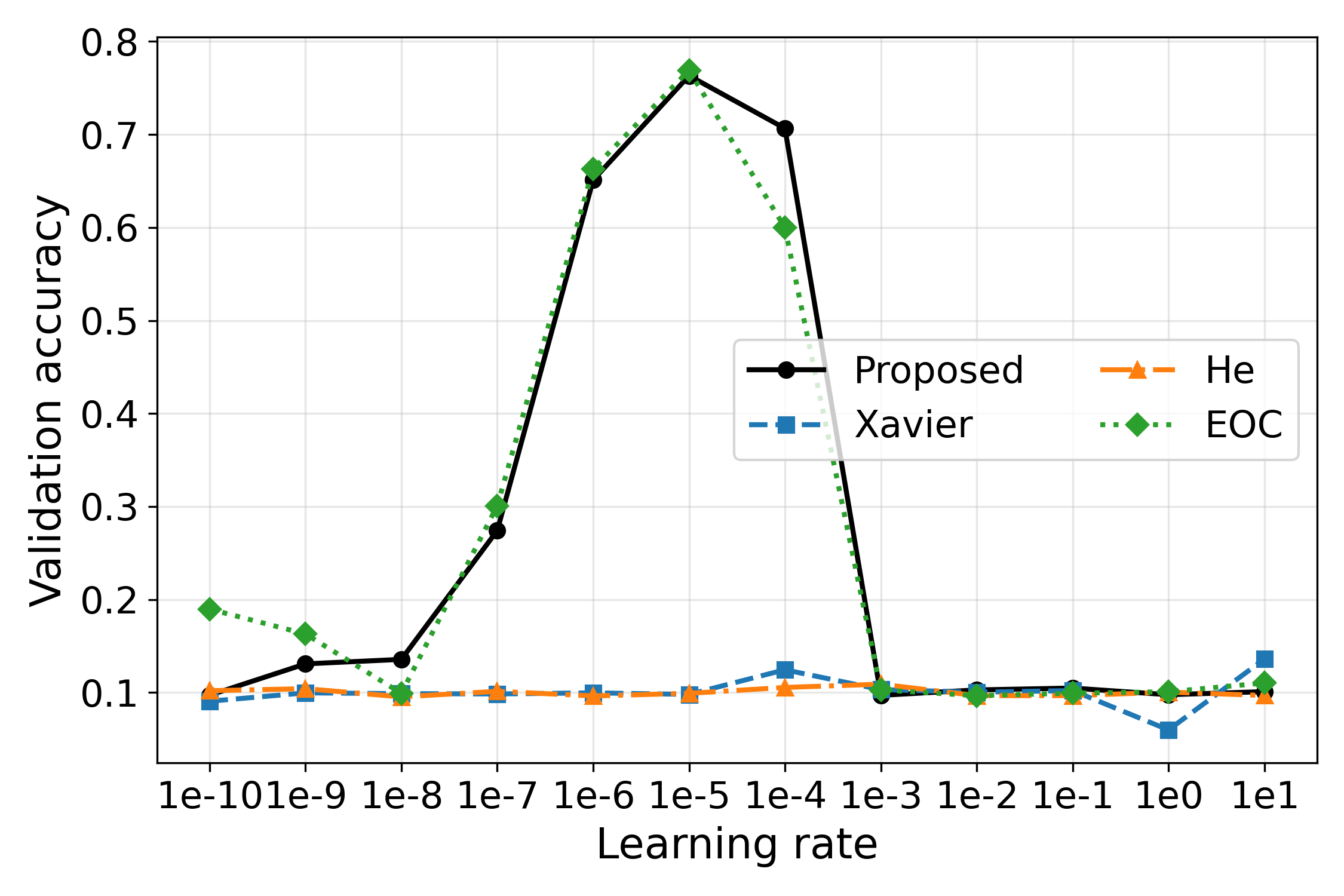}
    \caption{$2/\pi\arctan(10x)$}
\end{subfigure} &
\begin{subfigure}[b]{0.30\textwidth}
    \centering
    \includegraphics[width=\textwidth]{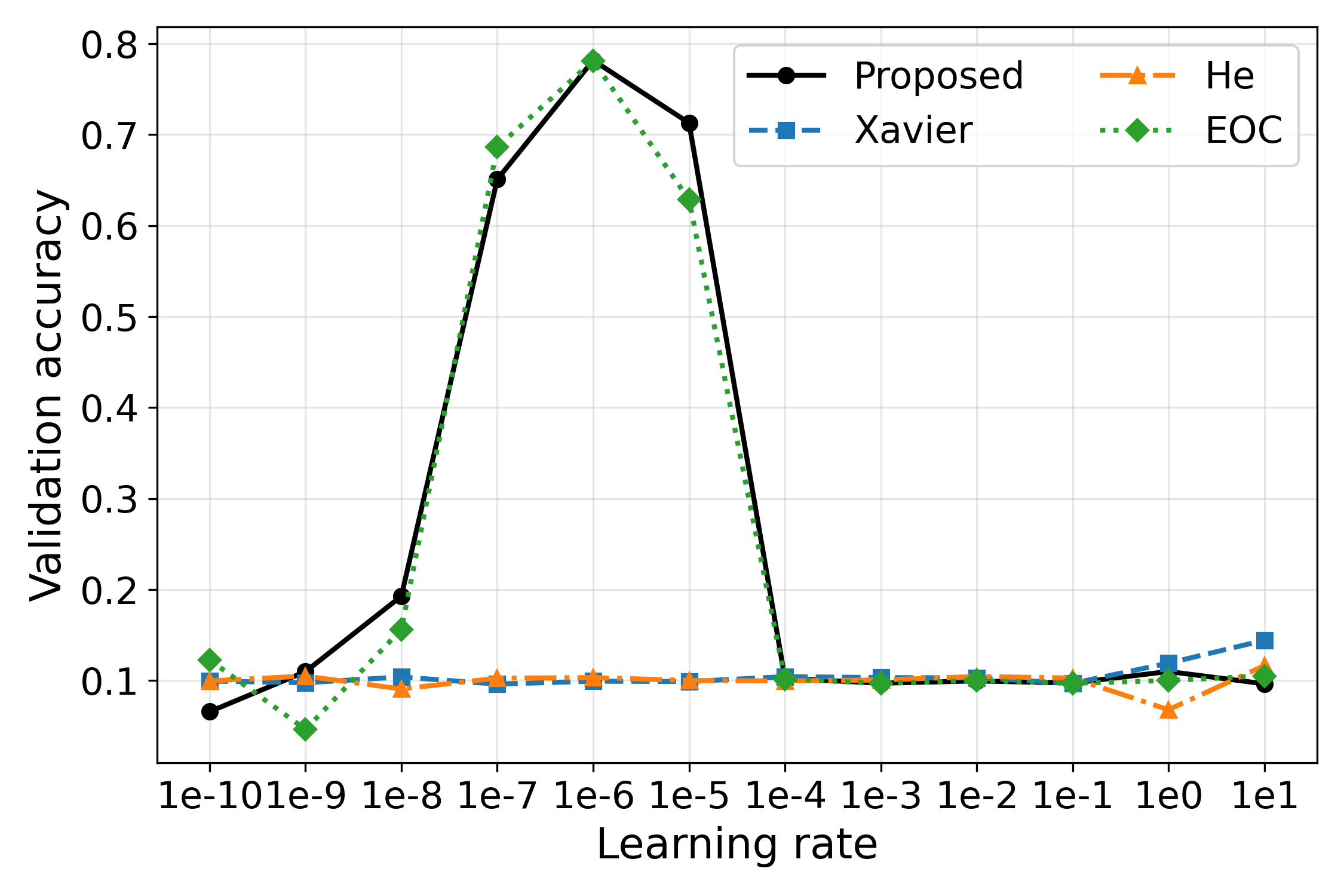}
    \caption{$2/\pi\arctan(100x)$}
\end{subfigure} &
\begin{subfigure}[b]{0.30\textwidth}
    \centering
    \includegraphics[width=\textwidth]{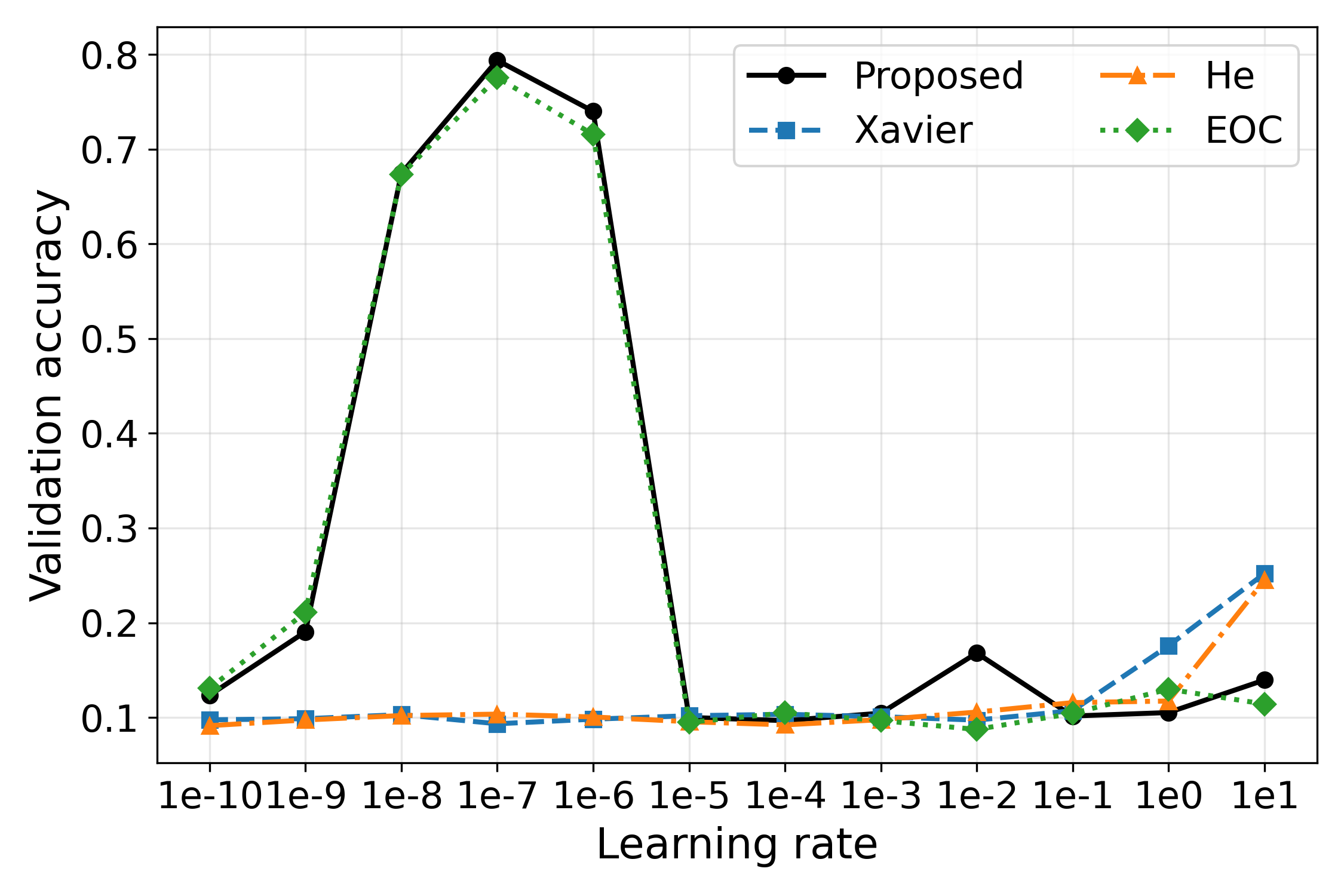}
    \caption{$2/\pi\arctan(1000x)$}
\end{subfigure} 
\end{tabular} 
  \caption{Learning rate accuracy curves on Fashion MNIST for a 20 layer, width 512 feedforward network
    with activation $f(x) =2/\pi\arctan(\alpha x)$.
    Each panel corresponds to a different activation scale
    $\alpha \in \{10^3, 10^2, 10^1, 1, 10^{-1}, 10^{-2}\}$.
    For each learning rate, we train for 200 iterations on a 10k training subset
    and report the validation accuracy.
    Curves compare four initializations: Proposed, Xavier, He, and EOC.
  }
\label{lrst4}
\end{figure}

\clearpage
\subsection{Scale Preserving Odd–Sigmoid Activations.}\label{scale_pinn}
Using a scaled activation function amounts to adjusting the effective range of both the input and output axes of the nonlinearity. For example, $\tanh(x)$ has output range $[-1,1]$, whereas $\alpha\tanh(x)$ has range $[-\alpha,\alpha]$. With our initialization we therefore ask whether signals remain well propagated over the full interval $[-\alpha,\alpha]$ when the activation is scaled. As shown in Figure~\ref{activations_range}, for several values of $\alpha$ the proposed initialization keeps the last layer activations of $\alpha\tanh(x)$ spread over $[-\alpha,\alpha]$ even in very deep networks (up to $L=10^5$), while Gaussian i.i.d.\ initialization rapidly drives the activations to saturate near zero. 
In all $\alpha$-scale experiments (Tables \ref{scale_table1}–\ref{scale_table6}), each initializer was given its own LR grid search. Even under this per initializer tuning, Gaussian and EOC schemes fail to train for large $\alpha$, while our method remains stable.

Since the proposed initialization lets us target the output scale via the choice of $\alpha$, it is natural to expect benefits on regression tasks where the range of the target $y$ matters. To test this, we evaluate our method in physics informed neural networks (PINNs), using scaled odd–sigmoid activations for Burgers’ equation and the Black–Scholes equation. The PINN setup and the precise PDE formulations are described in the following subsections.

\medskip

\begin{figure}[h!]
\centering 
\begin{tabular}{cccc}
\begin{subfigure}[b]{0.47\textwidth}
    \centering
    \includegraphics[width=\textwidth]{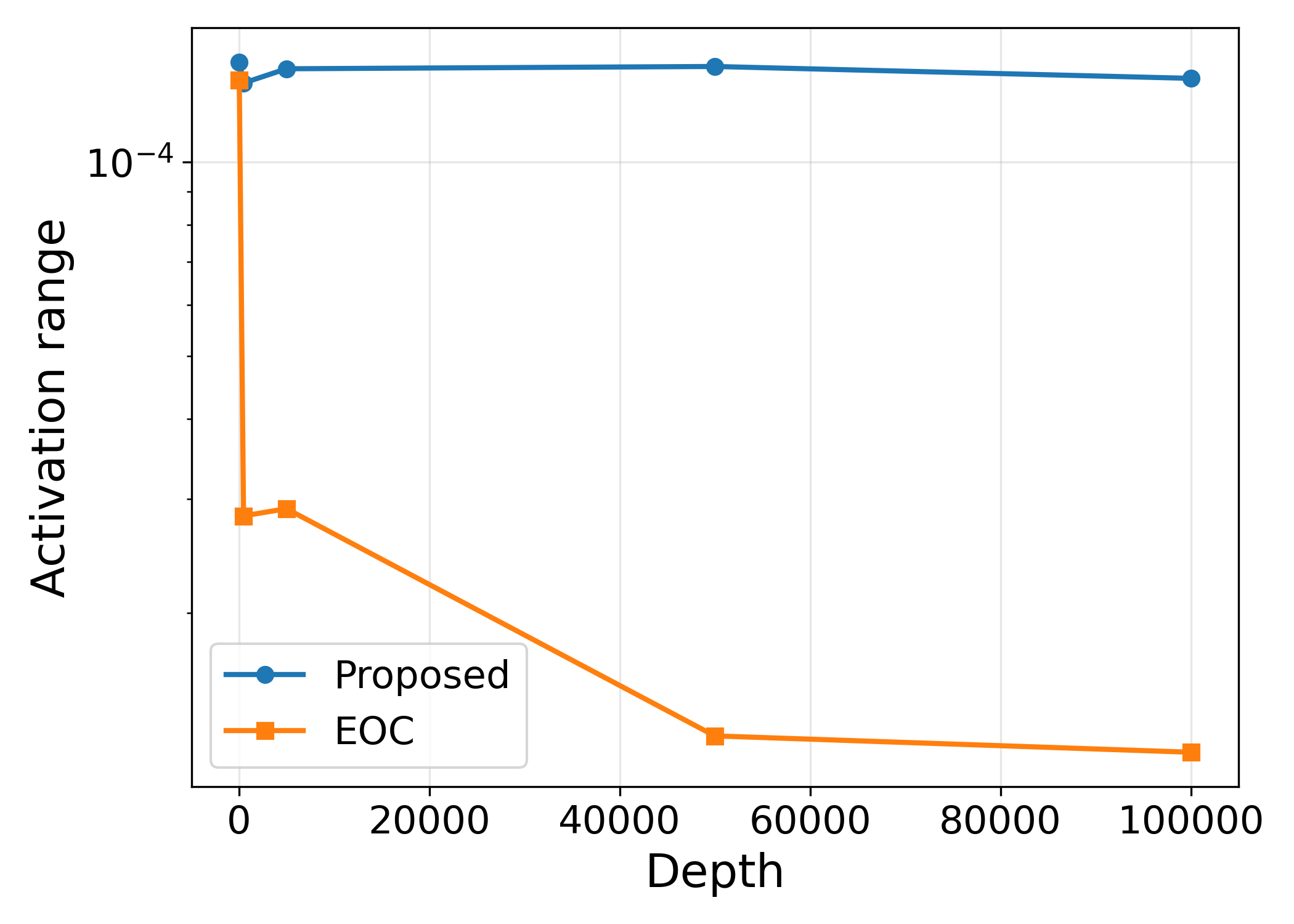}
    \caption{$\alpha=0.0001$}
\end{subfigure} &
\begin{subfigure}[b]{0.47\textwidth}
    \centering
    \includegraphics[width=\textwidth]{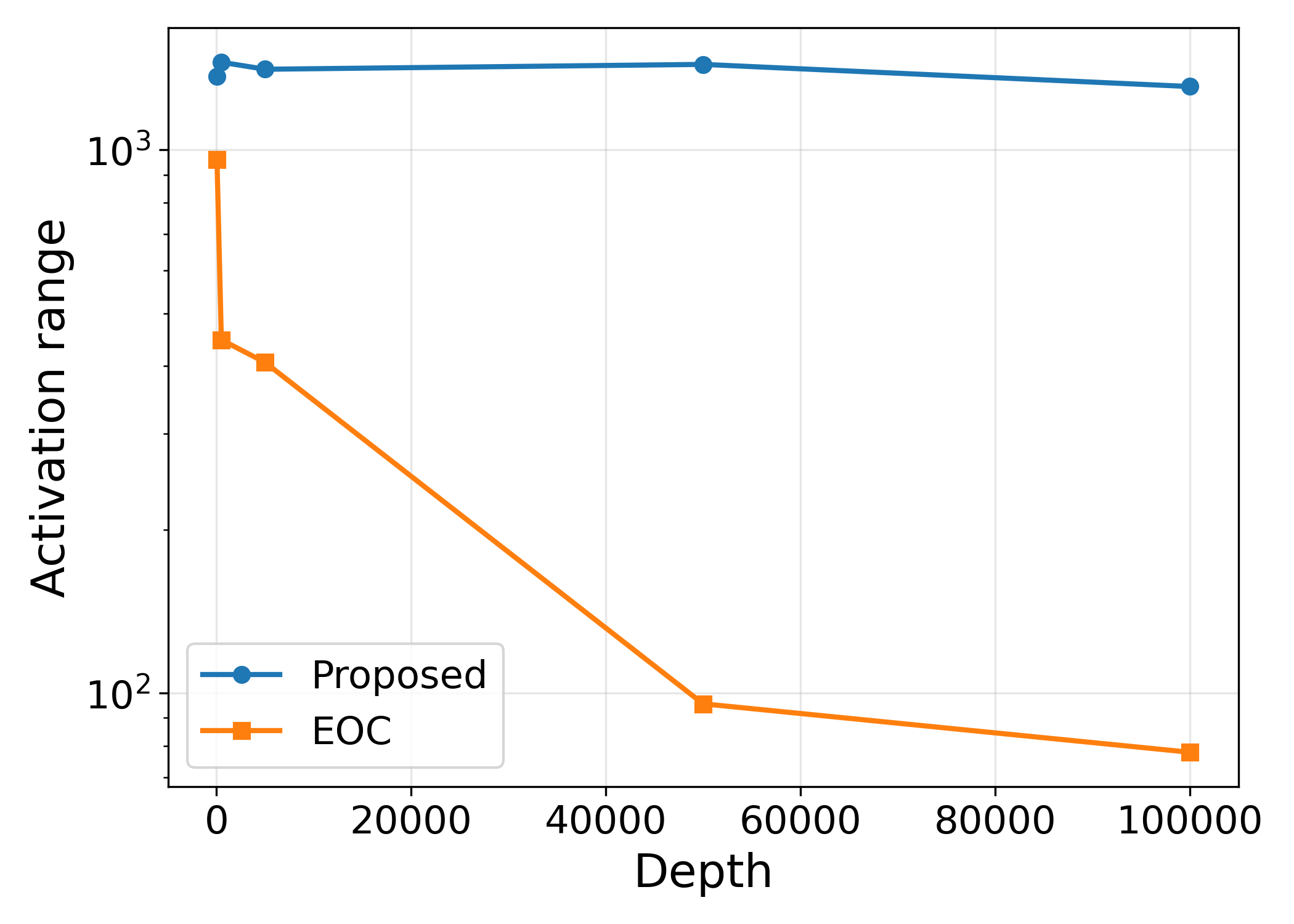}
    \caption{$\alpha=1000$}
\end{subfigure} 
\end{tabular} 
\caption{Activation range as a function of depth $L \in \{50, 500, 5{,}000, 50{,}000, 10^{5}\}$ for fully connected width 64 networks with activation $\alpha\tanh(x)$ under the Proposed and EOC initializations. Panel (a) uses $\alpha = 10^{-4}$ and panel (b) uses $\alpha = 10^{3}$.}
\label{activations_range}
\end{figure}

\begin{table}[h!]
  \centering
  \caption{Validation accuracy on MNIST~(left) and Fashion MNIST~(right) for a
  50 layer, width 128 fully connected neural network with activation $a\tanh(x)$.
  Each row corresponds to a different activation scale $a$, and for every $a$
  the learning rate is set to $\eta = 10^{-4}/a$ for both initializations.}  \label{tab:a_tanh_accuracy1}
  \small
  \setlength{\tabcolsep}{6pt}
  \renewcommand{\arraystretch}{1.2}

  \begin{tabular}{|C{1.5cm}|c|c|}
    \hline
    $a$ & MNIST (accuracy vs.\ epoch) & Fashion-MNIST (accuracy vs.\ epoch) \\
    \hline
    $10^{-2}$ 
    & \includegraphics[valign=m,width=0.28\textwidth]{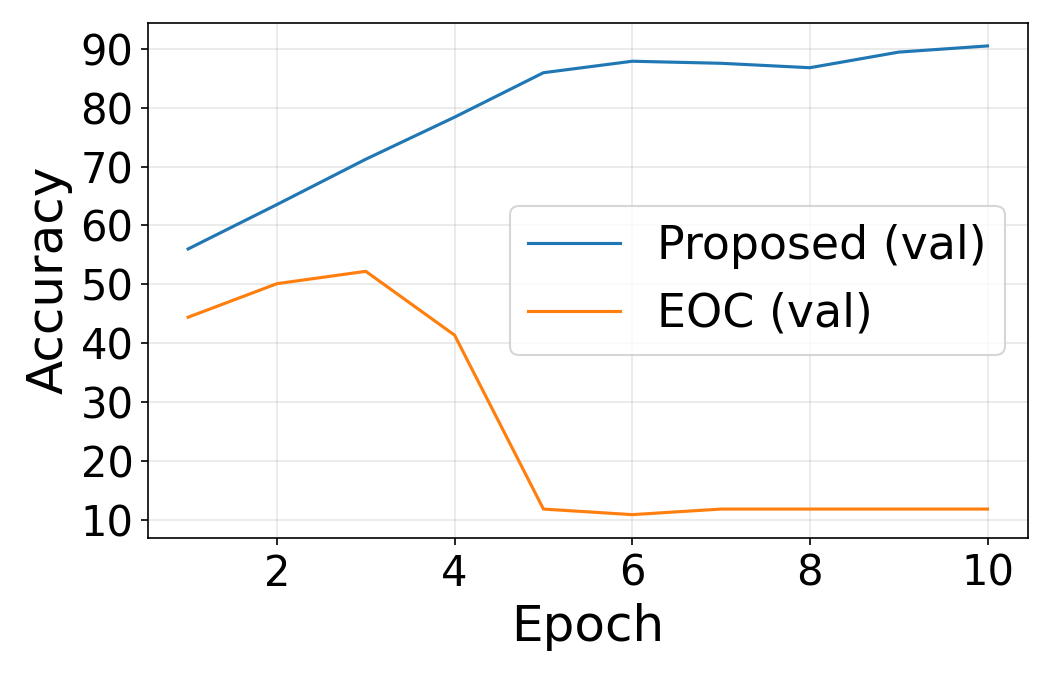}
    & \includegraphics[valign=m,width=0.28\textwidth]{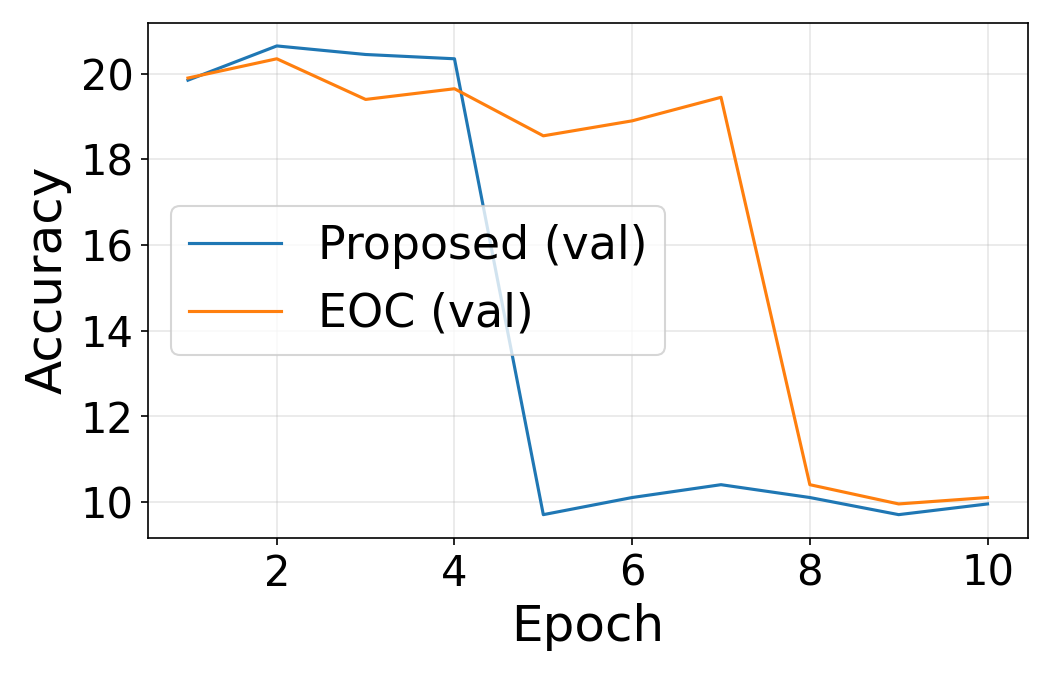}\\
    \hline
    $10^{-1}$
    & \includegraphics[valign=m,width=0.28\textwidth]{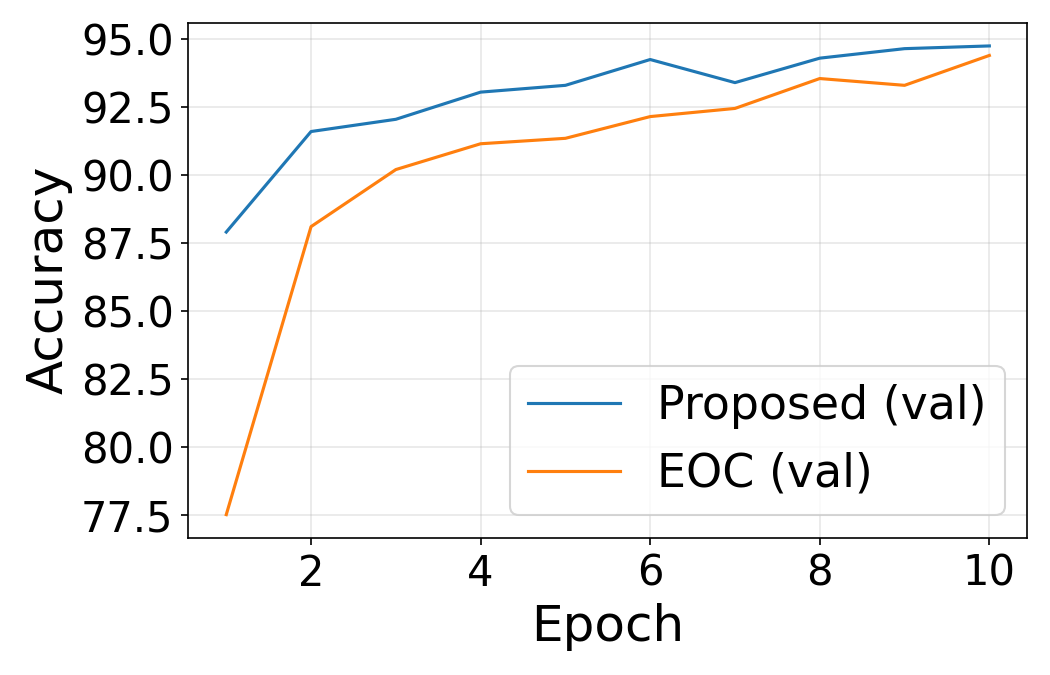}
    & \includegraphics[valign=m,width=0.28\textwidth]{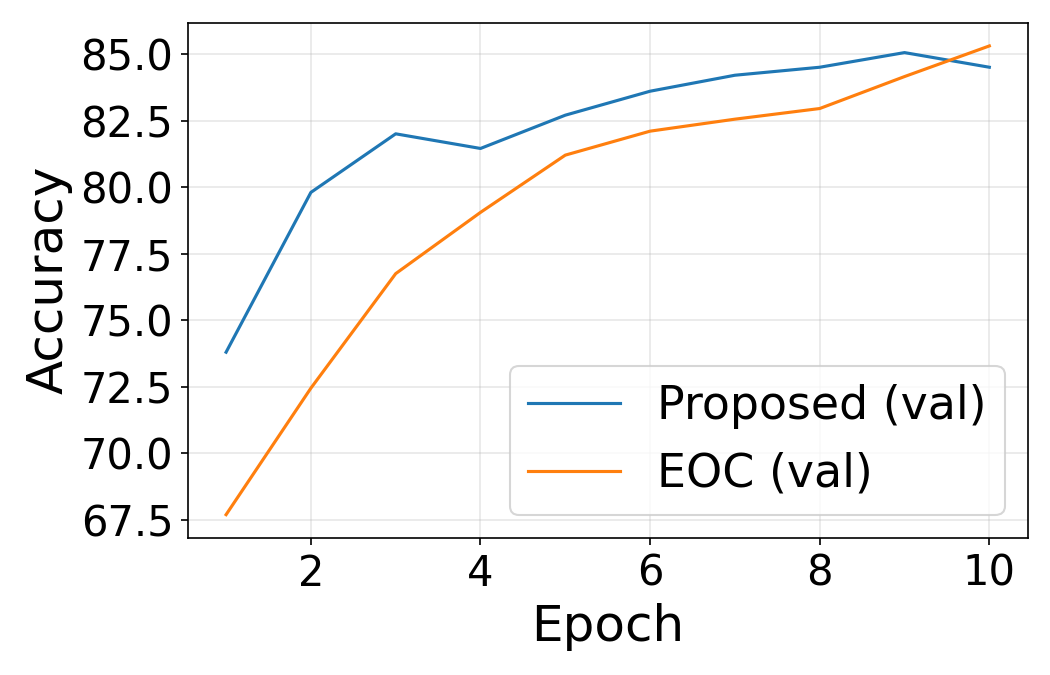}\\
    \hline
    1
    & \includegraphics[valign=m,width=0.28\textwidth]{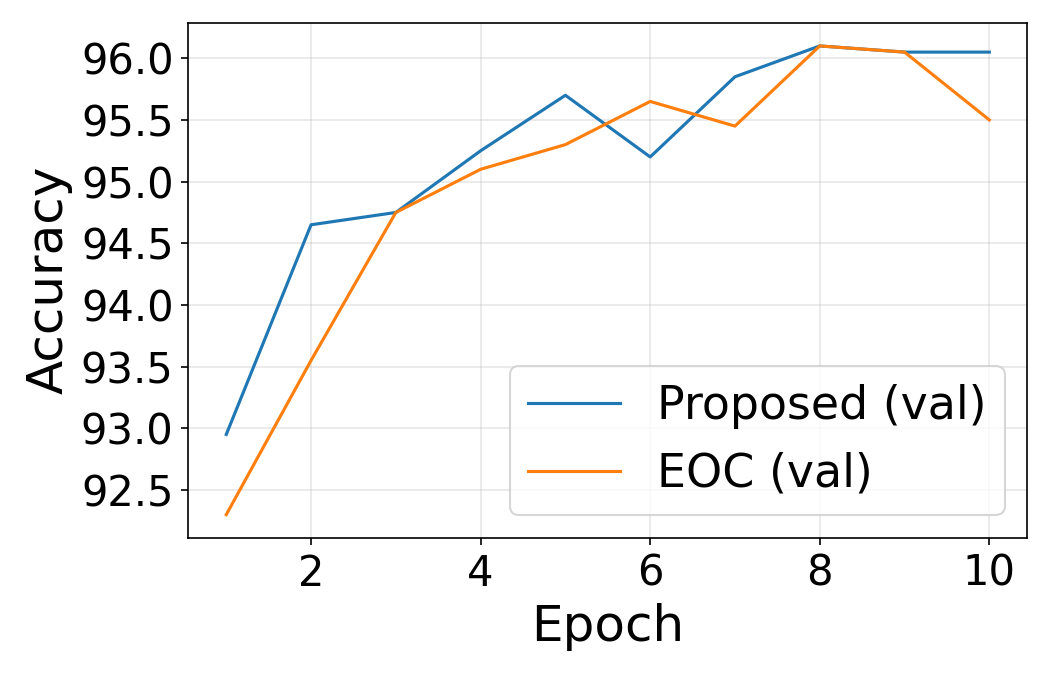}
    & \includegraphics[valign=m,width=0.28\textwidth]{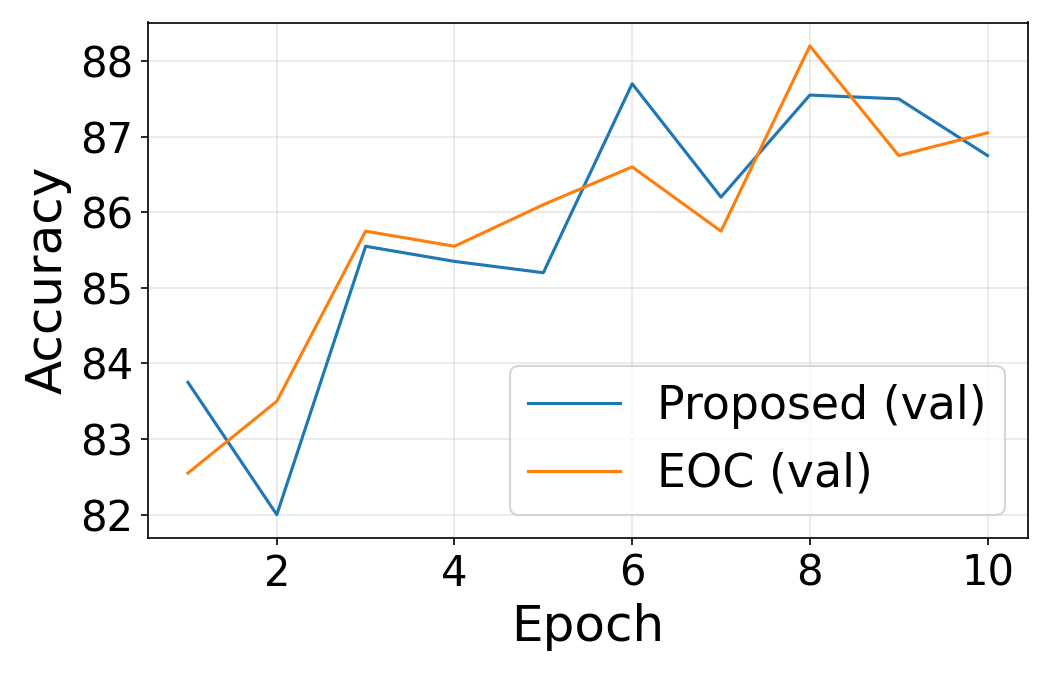}\\
    \hline
    $10^{1}$
    & \includegraphics[valign=m,width=0.28\textwidth]{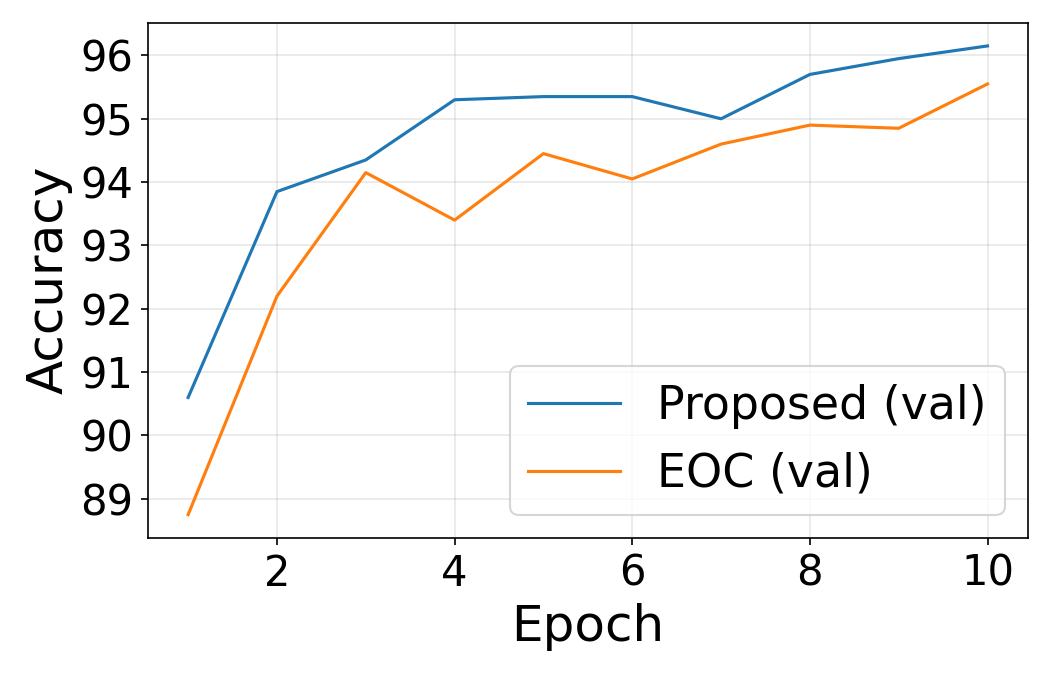}
    & \includegraphics[valign=m,width=0.28\textwidth]{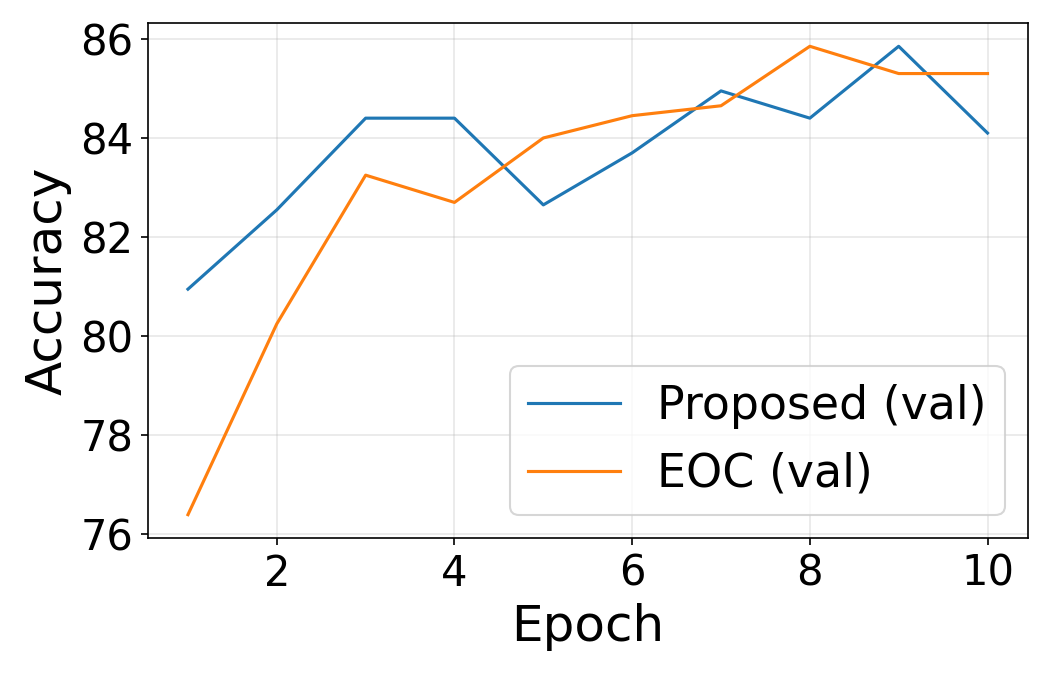}\\
    \hline
    $10^{2}$
    & \includegraphics[valign=m,width=0.28\textwidth]{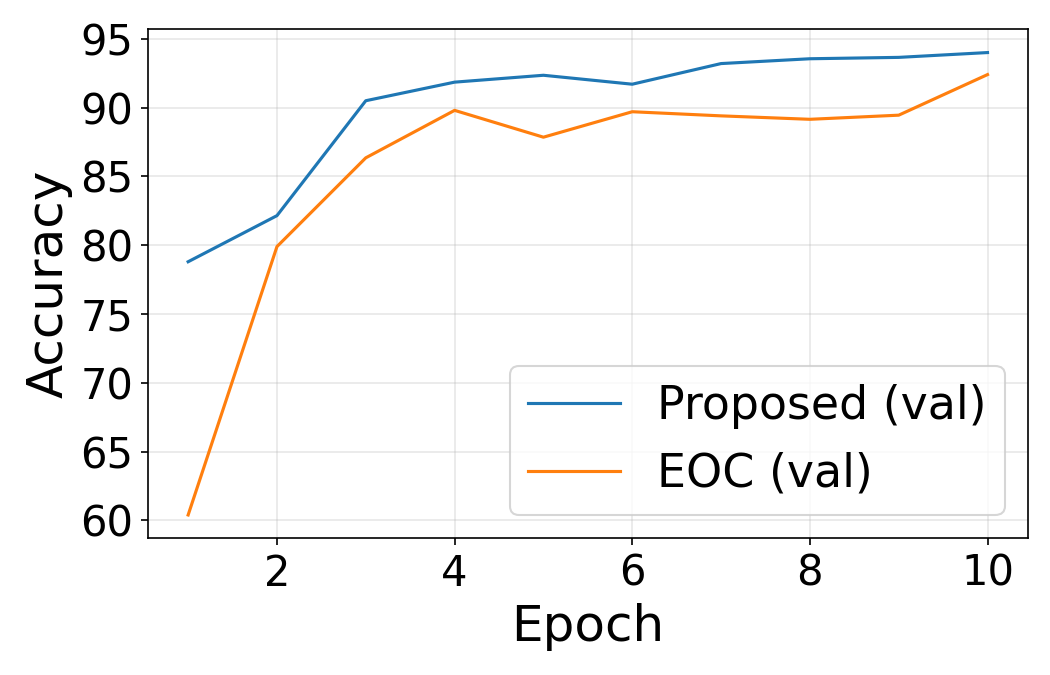}
    & \includegraphics[valign=m,width=0.28\textwidth]{figure/fMNIST_a1e+01_LR1e-05_W128_L50.png}\\
    \hline
    $10^{3}$
    & \includegraphics[valign=m,width=0.28\textwidth]{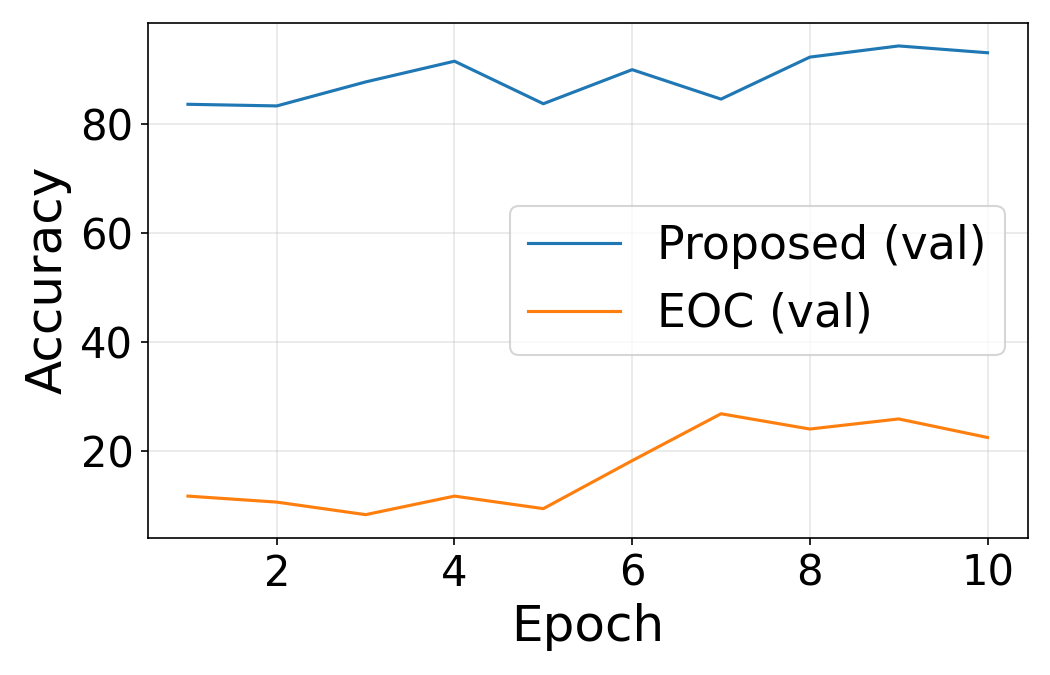}
    & \includegraphics[valign=m,width=0.28\textwidth]{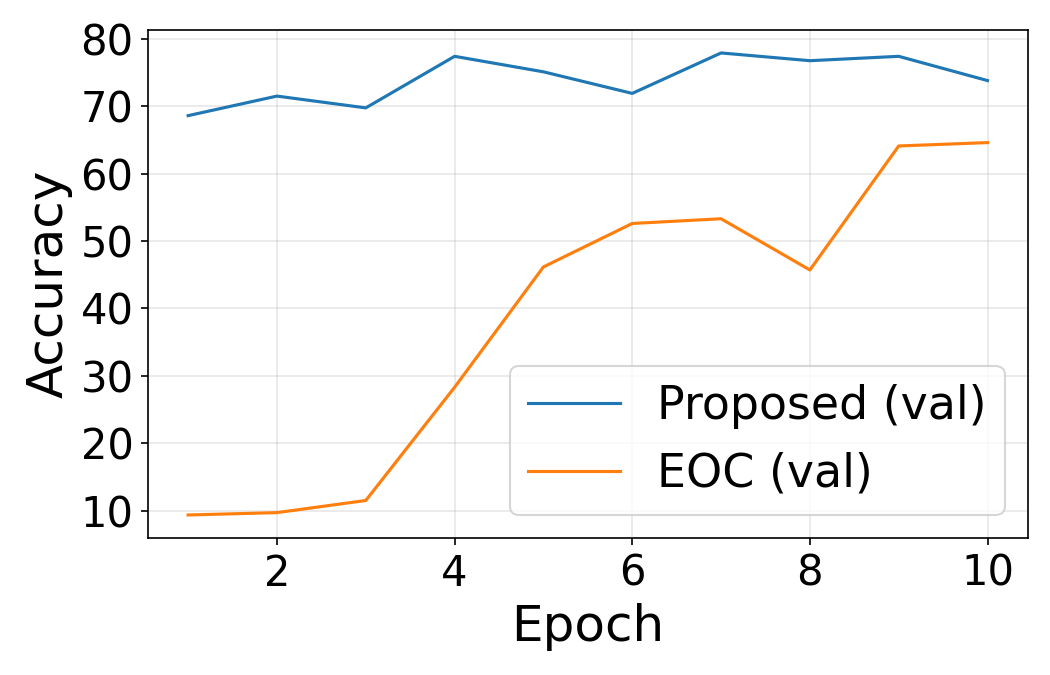}\\
    \hline
  \end{tabular}
\label{scale_table1}
\end{table}

\begin{table}[t]
  \centering
  \caption{Validation accuracy on MNIST~(left) and Fashion MNIST~(right) for a
  50 layer, width 128 fully connected neural network with activation $a\tanh(x)$.
  Each row corresponds to a different activation scale $a$, and for every $a$
  the learning rate is set to $\eta = 10^{-4}/a$ for both initializations.} \label{tab:a_tanh_accuracy2}

  \small
  \setlength{\tabcolsep}{6pt}
  \renewcommand{\arraystretch}{1.2}

  \begin{tabular}{|C{1.5cm}|c|c|}
    \hline
    $a$ & MNIST (accuracy vs.\ epoch) & Fashion-MNIST (accuracy vs.\ epoch) \\
    \hline
    $10^{4}$
    & \includegraphics[valign=m,width=0.28\textwidth]{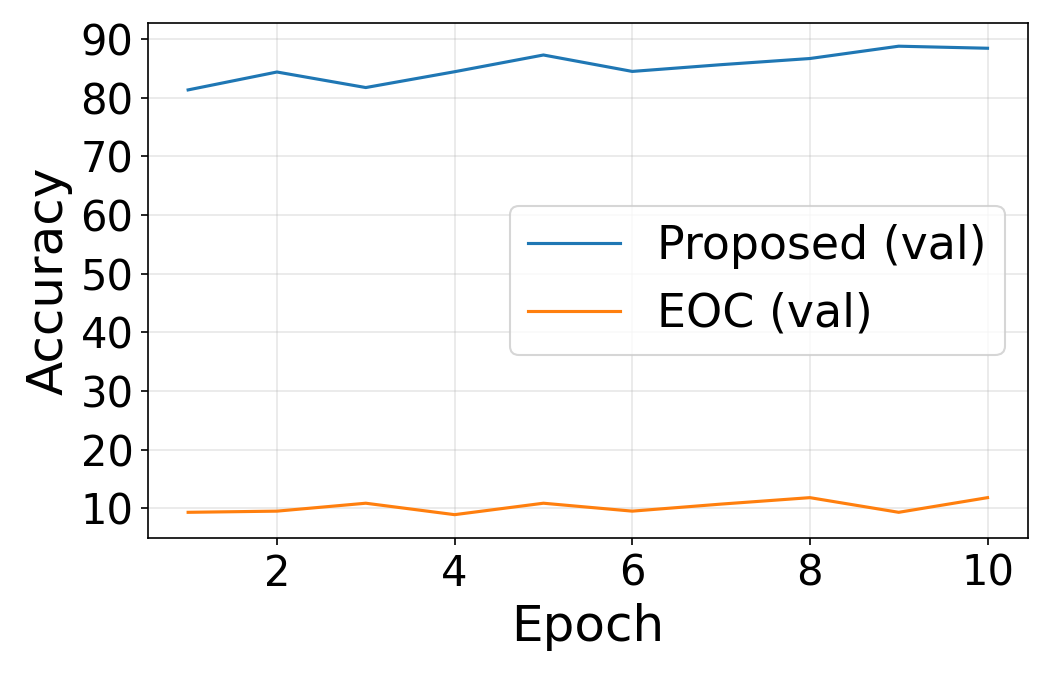}
    & \includegraphics[valign=m,width=0.28\textwidth]{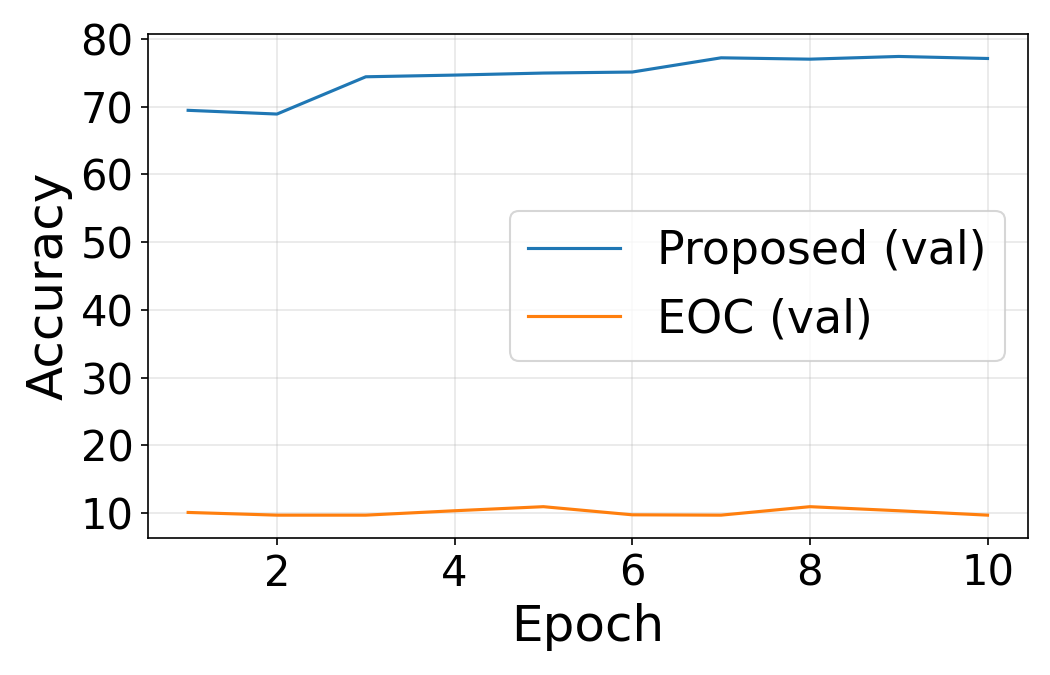}\\
    \hline
    $10^{5}$ 
    & \includegraphics[valign=m,width=0.28\textwidth]{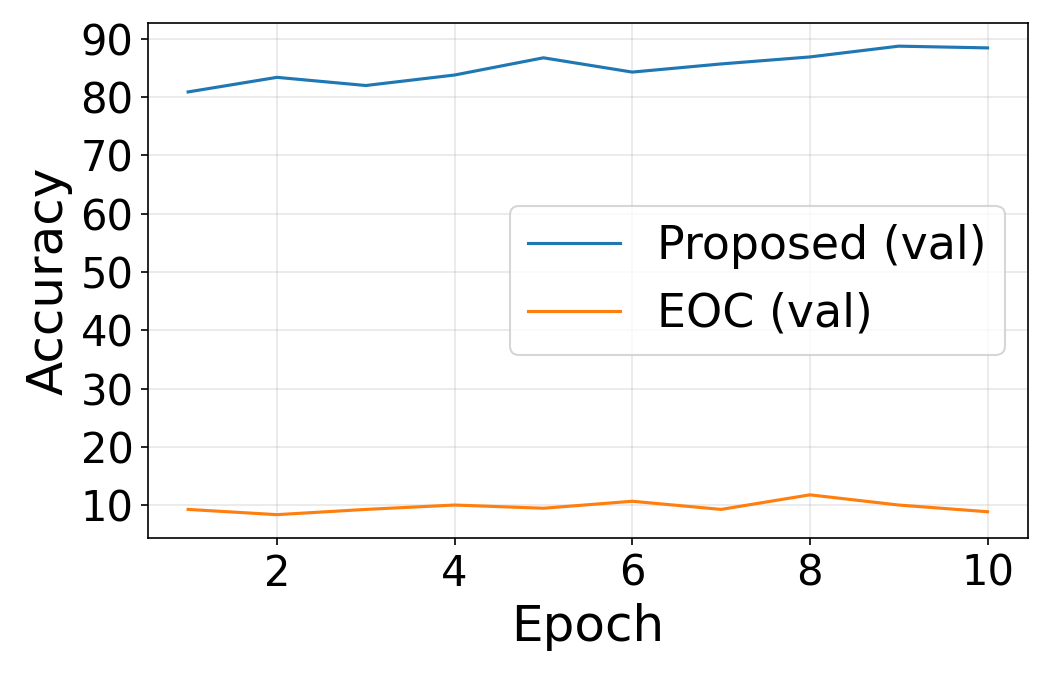}
    & \includegraphics[valign=m,width=0.28\textwidth]{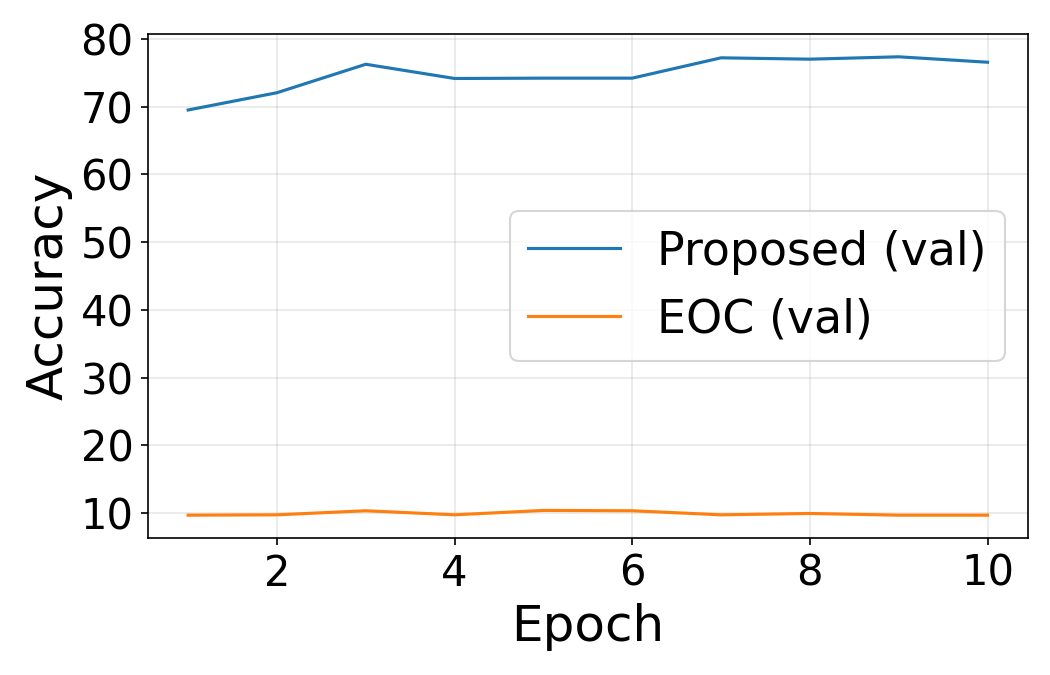}\\
    \hline
    $10^{6}$
    & \includegraphics[valign=m,width=0.28\textwidth]{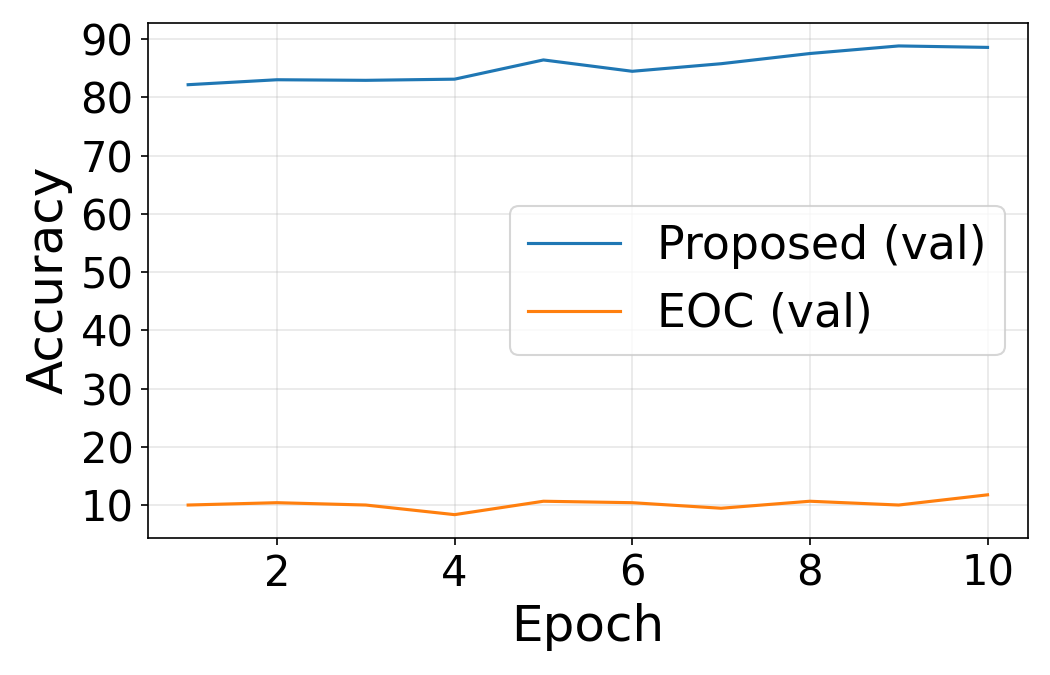}
    & \includegraphics[valign=m,width=0.28\textwidth]{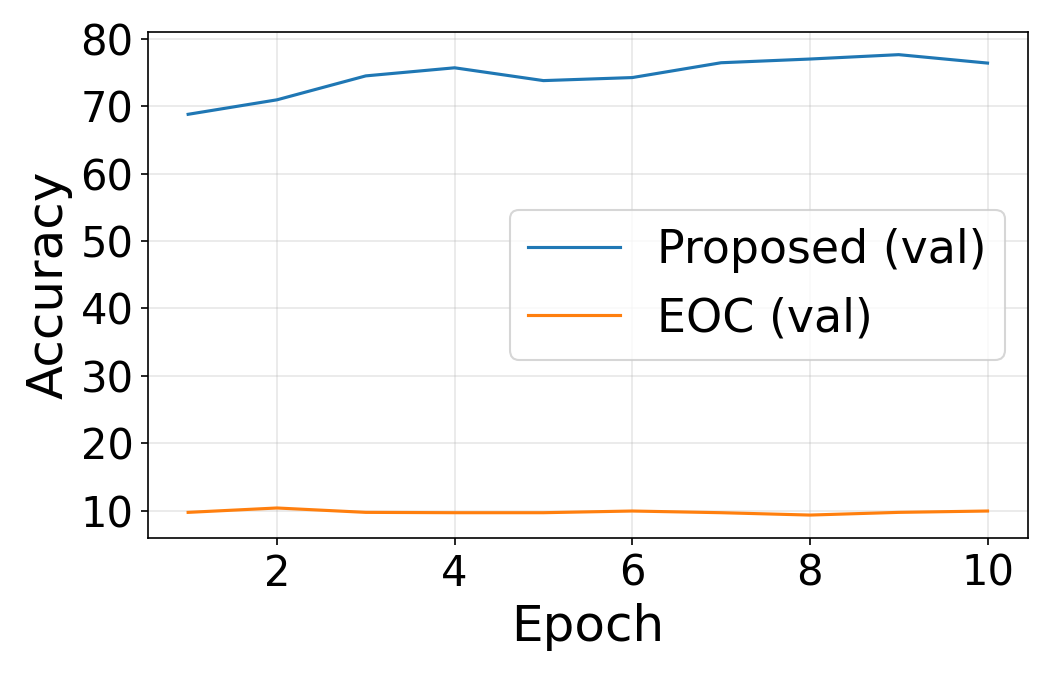}\\
    \hline
    $10^{7}$
    & \includegraphics[valign=m,width=0.28\textwidth]{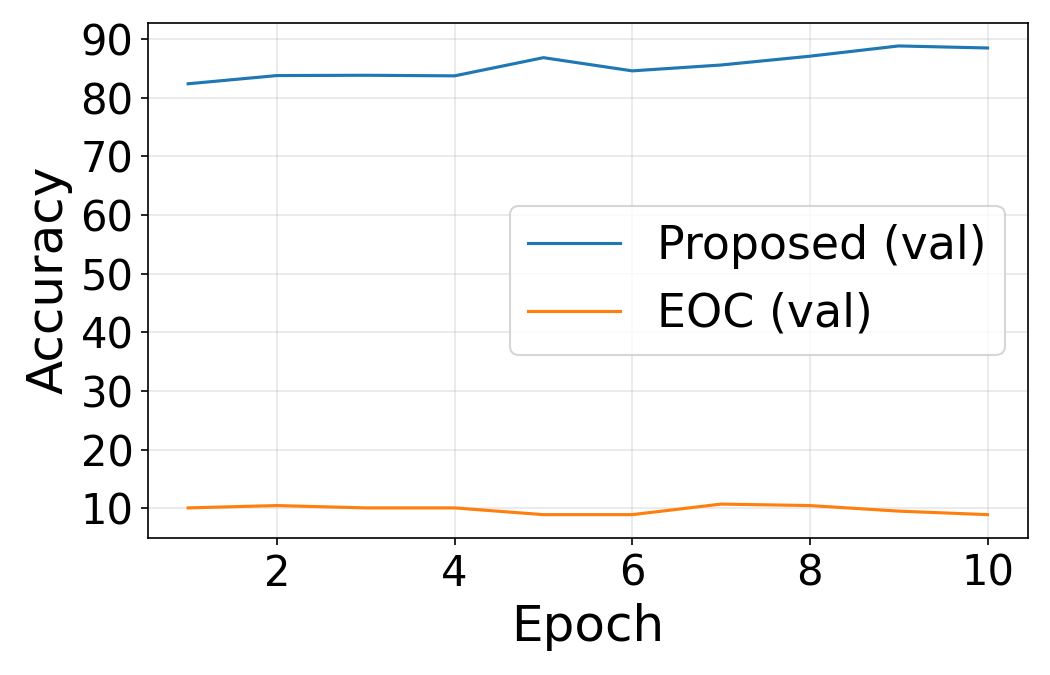}
    & \includegraphics[valign=m,width=0.28\textwidth]{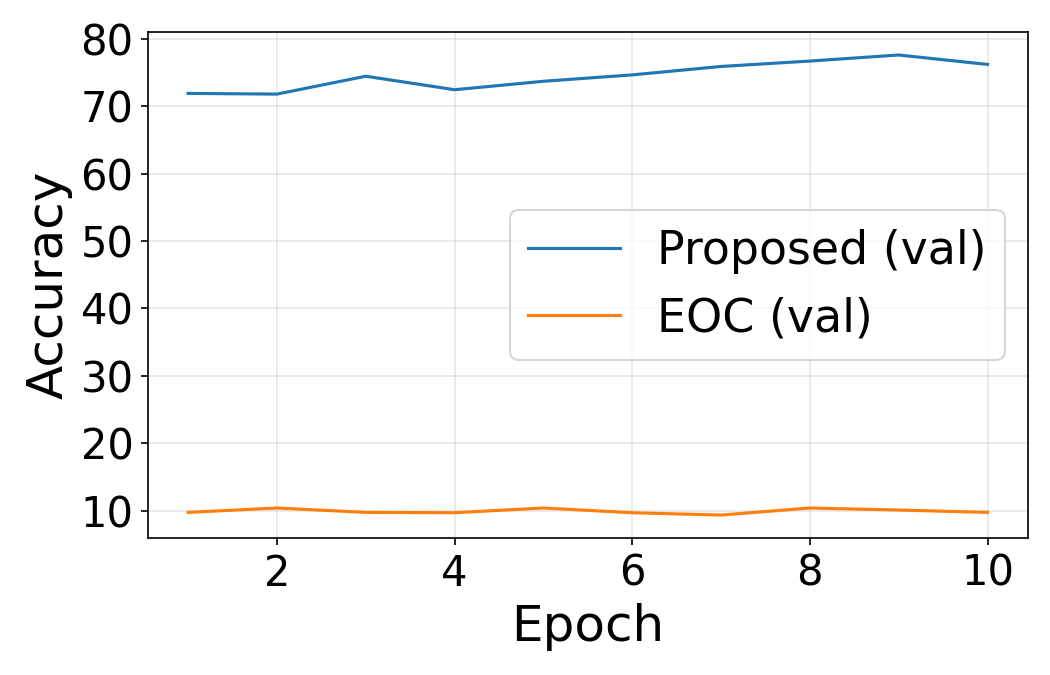}\\
    \hline
    $10^{8}$
    & \includegraphics[valign=m,width=0.28\textwidth]{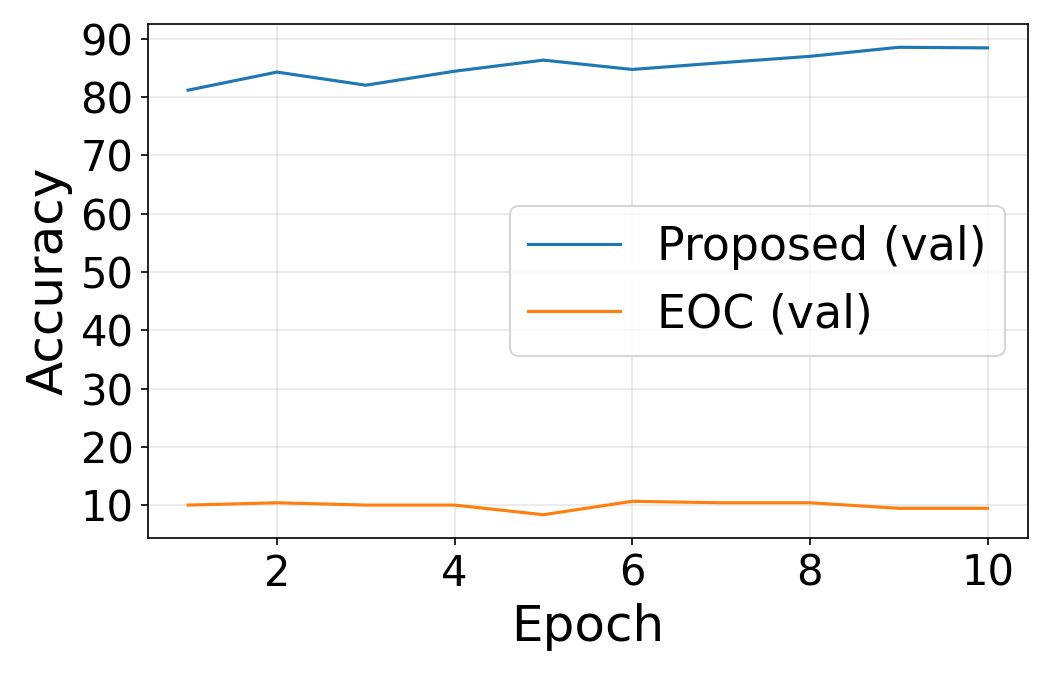}
    & \includegraphics[valign=m,width=0.28\textwidth]{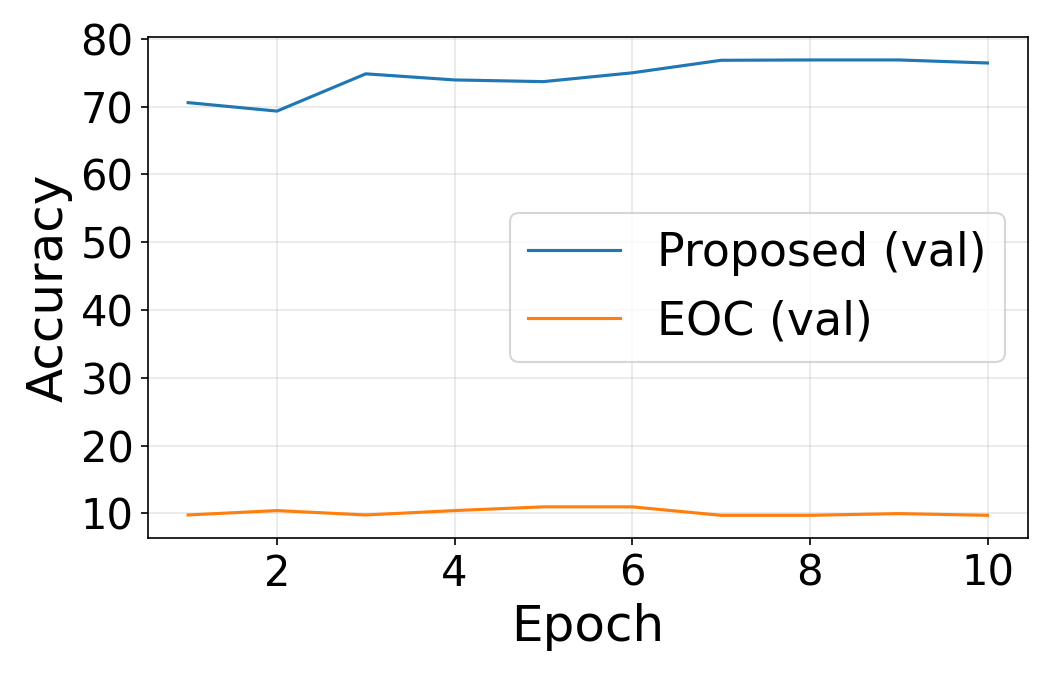}\\
    \hline
    $10^{9}$
    & \includegraphics[valign=m,width=0.28\textwidth]{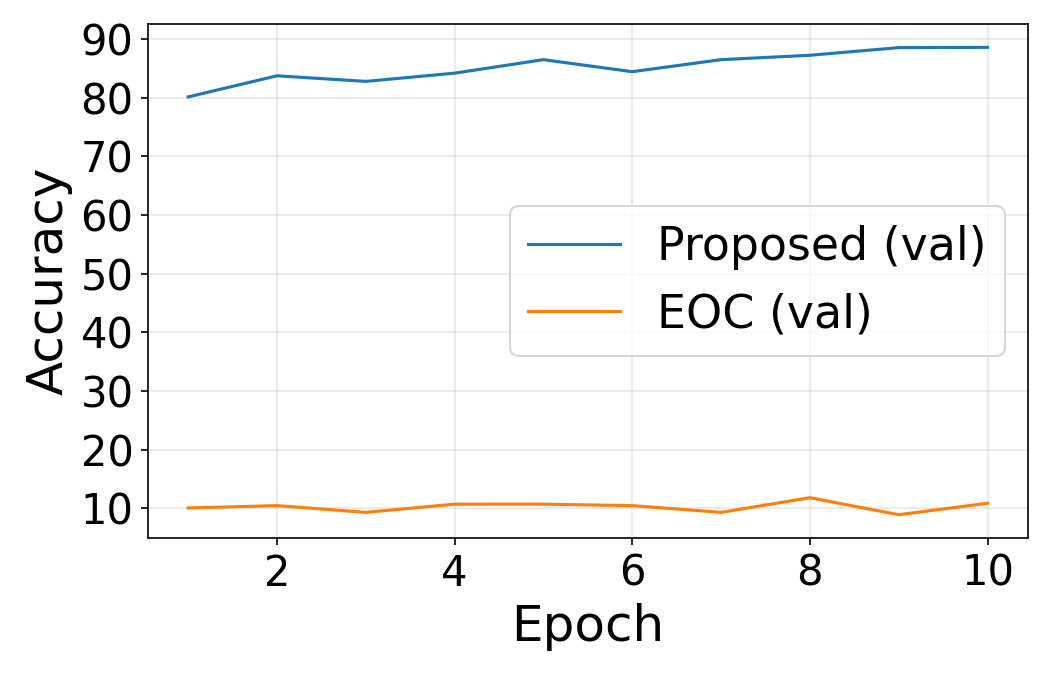}
    & \includegraphics[valign=m,width=0.28\textwidth]{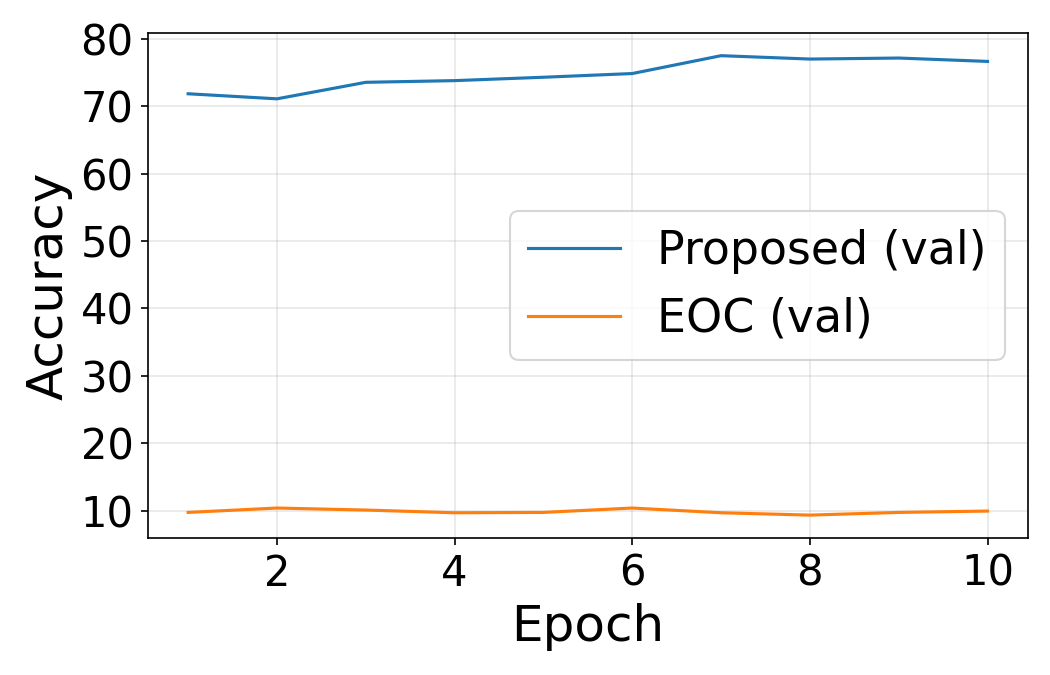}\\

    \hline
  \end{tabular}
  \label{scale_table2}

\end{table}

\begin{table}[t]
  \centering
  \caption{Validation accuracy on MNIST~(left) and Fashion MNIST~(right) for a
  50 layer, width 128 fully connected neural network with activation $a\arctan(x)$.
  Each row corresponds to a different activation scale $a$, and for every $a$
  the learning rate is set to $\eta = 10^{-4}/a$ for both initializations.}  \label{tab:a_atanh_accuracy1}
  \small
  \setlength{\tabcolsep}{6pt}
  \renewcommand{\arraystretch}{1.2}

  \begin{tabular}{|C{1.5cm}|c|c|}
    \hline
    $a$ & MNIST (accuracy vs.\ epoch) & Fashion-MNIST (accuracy vs.\ epoch) \\
    \hline
    $10^{-2}$
    & \includegraphics[valign=m,width=0.28\textwidth]{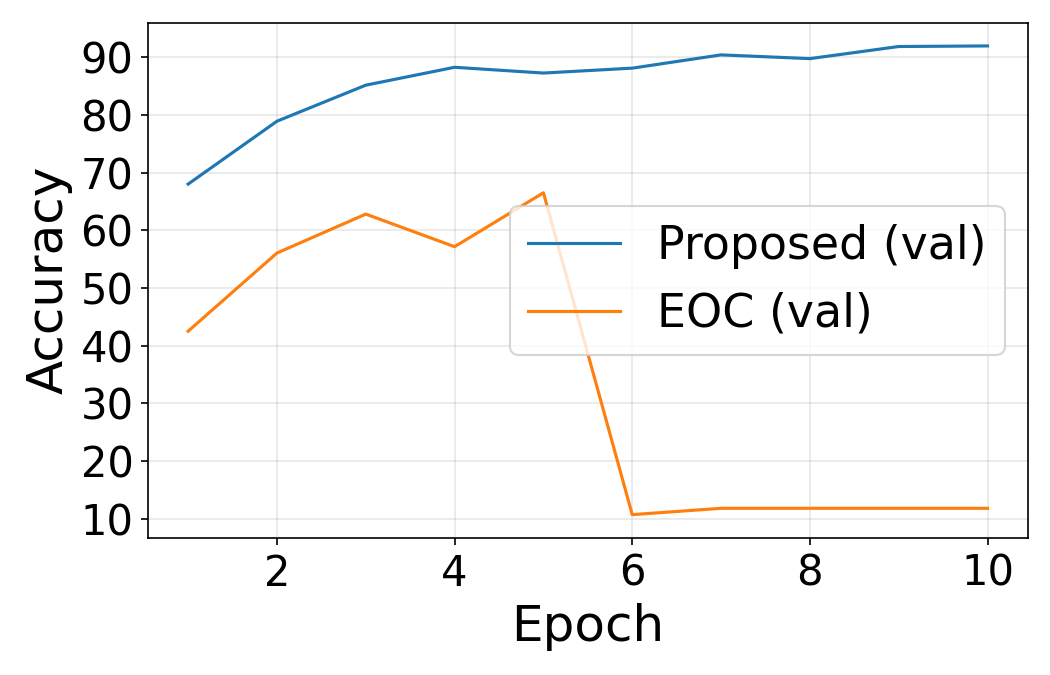}
    & \includegraphics[valign=m,width=0.28\textwidth]{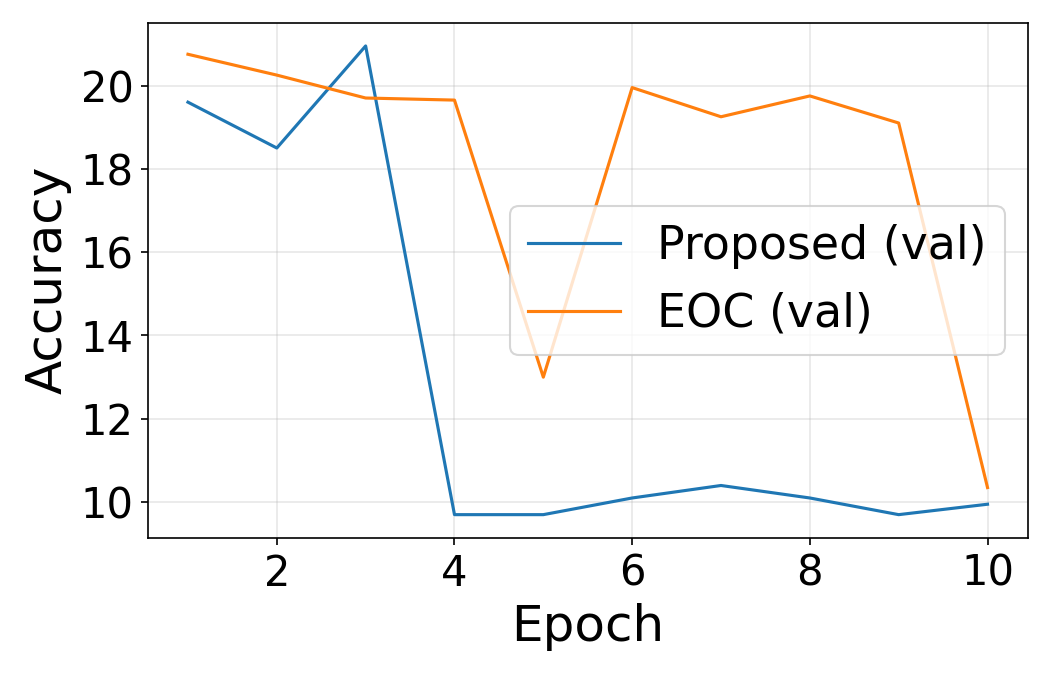}\\
    \hline
    $10^{-1}$
    & \includegraphics[valign=m,width=0.28\textwidth]{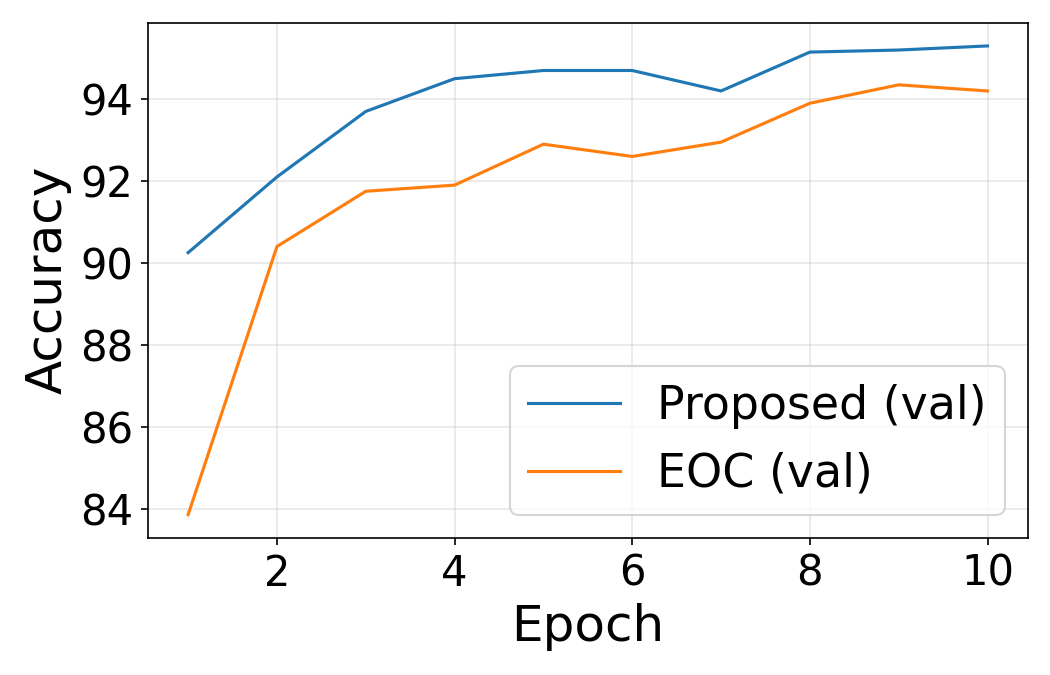}
    & \includegraphics[valign=m,width=0.28\textwidth]{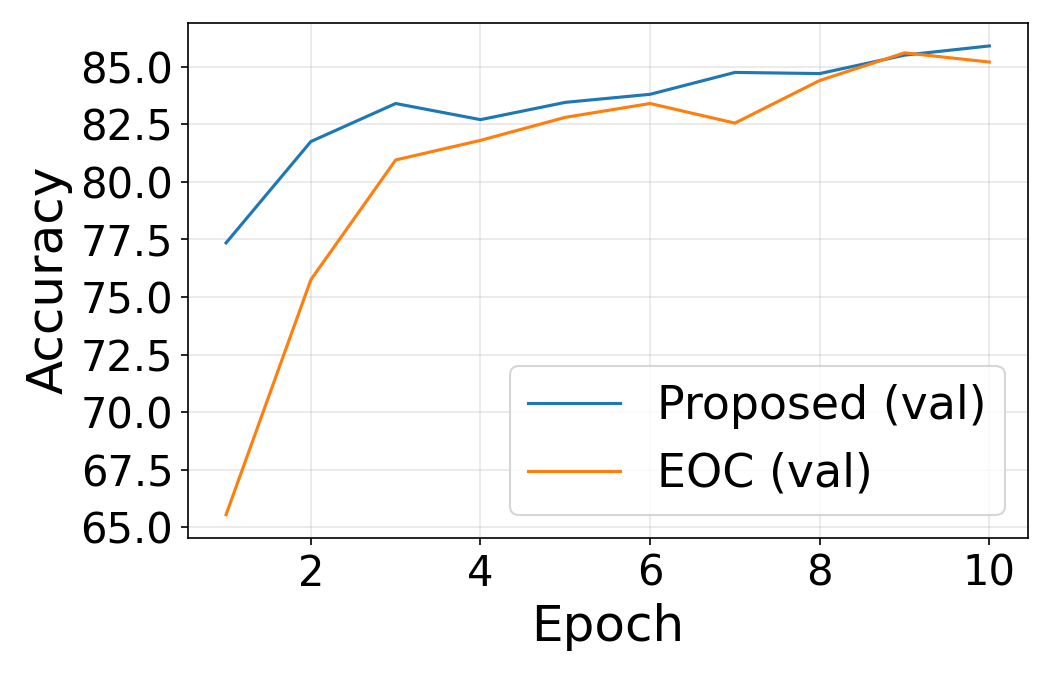}\\
    \hline
    1
    & \includegraphics[valign=m,width=0.28\textwidth]{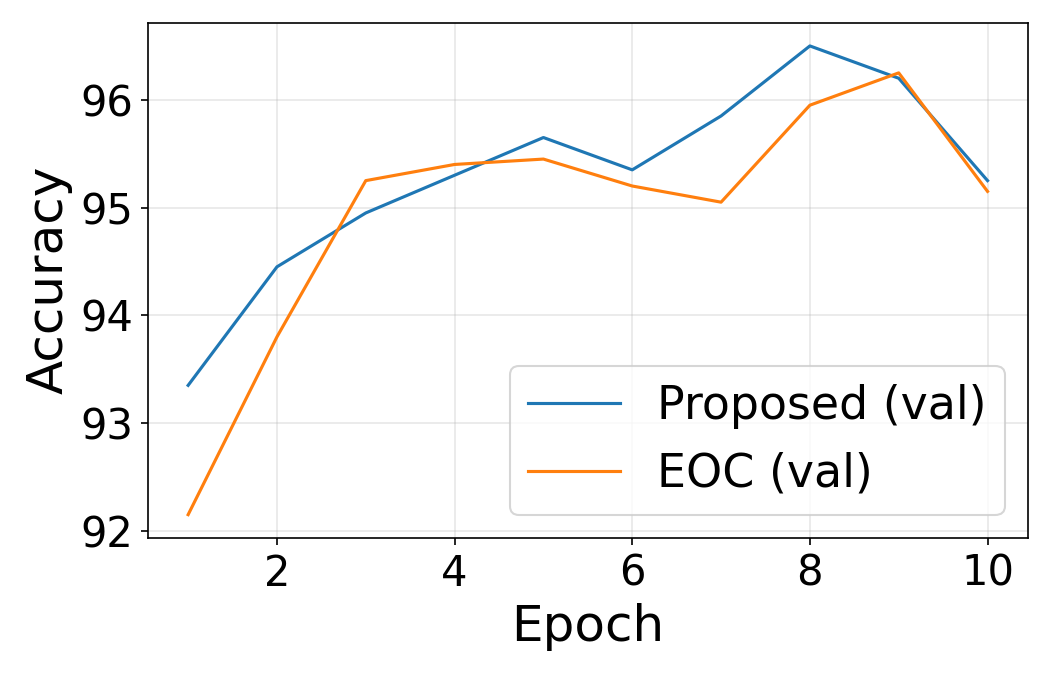}
    & \includegraphics[valign=m,width=0.28\textwidth]{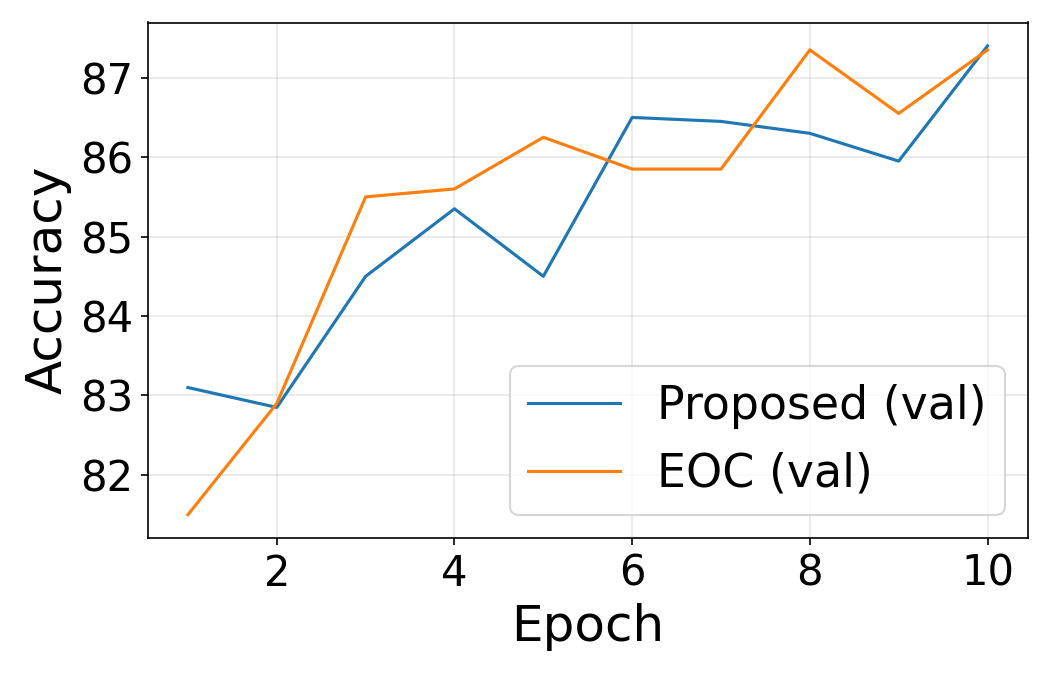}\\
    \hline
    $10^{1}$
    & \includegraphics[valign=m,width=0.28\textwidth]{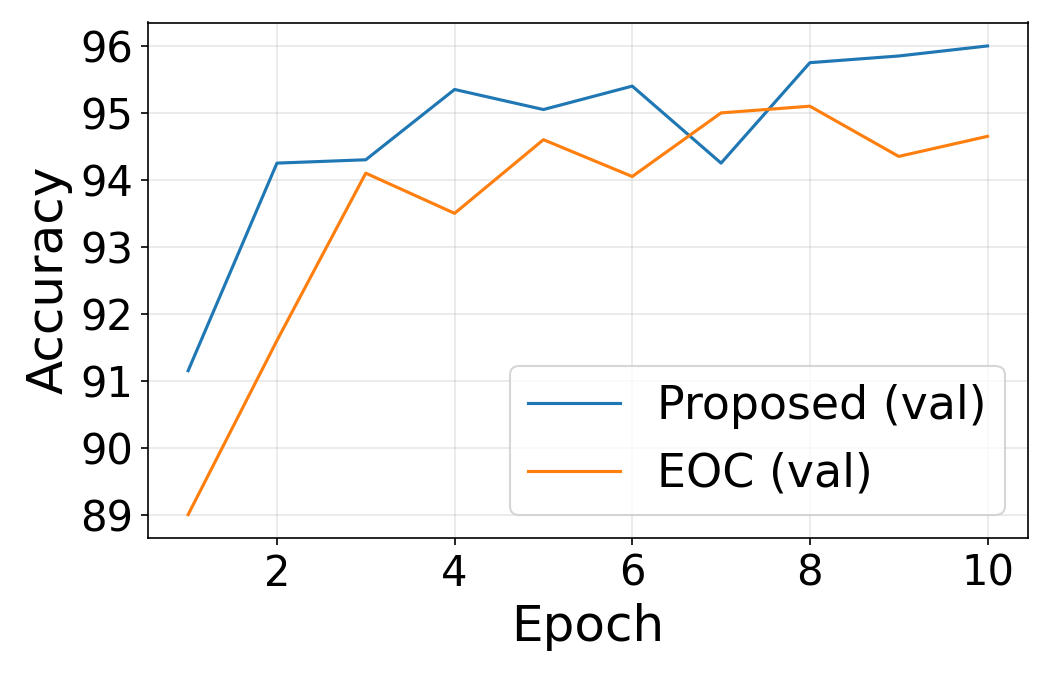}
    & \includegraphics[valign=m,width=0.28\textwidth]{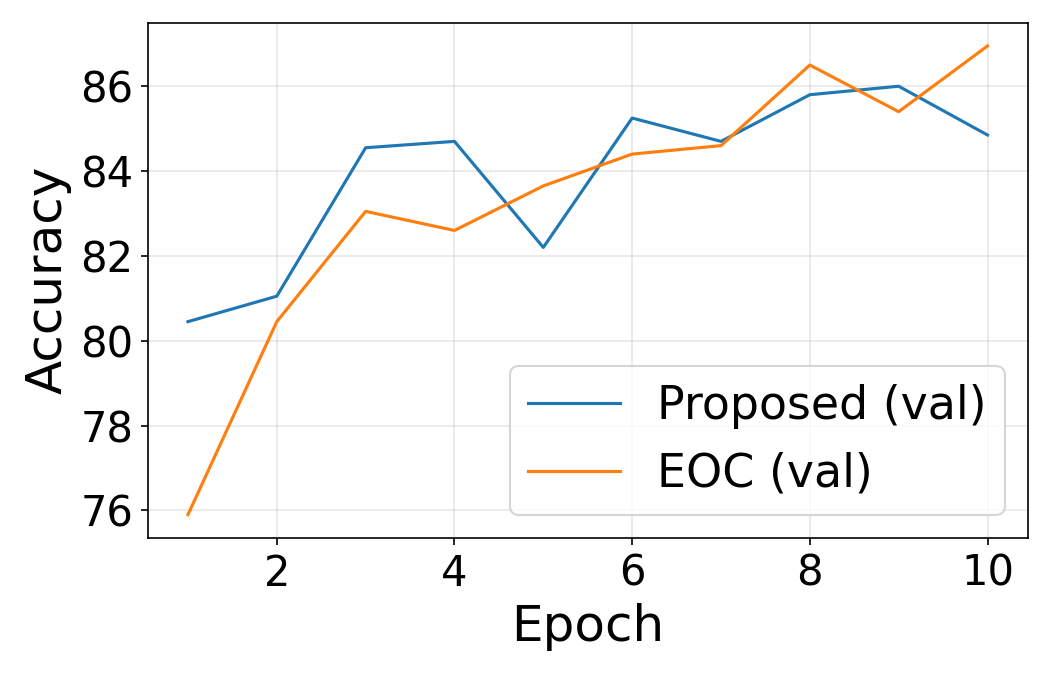}\\
    \hline
    $10^{2}$
    & \includegraphics[valign=m,width=0.28\textwidth]{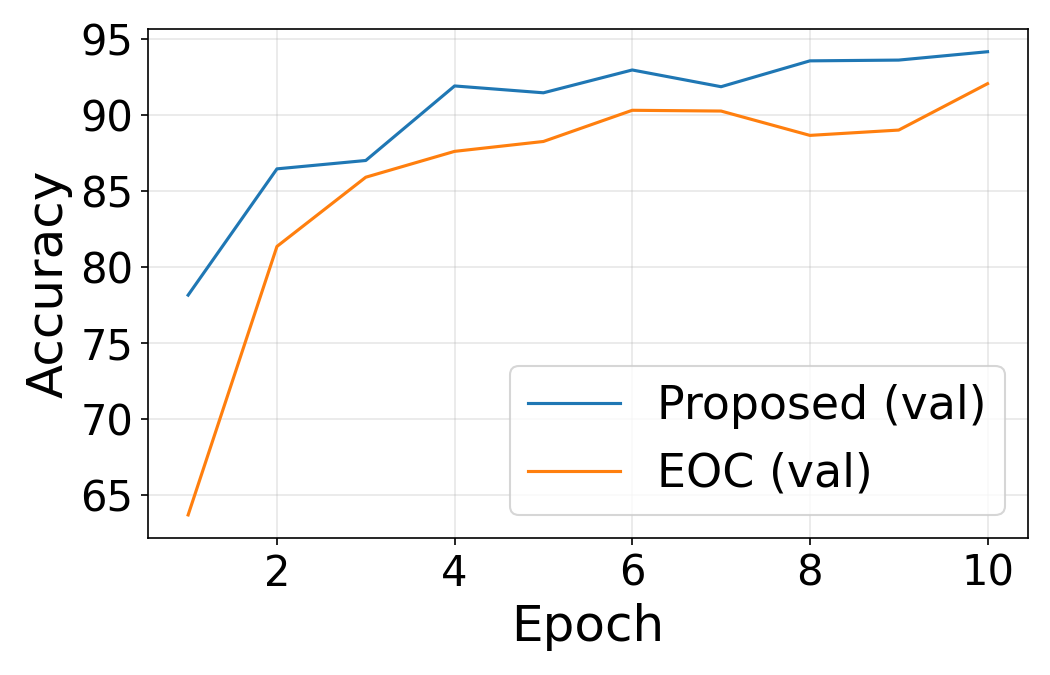}
    & \includegraphics[valign=m,width=0.28\textwidth]{figure/afMNIST_a1e+01_LR1e-05_W128_L50.png}\\
    \hline
    $10^{3}$
    & \includegraphics[valign=m,width=0.28\textwidth]{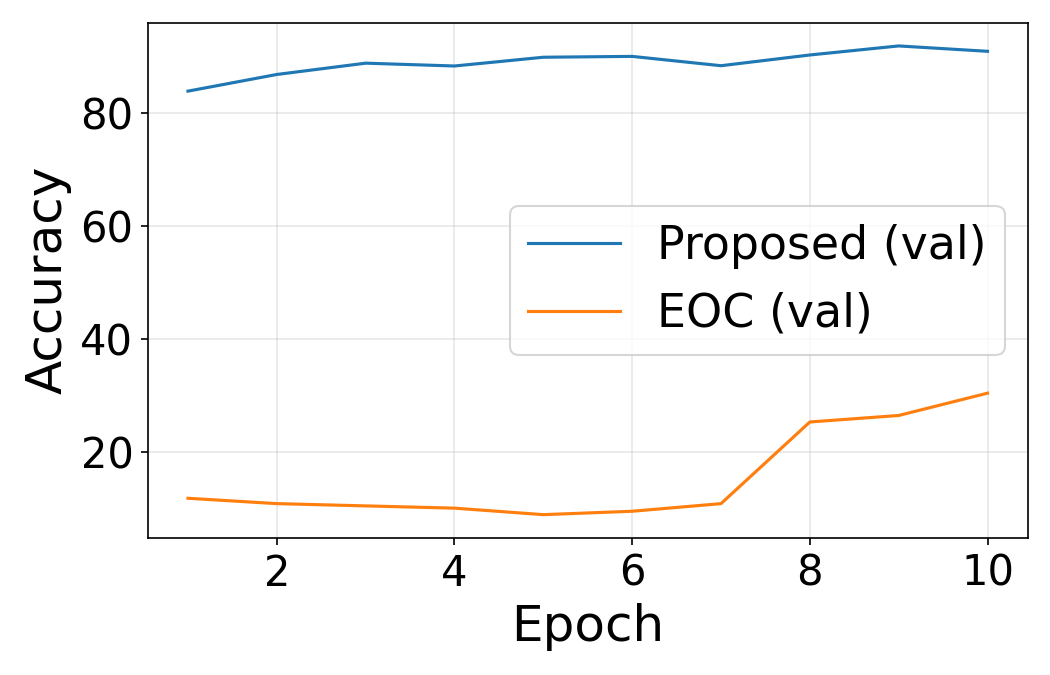}
    & \includegraphics[valign=m,width=0.28\textwidth]{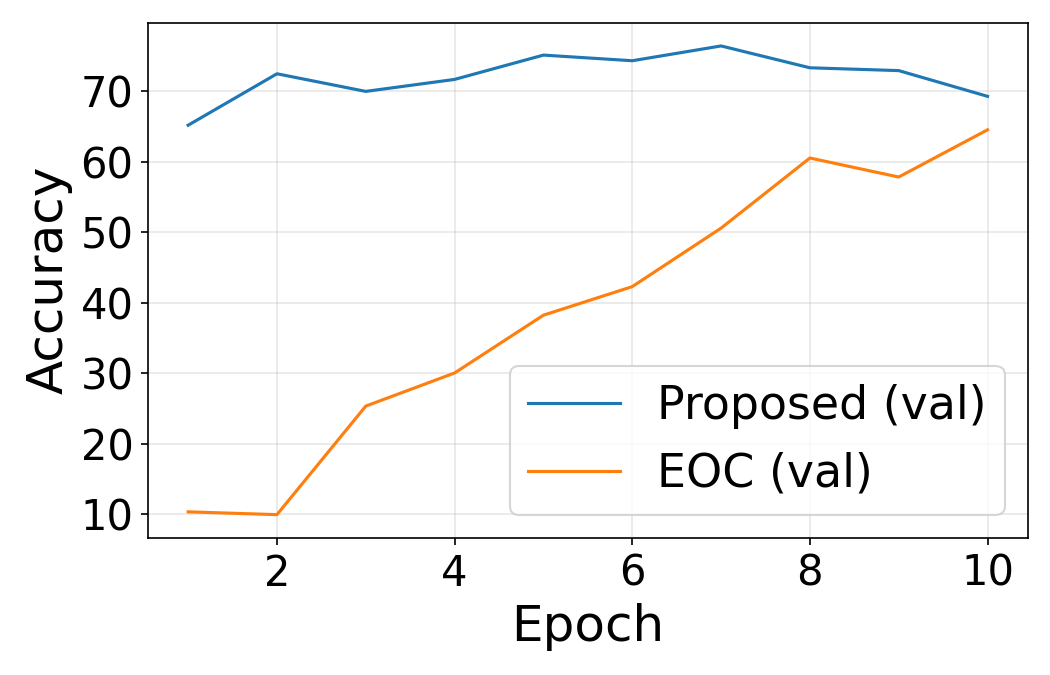}\\
    \hline
  \end{tabular}
  \label{scale_table3}

\end{table}

\begin{table}[t]
  \centering
  \caption{Validation accuracy on MNIST~(left) and Fashion MNIST~(right) for a
  50 layer, width 128 fully connected neural network with activation $a\arctan(x)$.
  Each row corresponds to a different activation scale $a$, and for every $a$
  the learning rate is set to $\eta = 10^{-4}/a$ for both initializations.}\label{tab:a_atanh_accuracy2}

  \small
  \setlength{\tabcolsep}{6pt}
  \renewcommand{\arraystretch}{1.2}

  \begin{tabular}{|C{1.5cm}|c|c|}
    \hline
    $a$ & MNIST (accuracy vs.\ epoch) & Fashion-MNIST (accuracy vs.\ epoch) \\
    \hline
    $10^{4}$
    & \includegraphics[valign=m,width=0.28\textwidth]{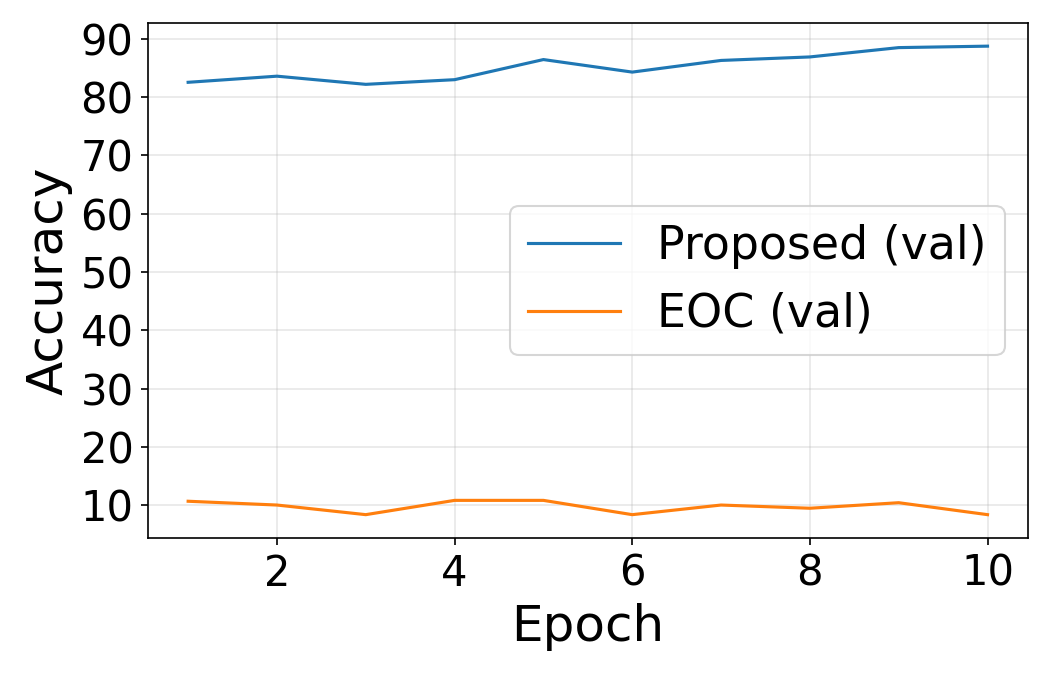}
    & \includegraphics[valign=m,width=0.28\textwidth]{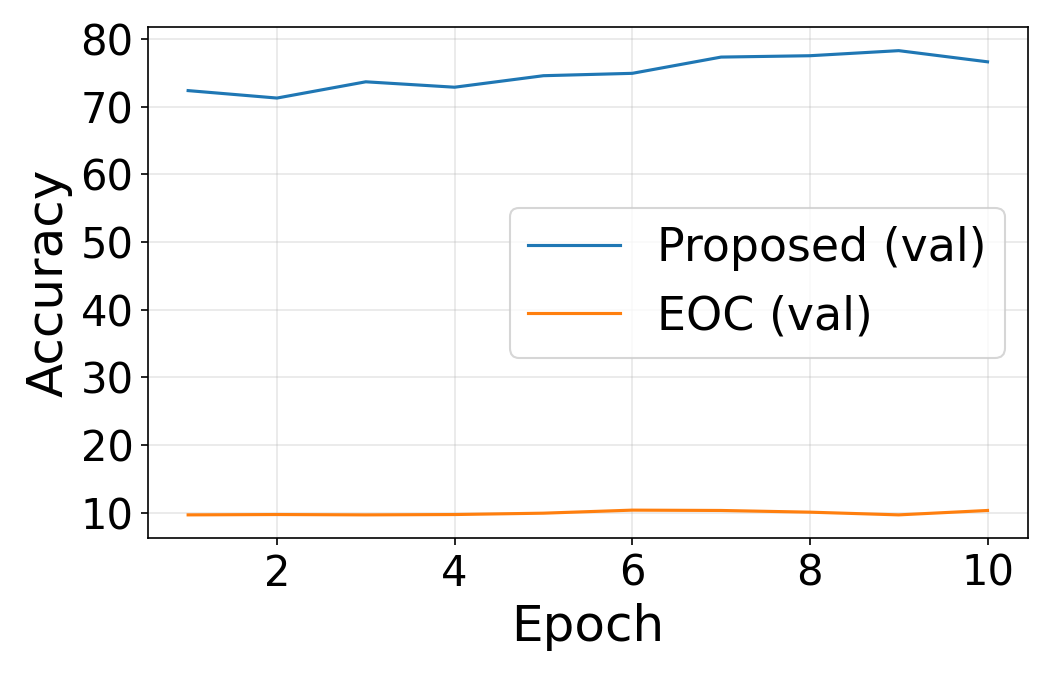}\\
    \hline
    $10^{5}$ 
    & \includegraphics[valign=m,width=0.28\textwidth]{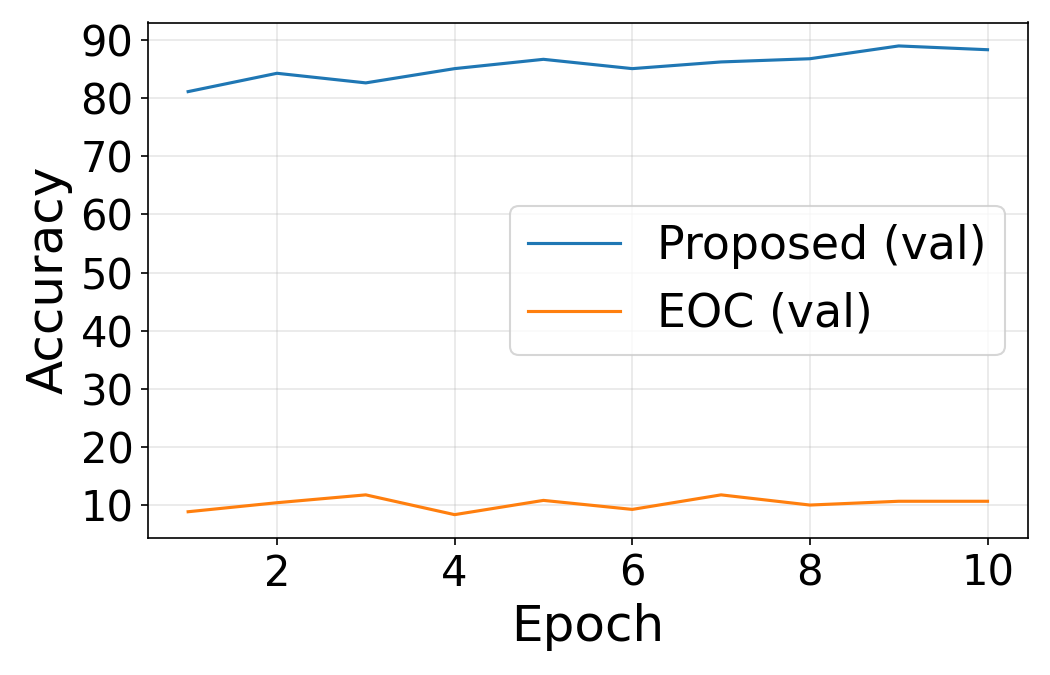}
    & \includegraphics[valign=m,width=0.28\textwidth]{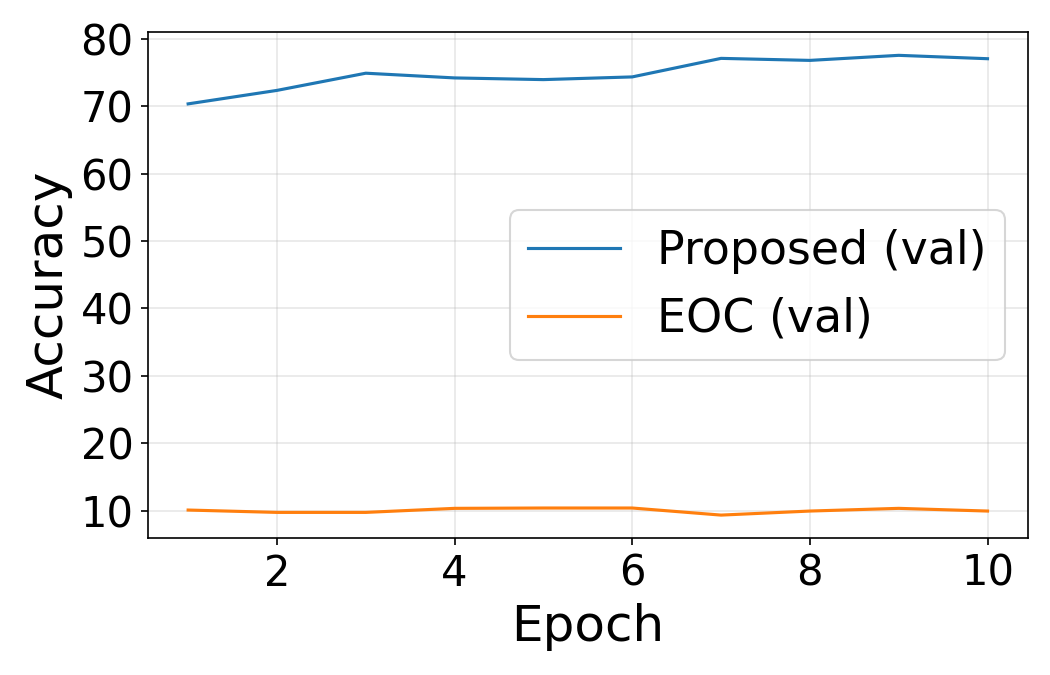}\\
    \hline
    $10^{6}$
    & \includegraphics[valign=m,width=0.28\textwidth]{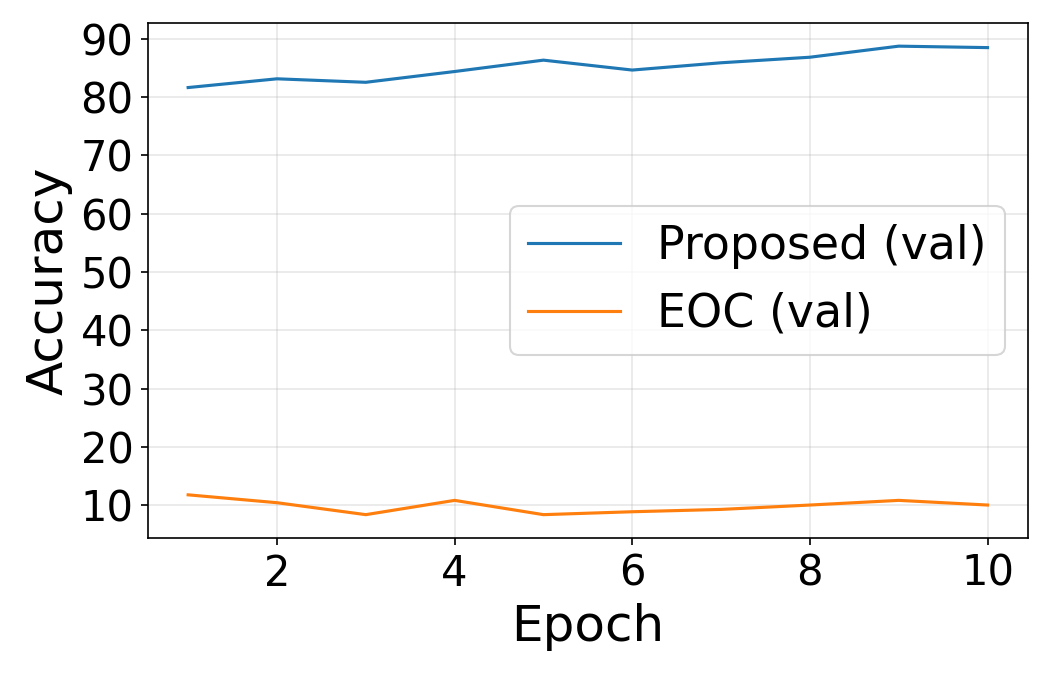}
    & \includegraphics[valign=m,width=0.28\textwidth]{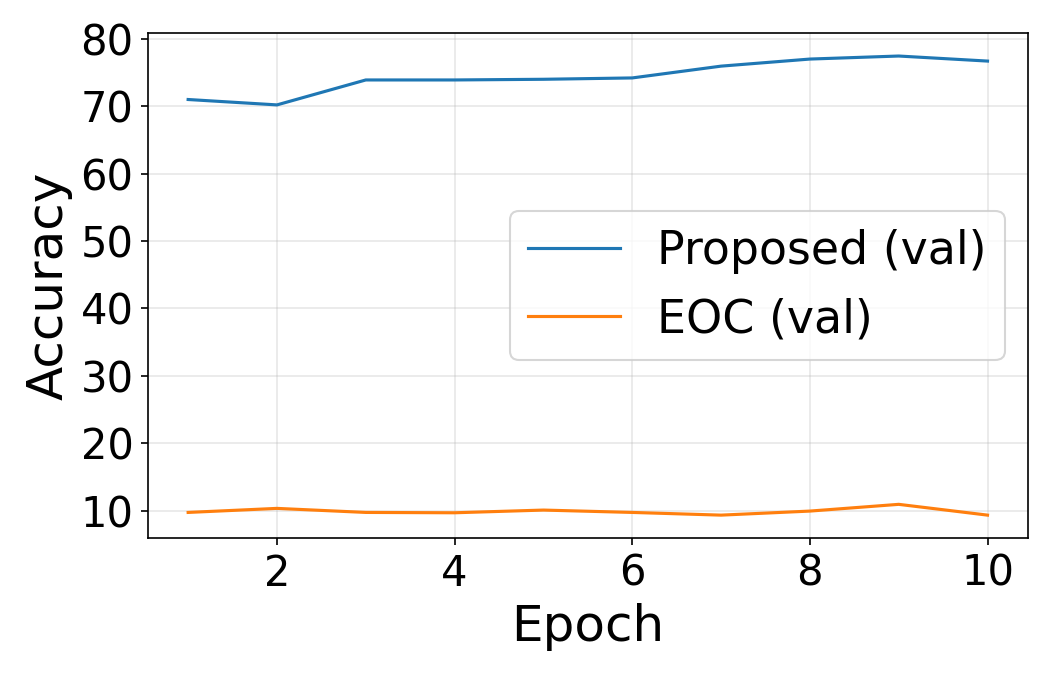}\\
    \hline
    $10^{7}$
    & \includegraphics[valign=m,width=0.28\textwidth]{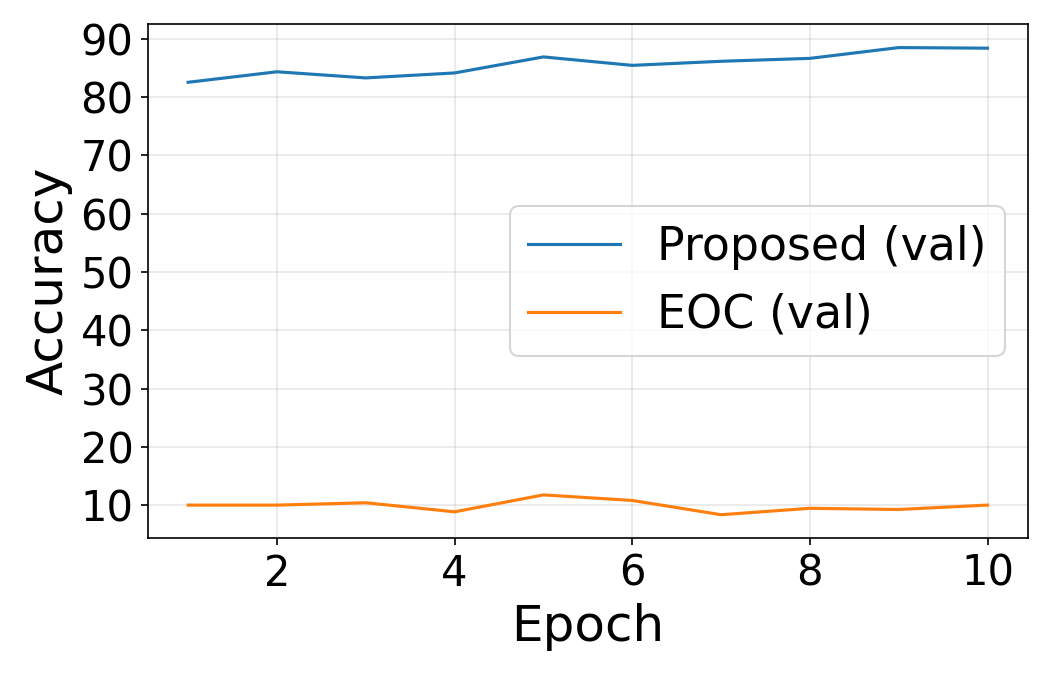}
    & \includegraphics[valign=m,width=0.28\textwidth]{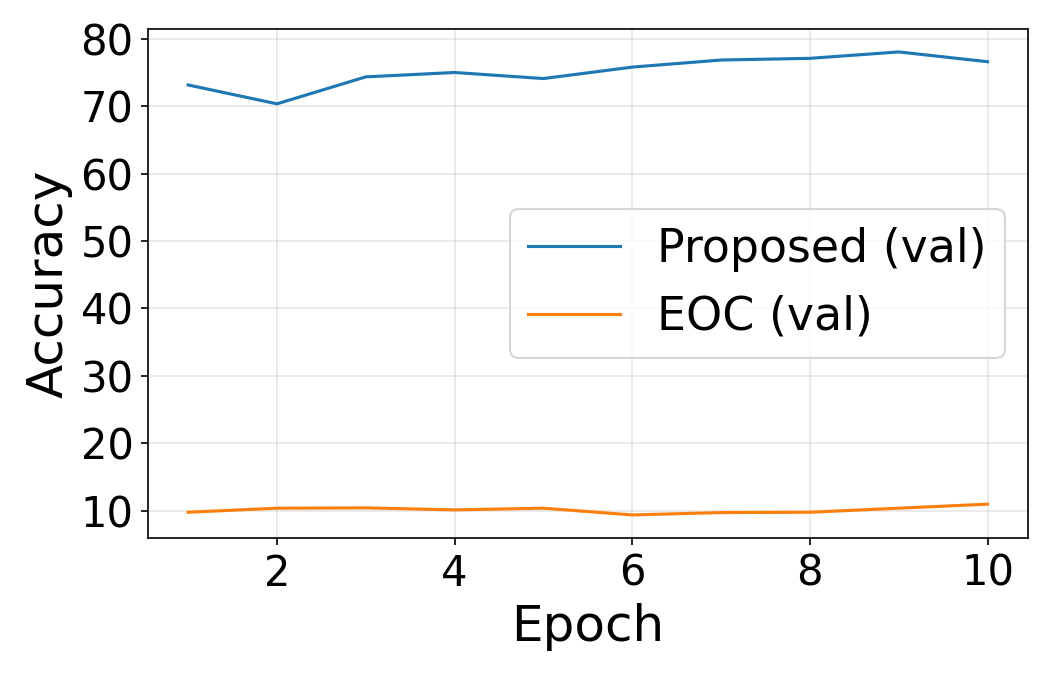}\\
    \hline
    $10^{8}$
    & \includegraphics[valign=m,width=0.28\textwidth]{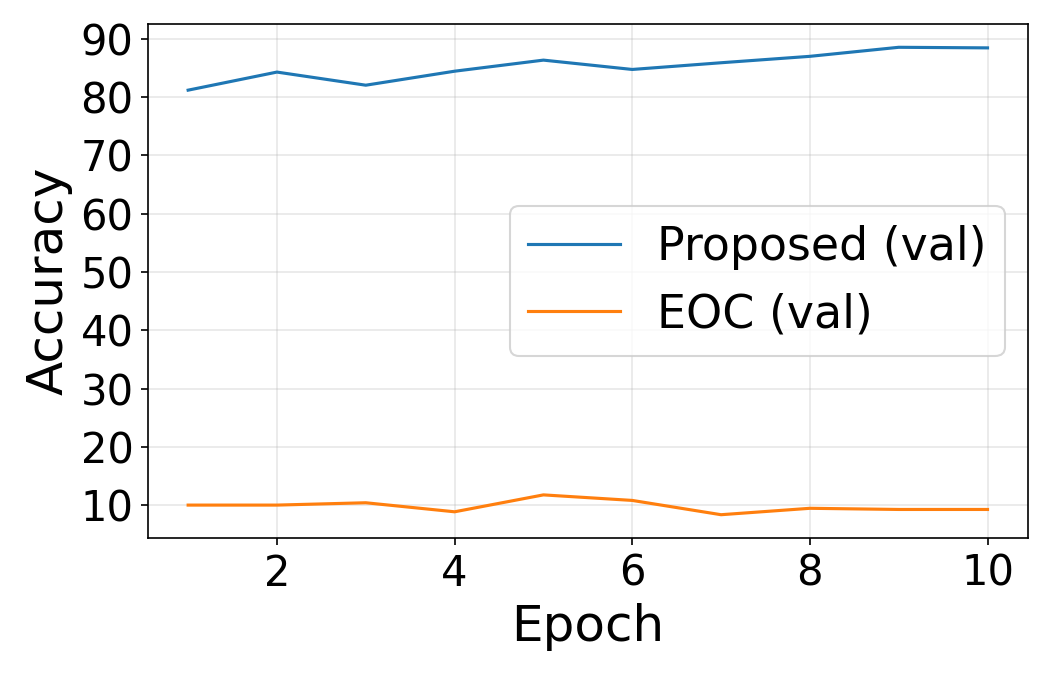}
    & \includegraphics[valign=m,width=0.28\textwidth]{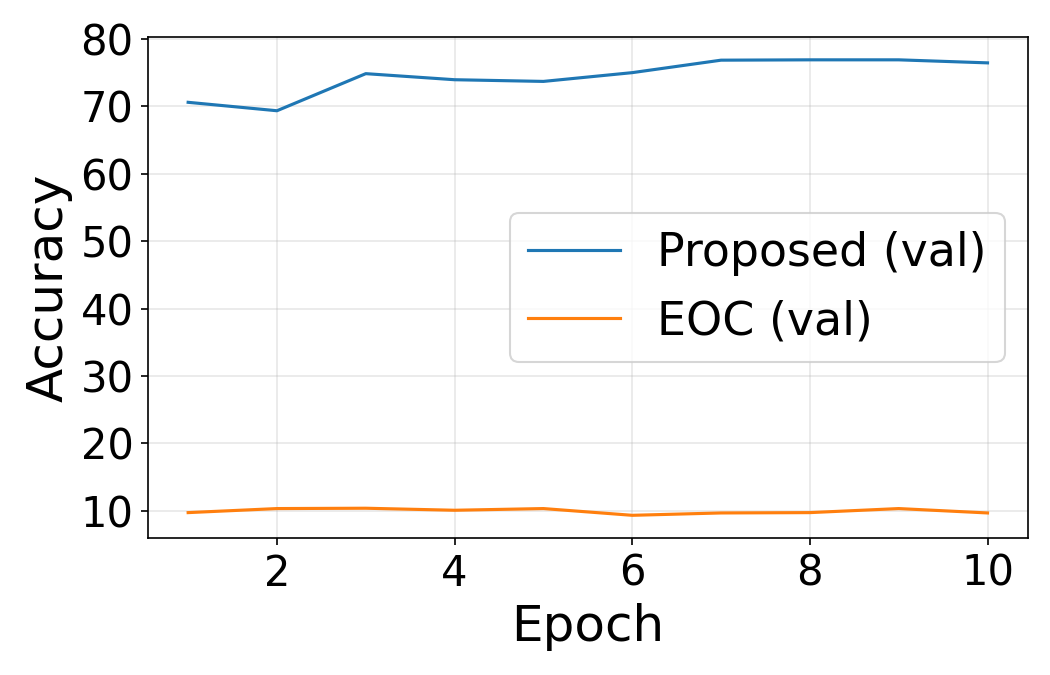}\\
    \hline
    $10^{9}$
    & \includegraphics[valign=m,width=0.28\textwidth]{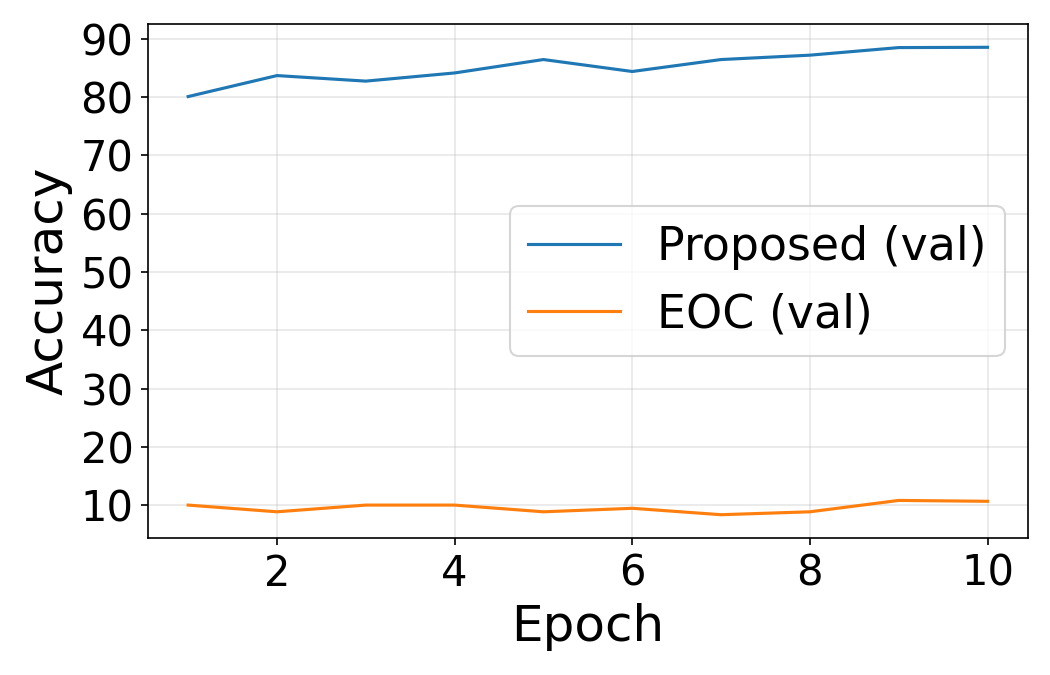}
    & \includegraphics[valign=m,width=0.28\textwidth]{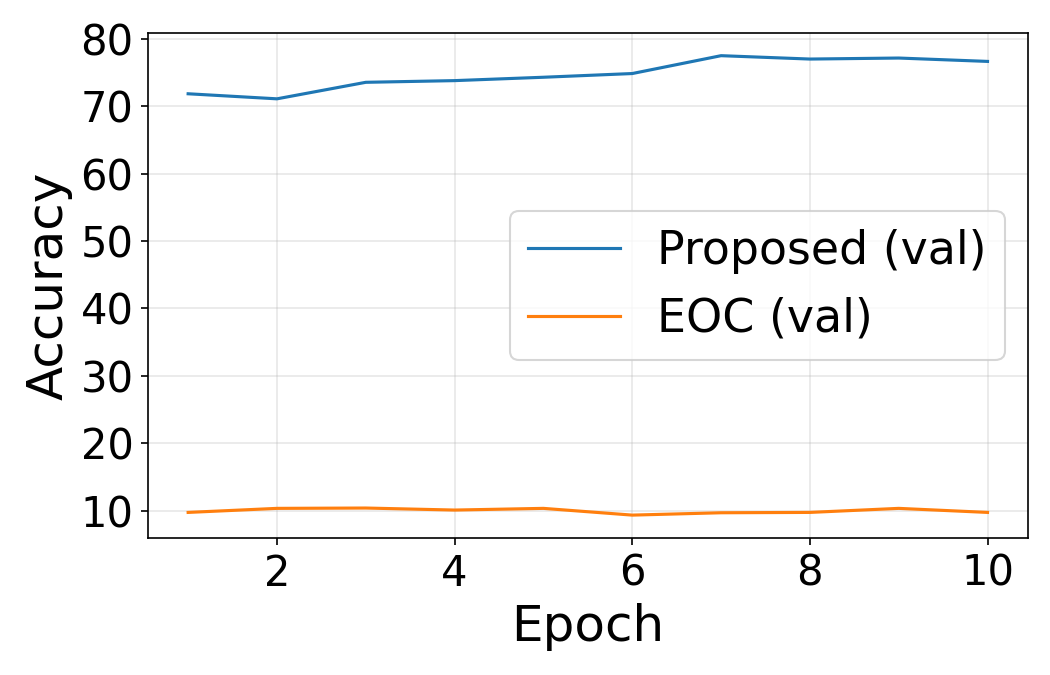}\\

    \hline
  \end{tabular}
\label{scale_table4}

\end{table}

\begin{table}[t]
  \centering
  \caption{Validation accuracy on MNIST~(left) and Fashion MNIST~(right) for a
  50 layer, width 128 fully connected neural network with activation $a\operatorname{softsign_2}(x)$.
  Each row corresponds to a different activation scale $a$, and for every $a$
  the learning rate is set to $\eta = 10^{-4}/a$ for both initializations.}\label{tab:a_soft_accuracy1}
  \small
  \setlength{\tabcolsep}{6pt}
  \renewcommand{\arraystretch}{1.2}

  \begin{tabular}{|C{1.5cm}|c|c|}
    \hline
    $a$ & MNIST (accuracy vs.\ epoch) & Fashion-MNIST (accuracy vs.\ epoch) \\
    \hline
    $10^{-2}$
    & \includegraphics[valign=m,width=0.28\textwidth]{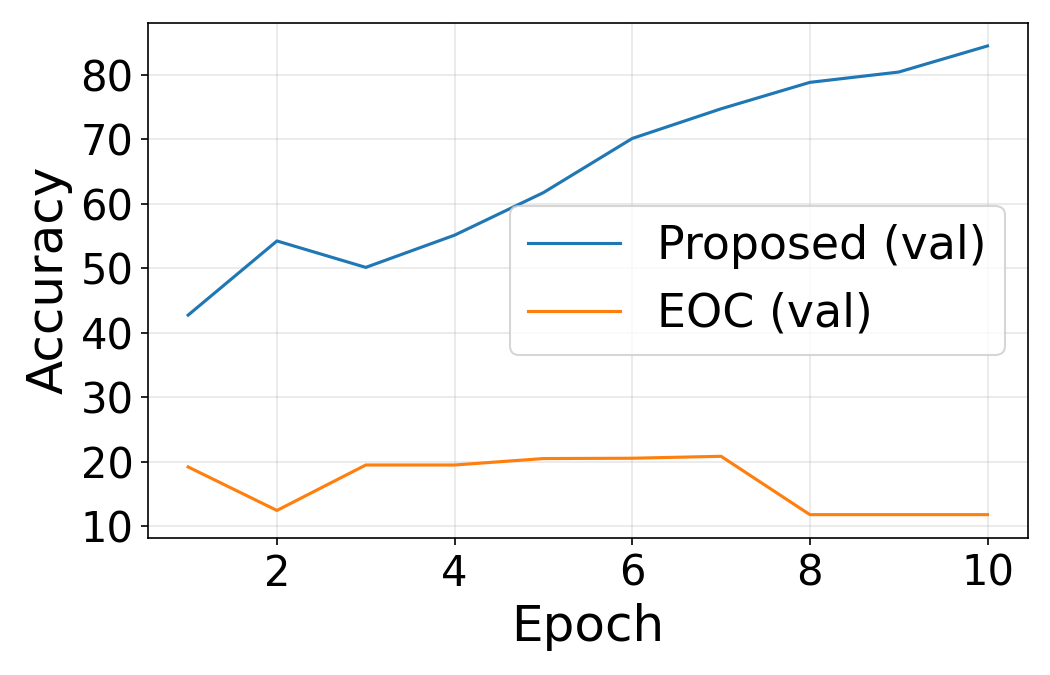}
    & \includegraphics[valign=m,width=0.28\textwidth]{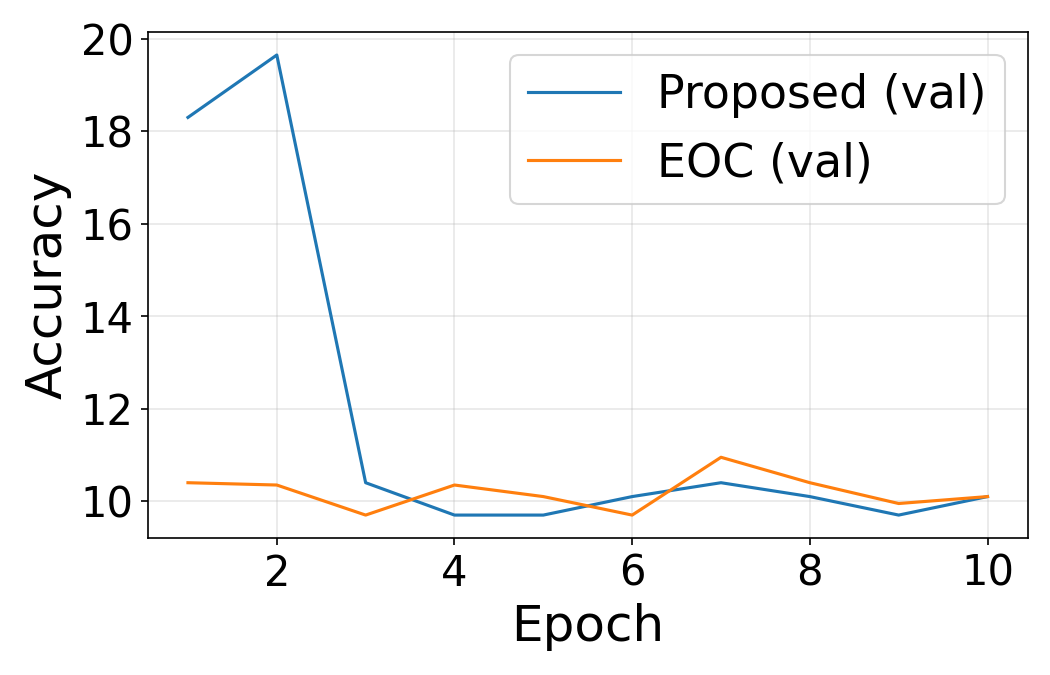}\\
    \hline
    $10^{-1}$
    & \includegraphics[valign=m,width=0.28\textwidth]{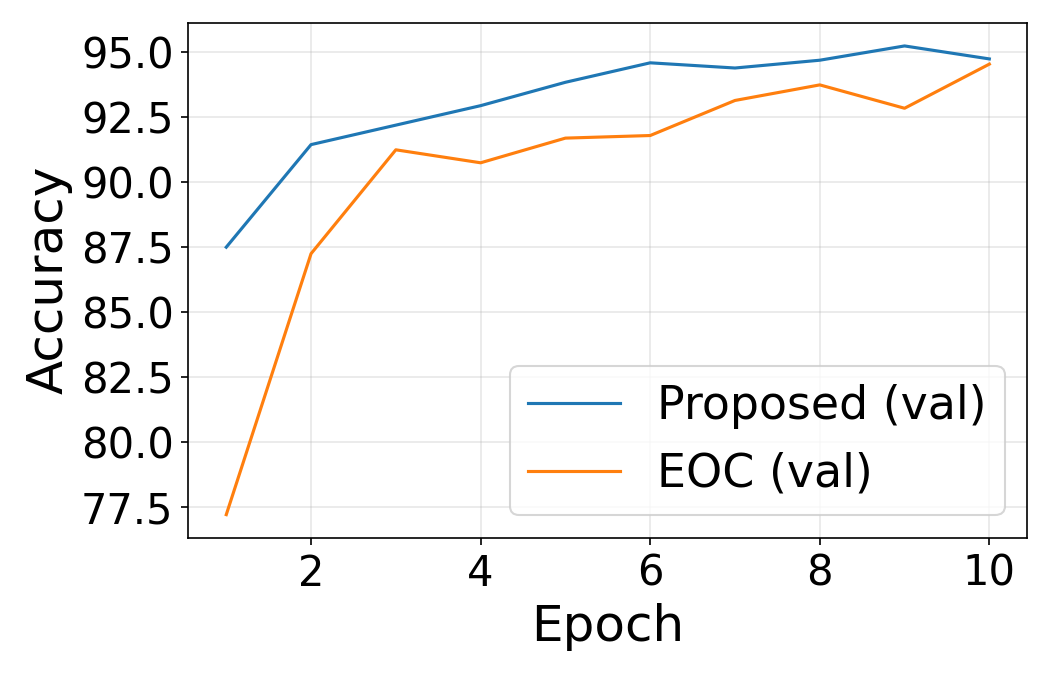}
    & \includegraphics[valign=m,width=0.28\textwidth]{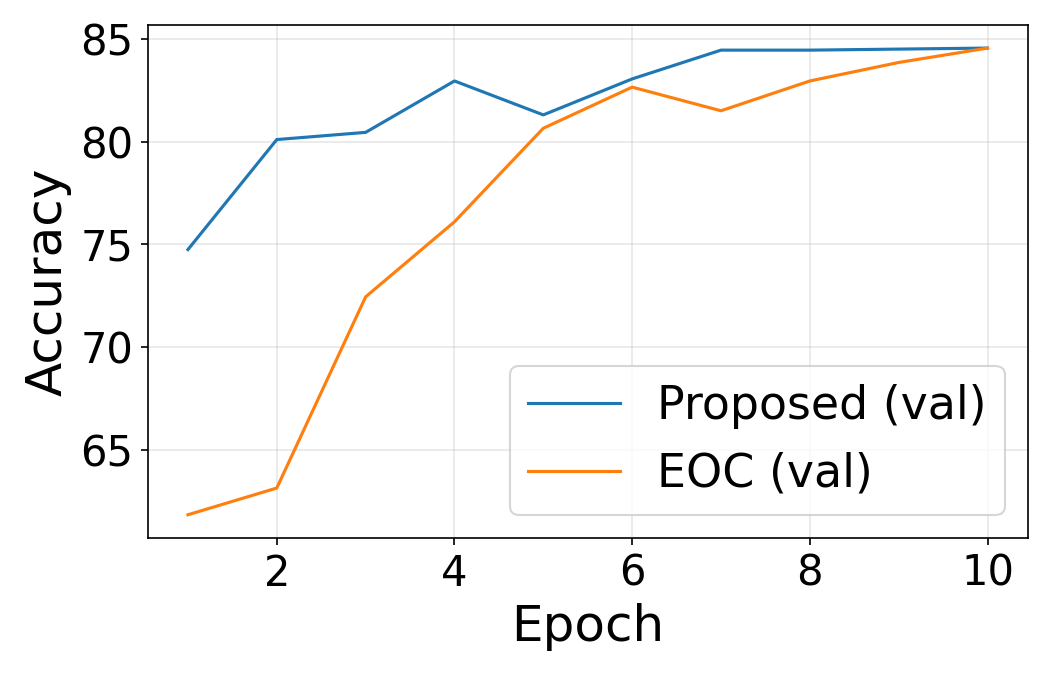}\\
    \hline
    1
    & \includegraphics[valign=m,width=0.28\textwidth]{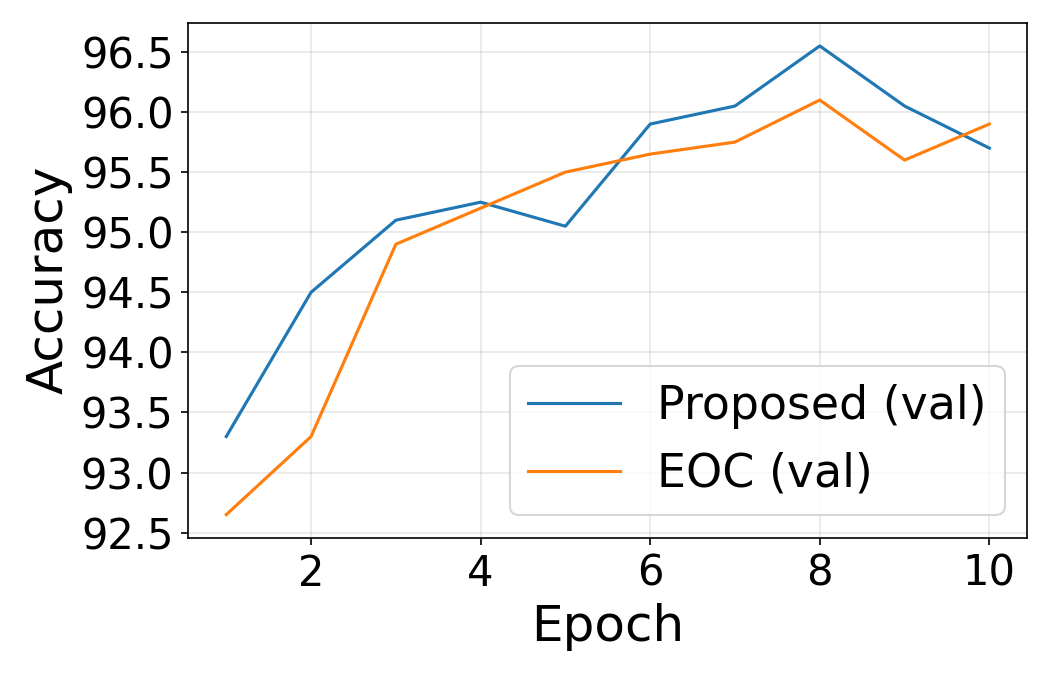}
    & \includegraphics[valign=m,width=0.28\textwidth]{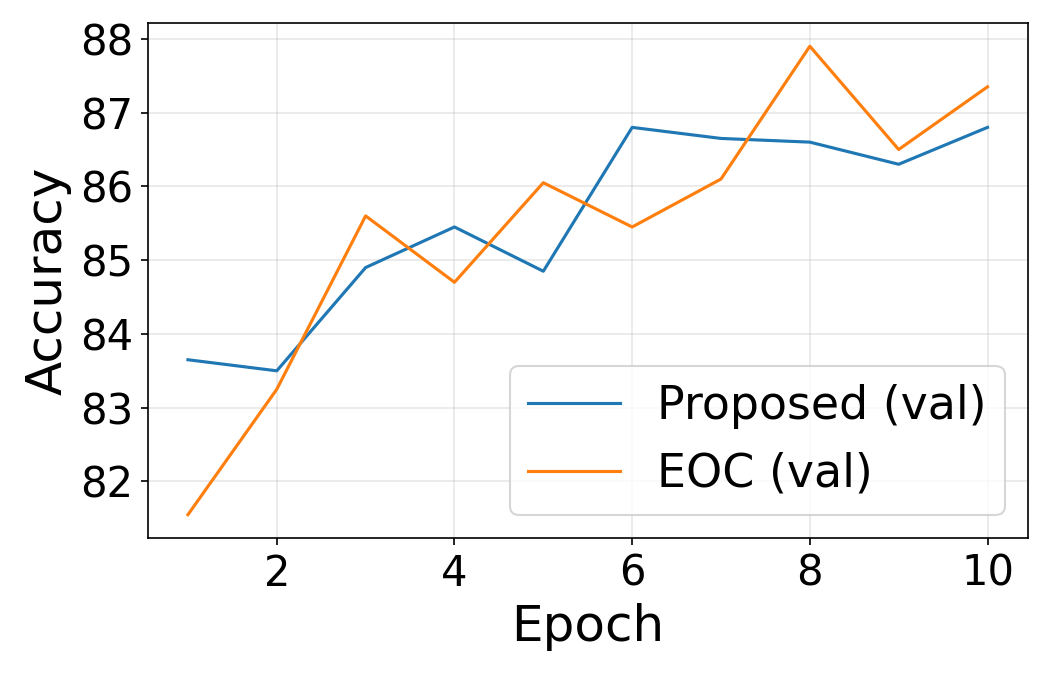}\\
    \hline
    $10^{1}$
    & \includegraphics[valign=m,width=0.28\textwidth]{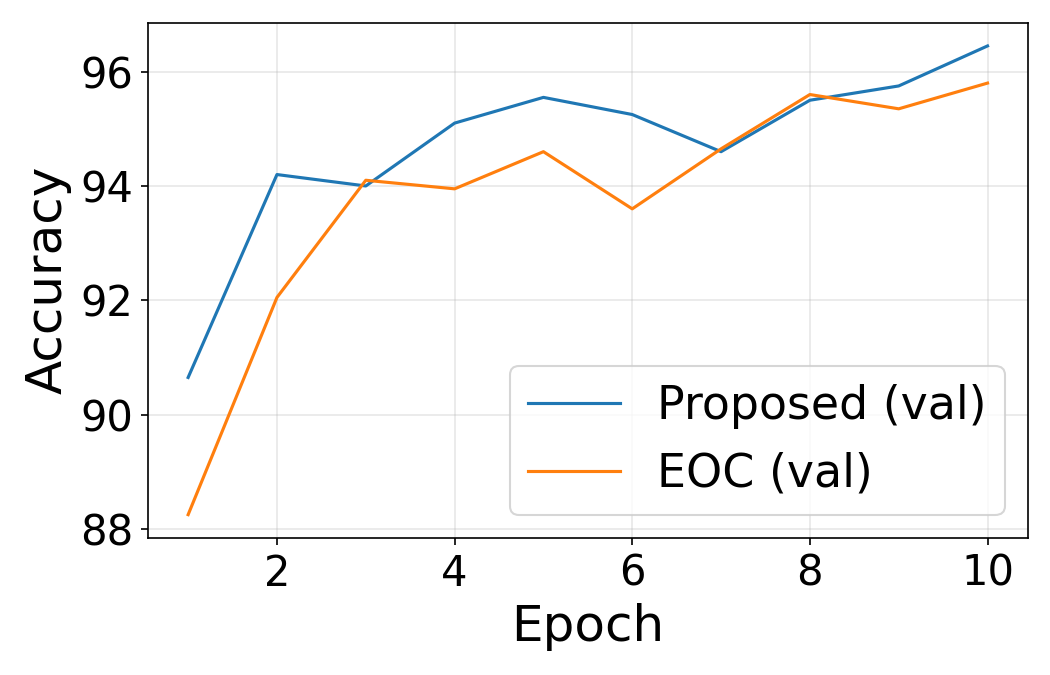}
    & \includegraphics[valign=m,width=0.28\textwidth]{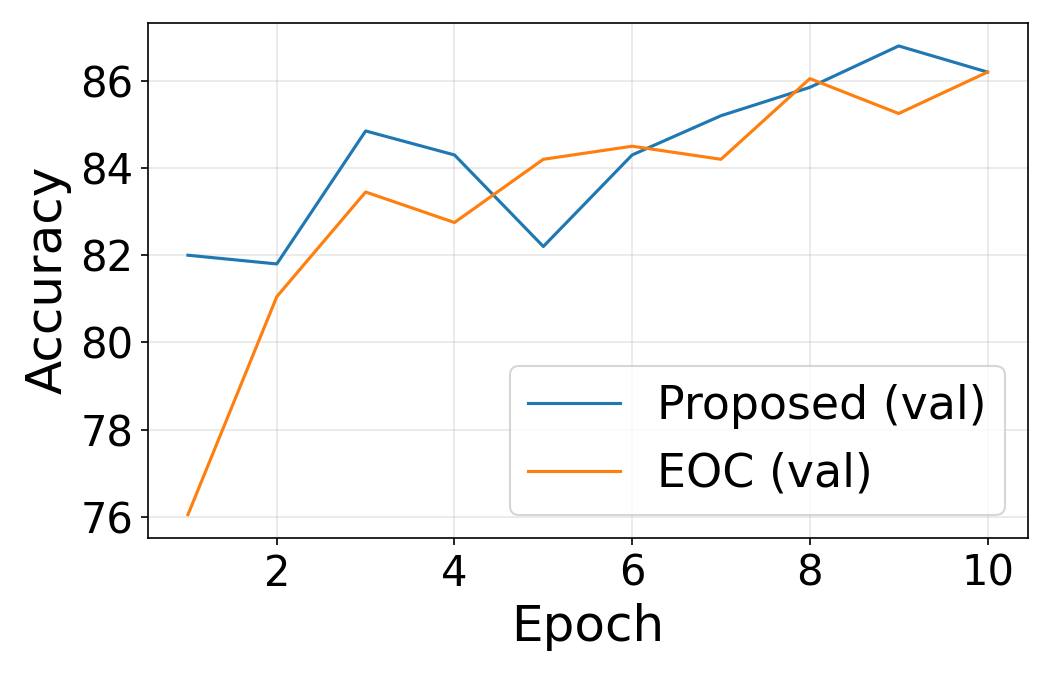}\\
    \hline
    $10^{2}$
    & \includegraphics[valign=m,width=0.28\textwidth]{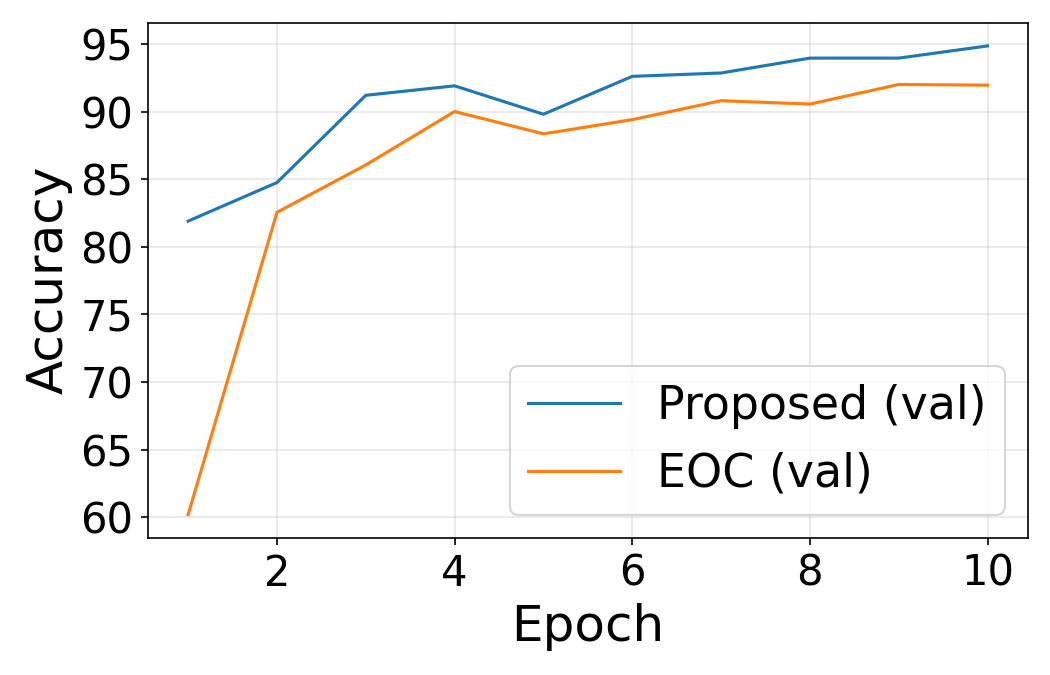}
    & \includegraphics[valign=m,width=0.28\textwidth]{figure/sfMNIST_a1e+01_LR1e-05_W128_L50.png}\\
    \hline
    $10^{3}$
    & \includegraphics[valign=m,width=0.28\textwidth]{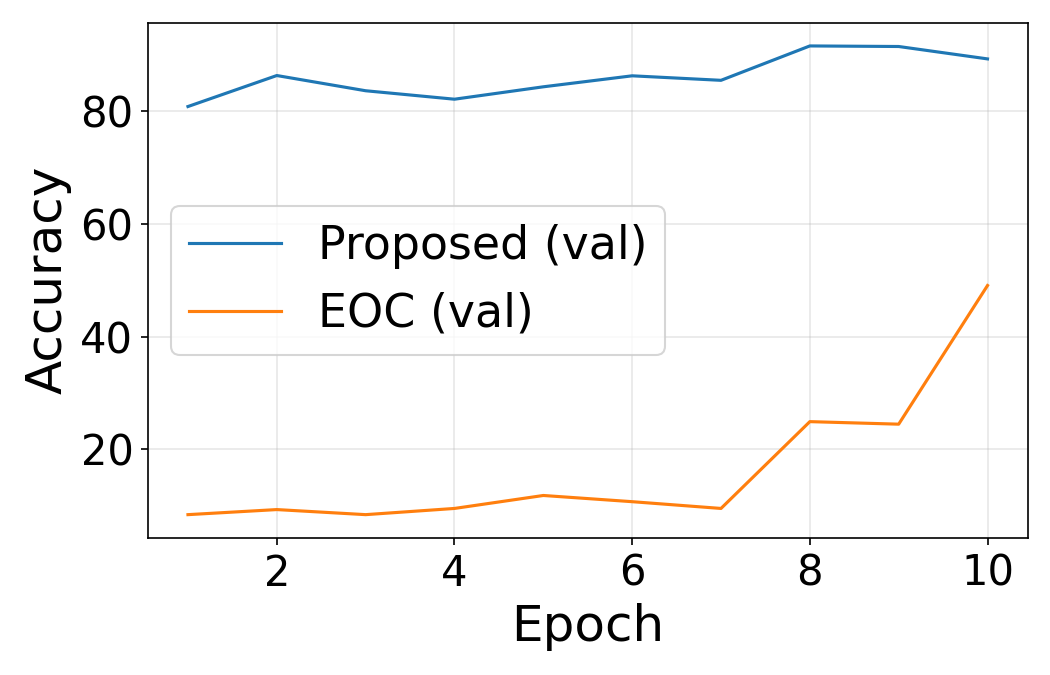}
    & \includegraphics[valign=m,width=0.28\textwidth]{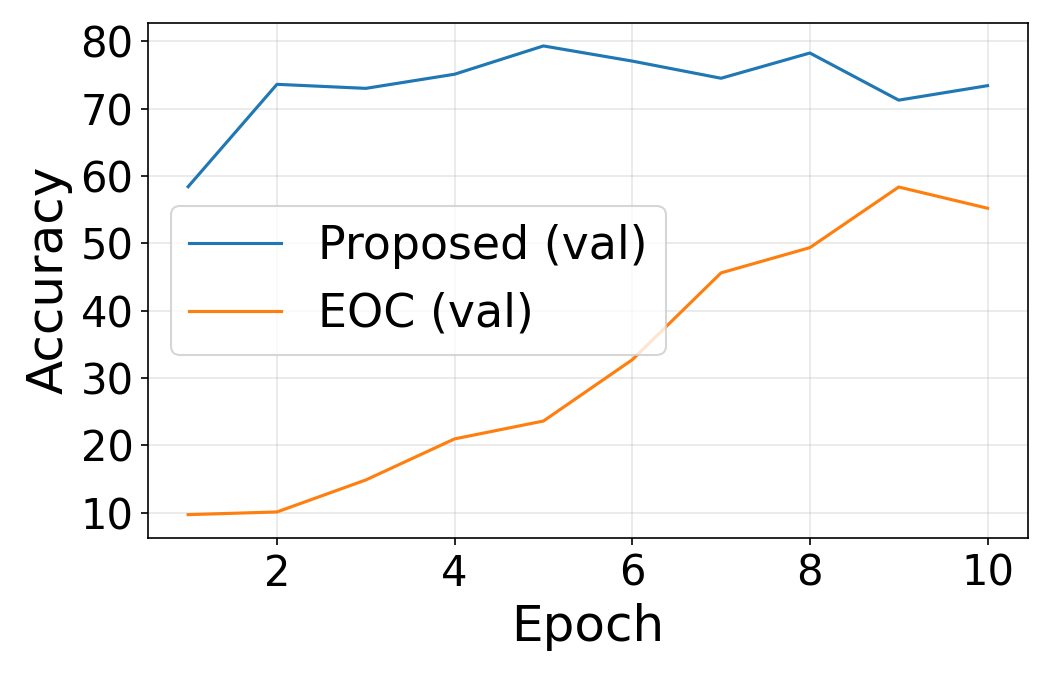}\\
    \hline
  \end{tabular}
\label{scale_table5}

\end{table}

\begin{table}[t]
  \centering
  \caption{Validation accuracy on MNIST~(left) and Fashion MNIST~(right) for a
  50 layer, width 128 fully connected neural network with activation $a\operatorname{softsign_2}(x)$.
  Each row corresponds to a different activation scale $a$, and for every $a$
  the learning rate is set to $\eta = 10^{-4}/a$ for both initializations.} \label{tab:a_soft_accuracy2}

  \small
  \setlength{\tabcolsep}{6pt}
  \renewcommand{\arraystretch}{1.2}

  \begin{tabular}{|C{1.5cm}|c|c|}
    \hline
    $a$ & MNIST (accuracy vs.\ epoch) & Fashion-MNIST (accuracy vs.\ epoch) \\
    \hline
    $10^{4}$
    & \includegraphics[valign=m,width=0.28\textwidth]{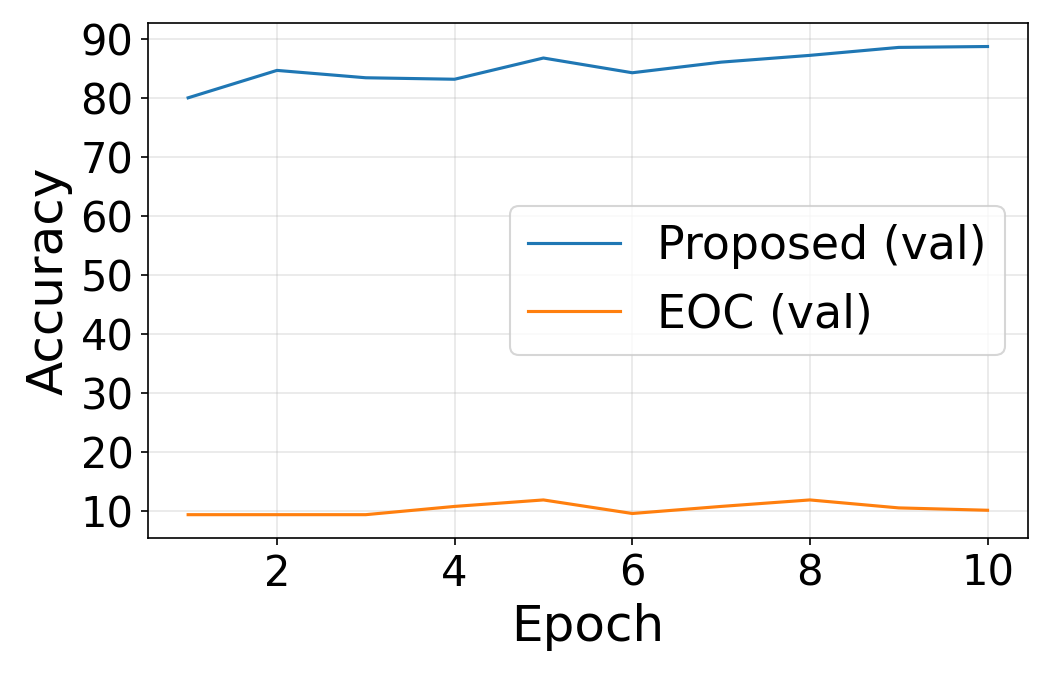}
    & \includegraphics[valign=m,width=0.28\textwidth]{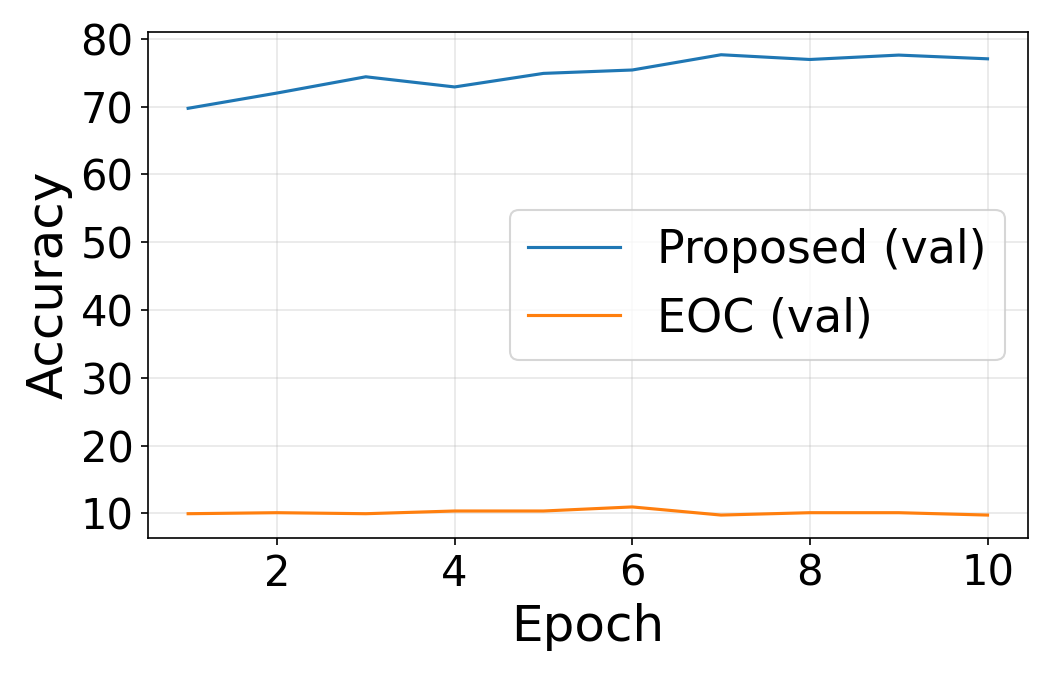}\\
    \hline
    $10^{5}$ 
    & \includegraphics[valign=m,width=0.28\textwidth]{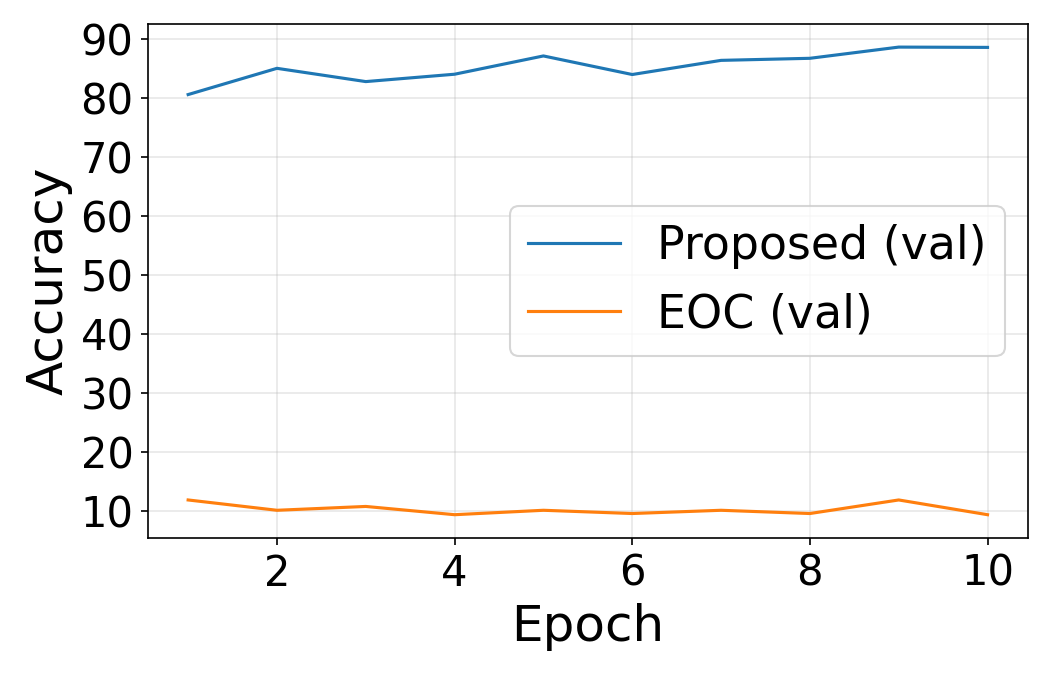}
    & \includegraphics[valign=m,width=0.28\textwidth]{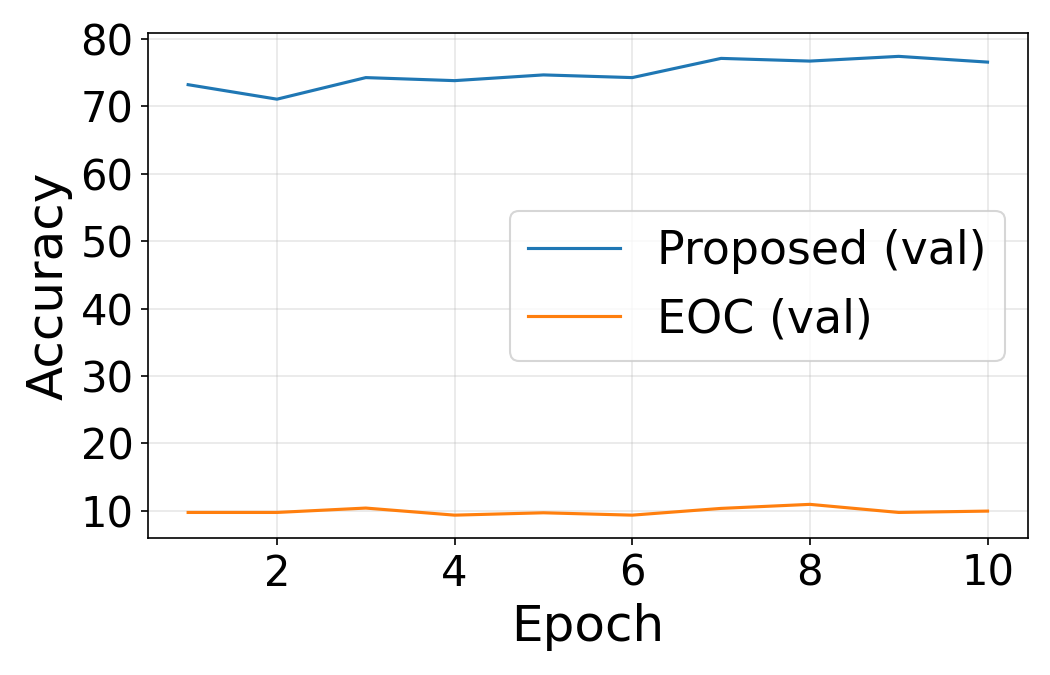}\\
    \hline
    $10^{6}$
    & \includegraphics[valign=m,width=0.28\textwidth]{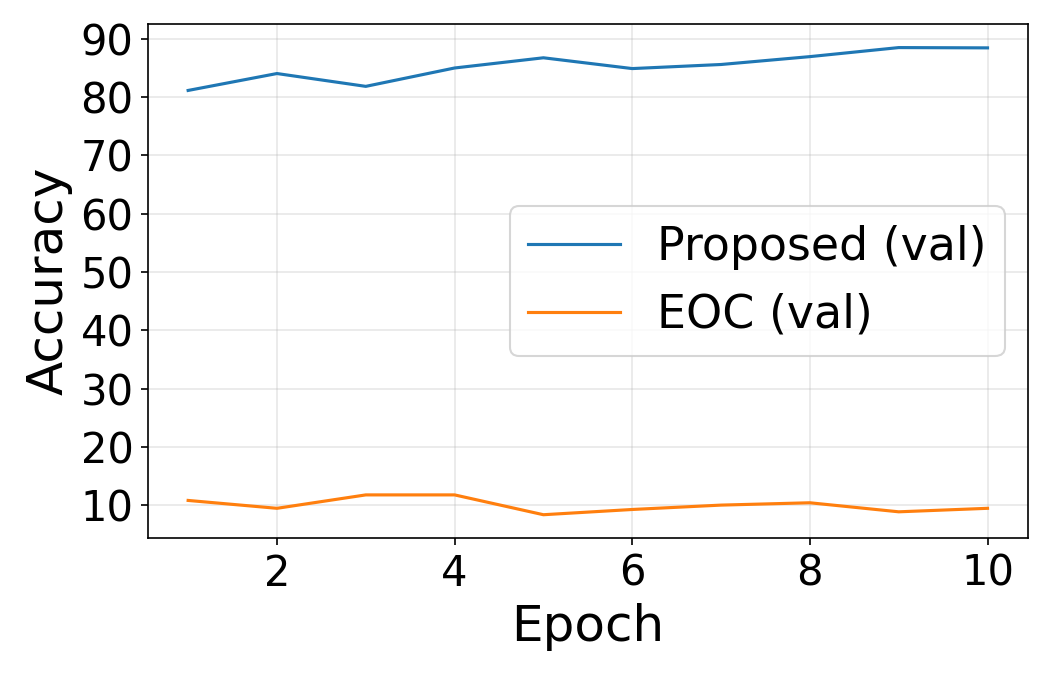}
    & \includegraphics[valign=m,width=0.28\textwidth]{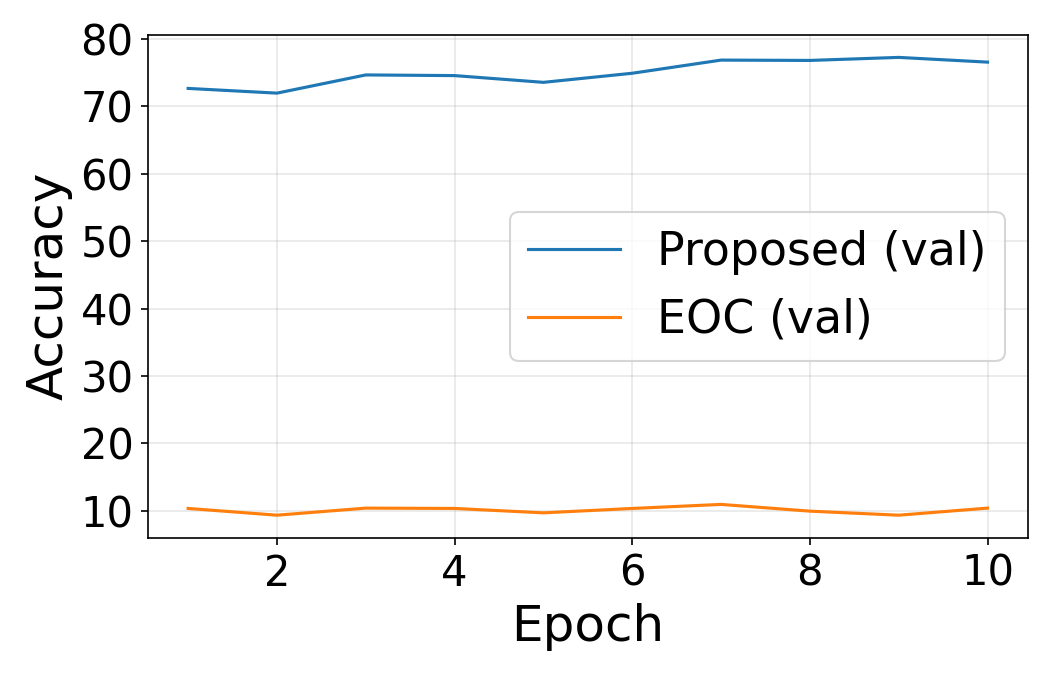}\\
    \hline
    $10^{7}$
    & \includegraphics[valign=m,width=0.28\textwidth]{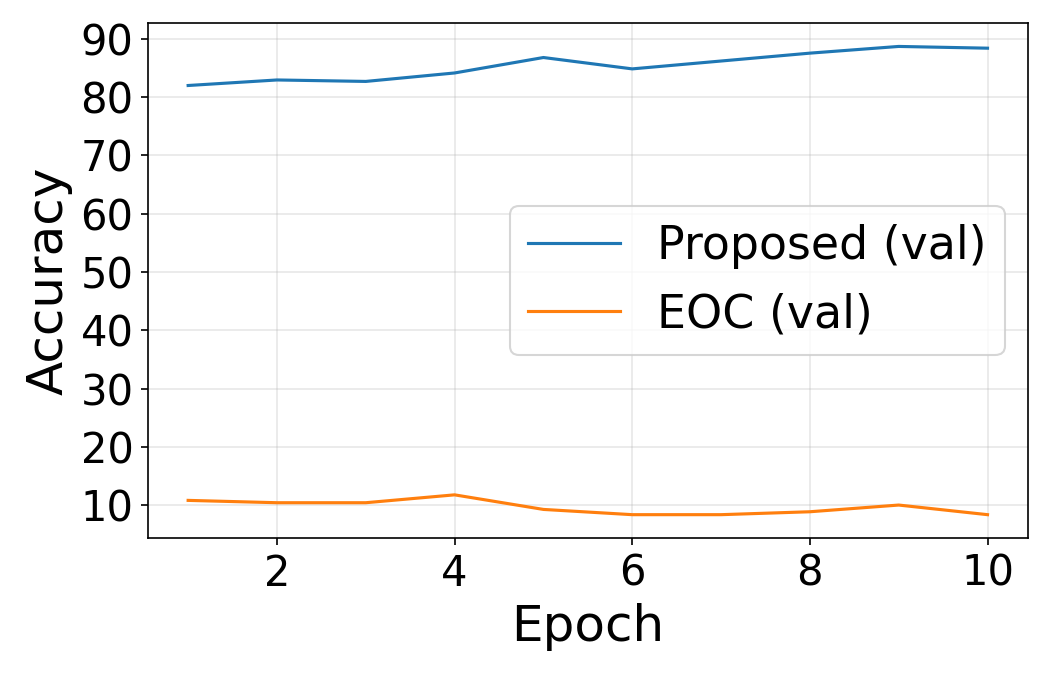}
    & \includegraphics[valign=m,width=0.28\textwidth]{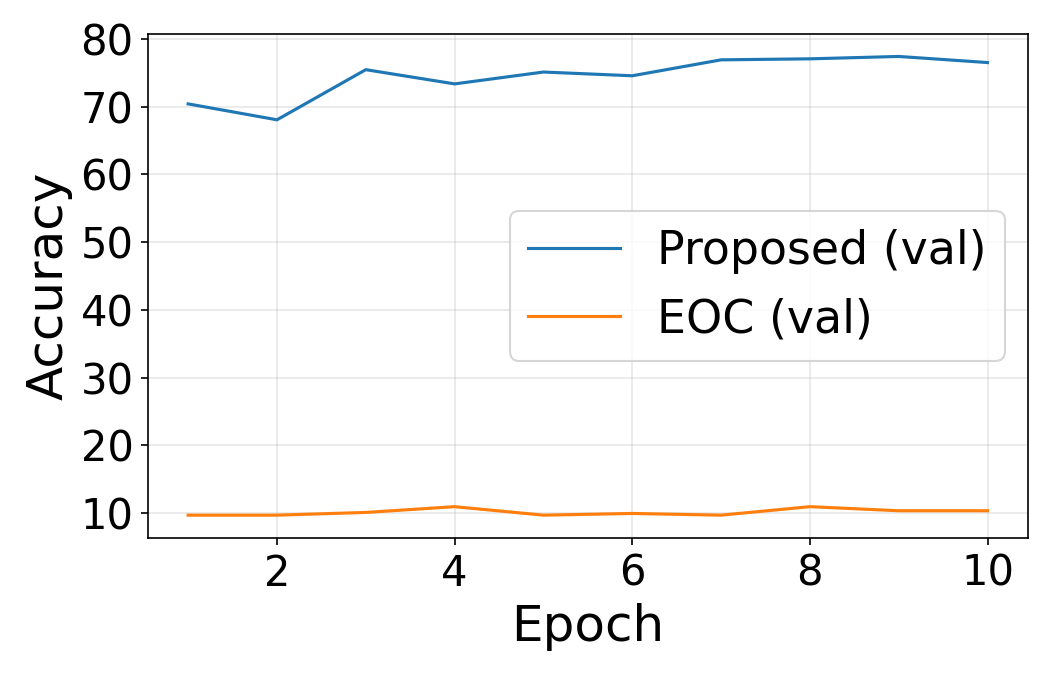}\\
    \hline
    $10^{8}$
    & \includegraphics[valign=m,width=0.28\textwidth]{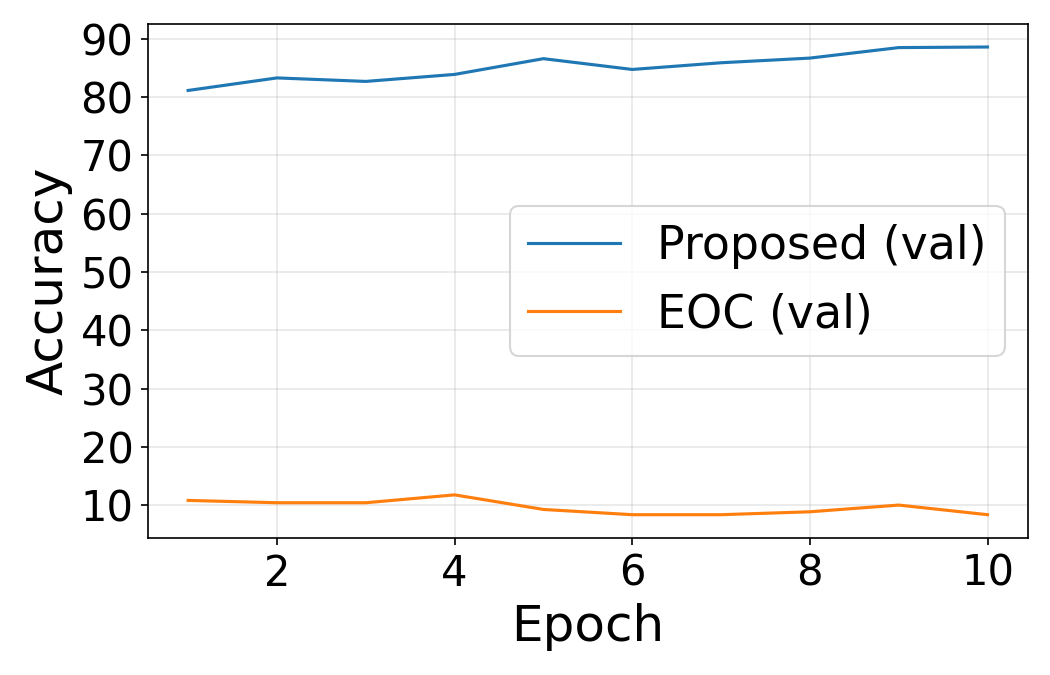}
    & \includegraphics[valign=m,width=0.28\textwidth]{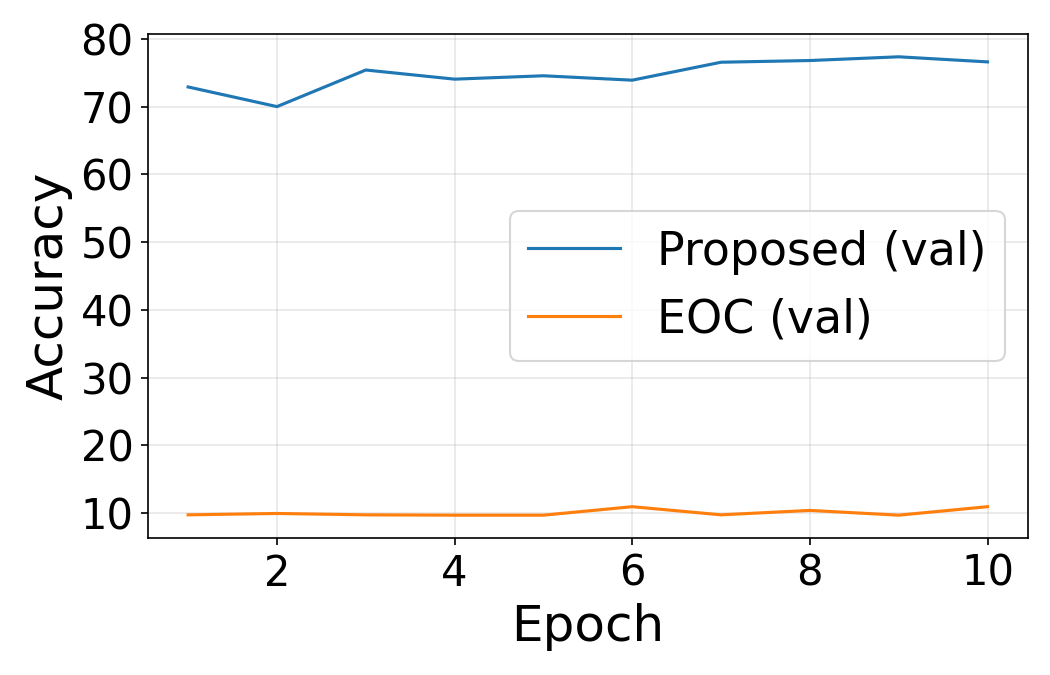}\\
    \hline
    $10^{9}$
    & \includegraphics[valign=m,width=0.28\textwidth]{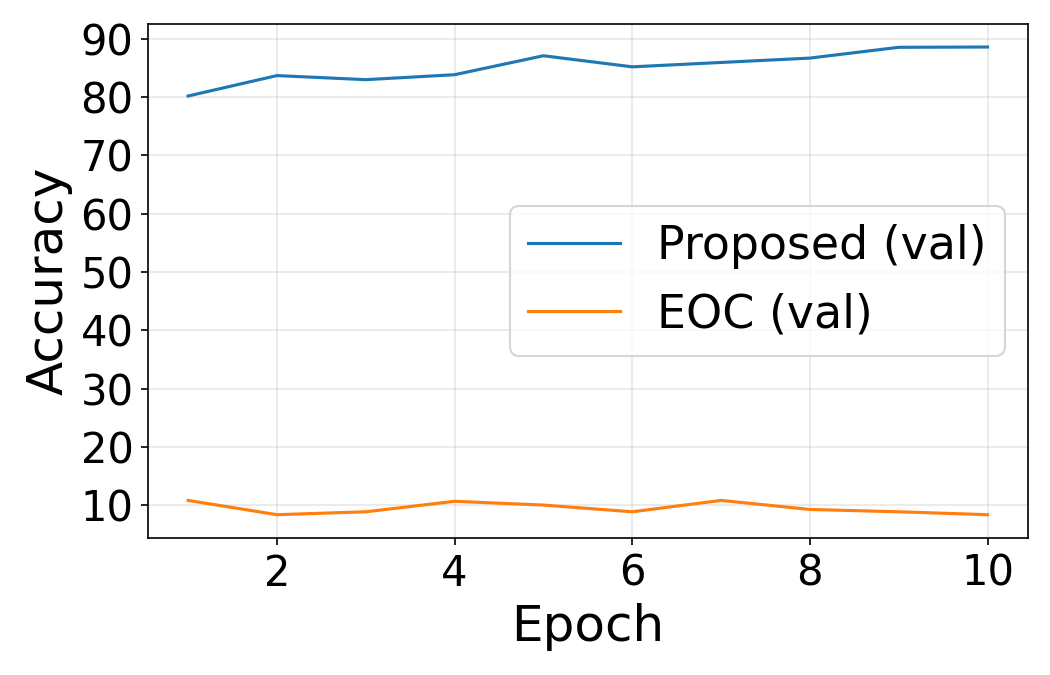}
    & \includegraphics[valign=m,width=0.28\textwidth]{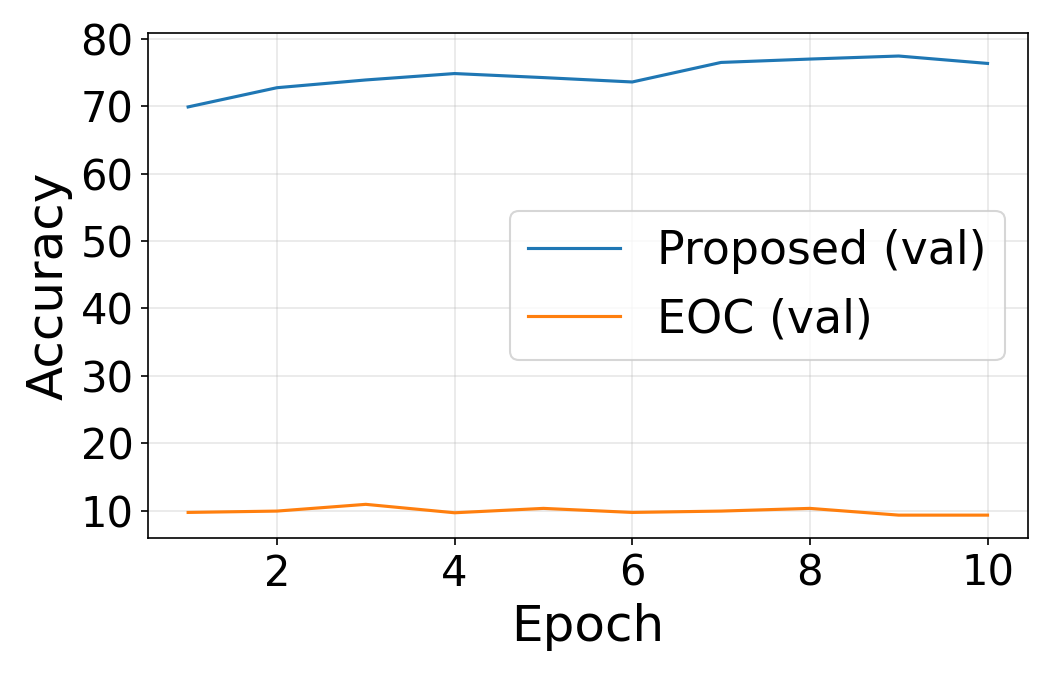}\\

    \hline
  \end{tabular}
\label{scale_table6}

\end{table}

\clearpage 

\paragraph{Common PINN configuration.}
For both PDE benchmarks we use fully connected PINNs with input $(\xi_1,\xi_2)$
corresponding to either $(S,\tau)$ or $(x,t)$, $50$ hidden layers, and width
$64$ in every hidden layer. The hidden activation is either $a\tanh(x)$ or
$b\,\mathrm{softsign}_2(x)=b\,x/\sqrt{1+x^2}$, and we compare the proposed
initialization against the EOC initialization for scales
$a,b \in \{0.1, 0.3, 0.5, 0.8, 1.0, 1.3, 1.5, 1.8, 2.0\}$. 
For the proposed scheme we set the structured weights using the negative-rate
calibration with target $p=0.49$ at depth $L=50$, and for EOC we compute
$(\sigma_b,\sigma_w)$ using the Gauss–Hermite fixed point equations for the
corresponding activation.
In all experiments we train with Adam for iterations
($600$ for Black--Scholes, $600$ for Burgers) followed by L-BFGS for iterations ($1000$ and $1000$, respectively).
The learning rate is scaled by the local gain of the activation via
\(\eta = 10^{-4} / f'(0)\).

\paragraph{Black--Scholes equation.}
We consider the Black--Scholes PDE in time-to-maturity form
\[
 -V_\tau + \tfrac12 \sigma^2 S^2 V_{SS} + r S V_S - r V = 0,
\]
with volatility $\sigma = 0.2$, interest rate $r = 0.05$,
asset price $S \in [S_{\min},S_{\max}] = [0,1]$,
and time-to-maturity $\tau \in [0, T_{\max}] = [0,1]$.
The initial condition at $\tau=0$ (maturity) is 
\[
V(S,0) = \max(S-K,0), \qquad K = 0.5,
\]
and the spatial boundary conditions are
\[
V(0,\tau) = 0,\qquad
V(S_{\max},\tau) = S_{\max} - K e^{-r\tau}.
\]
We sample $30,000$ collocation points $(S,\tau)$ uniformly
from $[0,1]\times[0,1]$, $n_{\mathrm{IC}} = 4{,}000$ points along $\tau=0$, and
$n_{\mathrm{SB}} = 4{,}000$ boundary points along $S=0$ and $S=S_{\max}$.

\paragraph{Burgers’ equation.}
For Burgers’ equation we solve
\[
u_t + u\,u_x - \nu u_{xx} = 0,
\]
with viscosity $\nu = 0.01/\pi$, spatial domain $x \in [X_{\min},X_{\max}] =
[-1,1]$, and time interval $t \in [0,T_{\max}] = [0,1]$.
The initial condition is
\[
u(x,0) = -\sin(\pi x),
\]
and we impose homogeneous Dirichlet boundary conditions
\[
u(X_{\min},t) = 0, \qquad u(X_{\max},t) = 0.
\]
We draw $20,000$ collocation points $(x,t)$ uniformly from
$[-1,1]\times[0,1]$, use $n_{\mathrm{IC}} = 1{,}000$ points along $t=0$, and
$n_{\mathrm{SB}} = 1{,}000$ boundary points along $x=X_{\min}$ and
$x=X_{\max}$. 

\begin{figure}[h!]
\centering 
\includegraphics[width=1\textwidth]{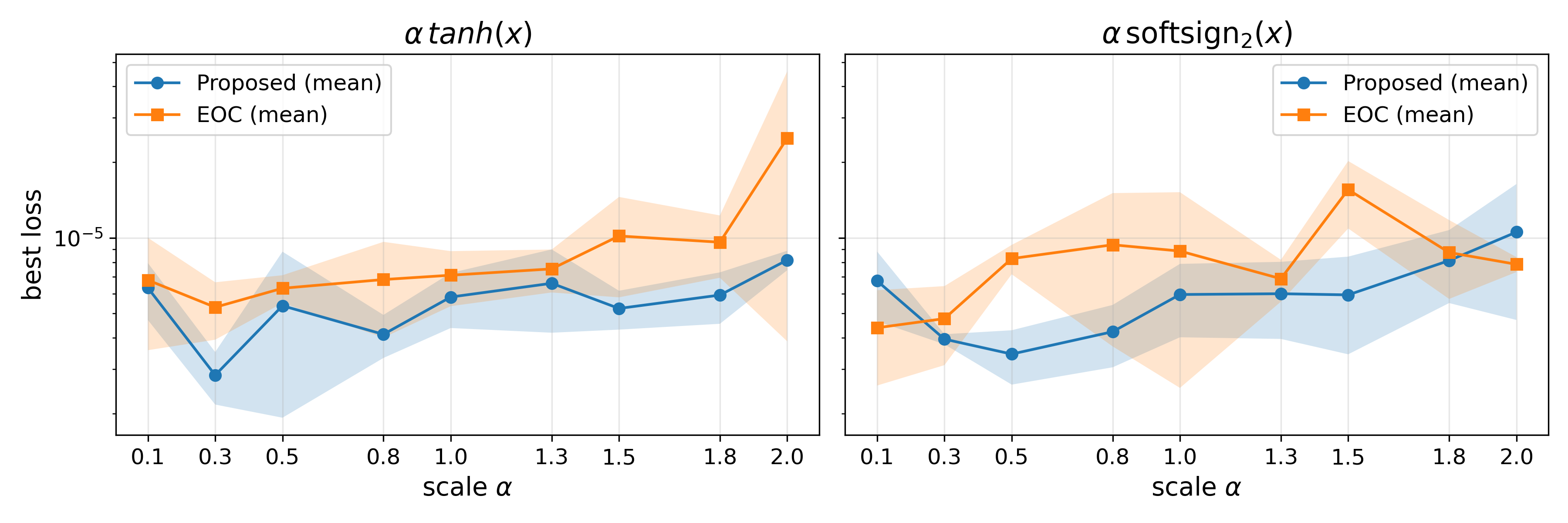}
\caption{
Best PINN loss versus activation scale for the
Burgers PINN (depth $50$, width $64$), comparing the proposed and EOC
initializations. Left: $a\tanh(x)$ with $a \in \{0.1,0.3,0.5,0.8,1.0,1.3,1.5,1.8,2.0\}$.
Right: $b\,\operatorname{softsign}_2(x)$ with the same set of scales.
}
\label{burgers}
\end{figure}

\end{document}

%% file: math_commands.tex
\usepackage{amsmath,amsfonts,bm}

\def\eqref#1{equation~\ref{#1}}

\def\1{\bm{1}}

\DeclareMathAlphabet{\mathsfit}{\encodingdefault}{\sfdefault}{m}{sl}
\SetMathAlphabet{\mathsfit}{bold}{\encodingdefault}{\sfdefault}{bx}{n}

\newcommand{\E}{\mathbb{E}}

\newcommand{\Var}{\mathrm{Var}}

